%% file: main.tex
\documentclass[article, 11pt]{article}

\usepackage{arxiv}
\usepackage[numbers]{natbib}
\usepackage{amsmath}
\usepackage{amsfonts}
\usepackage[title]{appendix}
\usepackage{multirow}
\usepackage{booktabs}
\usepackage{siunitx}
\usepackage{comment}
\usepackage[ruled,lined]{algorithm2e}
\usepackage{dsfont}
\usepackage{todonotes}
\usepackage[T1]{fontenc}
\usepackage{adjustbox}
\usepackage{enumitem}
\usepackage{bbm}
\usepackage{bm}
\usepackage{array}
\usepackage{capt-of}
\usepackage{subfigure}
\usepackage{miama}
\usepackage{wrapfig}

\newcommand{\printfnsymbol}[1]{%
  \textsuperscript{\@fnsymbol{#1}}%
}
\makeatother
\newenvironment{callig}{\fontfamily{qcr}\selectfont}{}
\newcommand{\textcallig}[1]{{\callig#1}}
\newcommand{\f}{v}
\newcommand{\g}{g}

\newcommand{\x}{x}

\newcommand{\expd}{\text{\textcallig{p}}}
\newcommand{\ex}{\Expl}
\newcommand{\prodd}{\Upsilon}
\def\Expl{\mathcal{E}}

\newcommand{\mubar}[1]{\bar{\mu}_{#1}}
\newcommand{\setlessell}{\mathcal{S}_\ell }
\newcommand{\bfb}{\mathbf{b} }
\DeclareMathOperator{\spann}{span}
\DeclareMathOperator{\calA}{\mathcal{A}}
\newcommand*{\defeq}{\stackrel{\text{def}}{=}}

\DeclareMathOperator*{\argmin}{arg\,min}
\newcommand{\norm}[1]{\left\lVert#1\right\rVert}

\usepackage{amsthm}

\newtheorem{example}{Example} 
\newtheorem{theorem}{Theorem}
\newtheorem{lemma}[theorem]{Lemma} 
\newtheorem{proposition}[theorem]{Proposition}

\newtheorem{definition}[theorem]{Definition}

\newtheorem{axiom}[theorem]{Axiom}
\newtheorem{claim}[theorem]{Claim}

\newcolumntype{?}[1]{!{\vrule width #1}}

\title{Faith-Shap: The Faithful Shapley Interaction Index}

\author{%
  Che-Ping Tsai, \ Chih-Kuan Yeh, \ Pradeep Ravikumar \\
  \\
  \texttt{\{chepingt, cjyeh, pradeepr\}@cs.cmu.edu} \\
  Department of Machine Learning\\
  Carnegie-Mellon University\\
  Pittsburgh, PA 15213 \\
}

\begin{document}

\maketitle

\begin{abstract}%
Shapley values, which were originally designed to assign attributions to individual players in coalition games, have become a commonly used approach in explainable machine learning to provide attributions to input features for black-box machine learning models. A key attraction of Shapley values is that they uniquely satisfy a very natural set of axiomatic properties. However, extending the Shapley value to assigning attributions to interactions rather than individual players, an \emph{interaction index}, is non-trivial: as the natural set of axioms for the original Shapley values, extended to the context of interactions, no longer specify a unique interaction index. Many proposals thus introduce additional less ``natural'' axioms, while sacrificing the key axiom of efficiency, in order to obtain unique interaction indices. In this work, rather than introduce additional conflicting axioms, we adopt the viewpoint of Shapley values as coefficients of the most faithful linear approximation to the pseudo-Boolean coalition game value function. By extending linear to $\ell$-order polynomial approximations, we can then define the general family of \emph{faithful interaction indices}. We show that by additionally requiring the faithful interaction indices to satisfy interaction-extensions of the standard individual Shapley axioms (dummy, symmetry, linearity, and efficiency), we obtain a \emph{unique} Faithful Shapley Interaction index, which we denote Faith-Shap, as a natural generalization of the Shapley value to interactions. We then provide some illustrative contrasts of Faith-Shap with previously proposed interaction indices, and further investigate some of its interesting algebraic properties. We further show the computational efficiency of computing Faith-Shap, together with some additional qualitative insights, via some illustrative experiments.
\end{abstract}

\section{Introduction}

Explaining the prediction of a black-box machine learning model via attributions to its features is an increasingly important task. Most approaches have focused on attributions to \emph{individual features}, which does not always suffice to provide insight into the model when there are heavy feature interactions. For instance, when explaining models with text input, we might also ask for attributions to phrases and sequences of words rather than just individual words. Similarly, in Question Answering (QA)~\citep{ye2021connecting}, it is of interest to measure attributions to query answer tuples, rather than just individual entities associated with answers. Such feature interactions are also salient with images as input, where instead of attributions to individual pixels, we might prefer attributions to groups of pixels. 

A large class of recent approaches for individual feature attributions reduces the task to a cooperative game theory problem. Given a machine learning model, a test point, and the underlying data distribution, one can devise a ``set value function'' that takes as input a set of features and outputs the value of that set of features. There are many choices for such a reduction to a set function~\citep{lundberg2017unified, sundararajan2019many, frye2020shapley, chen2020true}. We can then relate this to a cooperative game theory problem where the features are players, the set function above is the value function of the coalition game that specifies the value of various player coalitions, and we wish to derive feature attributions given such a value function. This meta-approach has led to a slew of explanation approaches when the goal is to obtain individual feature attributions. The key question we focus on in this paper is to obtain attributions to \emph{feature interactions} instead. In this setting, any feature interactions (up to a given order), along with each individual feature, should get some attribution score. This question has attracted some attention in the cooperative game theory and the explainable AI literature, with the broad strategy of extending popular approaches for individual feature attributions, such as Shapley and Banzhaf values \citep{shapley1953value,harsanyi1963simplified}, to the interaction context. But these existing proposals come with many caveats.

Part of the attraction of the cooperative game theory based explanations above is that for the case of individual feature attributions, if we stipulate some natural axioms such as linearity, symmetry, dummy, and efficiency (detailed in a later section), there exist unique attributions such as Shapley and Banzhaf (depending on the notion of efficiency). Thus we have both a strong axiomatic foundation to the explanations, as well as a very compelling uniqueness result that there can exist no other explanations that satisfy these axioms. These have thus led to an explosion of Shapley value based explanations in the XAI literature that assign attributions to features, data, and even concepts~\citep{lundberg2017unified, gromping2007estimators, lindeman1980introduction, owen2014sobol, owen2017shapley, datta2016algorithmic,ghorbani2019data, jia2019towards,yeh2020completeness}. However, when we move to the context of feature interactions, while the axioms above have natural extensions from the individual feature to the feature interaction context, they no longer result in a \emph{unique feature attribution value}. 

Approaches to address this have thus focused on adding additional less natural axioms to ensure uniqueness. One set of unique feature attributions --- Shapley interaction and Banzhaf interaction indices~\citep{grabisch1999axiomatic} --- derive unique attributions via a \emph{recursive axiom}, which specifies how higher-order feature attributions be derived from lower order feature interaction attributions (all the way to individual feature attributions). Thus, given the uniqueness at the level of individual feature attributions, we in turn get uniqueness at all levels of interaction attributions. One major caveat of these Shapley interaction and Banzhaf interaction indices is that they do not satisfy the efficiency axiom for interaction feature attributions, and hence can no longer be viewed as distributing the total contribution of the model prediction among all feature interactions. The other caveat is that the recursive axiom, while convenient to extend uniqueness from individual to interaction feature attributions, is much less ``natural'' when compared with the original Shapley axioms, which specifically defined the forms of first-order indices for certain value functions.
To address these caveats, \citet{sundararajan2020shapley} proposed the \emph{interaction distribution axiom} that entails distributing higher-order interactions to the topmost interaction indices at the expense of impoverished lower-order interactions. This makes the interaction attributions unique for unanimity games \citep{shapley1953value}, and since these act as a basis for set value functions, by linearity this ensures uniqueness of interaction attributions for general games. The caveat however is that the specified attribution distribution inordinately favors the topmost interactions, which in turn affects the usefulness of both the lower and highest-order interactions as we show in our examples. And arguably, the interaction distribution axioms too are much less natural when compared to the original Shapley axioms. Thus, there remains an open problem to specify a ``natural'' restriction or axiom that allows for unique interaction attributions.

An additional desideratum is that the feature interaction attributions be cognizant of the \emph{maximum interaction order} of the interaction attributions we require. For instance, with individual feature attributions, the maximum interaction order is one, while with pairwise feature attributions, the maximum interaction order is two. This would allow the explanations to be tailored to the set of possible interactions and satisfy the relevant axioms with respect to just these  interactions, instead of all possible subsets of feature interactions.

In this work, rather than devising potentially less natural axioms to ensure uniqueness, we work from yet another viewpoint of Shapley values, that they are faithful to the set value function: for all subsets, the sum of individual feature attributions over a subset should approximate the set value function evaluated on that subset. When formalized as a  weighted regression problem, this yields Shapley and Banzhaf values depending on the weights in the weighted regression~\citep{banzhaf1964weighted,ruiz1996least}. We then extend the above weighted regression to feature interactions up to a given maximum interaction order, which then yields what we call Faith-Interaction indices. We show that when  restricting to the class of Faith-Interaction indices, together with the (interaction extensions of the individual) Shapley axioms, we obtain a unique interaction index, which we term the Faith-Shap (for Faithful Shapley Interaction) index, which reduces to the individual feature Shapley values when the top interaction order is one. We thus posit Faith-Shap as the natural extension of Shapley values from individual features to interaction indices. Similarly, when the efficiency axiom is replaced by the generalized 2-efficiency axiom, we obtain a unique interaction index, which we term Faith-Banzhaf (for Faithful Banzhaf Interaction) index. The latter has also appeared in other guises in prior work \citep{hammer1992approximations,grabisch2000equivalent}. Unlike the other restrictive axioms discussed earlier, here we only require that the explanations be faithful to the model, which has always been a big attraction of Shapley values in the explainable AI (XAI) context. We corroborate the usefulness of these Faith-Interaction indices by contrasting them with prior indices in two illustrative coalition games, as well as real-world XAI applications. We then discuss the algebraic properties of Faithful Shapley Interaction index by relating them to cardinal indices, i.e. indices that can be expressed as a linear combination of marginal contributions, as well as in terms of approximations to multilinear extensions of the coalition set value function. An additional benefit of the Faith Interaction indices is that the estimation becomes much more efficient via leveraging the weighted linear regression formulation, which we validate in our experiments.

\def\model{f}
\def\X{\mathcal{X}}
\def\R{\mathbb{R}}
\section{Preliminaries}

\subsection{Notations}
Suppose we are given a black-box model $\model : \X \mapsto \mathbb{R}$, with input domain $\X \subseteq \R^d$; and suppose we wish to explain its prediction at a given test point $\x \in \X$.  Suppose also given the tuple $\model, \x$ (and possibly with additional information about the underlying data distribution on which $\model$ is trained on, and from which $\x$ is drawn), there is a well-defined set function $\f_{\x} : 2^d \rightarrow \mathbb{R}$. We can interpret such a set function as specifying the value of a subset of the set of $d$ features. Many popular explanations employ such a reduction of the model and its prediction context to set value functions; see \citet{ribeiro2016should, lundberg2017unified, sundararajan2020shapley} for many examples. When clear from the context, and for notational simplicity, we will often omit $\x$ and simply use $\f$ to denote the set function. Such a reduction allows us to leverage results from cooperative game theory, by relating the set of features to a set of players, and the set function above as specifying the values of coalitions of players.

We are then interested in quantifying the importance of interactions between different features up to some order $\ell \in [d]$. Note that in this context, when we mean interactions between features, we mean non-self interactions between distinct features, since self-self interactions could simply be identified with the individual features.
In other words, we require an importance function $\ex$ which for each coalition $S \subseteq [d]$ where $0 \leq |S| \leq \ell$, outputs a scalar $\ex_{S} (\f, \ell)$.
Let $\mathcal{S}_\ell$ denote the set of all subsets of $[d]$ with size less than or equal to $\ell$; the size of this set can be seen to be $d_{\ell} \defeq \sum_{j=0}^{\ell} { d \choose j}$.
We then use the shorthand $\ex(\f, \ell) = (\ex_{S}(\f, \ell))_{S \in \mathcal{S}_\ell} \in \R^{d_\ell}$. To simplify notation, we omit braces for small sets and write $T \cup i$ to represent $T \cup \{i\}$. 

\subsection{Definitions}
We begin by recalling the concept of discrete derivatives.
\begin{definition}
\label{def:discrete_derivative}
(Discrete Derivative) Given a set function $\f:2^d \mapsto \mathbb{R}$ and two finite disjoint coalitions $S, T \subseteq [d]$ with $S \cap T = \text{\O}$, the $S$-derivative of $\f$ at $T$, $\Delta_S (\f(T))$, is defined recursively as follows:
\begin{equation}\label{eq:discrete1}
    \Delta_i \f(T) = \f(T \cup i) - \f(T), \ \ \forall i \in [d], 
\text{ and }
\end{equation}
\begin{equation}\label{eq:discrete2}
    \Delta_S (\f(T)) = \Delta_i [\Delta_{S \backslash i} (\f(T)) ] = \sum_{L \subseteq S} (-1)^{|S|-|L|} \f(T \cup L), \forall i \in S.
\end{equation}
\end{definition}
The second equality in Eqn.~\eqref{eq:discrete2} can be shown via induction on $S$ \citep{fujimoto2006axiomatic}. 
As an illustration of discrete derivatives, for a subset $S$ of size $2$, the discrete derivative can be written as  
$$
\Delta_{\{i,j\}} \f(T) 
=  \f(T \cup \{i,j\}) -\f(T\cup j) - \f(T\cup i) +\f(T).
$$
$\Delta_{\{i,j\}} \f(T) $ captures the joint effect of features $i$ and $j$ co-occurring compared to the individual effects of $i$ and $j$. If $\Delta_{\{i,j\}} \f(T) > 0$ (resp. $<0$), we say $i$ and $j$ have positive (resp. negative) interaction effect in the presence of $T$ since the presence of $i$ increases (resp. decreases) the marginal contribution of $j$ to coalition $T$. 
Following the intuition from the two features example, the discrete derivative $\Delta_S (\f(T))$ can be viewed as a measurement of the  \emph{marginal interaction of $S$ in the presence of $T$}. When a set of features have a positive (negative) interaction effect, the discrete derivative is positive (negative). 
Discrete derivatives play a fundamental role in measurement of interaction effects.
As we will see in the following section, the Shapley and Banzhaf interaction indices can be viewed as a weighted average of $S$-derivatives over all subsets $T \subseteq [d] \backslash S$.

\vskip0.1in
\noindent
Next, let us recall the concept of the Möbius transform.
\begin{definition}
(Möbius transform) Given set function $\f:2^d \mapsto \mathbb{R}$, the Möbius transform of $\f(\cdot)$ is 
\begin{equation}
\label{eqn:def_mobius_transform}
a(\f,S) = \sum_{T \subseteq S} (-1)^{|S| - |T|} \f(T) 
\ \text{ for all } \ S \subseteq [d].
\end{equation}
\end{definition}
An important property~\citep{shapley1953value} of the Möbius transform is that any set function $\f(\cdot)$ can be expressed as:
\begin{equation}
\label{eqn:discopose_f}
\f = \sum_{R \subseteq [d]} a(\f,R) \, \f_R,
\end{equation}
where $\f_R$ for any $R \subseteq [d]$ has the form $\f_R(S) = 1$ if $S \supseteq R$ and $0$ otherwise; and is also known as a \emph{unanimity game} value function in game theory. Eqn.~\eqref{eqn:discopose_f} states that any set function can be expressed as a linear combination of these unanimity game value functions (so that $\{\f_R\}_{R \subseteq [d]}$ form a basis for real-valued set value functions), with the Möbius transforms $a(\f,R)$ as their coefficients. Note that if an interaction index satisfies the \textbf{interaction linearity axiom} (to be discussed in the sequel), the interaction index for general set value functions can be expressed as a linear combination of the interaction indices for unanimity games.

\section{Background: Axioms for Interaction Indices} 
\label{sec:game_axiom}
In this section, we present natural extensions of Shapley axioms for individual features to the feature interactions~\citep{grabisch1999axiomatic,sundararajan2020shapley}. We then discuss the key interaction indices proposed so far in the literature --- the Shapley interaction index, Banzhaf interaction index and Shapley-Taylor interaction index --- with respect to these axioms. In all these axioms, we allow for dependence on the maximum interaction order $\ell \in [d]$. A summarization of axioms that these interaction indices satisfy is in Table \ref{table:axioms}.

\begin{table}[ht]
\centering
\small
\resizebox{\textwidth}{!}{
\begin{tabular}{ | c | c | c | c | c | c | c | c | c | }
\hline 
Indices 
& \begin{tabular}{@{}c@{}} Interaction \\ linearity  \end{tabular}  
& \begin{tabular}{@{}c@{}} Interaction \\ symmetry  \end{tabular} &\begin{tabular}{@{}c@{}} Interaction\\ dummy  \end{tabular} & \begin{tabular}{@{}c@{}} Interaction\\ efficiency  \end{tabular} & \begin{tabular}{@{}c@{}} Interaction\\ recursive  \end{tabular}  & \begin{tabular}{@{}c@{}} Generalized \\  Interaction\\ 2-efficiency \end{tabular} & \begin{tabular}{@{}c@{}} Interaction\\ distribution  \end{tabular}  & \begin{tabular}{@{}c@{}} Is \\ Faith-Interaction \\ Index \end{tabular}  \\
\hline
\hline
\begin{tabular}{@{}c@{}} Shapley \\ Interaction \end{tabular}   &\checkmark &\checkmark &\checkmark  & & \checkmark & & & \\
\hline 
\begin{tabular}{@{}c@{}} Banzhaf \\ Interaction \end{tabular}  &\checkmark &\checkmark &\checkmark  & & \checkmark & \checkmark & & \\
\hline 
\begin{tabular}{@{}c@{}} Shapley \\ Taylor \end{tabular}  &\checkmark &\checkmark &\checkmark  & \checkmark & &  & \checkmark& \\
\hline 
\begin{tabular}{@{}c@{}} Faithful \\ Shapley \end{tabular}  &\checkmark &\checkmark &\checkmark  & \checkmark & & & & \checkmark \\
\hline 
\begin{tabular}{@{}c@{}} Faithful \\ Banzhaf \end{tabular} &\checkmark &\checkmark &\checkmark  & & & \checkmark   & & \checkmark  \\
\hline
\end{tabular}}
\caption{A table of axioms that different interaction indices satisfy.}
\label{table:axioms}
\end{table}

\begin{axiom}
(Interaction Linearity): For any maximum interaction order $\ell \in [d]$, and for any two set functions $\f_1$ and $\f_2$, and any two scalars $\alpha_1, \alpha_2 \in \mathbb{R}$, the interaction index satisfies: $\ex(\alpha_1 \f_1+ \alpha_2 \f_2,\ell) = \alpha_1 \ex(\f_1,\ell) + \alpha_2 \ex(\f_2,\ell)$.
\end{axiom}

The interaction linearity axiom states that the feature interaction index is a linear functional of the set function $\f(\cdot)$. It ensures that the corresponding indices scale with the value function $\f(\cdot)$.

\begin{axiom}
(Interaction Symmetry):  For any maximum interaction order $\ell \in [d]$, and for any set function $\f:2^d \mapsto \R$ that is symmetric with respect to elements $i, j \in [d]$, so that
$\f(S \! \cup  i) = \f(S \! \cup  j) \!$ for any $S \subseteq [d] \backslash \{i,j\}$, the interaction index satisfies: $\ex_{T \cup i}(\f,\ell) = \ex_{ T \cup j}(\f,\ell)$ for any $T \subseteq [d] \backslash \{i,j\}$ with $|T| < \ell$.
\end{axiom}

The interaction symmetry axiom entails that if the value function treats two features the same, their corresponding feature interaction index values should be the same as well.

\begin{axiom}
(Interaction Dummy): For any maximum interaction order $\ell \in [d]$, and for any set function $\f:2^d \mapsto \R$ such that $\f(S \cup i) = \f(S)$ for some $i \in [d]$ and for all $ S \subseteq [d] \backslash \{i\}$, the interaction index satisfies:  $\ex_{T}(\f, \ell) = 0$ for all $T \in \mathcal{S}_\ell$ with $i \in T$.
\end{axiom}

The interaction dummy axiom entails that a dummy feature $i \in [d]$ that has no influence on the function $\f$ should have no interaction effect with the other features.

\begin{axiom}
(Interaction Efficiency): For any maximum interaction order $\ell \in [d]$, and for any set function $\f:2^d \rightarrow \R$, the interaction index satisfies: $\sum_{S \in \mathcal{S}_\ell\backslash\text{\O}} \ex_S(\f,\ell) = \f([d]) - \f(\text{\O})$ and $\ex_{\text{\O}}(\f,\ell) = \f(\text{\O})$.
\end{axiom}
The interaction efficiency 
ensures that the interaction index distributes the total value $\f([d])$ among the different subsets in $\mathcal{S}_\ell$. This axiom lends itself a natural explanation of $\ex_S(\f,\ell)$: it represents the marginal contribution that the group $S$ makes to the total value, 
which has also been considered by \citet{sundararajan2020shapley}.
As we will detail in the sequel, some of the recently proposed interaction indices do not satisfy such an efficiency axiom. For instance, the chaining interaction and Shapley interaction indices only require the total sum of \emph{individual feature importances} to sum to $\f([d]) - \f(\text{\O})$, without consideration of the higher-order interaction importances. 

\paragraph{Challenge: Lack of Uniqueness:} These axioms are natural extensions to the interaction setting of classical axioms for individual feature attributions; see \citet{fujimoto2006axiomatic,grabisch1999axiomatic} for a counterpart of these interaction axioms without consideration of the maximum interaction order $\ell \in [d]$.
As \citet{sundararajan2020shapley} note, though the linearity, symmetry, dummy, and efficiency axioms uniquely specify a feature attribution when the maximum interaction order $\ell = 1$ (i.e. for individual feature attributions), they no longer do when $\ell > 1$. In other words, there could exist many interaction indices that all satisfy the axioms specified above. A big attraction of the individual Shapley value was its uniqueness given the corresponding individual attribution axioms. Accordingly, a line of work has focused on specifying additional axioms that together specify a unique interaction index.

\begin{axiom}(Recursive Interaction):
For any maximum interaction order $2 \leq \ell \leq d$, and for any set function $\f:2^d \rightarrow \mathbb{R}$, and for any $j \in [d]$, 
let the reduced set functions $\f^{[d] \backslash j}, \f_{\cup j}^{[d] \backslash j}:2^{d-1} \rightarrow \mathbb{R}$ be defined as: 
$$
\text{ for all }  \  T \subseteq [d] \backslash j, \ \ \f^{[d] \backslash j}(T) = \f(T),
\ \  \text{ and } \ \ 
\f_{\cup j}^{[d] \backslash j}(T) = \f(T \cup j) - \f(j) .
$$
Then the interaction index satisfies: $\ex_{S}(\f,\ell) = \ex_{S \backslash j}(\f_{\cup j}^{[d] \backslash j}, \ell) - \ex_{S \backslash j}(\f^{[d] \backslash j}, \ell), \ \ \forall S \in \setlessell \text{ with } |S| \geq 2$.
\end{axiom}

The recursive axiom above is an extension of the recursive axiom of~\citet{grabisch1999axiomatic} to account for arbitrary maximum interaction orders. The axiom can be informally interpreted as ``how does the presence or absence of feature $j$ influence the share of feature set $S$''. But more importantly (and the reason it is termed the recursive axiom) is that it specifies how higher-order interaction scores are \emph{uniquely determined} given lower-order interaction indices. By recursion, the higher-order interaction indices are thus uniquely specified given just the  singleton feature attributions. The reason this helps with uniqueness is that so long as the axioms entail unique singleton attributions, together with this recursive axiom, they would entail unique interaction attributions.  Thus, we argue that the recursive axiom is less ``natural'' compared to previously introduced axioms since the recursive axiom only ensures the uniqueness property, at the potential expense of other axiomatic properties.

\paragraph{Shapley Interaction Index:} 
~\citet{grabisch1999axiomatic} thus show that there is a unique interaction index that satisfies the interaction linearity, symmetry, dummy, and the recursive axioms (but not the interaction efficiency axiom), and whose restrictions to singleton sets correspond to Shapley values. They term this interaction index Shapley interaction index. 
This Shapley interaction index has the following closed form:
\begin{equation}
\label{eqn:closed_form_shap_inter}
    \ex_S^{\text{Shap}}(\f, \ell) =  \sum_{T \subseteq [d]/ S}  \frac{|T|! (d-|S|-|T|)!}{(d-|S|+1)!} \Delta_S(\f(T)),\ \ \  
    \forall S \in \setlessell.
\end{equation}
A critical caveat of the resulting Shapley interaction value is that it no longer satisfies the interaction efficiency axiom when the maximum interaction order $\ell>1$. Indeed, simply summing the contributions to singleton sets (i.e. the classical individual attribution Shapley values) is already equal to $\f([d]) - \f(\text{\O})$, so the only way for the interaction efficiency axiom to be satisfied if all the other interaction attributions sum to zero, which they do not. 

\paragraph{Banzhaf Interaction Index:} 
~\citet{grabisch1999axiomatic} further show that there is a unique interaction index that satisfies the interaction linearity, symmetry, dummy, and recursive axioms (but not the interaction efficiency axiom), and whose restrictions to singleton sets correspond to the Banzhaf values. They term this interaction index Banzhaf interaction index, which has the following closed form:
\begin{equation}
\label{eqn:closed_form_bzf_inter}
    \ex^{\text{Bzf}}_S(\f,\ell) =  \sum_{T \subseteq [d]/S}  \frac{1}{2^{d-|S|}} \Delta_S(\f(T)),\ \ \  
    \forall  S \in \setlessell.
\end{equation}
It can be again shown that the Banzhaf interaction index does not satisfy the interaction efficiency axiom even when $\ell=1$; though they do satisfy the generalized 2-efficiency axiom, which can be stated as follows.
\begin{axiom}(Generalized Interaction 2-Efficiency): 
Define the reduced function $\f_{[ij]}: 2^{d-1} \rightarrow \mathbb{R}$ given any $i,j \in [d]$ as $\f_{[ij]}(S) = \f(S)$ for all sets $S$ containing both $i$ and $j$, and $\f_{[ij]}(S \cup [ij]) = \f(S \cup \{i,j\})$ for all $S$ containing neither $i$ nor $j$. That is, the reduced function considers features $i$ and $j$ together as a group $[ij]$. Then the interaction index satisfies:  $\ex_{S \cup [ij]}(\f_{[ij]}, \ell) = \ex_{S \cup i}(\f, \ell) + \ex_{S \cup j}(\f, \ell)$ for all $S \subseteq [d] \backslash \{i,j \}$, and $\ell = |S|+1$.
\end{axiom}

The generalized interaction 2-efficiency axiom above is an extension of the generalized 2-efficiency axiom of \citet{grabisch1999axiomatic} to account for arbitrary maximum interaction orders.
It states that when features $i,j$ form a group in the set function $\f_{[ij]}$ with $d-1$ features, the importance of $S \cup [ij]$ equals the sum of importances of $S \cup i$ and $S \cup j$ with respect to the original set value function.
When $S = \text{\O}$ and $\ell=1$, it reduces to the classical 2-efficiency axiom~\citep{harsanyi1963simplified} that indicates that the importance of $[ij]$ as a group should be equal to the sum of importances of individual features $i$ and $j$.

\paragraph{Shapley Taylor Interaction Index:}
\citet{sundararajan2020shapley} stipulate an additional \emph{interaction distribution (ID) axiom}, which can be stated as follows.

\begin{axiom}(Interaction distribution~\citep{sundararajan2020shapley}):  
Define $\f_T$ parameterized by a set $T \subseteq [d]$ as $\f_T(S) = 0$ if $T \not \subseteq S$ and $\f_T(S) = 1$ otherwise. Then for all $\ell \in [d]$, and for all $S$ with $S \not \subseteq T $ and $|S| < \ell$, the interaction index satisfies: $\Expl_S(\f_T,\ell)=0$. 
\end{axiom}

The key idea behind the ID axiom is to uniquely specify an interaction index for unanimity games $\{\f_T\}_{T \subseteq [d]}$, given the interaction linearity, symmetry, dummy, and efficiency axioms. Since unanimity games form a basis for the set of all games, in the presence of interaction linearity axiom, we then get unique interaction indices.
They thus show that there exists a unique interaction index that satisfies interaction linearity, symmetry, dummy, efficiency, and interaction distribution axioms and which they term Shapley Taylor index (for reasons which will become clearer in a later section when we discuss algebraic properties of various interaction indices). The Shapley Taylor interaction index has the following closed form:
\begin{equation}
\label{eqn:def_taylor_shap}
\ex_S^{\text{Taylor}}(\f,\ell) = 
\begin{cases}
\Delta_S(\f(\text{\O})) & \text{, if } |S| < \ell. \\
\sum_{T \subseteq [d]/ S}  \frac{|T|!(d-|T|-1)!|S|}{d!} \Delta_S(\f(T)) & \text{, if } |S| = \ell. \\
\end{cases}
\end{equation}

A key advantage of this interaction index is that it depends on the maximum interaction order $\ell$, in contrast to previously proposed interaction indices such as the Shapley interaction and Banzhaf interaction indices. Indeed, in order for an interaction index to satisfy the interaction efficiency axiom for maximum interaction order $\ell$, it has to distribute the contributions among subsets in $\mathcal{S}_\ell$, and hence has to be cognizant of the maximum interaction order $\ell$.
However, a key caveat of the interaction distribution  axiom is that the specified attribution distribution inordinately favors the topmost interaction. As can be seen from Eqn.\eqref{eqn:def_taylor_shap}, the importance of a set $S$ with $|S| < \ell$ is only specified by the marginal contribution of $S$ in the presence of the empty set, and not the presence of other subsets $T \subseteq [d] \backslash S$. This impoverishes lower-order interactions, which in turn hurts the meaningfulness of both lower and highest-order interactions as we will show in Section \ref{sec:contrast}.

Thus a key open question that this section has made salient is: how do we more naturally constrain interaction indices beyond interaction linearity, symmetry, dummy, and efficiency axioms, so as to obtain a unique interaction index?

\section{Faith-Interaction Indices}
\label{sec:faithful_interaction_indices}

In this section, in contrast to additional axioms, we draw from another viewpoint of singleton Shapley feature attributions: that they are faithful to the underlying value function. 

\paragraph{Faithfulness of Singleton Shapley Values:} 
Given singleton feature attributions $\{\ex_i\}_{i \in [d]}$, we can require that:
\[ \f(S) \approx \sum_{i \in S} \ex_i,\; \forall S \subseteq [d].\]
Note that we can only ask for approximate rather than exact equality for all sets $S$, since exact equality would entail we solve $2^d$ linear equalities (corresponding to the subsets of $[d]$) with $d$ variables (corresponding to the $d$ singleton feature attributions $\{\ex_i\}_{i \in [d]}$), which may not always have a feasible solution. One approach to formalize such approximate equality is via weighted regression:
\begin{equation}
\label{eqn:singleton_regression}
\min_{\ex \in \mathbb{R}^{d+1}} \sum_{S \subseteq [d]}\mu(S)\left(v(S) - \ex_{\text{\O}} - \sum_{i \in S} \ex_i\right)^2,    
\end{equation}
where $\mu: 2^{[d]} \mapsto \R^{+} \cup \{\infty\}$ is some weighting over the subsets $S \subseteq [d]$ which can be interpreted as the importance of different coalitions. 
Note that the range of $\mu$ is the extended positive reals. When $\mu(S) = \infty$ for some sets $S$, we can interpret the above as solving the constrained problem:
\begin{equation*} 
\min_{\ex \in \R^{d+1}}  \sum_{S \subseteq [d]\,:\, \mu(S) < \infty}\mu(S)\left(v(S) - \sum_{i \in S} \ex_i\right)^2 
\text{s.t.} \;  \f(S) = \sum_{i \in S} \ex_i,\; \forall S:\, \mu(S) = \infty.
\end{equation*}

It has been shown that we can recover the singleton Shapley values as the solution of the weighted regression problem above by setting $\mu(S) \propto \frac{d-1}{\binom{d} {|S|}\,|S|\,(d-|S|)}$ and $\mu(\text{\O}) = \mu([d]) = \infty$~ \citep{charnes1988extremal}. And we can recover singleton Banzhaf values by using the uniform distribution $\mu(S) = 1/2^d$~\citep{hammer1992approximations}.

\paragraph{From Singleton Attributions to Interaction Indices:}
In this section, we consider the generalization of the above to \emph{interaction indices}, so that we now require:
\[\f(S) \approx \sum_{T \subseteq S, |T| \leq \ell} \Expl_T(\f,\ell),\; \forall S \subseteq [d].\]
Again here we ask for approximate rather than exact equality since when the order of interactions is less than the number of features, so that $\ell < d$, the latter would entail we solve $2^d$ linear equalities with $d_\ell$ variables, which may not always have a feasible solution. Accordingly, we consider the following weighted regression problem as a formalization of the above:
\begin{equation}
\label{eqn:weighted_regression}
\ex(\f, \ell) \ = \argmin_{\Expl \subseteq \mathbb{R}^{d_\ell} } 
\sum_{S \subseteq [d]}  \mu(S) \left( \f(S) - \sum_{T \subseteq S , |T| \leq \ell}\Expl_T(\f,\ell) \right)^2,
\end{equation}
where $\mu:2^d \rightarrow  \R^+ \cup \{\infty\}$ is a coalition weighting function. And as before of $\mu(S) = \infty$ for some sets $S$, we can interpret above as solving the constrained problem:
\begin{align}
\ex(\f, \ell) \ 
&= \argmin_{\Expl \subseteq \mathbb{R}^{d_\ell} } \sum_{S \subseteq [d]\,:\, \mu(S) < \infty}  \mu(S) \left( \f(S) - \sum_{T \subseteq S , |T| \leq \ell}\Expl_T(\f,\ell) \right)^2 \nonumber \\
\text{s.t.} \ & \f(S) = \sum_{T \subseteq S , |T| \leq \ell}\Expl_T(\f,\ell), \;\;\forall S :\, \mu(S) = \infty.
\label{eqn:constrained_weighted_regresion}
\end{align}
We note that the range of the weighting function $\mu$ is not allowed to include zero since it is a necessary condition to ensure that there exists a unique minimizer (See Proposition \ref{pro:unique_minimizer} in the Appendix). This is not an issue in practice since we can always choose an arbitrary small positive value instead of zero to approximate the intended constraint that $\mu(S)=0$ for some $S \subseteq [d]$.

We can also see from Eqn.~\eqref{eqn:weighted_regression} that when the weighting function is infinite for many subsets, this entails corresponding equality constraints on the interaction index, which 
may not have a feasible solution. We thus consider the following set of what we term \emph{proper} weighting functions.
\begin{definition}
\label{def:proper_weighting_f}
(Proper weighting function) We say that a weighting function $\mu:2^d \mapsto \mathbb{R}^{+} \cup \{ \infty\}$ is proper if $\mu(S)$ is finite for all $S \subseteq [d]$ with $1 \leq S \leq d-1$.
\end{definition}
This then leads to our definition of Faith-interaction indices.
\begin{definition}
\label{def:faithfulness}
(Faith-Interaction Indices): 
We say that $\Expl$ is a Faith-Interaction index, given any set value function $\f:2^d \rightarrow \mathbb{R}$ and any maximum interaction order $\ell \in [d]$, if there exists a proper weighting function $\mu:2^d \rightarrow \mathbb{R}^{+} \cup \{\infty\}$ such that $\Expl(\f,\ell)$ minimizes the corresponding weighted regression objective in Eqn.\eqref{eqn:constrained_weighted_regresion}.
\end{definition}

When the coalition weighting function $\mu$ is fully finite so that $\mu(S)$ are finite for all sets $S \subseteq [d]$, Faith-interaction indices have a simple closed-form expression as detailed in the following proposition.
\begin{proposition}
\label{pro:closed_matrix_form_solution} 
Any Faith-Interaction index $\Expl(\f,\ell)$ with respect to a finite weighting function $\mu(\cdot)$ has the form:
\begin{equation}
\label{eqn:opt_binary} 
\small
{\Expl}(\f,\ell) = \left(\sum_{S \subseteq [d]} \mu(S){\expd}(S) {\expd}(S)^T \right)^{-1}\!\!\!\sum_{S \subseteq [d]} \mu(S) \f(S) {\expd}(S), 
\end{equation}
where $\expd: 2^{[d]} \rightarrow \{0,1\}^{d_\ell}$ is specified as: $\expd(S)[T] = \mathbbm{1}[(T \subseteq S)]$ for any $T \in \mathcal{S}_\ell$.
\end{proposition}

When the coalition weighting function $\mu(\cdot)$ is not fully finite, we have a linearly constrained least squares problem that does not have a closed form, but whose solution can be characterized via its Lagrangian (see more details in Proposition \ref{pro:constrained_closed_matrix_form_solution} in the Appendix).

\subsection{Axiomatic Characterization of Faith-Interaction Indices}
\label{sec:aximomatic_characterization}

In this section, we investigate the axiomatic properties of our class of Faith-Interaction indices. We first show that all faith-interaction indices satisfy the interaction linearity axiom.

\begin{proposition}
\label{pro:linearity}
Faith-Interaction indices $\Expl$ satisfy the interaction linearity axiom.
\end{proposition}
For Faith-Interaction indices corresponding to finite coalition, weighting functions $\mu(\cdot)$, this result easily follows from Proposition \ref{pro:closed_matrix_form_solution} that these are linear functionals of the set value function $\f(\cdot)$. For Faith-Interaction indices where the weighting function is no longer finite for some sets $S \in \{ \text{\O}, [d] \}$, they solve a linearly constrained least squares problem which does not have a closed-form solution. But by a more nuanced analysis of its Lagrangian, we can again show that the interaction indices are linear functionals of the set value function $\f(\cdot)$.

We next show that Faith-Interaction indices also satisfy the interaction symmetry axiom provided that the weighting functions are permutation invariant (``symmetric''), and hence the weighting functions only depend on the size of the set.

\begin{proposition}
\label{pro:symmetry}
Faith-Interaction indices $\Expl$ satisfy the interaction symmetry axiom if and only if the weighting functions are permutation invariant, and hence only depend on the size of the set so that $\mu(S)$ is only a function of $|S|$.

\end{proposition}
We next consider the dummy axiom.
\begin{proposition}
\label{pro:dummy}
Faith-Interaction indices $\Expl$ satisfy the interaction dummy axiom if the features behave independently of each other when forming coalitions in the weighting function so that the coalition weighting functions can be expressed as $\mu(S) \propto \prod_{i \in S} p_i \prod_{j \not \in S} (1-p_j)$ for all $S \subseteq [d]$, where $0 < p_i < 1$ is the probability of the feature $i$ to be present. 
\end{proposition}
Proposition \ref{pro:dummy} implies that a dummy feature has no impact on other features when the weighting function treats features independently. 

So far, we have analyzed when Faith-Interaction indices satisfy the interaction linearity, symmetry, and dummy axioms. When they satisfy all three simultaneously, and the coalition weighting function is finite, then we can show that the latter has a specific algebraic form.

\begin{theorem}
\label{thm:thmdummysymm}
Faith-Interaction indices $\Expl$ with a finite weighting function 
satisfy the interaction linearity, symmetry, and dummy axioms if and only if the weighting function $\mu(\cdot)$ has the following form: 
{\small
\begin{align}
\mu(S) & \propto \sum_{i=|S|}^{d} {d- |S| \choose i-|S|}(-1)^{i-|S|} g(a,b,i), \  \text{ where }
g(a,b,i) = 
\begin{cases}
1 & \text{ if } \ i = 0 \\
\prod_{j=0}^{j=i-1} \frac{a(a-b) + j(b-a^2)}{a-b + j(b-a^2)}
& \text{ if } \   1 \leq i \leq d,\\
\end{cases}
\label{eqn:thmdummysymm}
\end{align} 
}
for some $a,b \in \R^{+}$ with $a>b$ such that $\mu(S) > 0$ for all $S \subseteq [d]$. 
\end{theorem}
Theorem \ref{thm:thmdummysymm} shows the surprising fact that Faith-Interaction indices satisfying the interaction linearity, symmetry, and dummy axioms with finite weighting functions have only two degrees of freedom: $a, b \in \R$. Given these, we can fully specify the weighting function, and hence the corresponding Faith-Interaction indices. In Appendix \ref{sec:additional_guidance}, we additionally show that the condition $1 > a > b \geq a^2 > 0$ ensures that $\mu(\cdot)$ is positive everywhere and also provides generalized guidance on setting the values $a,b$.

\paragraph{Faith-Banzhaf Interaction Index:} As a first application of this theorem, suppose in addition to the three axioms above, we additionally require the  Faith-Interaction indices to satisfy generalized 2-efficiency. The following theorem shows that there is a unique Faith-Interaction index satisfying these four axioms, which we term the Faith-Banzhaf index.
\begin{theorem}
\label{thm:faith-banzhaf}
(Faith-Banzhaf)
For any $d \geq 3$, there is a unique Faith-Interaction index that satisfies the interaction linearity, symmetry, dummy, and generalized 2-efficiency axioms, with its coalition weighting function given as $\mu(S) \propto \frac{1}{2^d}$ for all $S \subseteq [d]$. We term this unique interaction index as \textbf{Faithful Banzhaf Interaction index} (Faith-Banzhaf), which has the form:
\begin{equation}
\label{eqn:faith_banzhaf2}
\ex^{\text{F-Bzf}}_S(\f,\ell) = a(\f,S) + (-1)^{\ell - |S|} \sum_{T \supseteq S, |T| > \ell} \left(\frac{1}{2} \right)^{|T| -|S|} {|T| - |S| - 1 \choose \ell - |S|} a(\f,T), \forall S \in \setlessell,
\end{equation}
where $a(\f,\cdot)$ is the Möbius transform of $\f(\cdot)$. 
Moreover, its highest-order interaction terms coincide with corresponding interaction terms from the Banzhaf interaction index introduced earlier:
\begin{equation}
\Expl_S^{\text{F-Bzf}}(\f,\ell) = 
\sum_{T \subseteq [d] \backslash S}\frac{1}{2^{d-|S|}} \Delta_S(\f(T))
\ \ \text{ for all } S \in \setlessell 
\text{ with } |S| = \ell.
\end{equation}
\end{theorem}
Our derivation of Faith-Banzhaf indices follows the pseudo-Boolean function approximation results from \citet{grabisch2000equivalent}. 

\paragraph{Faith-Shapley Interaction Index:}
When moving from generalized 2-efficiency to the more natural interaction efficiency axiom, we have the following proposition.

\begin{proposition}
\label{pro:efficiency}
Faith-Interaction indices satisfy the interaction efficiency axiom if and only if
the weighting functions satisfy $\mu(\text{\O}) = \mu([d]) = \infty$.
\end{proposition}
That the condition in the proposition is sufficient is a straight-forward consequence of the fact that $\mu(\text{\O}) = \mu([d]) = \infty$ entails that the corresponding linear constraint be exactly satisfied, so that: $\sum_{S \in \mathcal{S}_\ell} \ex_S(\f,\ell) = \f([d])$ and $\ex_{\text{\O}}(\f,\ell) = \f(\text{\O})$, which is precisely the interaction efficiency axiom. We now have the machinery to present our main result on the unique Faith-Interaction index that satisfies the four (interaction counterparts of the) standard axioms that the singleton Shapley value satisfies. 

\begin{theorem}
\label{thm:faith_shap}
(Faith-Shap)
There is a unique Faith-Interaction index that satisfies the interaction linearity, symmetry, dummy, and efficiency axioms, with its coalition weighting function given as:
\begin{equation}
\label{eqn:faithshap_weight}
\mu(S) \propto
\frac{d-1}{\binom{d} {|S|}\,|S|\,(d-|S|)}
\text{   for all   } S \subseteq [d]
\text{   with   } 1 \leq |S| \leq d-1, 
\text{   and   } \mu(\text{\O}) = \mu([d]) = \infty.
\end{equation}
We term this unique interaction index as the \textbf{Faithful Shapley Interaction index} (Faith-Shap), which has the form:
\begin{equation}
\label{eqn:faith_shapley}
\ex^{\text{F-Shap}}_S(\f,\ell) = 
a(\f,S) + (-1)^{\ell - |S| }\frac{|S|}{\ell + |S|} {\ell \choose |S| } \sum_{T \supset S, |T| > \ell} \frac{ {|T| - 1 \choose \ell }}{ { |T| + \ell -1 \choose \ell + |S|}} a(\f,T), \;\; \forall S \in \mathcal{S}_\ell,
\end{equation}
where $a(\f,\cdot)$ is the Möbius transform of $\f(\cdot)$. 
Moreover, its highest-order interaction terms can be expressed as a weighted average of discrete derivatives:
\begin{equation}
\label{eqn:faith_shapley_highest_order}
\Expl_S^{\text{F-Shap}}(\f,\ell) = \frac{(2\ell -1)!}{((\ell-1)!)^2 } 
\sum_{T \subseteq [d] \backslash S}\frac{(\ell+|T|-1)!(d-|T|-1)!}{(d+\ell-1)!}   \Delta_S(\f(T))
\ \ \text{ for all } S \in \setlessell 
\text{ with } |S| = \ell.
\end{equation}
\end{theorem}

When the maximum interaction order $\ell=1$, so that we only require singleton feature contributions, the explanation coincides with the classical singleton Shapley values. Thus for larger orders with $\ell > 1$, Faith-Shap can be seen to be a ``natural'' generalization of the first-order Shapley value. Note that the set of axioms it satisfies are (interaction extensions of) the classical linearity, symmetry, dummy, and efficiency axioms. As noted before in an interaction context these axioms alone do not uniquely specify an interaction index. In contrast to the less intuitive axioms such as recursive and interaction distribution axioms, we merely require an interaction extension of the faithfulness property of singleton Shapley values: that the interaction Shapley values approximate the given set value function for all possible subsets.

\section{Contrasting Faith-Interaction with other Interaction Indices}
\label{sec:contrast}

In this section, we compare our Faith-Interaction indices, specifically Faith-Shap, with the other interaction indices introduced earlier.

\paragraph{Comparison with Shapley Interaction and Banzhaf Interaction Indices:}
As noted earlier, the Shapley interaction and Banzhaf interaction indices do not satisfy the interaction efficiency axiom, which states that the sum of interaction weights should equal the difference between the value function evaluated over the complete and empty sets. A critical advantage of the interaction efficiency axiom is that it forces the interaction index to distribute a fixed contribution (difference between the value function evaluated over the complete and empty sets) among the different interactions; without such a distributive requirement, the resulting weights can become quite non-intuitive. For instances of such non-intuitive behaviors, we refer to \citet{sundararajan2020shapley}, who provided many simple examples where the sum of Shapley interaction values over all subsets diverges as the number of features increases, even when the value function is bounded and $\f([d])=1$. Another caveat with these two interaction indices is that they are not cognizant of the maximum interaction order, and hence we cannot compute Shapley values that differ with varying maximum interaction orders. 

\paragraph{Comparison with Shapley Taylor index:}
The Shapley Taylor index does satisfy the four axioms of interaction linearity, symmetry, dummy, and efficiency. However, as noted earlier these four axioms do not uniquely determine interaction indices. The fifth axiom Shapley Taylor index then imposes for uniqueness is the interaction distribution axiom, which has caveats of imbalanced distributions of values to coalitions of different orders, namely, inordinately favoring the maximum interaction order. In particular, the interaction distribution axiom states that higher than max-order interaction values (order $>\ell$) be distributed to the max-order interactions (order $=\ell$), but these max-order terms end up unable to solely explain all higher-order interactions.  On the other hand, it entails lower than max order interactions (order $<\ell$) that do not take into account sub-coalitions other than the empty set, which can be contrasted for instance with singleton Shapley value that explicitly takes into account even higher order coalitions that contain the single feature. Thus the interaction distribution  has the consequence of making both lower and max-order interactions less faithful to the model.

In contrast, in our Faith-Interaction indices, even lower-order interaction weights take into account all possible coalitions, and where the weights are balanced so that the overall set of interaction indices optimally approximates the behavior of the underlying value function. 

\subsection{Examples}
\label{sec:examples}

\paragraph{Example 1: }
We illustrate the difference between these interaction indices using a function with diminishing marginal utility. Consider the following value function with $11$ features: 
\begin{equation}
\label{eqn:example1}
\f(S) = 
\begin{cases}
0  & \text{ , if } |S| \leq 1. \\
|S| - p \times { |S| \choose 2 }  & \text{ , otherwise.} \\
\end{cases}
\end{equation}
This function represents the payoff when any subset of $11$ people work on a task. Each person contributes $1$ unit to the overall payoff, and the task requires at least $2$ people. However, the marginal utility is diminishing in nature, since any two people also have a probability of $p$ of being non-cooperative. 
Given this payoff function, it is worth reflecting on what the attributions to individuals should be. While it might seem that zero is a good value since at least two people are needed for the task, this attribution would only correspond to the \emph{marginal contribution} of an individual player i.e. how much a player would contribute when they are by themselves. Whereas we would like our attributions to also take into account larger coalitions, and marginal contributions to such larger coalitions: this is one of the motivations for considering coalitional game-theoretic indices. Once we do so, then it can be seen that an individual effect of one is much more reasonable. Similarly, we would expect that the pairwise interaction effects be close to $-p$.

In Table \ref{table:example1}, we list the values for different interaction indices for $p=0.1, 0.2$. When the maximum interaction order $\ell=1$, all indices are similar since their restrictions to singleton are the Shapley/Banzhaf values.
When the maximum interaction order $\ell=2$, our Faith-Shap accurately captures individual contribution and pairwise interaction effects by
assigning $0.95/0.95$ and $-0.091/-0.191$ for order $1$ and $2$ and for $p=0.1/0.2$ respectively, which are very close to the intuitions we outlined earlier. 
However, the Shapley Taylor index assigns the individual effect of $i$ by using the marginal $\f(\{i\}) - \f(\text{\O})$, which can be highly inaccurate since such a marginal contribution does not take into account marginal contributions to larger coalitions. 

For $p=0.1$, Shapley Taylor along with Interaction Shapley assigns a positive/zero value to the interaction effect, which suggests that forming groups has complimentary/no effects. On the contrary, Banzhaf interaction and Faith-Banzhaf give negative values for interaction between players, which correctly reflects the decrease in the marginal utility of this game.

For $p=0.2$, the Shapley Taylor index is uniformly zero for any order. This highlights the other drawback of the Shapley Taylor index: the impoverished lower-order interaction indices make the max-order indices less faithful to the model. Specifically, for $p = 0.2$, and with $d = 11$ players, we can see that $\f([d]) = \f(\emptyset) = 0$. We have already seen that $\ex^{\text{Taylor}}_{\{i\}}(\f,\ell) = 0$, for $i \in [d]$. For $\ell = 2$, we then have that the summation of the max-order (i.e. order two) indices equals  $\f([d]) - \f(\emptyset) - \sum_{i=1}^{d} \ex^{\text{Taylor}}_{\{i\}}(\f,\ell) = 0$ by the efficiency axiom. Since all max-order indices have the same value by the symmetry axiom, the max-order indices are uniformly zero. In this case, the Shapley Taylor indices do not take into account the function values $\f(S)$ with $ \ell \leq |S| < d$, and can be arbitrarily unfaithful to these orders. 
Here, the Banzhaf interaction and Faith-Banzhaf again correctly reflect the  negative interaction between players. However, the Banzhaf interaction value gives a value close to $0$ for the first-order indices. Taken together with its negative interaction effects, it might seem that coalitions can only be hurtful to the payoff, which is misleading since the total utility is positive when 2 to 10 players are present. On the other hand, our  Faith-Banzhaf gives a positive value close to $1$ for individual effects of order $1$. Taken together with its negative interaction effects, the value given by the Faith-Banzhaf seems more intuitive: every single player contributes to the utility, while each pair of players hurts the utility. 

Another instructive viewpoint for interaction values is by inspecting their utility for approximating the overall payoff function. In Figure \ref{fig:example1_1} and \ref{fig:example1_2}, we approximate the function $\f(S)$ using $\sum_{T \subseteq S, |T| \leq 2} \Expl_T(\f,\ell)$ for different interaction indices. We can see that our Faith-Shap/Faith-Banzhaf are (almost) faithful to all orders except for $|S|=1$. However, the Shapley Taylor index is only fully faithful to the model when the order is $0,1,11$, and curves for other interaction indices are unfaithful.

\begin{table}[ht]
\centering
\small
\begin{tabular}{ | c | c ?{1pt} c | c ?{1.5pt} c ?{1pt} c | c |}
    \hline 
    \multirow{3}{*}{Indices}
    & 
    \multicolumn{3}{c?{1.5pt}}{$p=0.1$} & \multicolumn{3}{c|}{$p=0.2$}  \\
    \cline{2-7}
    & $\ell=1$ &
    \multicolumn{2}{c?{1.5pt}}{$\ell=2$} &
     $\ell=1$ & \multicolumn{2}{c|}{$\ell=2$}  \\
    \cline{2-7}
     & Order 1 & Order 1 & Order 2 & Order 1 & Order 1 & Order 2\\
    \hline \hline
    Faith-Shap & 0.5 &  0.95 & -0.091  & 0 & 0.95 & -0.191\\
    \hline
    Shapley Taylor & 0.5 & 0 & 0.1  & 0  & 0 & 0 \\
    \hline 
    Interaction Shapley & 0.5 & 0.5 & 0 &  0 & 0 & -0.1\\
    \hline 
    \hline
    Banzhaf Interaction & 0.51 & 0.51  & -0.113 &  0.009 & 0.009 & -0.213 \\
    \hline
    Faith-Banzhaf & 0.51 & 1.08 & -0.113 &  0.009 & 1.08&  -0.213 \\
    \hline
   \end{tabular}
   \caption{Values for different interaction indices of different orders for $p=0.1,0.2$ with different maximum interaction orders. Note that the value function is symmetric with respect to players, so we use order 1 and 2 to denote importance scores of any single player and interaction of any two players.
   Note that $\ex^{\text{F-Shap}}_{\text{\O}}(\f,\ell) = 0$ and $\ex^{\text{F-Bzf}}_{\text{\O}}(\f,\ell) = -0.24$ for both $p=0.1$ and $p=0.2$.}
\label{table:example1}
\end{table}

\begin{table}[ht]
\begin{minipage}[b]{0.48\linewidth}
\centering
\includegraphics[width=\linewidth]{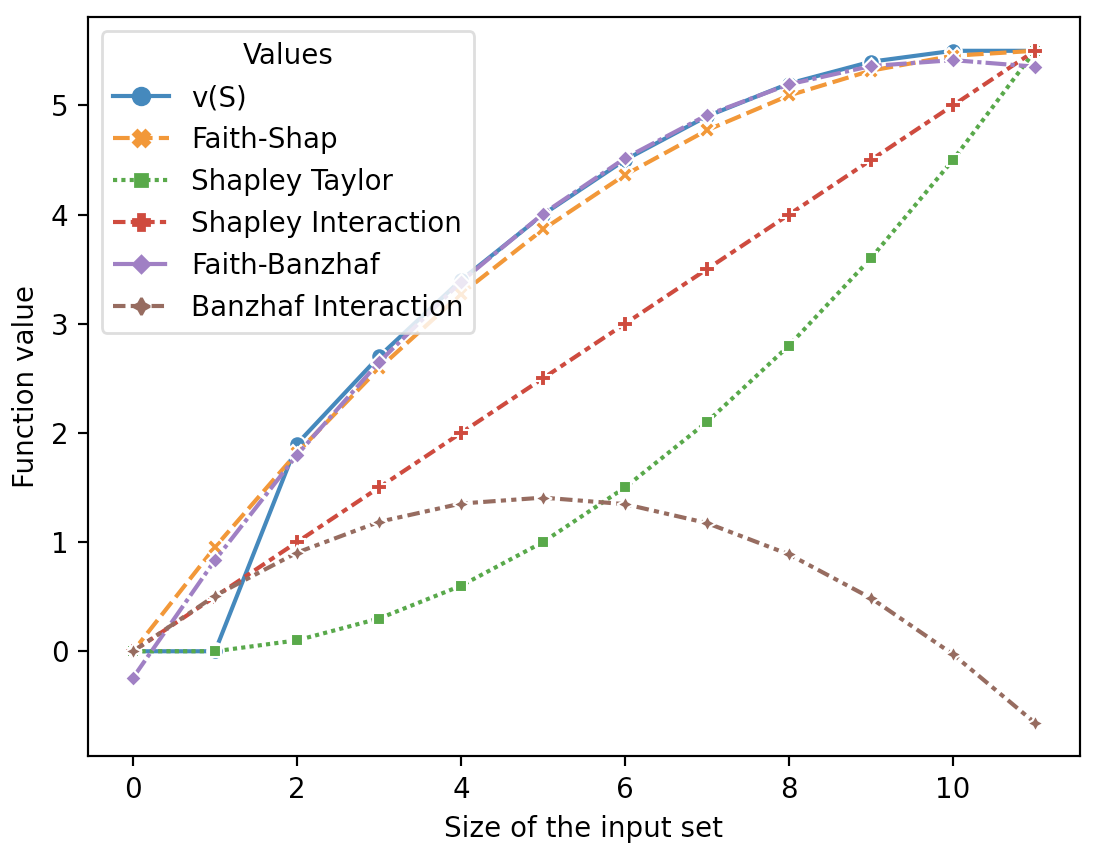}
\captionof{figure}{Function approximation of Eqn.\eqref{eqn:example1} using different interaction indices for $p=0.1$ with the maximum interaction order $\ell=2$.}
\label{fig:example1_1}
\end{minipage}\hfill
\begin{minipage}[b]{0.48\linewidth}
\centering
\includegraphics[width=\linewidth]{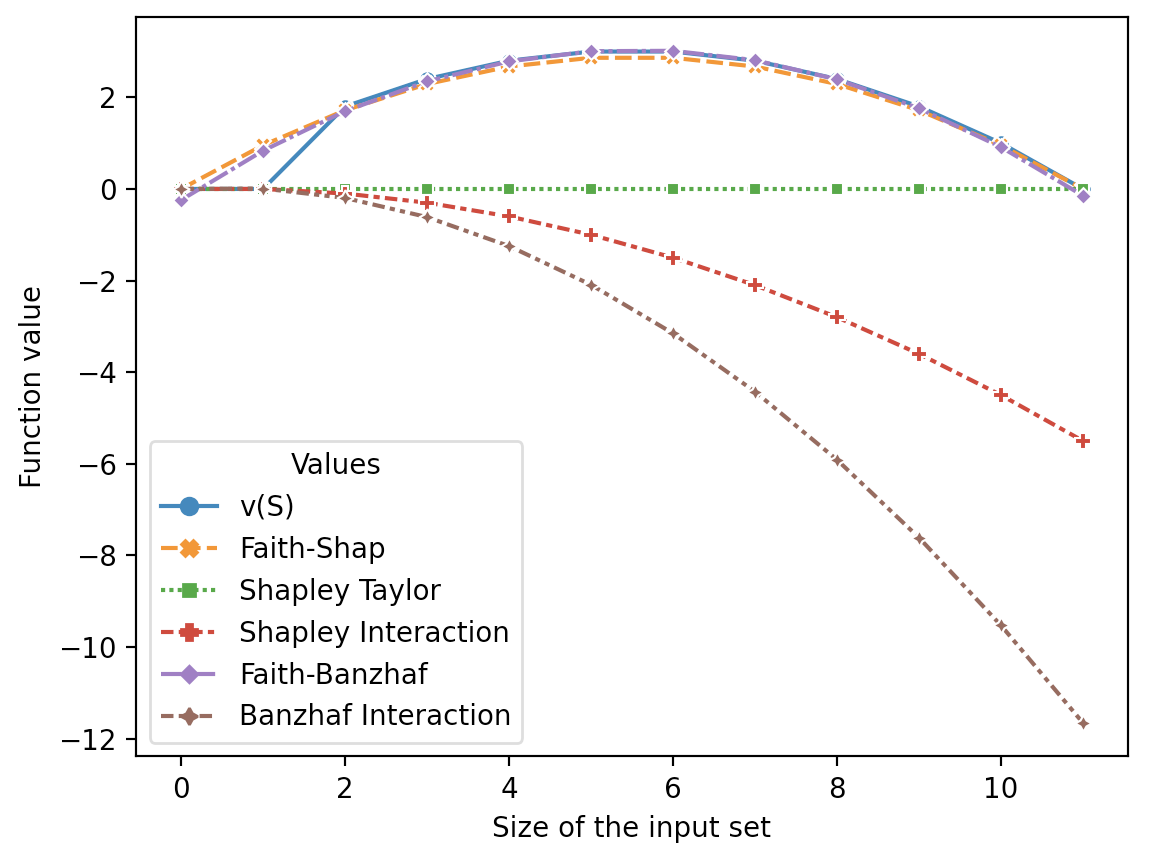}
\captionof{figure}{Function approximation of Eqn.\eqref{eqn:example1} using different interaction indices for $p=0.2$ with the maximum interaction order $\ell=2$. }
\label{fig:example1_2}
\end{minipage}
\end{table}

\paragraph{Example 2:}
We provide another example, this time with increasing marginal utility. Consider a family who is in the wind energy business, with $d=11$ family members. Currently, the family owns 1 wind turbine, and they can get 3 units of revenue per wind turbine they own. Now, each family member is considering whether to manage a wind turbine. To build $x$ wind turbines, the cost is described by the function $\text{cost}(x) = x + 2\log(x+1)$, as they may get a discount from the constructor to build more wind turbines at the same time. If exactly one member chooses to manage a wind turbine, the building cost will be 0 since the family already owns one wind turbine. The total revenue for the family when $S$ is the set of members that participate in building new wind turbines can be described by the following function:

\begin{equation}
\label{eqn:example2}
\f(S) = 
\begin{cases}
0 & \text{ , if } |S| = 0. \\
3 & \text{ , if } |S| \leq 1. \\
3|S| - (|S| - 2 \log(|S|+1)) & \text{ , if } 2 \leq |S| \leq 11.\\
\end{cases}
\end{equation} 
This function has an increasing marginal utility since the marginal cost is decreasing. Therefore, we would expect the interaction effect to be positive. However, from Table \ref{table:example2}, only Faith-Shap, Faith-Banzhaf and Banzhaf interaction indices capture this effect.

Moreover, the Faith-Shap and Faith-Banzhaf indices have the following intuitive interpretation: Having one more member joining the family business increases the total revenue by 1.20/1.19 unit, with 0.07/0.09 additional unit of revenue when two members join together since they are cooperative. In contrast, we can not interpret the Banzhaf interaction index for orders 1 and 2 jointly since it is not cognizant of the maximum interaction order $\ell$.

\begin{table}[ht]
\begin{minipage}[b]{0.58\linewidth}
\centering
\small
\begin{tabular}{ | c | c ?{1pt} c | c |}
    \hline 
    \multirow{2}{*}{Indices} & $\ell= 1$ & \multicolumn{2}{c|}{$\ell=2$} \\
    \cline{2-4}
    & Order 1 & Order 1 & Order 2 \\
    \hline \hline
    Faith-Shap & 1.55 & 1.20 & 0.07 \\
    \hline
    Shapley Taylor & 1.55 & 3 & -0.29 \\
    \hline 
    Shapley Interaction & 1.55 &  1.55  & -0.12 \\
    \hline 
    \hline 
    Faith-Banzhaf & 1.65 & 1.19 & 0.09 \\
    \hline
    Banzhaf Interaction & 1.65 &  1.65 & 0.09 \\
    \hline
   \end{tabular}
\caption{\small{Values for different interaction indices of different orders with the maximum interaction order $\ell=2$. Note that the value function is symmetric with respect to players, so we use order 1 and 2 to denote importance scores of any single player and interaction of any two players. Note that $\ex^{\text{F-Shap}}_{\text{\O}}(\f,\ell) = 0$ and $\ex^{\text{F-Bzf}}_{\text{\O}}(\f,\ell) = 0.48$ for the indices corresponding to empty sets. }}
\label{table:example2}
\end{minipage}\hfill \hfill
\begin{minipage}[b]{0.41\linewidth}
\centering
\includegraphics[width=0.9\linewidth]{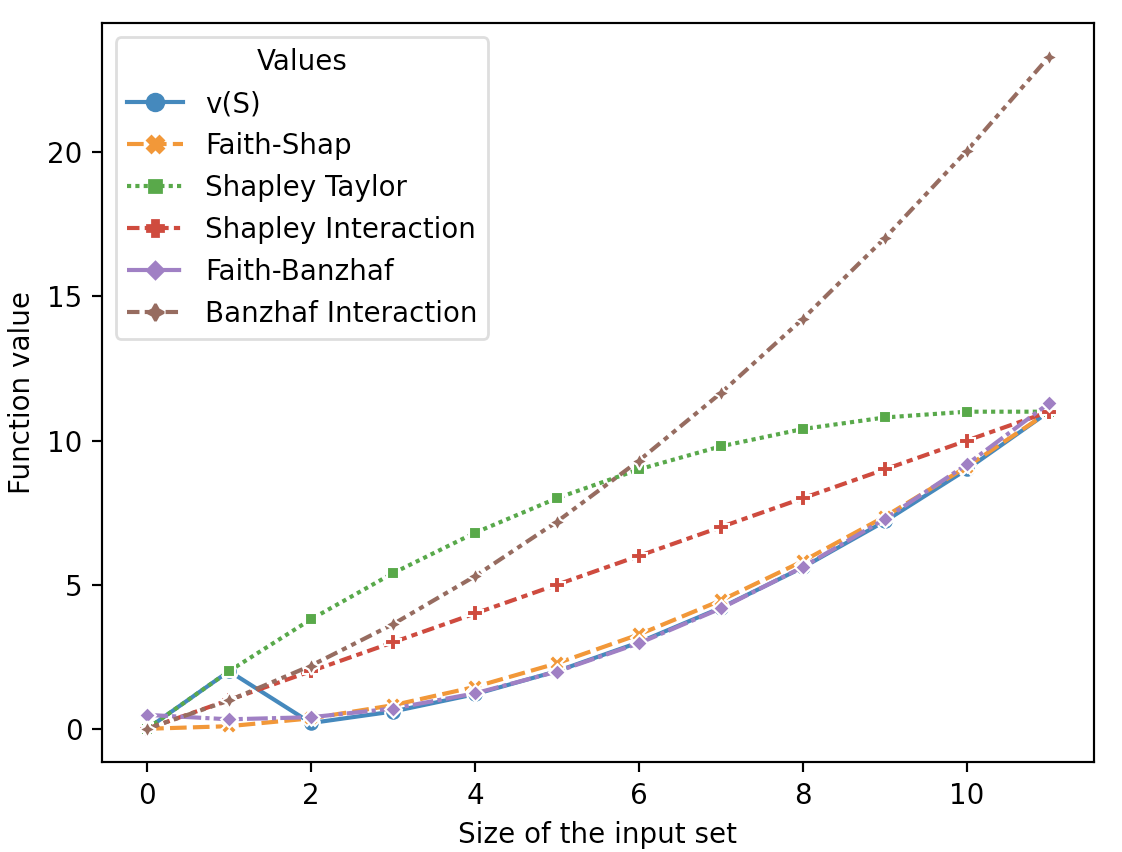}
\captionof{figure}{Function approximation of Eqn.\eqref{eqn:example2} using different interaction indices with the maximum interaction order $\ell=2$.}
\label{fig:example2}
\end{minipage}
\end{table}

\section{Computation of Faithful Shapley Interaction Index}

The exact computation of Faith-Shap indices via Eqn.\eqref{eqn:faith_shapley} is intractable in general since it involves computations of $\mathcal{O}(2^d)$ Möbius transforms of different subsets. However, when the value function is lower-order, the Faith-Shap interaction indices can be computed in polynomial time. 
\begin{definition}
(Order of value functions) A value function $\f$ has order $\ell_{\f} \in \mathbb{N}$ if $\ell_{\f}$ is the smallest integer such that $a(\f, R) = 0$ for all $R \subseteq [d]$ with $|R| > \ell_{\f}$.
\end{definition}

From Eqn.\eqref{eqn:discopose_f}, this entails the value function with order $\ell_{\f}$ can be written as $\f = \sum_{R \subseteq [d], |R| \leq \ell_{\f}}a(\f,R) v_R$. When a game only involves the cooperation of a small number of players, its value function can usually be written as a summation of lower-order basis functions. For instance, a basis function $\f_R$ has order 1,
example value functions in Section \ref{sec:examples} have order 2, the value function~\citep{garg2012novel} defined as summations of pairwise distances within a coalition used in clustering also have order 2. When $\ell_{\f} = \mathcal{O}(\ell) > \ell$, by using Eqn.\eqref{eqn:faith_shapley},
the exact computation of Faith-Shap indices of a subset $S$ only requires time complexity $\mathcal{O}(d^{O(\ell)})$ since we only need to consider ${d \choose \ell_{\f} - |S|} = \mathcal{O}(d^{O(\ell)})$ Möbius transforms of subsets.

For the computation of the Faith-Shap for general functions, computing the exact Faith-Shap values requires $2^d$ model evaluations. We sample each coalition $S \subseteq [d]$ with probability $\propto \frac{d-1}{{d \choose |S|}|S|(d-|S|)}$, and solve 
\begin{align}
\argmin_{\Expl \subseteq \mathbb{R}^{d_\ell} } \frac{1}{n} \sum_{i=1}^n  \left( \f(S) - \sum_{T \subseteq S , |T| \leq \ell}\Expl_T(\f,\ell) \right)^2,
\text{s.t.}\sum_{T \subseteq [d] , |T| \leq \ell}\Expl_T(\f,\ell) = \f([d]), \text{ and } \Expl_{\text{\O}}(\f,\ell) = \f(\text{\O}).
\label{eqn:faith_shap_estimation}
\end{align}
We empirically show the sampling approach provides more accurate estimates with fewer model evaluations in Section \ref{sec:exp_compute}. We defer deriving approximation results for sampling as well as other approximation methods for computing Faith-Shap to future work. There is a rich body of work on developing such approximation results for the first-order Shapley values (see \citet{mitchell2022sampling,covert2020improving}, which could be extended to the Faith-Shap setting.

\section{Algebraic Properties of Faith-Interaction Indices}

In the following two sub-sections, we discuss how Faith-Shap can be represented as a cardinal index, as well as through the lens of a multilinear approximation.

\subsection{Cardinal Indices}

\citet{grabisch1999axiomatic} show that any interaction index (they only consider the classical case with maximum interaction order $\ell = d$) that satisfies the  linearity, dummy, and symmetry axioms necessarily has the following form:
\begin{equation}
\label{eqn:cardinal_interaction indices}
\ex_S(\f, d) = \sum_{T \subseteq [d]\backslash S} p^{|S|}_{|T|} \Delta_{S} \f(T), \ \ \forall S \subseteq [d],
\end{equation}
and for some family of constants $\{p^s_{t}\}_{s \in [0:d],t \in [0:d-s]}$.
They term this class of interaction indices as \emph{cardinal} interaction indices.
Of course this is a large class, and it is not apriori clear how to further constrain the indices so as to get specific values for the constants $\{p^s_{t}\}$. We remark in passing that Shapley and Banzhaf interaction indices impose additional structure on the constants $\{p^s_{t}\}$. 

We can also consider the class of probabilistic interaction indices:
\begin{align*}
    \ex_S(\f,d) = \sum_{T \subseteq [d]\backslash S} p^{S}_{T} \Delta_{S}v(T), 
\end{align*}
where for any $S \subseteq [d]$, the constants $\{p^{S}_{T}\}_{T \subseteq [d]\backslash S}$ form a probability distribution on $[d]\backslash S$. We can then define cardinal-probabilistic indices as those indices that are both cardinal and probabilistic interaction indices, so that $p^{S}_{T} = p^{|S|}_{|T|}$, for some family of constants $\{p^s_{t}\}_{s \in [1:d],t \in [0:d-s]}$ that satisfy:
\[\sum_{t = 0}^{d-s}{d-s\choose t } p^s_t = 1.\]
\citet{fujimoto2006axiomatic} shows that indices that satisfy certain additivity, monotonicity, symmetry, and dummy partnership axioms are necessarily cardinal probabilistic indices. As \citet{fujimoto2006axiomatic} shows, Shapley and Banzhaf interaction indices do fall into this class. 

One could of course extend these notions of cardinal, probabilistic, and cardinal-probabilistic indices to be cognizant of the maximum interaction order $\ell \in [d]$. It is an interesting open question to investigate extensions of results of \citet{fujimoto2006axiomatic} to such a sub-class of cardinal-probabilistic indices cognizant of the max-interaction order. In this section, we provide a modest initial result along these lines, focusing on the top interaction level of the interaction index.
\begin{proposition}
\label{pro:faith_shap_cardinal_prob}
For any maximum interaction order $1 \leq \ell \leq d$, and for any set value function $\f:2^d \mapsto \mathbb{R}$,
the top level of the Faithful Shapley Interaction index can be expressed as a cardinal-probabilistic index:
\begin{equation}
\label{eqn:faith_shap_cardinal_prob}
\ex^{\text{F-Shap}}_{S}(\f,\ell) 
= \sum_{T \subseteq [d] \backslash S} p^\ell_{|T|} \Delta_S(\f(T)), \ \ \forall S \subseteq [d] \text{ with } |S| = \ell,
\end{equation}
where  $p^{\ell}_t = \frac{(2\ell -1)!(\ell+t-1)!(d-t-1)!}{((\ell-1)!)^2 (d+\ell-1)!}$. Moreover, it satisfies
$\sum_{t=0}^{d-\ell} {d-\ell \choose t} p^\ell_t = 1$.
\end{proposition}
Therefore, the top level of the Faithful Shapley Interaction index captures the interactions of features in $S$ \emph{in the presence of all subsets $T \subseteq [d] \backslash S$.} 

\def\g{g}
\subsection{Multilinear Formulation}

Any set value function $\f: 2^{[d]} \mapsto \R$ has a unique multi-linear extension $\g: [0,1]^{d} \mapsto \R$, also referred to the \textit{Owen multilinear extension}~\citep{owen1972multilinear}, given as:
\[
g(x) := \sum_{T \subseteq [d]}\f(T) \prod_{i \in T}x_i \prod_{i \not\in T}(1 - x_i)
,\quad \forall x \in [0,1]^d.
\]
For any set $S \subseteq [d]$, with $S = \{i_1,\hdots,i_s\}$, denote its $S$-derivative as $\Delta_{S}g(x) := \frac{\partial^{s}g(x)}{\partial x_{i_1}\hdots \partial x_{i_s}}$.

\subsubsection{Path Integrals} 
\citet{grabisch2000equivalent} show that Shapley interaction index can be written as:
\[ \ex^{\text{Shap}}_S(\f,d) = \int_{x=0}^{1} \Delta_{S} g(x,\hdots,x) dx, \  \forall S \subseteq [d].
\]
That is, we can obtain the Shapley interaction index by integrating the $S$-derivative along the diagonal of the unit hypercube.

On the other hand, the Banzhaf interaction index can be written as:
\[ \ex^{\text{Bzf}}_S(\f, d) = \int_{x \in [0,1]^d} \Delta_{S} g(x) dx, \  \forall S \in \mathcal{S}_d.
\]
That is, we can obtain the Banzhaf interaction index by integrating the  $S$-derivative over the entire unit hypercube. In this case, it also has the closed form: $\Delta_S g(1/2,\hdots,1/2)$.

\citet{fujimoto2006axiomatic} show that any cardinal probabilistic index $\ex$ has the form:
\[ \ex_S(\f, d) = \int_{x=0}^{1} \Delta_{S} g(x,\hdots,x) dF_{|S|}(x), \  \forall S \in \mathcal{S}_d,
\]
for some family of CDFs $\{F_{s}\}_{s \in [d]}$. That is, we can obtain any cardinal probabilistic index by integrating the $S$-derivative along the diagonal of the unit hypercube with respect to some distribution over $[0,1]$. 

It is an interesting open question whether we could extend these results from \citet{grabisch2000equivalent} and  \citet{fujimoto2006axiomatic} to interaction indices that are cognizant of the maximum interaction order $\ell \in [d]$.  In this section, we provide a modest initial result along these lines, focusing on the top interaction level of the interaction index.

\begin{proposition}
\label{pro:faith_shap_path_integral}
For any maximum interaction order $1 \leq \ell \leq d$, and for any set function $\f:2^d \mapsto \mathbb{R}$, the top level of the Faithful Shapley Interaction index value can be expressed as: 
\begin{equation}
\label{eqn:faith_shap_path_integral}
\ex^{\text{F-Shap}}_{S}(\f,\ell) 
= \int_{x=0}^1 \Delta_S g(x,\cdots,x) dI_x(\ell,\ell), \ \ \forall S \in \setlessell \text{ with } |S| = \ell,
\end{equation}
where $I_x(\ell,\ell)$ is cumulative distribution function of the beta distribution $B(x;\ell,\ell)$.
\end{proposition}

\subsubsection{Taylor Expansion}

In contrast to path integrals, \citet{sundararajan2020shapley} use the Taylor expansion of $g(\mathbf{1}) = \f([d])$ around $g(\mathbf{0}) = \f(\text{\O})$
Taylor derivations to derive their interaction index.
Specifically, they show that Shapley Taylor index $\ex^{\text{Taylor}}_S(\f,\ell)$ is equal to the $|S|^{\text{th}}$ term of the $(\ell-1)^{\text{th}}$ order Taylor expansion of $g(\cdot)$ with Lagrange remainder: 
\begin{align*}
g(\mathbf{1})
& = \sum_{j=0}^{\ell-1} \frac{g^{(j)}(\mathbf{0})}{j!}
g(\mathbf{0}) + \int_{x=0}^{1}\frac{(1-x)^{\ell-1}}{(\ell-1)!}g^{(\ell)}(x,\cdots,x) dx \\
& =  \sum_{j=0}^{\ell-1}
\sum_{|S|=j} \Delta_S g(\mathbf{0})
+ \sum_{|S| = \ell} \int_{x=0}^{1} \ell (1-x)^{\ell-1} \Delta_S g(x,\cdots,x) dx \\ 
& \, \, \text{\citep[Theorem~3]{sundararajan2020shapley}} \\
& =  \sum_{j=0}^{\ell-1}
\sum_{|S|=j}
\ex_S^{\text{Taylor}}(\f,\ell) 
+ \sum_{|S| = \ell}\ex^{\text{Taylor}}_S(\f,\ell),
\end{align*}
where $g^{(j)}(x)$ is the $j^{\text{th}}$ derivative of the function $g(x,\cdots,x)$, $\ex_S^{\text{Taylor}}(\f,\ell)  = \Delta_S g(\mathbf{0})$ for $|S| < \ell$ and $\ex_S^{\text{Taylor}}(\f,\ell)  
= \int_{x=0}^{1} \ell (1-x)^{\ell-1} \Delta_S g(x,\cdots,x) dx$ with $|S| = \ell$.
This can be seen to result in impoverished lower-order subset interactions, which now no longer take into account higher-order coalitions that include that subset. 

\subsubsection{Pseudo-Boolean Function Approximation}

While we have so far discussed the continuous multi-linear extension of a set value function $\f: 2^{[d]} \mapsto \R$, we can also simply consider its equivalent pseudo-Boolean counterpart $\g \in \mathcal{F}$ with $\mathcal{F} =\{ \g: \{0,1\}^{d} \mapsto \R \}$:
\[
\g(x) := \sum_{T \subseteq [d]}\f(T) \prod_{i \in T}x_i \prod_{i \not\in T}(1 - x_i),\quad \forall x \in \{0,1\}^d.
\]
One can also derive the pseudo-Boolean function $\g_{\ex}$ corresponding to interaction indices $\ex$, and ask for interaction indices with pseudo-Boolean counterparts $\g_{\ex}$ that best approximate the pseudo-Boolean counterpart $\g$ of the set value function. Specifically, given a maximum interaction order $\ell \in [d]$ and an interaction index $\ex \in \mathbb{R}^{d_\ell}$, its pseudo-Boolean counterpart $\g_{\ex} \in \mathcal{F}$ is defined as:
$$
\g_{\ex}(x) := \sum_{T \subseteq [d], |T| \leq \ell} \ex_T(\f,\ell)  \prod_{i \in T}x_i, \quad \forall x \in \{0,1\}^d.
$$

\citet{hammer1992approximations} and \citet{grabisch2000equivalent} consider solving for the best $\ell_2$-norm approximation by the function $\g_{\ex}(\cdot)$ with degree up to $\ell$. That is, $\norm{\g-\g_{\ex}}_2 =  \sqrt{\sum_{x \in \{0,1\}^d} (\g(x) - \g_\ex(x))^2}$. 
Using this perspective, we can see that Faith-Banzhaf interaction indices can in turn be related to such a function approximation:
$$
\ex^{\text{F-Bzf}}(\f,\ell)
= \min_{\ex \in \mathbb{R}^{d_\ell}} \norm{g(x) - g_{\ex}(x)}_2
= \min_{\ex \in \mathbb{R}^{d_\ell}} \sum_{S \subseteq [d]} \left(\f(S) - \sum_{T \subseteq S, |T| \leq \ell} \ex_T \right)^2,
$$
and where the solution has the closed-form expression we detail in Theorem \ref{thm:thmdummysymm}. 

For the singleton attribution case, with max order $\ell=1$, \citet{ding2008formulas} and \citet{ruiz1998family} consider $\mu$-norm function approximations $\|g(x) - g_{\ex}(x)\|_{\mu} = \sqrt{\sum_{x \in \{0,1\}^d}\mu(x)(g(x)-g_\ex(x))^2 }$, but where $\mu$ only depends on $\norm{x}_1$, and where $\mu(\mathbf{0})$ and $\mu(\mathbf{1})$ can both be infinity. \citet{ding2008formulas} provide a closed-form expression for $g_\ex(x)$, while  \citet{ruiz1998family} analyze its axiomatic properties.

For the specific case where the probability of coalition $S$ can be expressed as $\mu(x) = \prod_{i: x_i=1} p_i \prod_{j:x_j=0} (1-p_j)$ for some $0 < p_i < 1$ indicating the probability of the feature $i$ being present, 
\citet{ding2010transforms} and \citet{marichal2011weighted} considers solving the best $\ell^{\text{th}}$ order polynomial approximation under  $\|\cdot\|_{\mu}$ norm.

In contrast to the above work, our developments could be cast as pseudo-Boolean approximations for the general weighted norm case $\|\cdot\|_{\mu}$, for general weighting functions $\mu(\cdot)$ without stringent structural assumptions, and while allowing for arbitrary maximum interaction orders $\ell \in [d]$.

\section{Experiments}
We first provide some experiments validating the relative computational efficiency of computing our Faith-Interaction indices, followed by quantitative and qualitative demonstrations of their use as explanations of ML models over a language dataset.

The language dataset we use throughout the experiment is the simplified IMDB \citep{maas2011learning} dataset, where the model only uses the first two sentences of movie reviews as input, and predicts the probability of the reviews being positive. 
The model being explained is a BERT language model \citep{devlin2018bert} with $0.82$ accuracy on the test set.

\subsection{Computational Efficiency}
\label{sec:exp_compute}
Exact computation of interaction indices that aggregate over all possible feature subsets exactly typically requires $2^d$ model evaluations (with $d$ features) which is impractical in most machine learning applications.
A key advantage of our Faith-Interaction indices, as compared to other recently proposed interaction indices such as the Shapley Taylor index and Shapley Interaction index, is that they can be computed by solving a weighted least squares problem. 
As we empirically show in this section, this enables us to provide more accurate estimates with fewer model evaluations, compared to the other recent approaches that employ permutation-based sampling methods.

\begin{wraptable}{r}{8cm}
\centering
\resizebox{0.5\textwidth}{!}{
\begin{tabular}{|c|c|c|} 
\toprule
Methods & Simplified IMDB & Bank marketing \\
\midrule
Shapley Taylor & 2781.2 & 7368.3\\
\midrule
Shapley Interaction & 3960.4 & 10421.7\\
\midrule
Faith-Shap & \textbf{887.4} & \textbf{893.7} \\
\bottomrule
\end{tabular}}
\caption{Run-time comparison of different Shapley values. Each value represents the number of evaluations required to achieve averaged squared distance less than $10^{-3}$.}
\label{table:runtime}
\end{wraptable} 

\paragraph{Setup:}

To demonstrate the computational efficiency of Faith-Interaction indices, we compare our proposed Faith-Shap with Shapley interaction and Shapley Taylor interaction indices using different estimation methods.
For the Faith-Shap interaction index, 
We use Eqn.\eqref{eqn:faith_shap_estimation} and solve the corresponding linear regression problem with $\ell_1$ regularization, and regularization parameter $\alpha = 10^{-3}$ and $\alpha = 10^{-6}$ for the simplified IMDB dataset and the bank dataset. For the Shapley Taylor interaction and Shapley Interaction indices, we use the permutation-based sampling methods ( see exact algorithms in the Appendix \ref{sec:additional_exp}).

We compare these indices in two datasets: 
(1) language dataset: we randomly choose $50$ samples with $d=15$ words from the test set of simplified IMDB and set $\ell=2$. We treat each word in each text sentence as an input feature and set the baseline to empty text so that we simply remove a word if it does not appear in a coalition. We use BERT prediction scores as outputs of the value function.
(2) Portuguese marketing dataset~\citep{moro2014data}: this is a tabular dataset with $d=17$ features. We train an xgboost model~\citep{chen2016xgboost} with $90.5\%$ accuracy on the test set. We also compare them on synthetic sparse functions in the Appendix.

To measure how close an interaction index is to its ground-truth value, we use two evaluation metrics: 
(1) averaged squared distance of all top indices, $\norm{\ex - \ex^{\text{est}}}_2^2/{d \choose \ell}$,
and 
(2) precision at $10$, which we measure the proportion of top-10 feature interactions (with respect to absolute value) in the top-10 ground-truth interactions as top interactions are more critical when these indices are used in XAI.
(3) We also report run-time, as measured by the number of model evaluations required to achieve averaged squared distance to be smaller than $10^{-3}$.

We also note that we drop the lower-order indices and only compare top-order indices (order=$\ell$)
since computing lower-order Shapley Taylor indices are trivial. 
We sampled all $2^d$ different coalitions to compute the ground truth of each index. 
Each evaluation metric is reported by averaging $50$ different inputs with $20$ different random seeds.

\paragraph{Results:}

From Figure \ref{fig:computation_efficiency} and Table \ref{table:runtime}, we see that Faith-Shap can be estimated more accurately and uses fewer model evaluations: in both language data and sparse settings, as well as in terms of all evaluation metrics.

\begin{figure}[!ht]
\centering
      \subfigure[\label{fig:nlp_order2_l2}\small{ Averaged squared distance (language).}]{\includegraphics[width=0.24\textwidth]{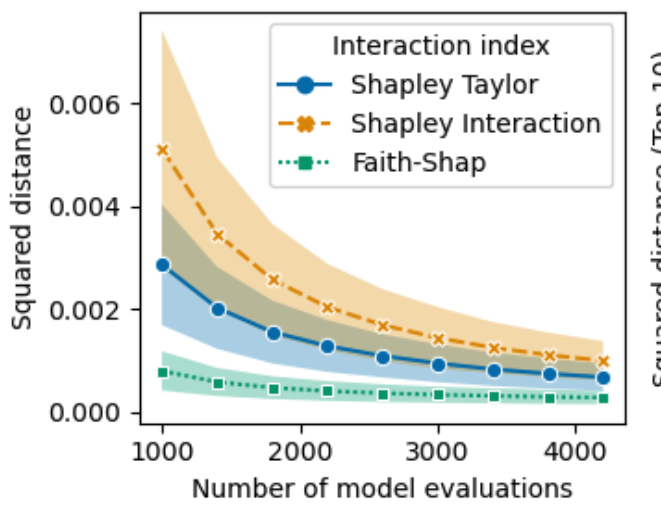}}\hfill
      \subfigure[\label{fig:nlp_order2_prec}\small{ Precision@10  (language).}]{\includegraphics[width=0.24\textwidth]{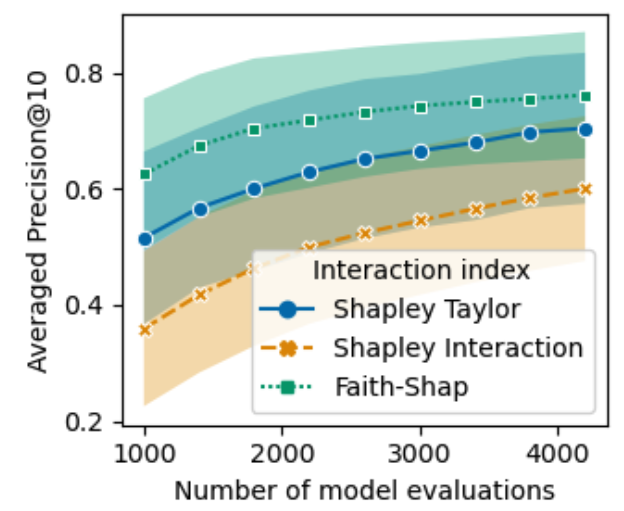}}\hfill
      \subfigure[\label{fig:sparse_l2}\small{ Average squared distance (bank).}]{\includegraphics[width=0.24\textwidth]{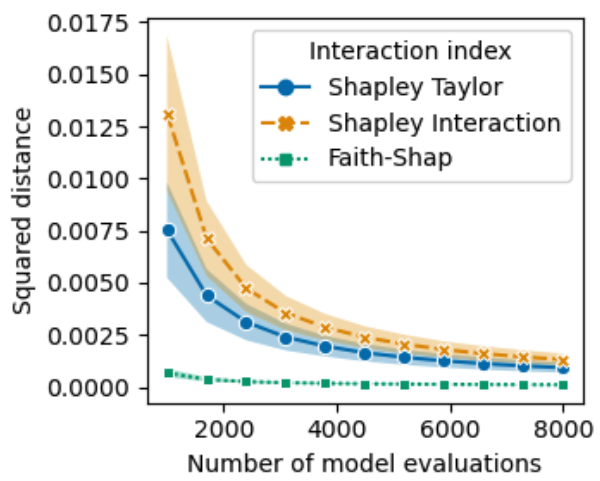}}\hfill
      \subfigure[\label{fig:sparse_prec}\small{ Precision@10 (bank).}]{\includegraphics[width=0.24\textwidth]{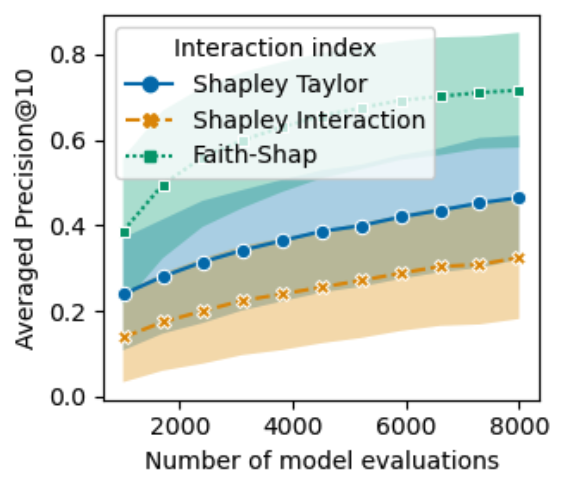}}\hfill
\caption{Comparison of Faith-Shap, Shapley Taylor and Shapley interaction indices
in terms of computational efficiency in language data and synthetic sparse functions. The shaded area indicates the total area between one standard deviation above and below. }
\label{fig:computation_efficiency}
\end{figure}

\subsection{ Explanations on a Language Dataset} \label{sec:nlp}

In this section, we use our Faith-Shap interaction index to explain the BERT model on the simplified IMDB dataset. The experimental setting is the same as described in Section \ref{sec:exp_compute} except that 
we set the maximum interaction order $\ell = 2$, the regularization parameter $\alpha = 10^{-3}$ for Lasso, and sample 4000 coalitions for each text in the simplified IMDB. 
Table \ref{table:imdb} shows some of the interesting interactions we found.

\begin{table*}[h!]
\small
\center
\resizebox{1\linewidth}{!}{
\begin{tabular}{m{0.05\linewidth}|m{0.8\linewidth}|m{0.08\linewidth}|m{0.08\linewidth}} 
\toprule
Index & Sentences (bold words are the interactions with the highest (absolute) importance values) & Model Prediction & Interaction score \\
\midrule
1 & I have \textbf{Never forgot} this movie. All these years and it has remained in my life. & Positive & 0.818 \\
\midrule
2 & TWINS EFFECT is a poor film in so many respects. The \textbf{only good} element is that it doesn't take itself seriously..
& Negative & -0.375 \\
\midrule
3 & I rented this movie to get an easy, entertained view of the history of Texas. I got a \textbf{headache instead}. & Negative & 0.396 \\
\midrule
4 & Truly \textbf{appalling waste} of space. Me and my friend tried to watch this film to its conclusion but had to switch it off about 30 minutes from the end. & Negative & 0.357 \\
\midrule
5 & I still remember watching Satya for the first time. I was completely \textbf{blown away}. & Positive & 0.283 \\
\bottomrule
\end{tabular}}
\caption{Top interactions of different examples on IMDB. See more results in Appendix \ref{sec:additional_exp}.}
\label{table:imdb}
\end{table*}
In the first two examples, we see non-complementary interaction effects. In the first example, while the importance values of the individual words ``Never'' and ``forgot'' are negative (as shown in Tables \ref{tab:index3}, \ref{tab:index5} in the Appendix), their joint effect as shown in the table here is extremely positive. Similarly, for the second, the words ``only'' and ``good'' are individually positive, while their joint effect is strongly negative. The fourth and fifth examples show more subtle non-complementarity effects. In the fourth example, while the individual words ``headache'' and ``instead'' have negative importance scores, their joint effect is positive, since the total effect of the phrase is less than the sum of the individual importance of these two words. The last example shows the effect of complementarity: words in a phrase are only meaningful when all words are present, and hence have a positive interaction effect. 

In Appendix \ref{sec:additional_exp}, we further show the top-15 important interactions and compare them to those from Shapley Taylor index, Shapley interaction index, Integrated Hessian~\citep{janizek2021explaining} 
and Archipelago~\citep{tsang2020extracting}.

 We find that although Faith-Shap, Shapley Taylor index, and Shapley interaction index capture similar feature interactions, the later two methods are not able to find meaningful singleton features. The reasons are (1) the first-order terms of the Shapley Taylor indices are trivial, which is the difference between predicted probabilities of a sentence containing only one word and an empty sentence (a baseline) (2) importance scores for the first order Shapley value and Shapley interaction index are not comparable since the Shapley interaction index does not satisfy the efficiency axiom. For integrated Hessian, we empirically find that the BERT model assigns higher values for self-interactions and punctuation marks.

\section{Related work}

\paragraph{Related work in cooperative theory:} in cooperative game theory, a set function $\f(\cdot)$ with $\f(\emptyset) = 0$ corresponds to a transferable utility game (TU-game), and a set function with order $\leq \ell$ is called an $\ell$-additive TU-game~\citep{grabisch2016set}. Therefore, our approach can be viewed as a least squares approximation of a TU-game by an $\ell$-additive TU-game; see for instance Eqn.~\eqref{eqn:constrained_weighted_regresion}.
Variants and special cases of this least squares approximation problem have been studied in the cooperative game theory field. For $\ell=1$, \citet{charnes1988extremal} first give general solutions when the weighting function is symmetric and positive, %
and show that the Shapley value results from a particular choice of the weighting function. \citet{ruiz1996least,ruiz1998family} consider the same setting, and study the axiomatic properties of the solutions of the least squares problems. \citet{ding2008formulas} further generalizes the previous results by considering the cases where some weights are allowed to be zero.
For the case where maximum interaction order $\ell>1$, \citet{hammer1992approximations} and \citet{ grabisch2000equivalent} solve the least squares problem when the weighting function is a constant, and show that the top-level coefficients coincide with those of the Banzhaf interaction indices of order $\ell$.
\citet{ding2010transforms} and \citet{marichal2011weighted} consider a certain weighted version of the problem, and propose weighted Banzhaf interaction indices. \citet{grabisch2020k} consider the approximation problem under the constraints that both TU-games yield the same Shapley value. 
\citet{marichal1999chaining} extend the Shapley value and propose the chaining interaction index whose definition is based on maximal chains of ordered sets. 
For more details on this line of work, see the recent book~\citep{grabisch2016set}. From the lens of TU-game approximation, our work could be viewed as allowing for general weighting functions $\mu(\cdot)$ without stringent structural assumptions, as well as arbitrary maximum interaction orders $\ell \in [d]$.

While the Shapley value focuses on a fair allocation among players, there exist other solution concepts in cooperative game theory that have different purposes. For example,  core~\citep{gillies1953some} allocates the total payoff in a stable manner, nucleolus~\citep{schmeidler1969nucleolus} is a solution lying in the core with unique axiomatic properties, and the Nash bargaining solution~\citep{nash1950bargaining} focuses on two-player bargaining problems. Extending these concepts to interaction contexts may lead to different solutions with different properties, and are interesting topics for future work.

\paragraph{Feature attribution in XAI:} When Faith-Shap is used in XAI, it can be seen as a local and post-hoc approach that extracts singleton features and feature interaction importances for a given prediction. It can be viewed as an order $\ell$ polynomial model,  with desired axiomatic properties, that explains how a black-box model behaves locally. Explaining complex models with an interpretable \emph{local surrogate model} has been substantially studied in XAI.
LIME~\citep{ribeiro2016should} use a local linear model to describe a prediction made by the model being explained. Model Agnostic Supervised Local Explanations (MAPLE)~\citep{plumb2018model} utilize local linear modeling and dual interpretation of random forests. 
AnchorLIME~\citep{ribeiro2018anchors} uses IF-THEN rules to generate explanations. Model Understanding through Subspace Explanations (MUSE)~\citep{lakkaraju2019faithful} explains how the model behaves in subspaces characterized by certain features of interest. Kernel SHAP~\citep{lundberg2017unified} can be viewed as first-order Faith-Shap. These approaches assign credits to each individual feature based on how much it influences the models' prediction and do not aim to explain how feature interactions affect the model.

\paragraph{Feature interactions in XAI:} Feature interactions have also been investigated in the machine learning community.
\citet{tsang2017detecting} detect feature interactions by examining weight matrices of DNNs. 
\citet{tsang2018neural} disentangle complex feature interactions within DNNs by forcing the weights matrices to be block-diagonal.
\citet{singh2018hierarchical} build hierarchical explanations within a feed-forward neural network using hierarchical clustering of features. 
\citet{cui2019learning} and \citet{janizek2021explaining} explain pairwise interactions in neural networks, and Bayesian neural networks respectively via second-order derivatives. 
\citet{lundberg2020local} quantify feature interactions in tree-based models using the Shapley interaction index. \citet{tsang2020does} proposes Archipelago, which quantifies the interaction within a feature group $S$ via the marginal importance $\f(S) - \f(\emptyset)$. 

While these approaches have taken significant steps towards understanding feature interactions, they are limited to a certain kind of model architecture.
\citet{tsang2017detecting} and \citet{tsang2018neural} can only be applied to feed-forward neural network architectures, but not LSTMs and CNNs. 
While \citet{singh2018hierarchical} can be applied to LSTMs and CNNs, it is unclear how to apply it to recent innovations such as transformers. The approach of  \citet{cui2019learning} can only be applied to Bayesian neural networks, and \citet{janizek2021explaining} can only be applied to models where second-order derivatives exist everywhere. \citet{lundberg2020local} only study tree-based models.
While Archipelago~\citep{tsang2020does} is a post-hoc explanation approach that can be applied to any model, 
Archipelago measures the importance of a feature group as a whole, while Faith-Shap measures the marginal effects of interaction among feature groups. Also, the Archipelago does not obey the dummy axiom and satisfies the efficiency axiom only for certain kinds of functions.

\section{Conclusion}

Deriving unique interaction indices that satisfy the interaction extensions of the individual Shapley axioms has been a long-standing open problem. Existing approaches introduce additional less natural axioms, with some even sacrificing natural ones such as efficiency, in order to specify unique interaction indices. In this work, we take the alternate route of considering the family of what we term faithful interaction indices, which similar to individual Shapley values, aim to approximate the given set value function for all feature subsets. We show that when restricting to the class of faithful interaction indices, we obtain a unique interaction index that satisfies the interaction extensions of the individual Shapley axioms, which we term the Faithful Shapley Interaction Index (Faith-Shap). We  show the benefits of the faithful Shapley interaction index via specific games of interest where there is diminishing return and increasing return and connect the Faith-Shap to cardinal probabilistic indices and multilinear approximations. Finally, we show that Faith-Shap is efficient to estimate thanks to its connection to weighted linear regression in sparse settings, and provide some qualitative results for their use as explanations of machine learning models on a real language dataset.

\section{Acknoledgement}
The authors would like to thank Michel Grabisch for his generous feedback and thank Hung-Hsun Yu for providing assistance in deriving the closed-form solution of Faith-Shap.

\begin{small}\bibliography{ref}
\bibliographystyle{plainnat}
\end{small}

\newpage

\appendix

\section{Organization}

The Appendices contain additional technical content and are organized as follows: 
In Appendix \ref{sec:additional_exp}, we provide details for sampling algorithms for different indices and supplementary results for different setups for the computational efficiency experiment in Section \ref{sec:exp_compute}.
In Appendix \ref{sec:complete_expl}, we give experimental details and show the detailed results of the Faithful Shapley Interaction value and Shapley Taylor indices.
In Appendix \ref{sec:additional_guidance}, we provide additional guidance on Theorem, where we clarify how to choose the parameters $a,b$ to design Faith-Interaction indices. 
In Appendix \ref{sec:aux_thm_results}, we provide auxiliary theoretical results of the Faith-Interaction indices, which will be subsequently used in our proof of main theorems. 
Finally, in Appendix \ref{sec:proofs_propositions}, \ref{sec:proofs_theorems} and \ref{sec:proofs_claims}, we provide the proof of propositions, theorems, and claims respectively.

\section{Experimental Details and Supplementary Results of Computational Efficiency}
\label{sec:additional_exp}
In this section, we provide implementation details of the sampling algorithms for different indices as well as supplementary experimental results for computational efficiency experiments.

The sampling algorithms for the Shapley Taylor and Shapley interaction indices are shown in Algorithm \ref{alg:taylor_shap} and \ref{alg:shap_interaction}. These algorithms are based on the fact that these two indices are the expected value of discrete derivatives over different ordering processes~\citep[Section 2.2]{sundararajan2020shapley}. These algorithms are more efficient since they may use $\f(S)$ of the same coalition $S$ to compute indices of different subsets.
Also, to measure the run-time of each index, we measure its average squared distance every 200/300 model evaluation.

We also measure computation efficiency on the synthetic sparse functions, which is constructed as follows: we parameterize the synthetic sparse function $\f:\{0,1\}^d \rightarrow \mathbb{R}$ with $\sum_{i=1}^N a_i \prod_{j \in S_i} x_j$, where $x = \{x_1,\cdots,x_d \} \in \{0,1 \}^d$ are the input of the value function,  $S_1,S_2,\cdots, S_N$ are subsets of $[d]$ and $a_1,\cdots,a_N$ are coefficients. We set $d=70$, $N=30$, $\ell=2$ and $d=90$, $N=10$, $\ell=2$, sample each $a_i$ uniformly over $[-\frac{i}{10},\frac{i}{10}]$. Each $S_i$ is uniformly sampled over subsets of $|S|$ with sizes $\leq 5$ and $\leq 10$, respectively. We use Eqn.\eqref{eqn:faith_shapley} to compute the ground truth of interaction indices for sparse synthetic functions. The results are shown in Figure \ref{fig:extra_sparse}. We can see that Faith-Shap is more efficient in terms of all evaluation metrics.

\begin{figure}[!ht]
\centering
      \subfigure[\label{fig:appendix_d90_square2}\small{ 
      Averaged squared distance for $d=70$, $\ell=2, N = 30, |S| \leq 5$. }]{\includegraphics[width=0.23\textwidth]{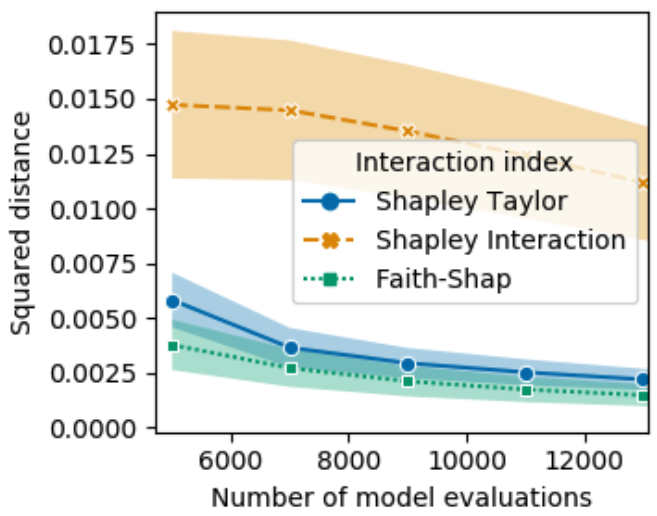}}\hfill
      \subfigure[\label{fig:appendix_d90_precision2}\small{ Precision@10 for $d=70$, $\ell=2, N = 30, |S| \leq 5$.}]{\includegraphics[width=0.23\textwidth]{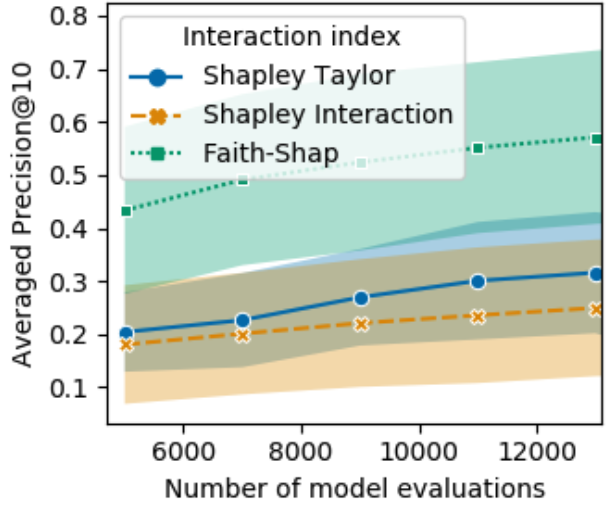}}\hfill
      \subfigure[\label{fig:appendix_d90_square1}\small{ 
      Averaged squared distance for $d=90$, $\ell=2, N = 20, |S| \leq 10$. }]{\includegraphics[width=0.23\textwidth]{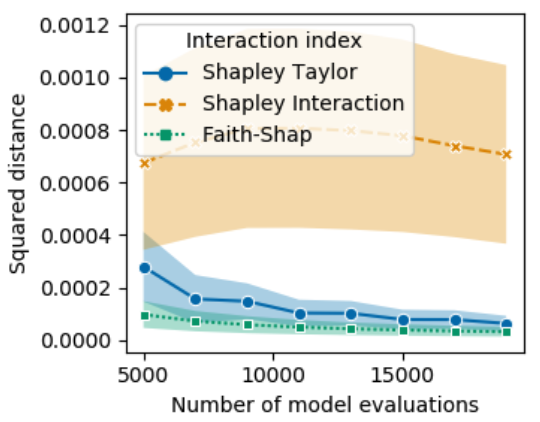}}\hfill
      \subfigure[\label{fig:appendix_d90_precision1}\small{Precision@10 for $d=90$, $\ell=2, N = 20, |S| \leq 10$. }]{\includegraphics[width=0.23\textwidth]{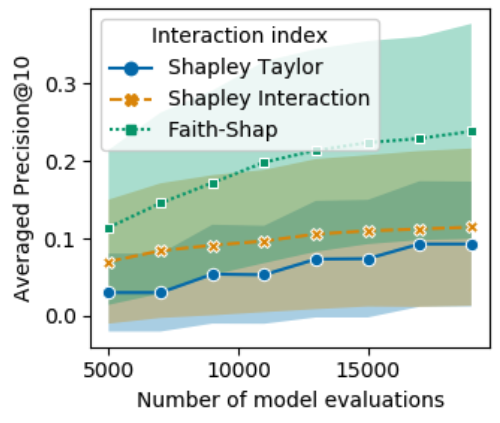}}\hfill
\caption{Comparison of Faith-Shap, Shapley Taylor and Shapley interaction indices
in terms of computational efficiency on synthetic sparse functions for $d=70$, $\ell=2, N = 30, |S| \leq 5$, and $d=90$, $\ell=2, N = 20, |S| \leq 10$.}
\label{fig:extra_sparse}
\end{figure}

\begin{algorithm}[!ht]
\caption{Permutation-based sampling algorithm for the top-order Shapley Taylor index}
\SetKwInOut{Input}{input}
\SetKwInOut{Output}{output}
\Input{a value function $\f:2^d \mapsto \mathbb{R}$, maximum order $\ell$.} 

\SetKwBlock{Beginn}{beginn}{ende}
 \Begin{
    $\text{sum}[S] \leftarrow 0 $  for all sets $S \subseteq [d]$ with size $\ell$. \\
    $\text{count}[S] \leftarrow 0 $ for all sets $S \subseteq [d]$ with size $\ell$.  \\
    \For{$t=1,2,...$}{
        $\pi \leftarrow \{i_1,\cdots, i_d\} $ be a random ordering of $\{1,2,\cdots,d\}$. \\
        \For{ all set $S \subseteq [d]$ with size $\ell$}{
        $i_k \leftarrow$ the leftmost element of $S$ in the ordering $\pi$. \\
        $T \leftarrow \{i_1,\cdots, i_{k-1} \}$ the set of predecessors of $i_k$ in $\pi$. \\
        $\text{sum}[S] \leftarrow \text{sum}[S] + \Delta_S(\f(T)) $. \\
        $\text{count}[S] = \text{count}[S] +1$. \\
        }
    }
    $\text{indices}[S] \leftarrow \text{sum}[S] / \text{count}[S]$  for all sets $S \subseteq [d]$ with size $\ell$. \\
    \Return $\text{indices}$  \\
}
\label{alg:taylor_shap}
\end{algorithm}

\begin{algorithm}[!ht]
\caption{Permutation-based sampling algorithm for the Shapley Interaction index.}
\SetKwInOut{Input}{input}
\SetKwInOut{Output}{output}
\Input{a value function $\f:2^d \mapsto \mathbb{R}$, maximum order $\ell$.} 

\SetKwBlock{Beginn}{beginn}{ende}
 \Begin{
    $\text{sum}[S] \leftarrow 0 $  for all sets $S \subseteq [d]$ with size $\ell$. \\
    $\text{count}[S] \leftarrow 0 $ for all sets $S \subseteq [d]$ with size $\ell$.  \\
    \For{$t=1,2,...$}{
        $\pi \leftarrow \{i_1,\cdots, i_d\} $ be a random ordering of $\{1,2,\cdots,d\}$. \\
        \For{$k= 1,\cdots,d-\ell+1$}{
        $S \leftarrow \{i_k,\cdots,i_{k+\ell-1} \}$. \\
        $T \leftarrow \{i_1,\cdots, i_{k-1} \}$ the set of predecessors of $i_k$ in $\pi$. \\
        $\text{sum}[S] \leftarrow \text{sum}[S] + \Delta_S(\f(T)) $. \\
        $\text{count}[S] = \text{count}[S] +1$. \\
        }
    }
    $\text{indices}[S] \leftarrow \text{sum}[S] / \text{count}[S]$  for all sets $S \subseteq [d]$ with size $\ell$. \\
    \Return $\text{indices}$  \\
}
\label{alg:shap_interaction}
\end{algorithm}

\newpage

\section{Experimental details for Language Dataset}
\label{sec:complete_expl}

For the dataset, the Internet Movie Review Dataset (IMDb)~\citep{maas2011learning} consists of 50,000 binary labeled movie reviews. Each review is annotated as a positive or negative review. We used 25,000 reviews for training and 25,000 reviews for evaluation. 

Here, the set function $\f(x)$ represents the predicted probability of input texts being positive sentiment, which is between 0 and 1. We remove a word in a text sequence if the corresponding entry of the word in a binary perturbation variable $x$ is 0. 
we use 4000 samples to estimate both Faithful Shapley Interaction indices and Shapley Taylor indices. We use Lasso with regularization parameter $\alpha = 0.001$ to estimate 
Faithful Shapley Interaction indices and permutation-based sampling method to estimate the highest order Shapley Taylor indices $(\ell=2)$.

For comparison with other feature interactions methods in XAI, we provide top-15 important features (interactions) for Faith-Shap, Shapley Taylor interaction indices, Shapley Interaction indices, Integrated Hessian~\citep{janizek2021explaining}, and Archipelago~\citep{tsang2020extracting}
in Table \ref{tab:index1} to \ref{tab:index5}. We note that Archipelago is run with the number of interactions $k=3$, and its usage is slightly different than our methods: it constructs a feature hierarchy as an explanation rather than measuring the importance scores of feature interactions.

\include{tables}

\section{Additional Guidance on Theorem \ref{thm:thmdummysymm} }
\label{sec:additional_guidance}
In this section, we clarify how to use Theorem \ref{thm:thmdummysymm} to design Faith-Interaction indices satisfying interaction linearity, symmetry, and dummy axioms by first explaining Theorem \ref{thm:thmdummysymm} and then providing some examples.

Theorem \ref{thm:thmdummysymm} states that the finite weighting function must be in the following form:
\begin{align*}
\mu(S) & \propto \sum_{i=|S|}^{d} {d- |S| \choose i-|S|}(-1)^{i-|S|} g(a,b,i), \  \text{ where }
g(a,b,i) = 
\begin{cases}
1 & \text{ , if } \ i = 0. \\
\prod_{j=0}^{j=i-1} \frac{a(a-b) + j(b-a^2)}{a-b + j(b-a^2)}
& \text{ , if } \   1 \leq i \leq d. \\
\end{cases}
\end{align*} 
for some $a,b \in \R^{+}$ with $a>b$ such that $\mu(S) > 0$ for all $S \subseteq [d]$. 

To better understand this formula, some questions need to be answered: 
(1) What kind of $a,b$ makes $\mu(S) > 0$ for all $S \subseteq [d]$? 
(2) What is the physical meaning of the parameters $a$ and $b$?

To answer (1), we show that a simple condition $1 \geq a > b \geq a^2 > 0$ suffices to make $\mu(S) > 0$ for all $S \subseteq [d]$.
\begin{proposition}
\label{pro:ab_condition}
When $a,b \in \mathbb{R}^+$ such that $1 \geq a > b \geq a^2 > 0$, we have 
\begin{align*}
\sum_{i=|S|}^{d} {d- |S| \choose i-|S|}(-1)^{i-|S|} g(a,b,i) > 0
\ \ \text{ for all } S \subseteq [d],
\end{align*} 
where $g(a,b,i)$ is defined in Eqn.\eqref{eqn:thmdummysymm}. 
\end{proposition}
We delayed the proof of this proposition to Appendix \ref{sec:proofs_propositions}. We note that it is only a sufficient condition for selecting $a$ and $b$: For some small $d \in \mathbb{N}$, we may have some $a,b$ such that $1 > a^2 > b >0$ but makes $\mu(S) > 0 $ for all $S \subseteq [d]$. However, if $a = \mubar{1}$ and $b = \mubar{2}$ need to make the weighting function positive for all $d \in \mathbb{N}$, we must have the condition  $1 \geq a > b \geq a^2 > 0$.

For question (2), we show that $\mubar{i} = g(a,b,i)
= \sum_{L \supseteq S} \mu(L)$ for subsets $S \subseteq [d]$ with $|S| = i$  in the proof in Section \ref{sec:proof_theorem_dummysymm}. Here, $\mubar{i}$ is defined as the total weight of coalitions containing a group of features of size $i$ (any group with size $i$ will work due to the interaction symmetry axiom).
By plugging in $i=1,2$, we get $a = \mubar{1}$ and $b = \mubar{2}$ are the total weights of coalitions containing a single feature and a pair of features. 

In the following, we give some special cases with particularly chosen $a$ and $b$ to provide an intuition of Theorem \ref{thm:thmdummysymm},

\begin{example}
\label{exmp:banzhaf}
When $a=0.5$ and $b=0.25$, the weighting function $\mu(\cdot)$ with respect to Theorem \ref{thm:thmdummysymm} is $\mu(S) = 1/2^d$ for all $S \subseteq [d]$. In this case, the explanations $\Expl_T(\f,\ell)$ equals the Banzhaf Interaction value up to order $\ell$ for all $|T| = \ell$, which has the form $\ex_T(\f,\ell) = \sum_{S \subseteq [d] \backslash T} \Delta_T \f(S) / 2^{d-|S|}$.
\end{example}

In this example, the Banzhaf interaction value satisfies interaction linearity, symmetry, and dummy axioms \cite{fujimoto2006axiomatic}, which coincides with Theorem \ref{thm:thmdummysymm}.
We also provide another guideline to design the values of $a, b$ based on the desired $\frac{\mu_{d}}{\mu_{d-1}}$ and $\frac{\mu_{d-1}}{\mu_{d-2}}$, where $\mu_i = \mu(S)$ when $|S| = i$. \footnote{ $\mu_i$ can be defined since the interaction symmetry axiom ensures that all coalitions with equal size have equal weights.}

\begin{proposition}
\label{pro:guideline_ab}
\begin{equation*}
\text{Let }
\frac{\mu_{d}}{\mu_{d-1}} = r_1 
\ \text{ and } \ 
\frac{\mu_{d-1}}{\mu_{d-2}} = r_2 
\ \text{ with} \ 
r_1 > r_2 > \frac{(d-2)r_1}{ d -1 + r_1} > 0,
\end{equation*} 
then $a$ and $b$ can be represented as functions of $r_1$ and $r_2$:
\begin{equation}
\label{eqn:r1r2_to_ab}
a = \frac{ r_1(r_2+1) -(d-1)(r_1-r_2) }
{(r_1+1)(r_2+1) - (d-1)(r_1-r_2)  } 
\ \text{ and } \ 
b = \frac{ r_1(r_2+1)  -(d-2)(r_1-r_2) }
{(r_1+1)(r_2+1) - (d-2)(r_1-r_2) } a.
\end{equation}
In this case, $a$ and $b$ satisfy $1 > a > b \geq a^2 > 0$, which implies $\mu_i > 0$ for all $0 \leq i \leq d$ .
\end{proposition}

This proposition provides a guideline to design a unique interaction value that satisfies interaction linearity, symmetry, dummy axioms  based on given values of $\frac{\mu_{d}}{\mu_{d-1}}$ and $\frac{\mu_{d-1}}{\mu_{d-2}}$. For example, if the coalition $\mu_t$ has a higher probability to form when $t$ is large, such as the case when the features of an image is explained. As an example, we may set $\frac{\mu_{d}}{\mu_{d-1}} = 10$. We then have $10 > r_2 > \frac{d-2}{d+9} 10$, and we can set $r_2 = 9$ when $d<101$. This narrows down a unique interaction value that satisfies these three axioms and the conditions of $\frac{\mu_{d}}{\mu_{d-1}} = 10$ and $\frac{\mu_{d-1}}{\mu_{d-2}} = 9$.

\newpage

\section{Auxiliary Theoretical Results}
\label{sec:aux_thm_results}

In this section, we provide auxiliary theoretical results of the Faith-Interaction indices. These properties are useful in the proof of our main theorems. The proof are delayed to Appendix \ref{sec:proofs_propositions}

First of all, we show that if the coalition weighting function $\mu(\cdot)$ is finite,  Eqn.\eqref{eqn:weighted_regression} is strictly convex.

\begin{proposition}
\label{pro:strictly_convex}
If the coalition weighting function $\mu(\cdot)$ is finite such that $\mu(S) \in \mathbb{R}^+ $ for all $S \subseteq [d]$, Eqn.\eqref{eqn:weighted_regression} is strictly convex.
\end{proposition}
Given that Eqn.\eqref{eqn:weighted_regression} is strictly convex, we next show that the minimization problems have a unique minimizer.
\begin{proposition}
\label{pro:unique_minimizer}
The (constrained) regression problems defined in  Eqn.\eqref{eqn:constrained_weighted_regresion} with a proper weighting function $\mu$ ( Definition \ref{def:proper_weighting_f})
have a unique minimizer.
\end{proposition}
This proposition is a straightforward application of the following fact: For a minimization problem with linear constraints, if the objective is strictly convex, then it has a unique minimizer.

Also, we note that having a positive measure for all subsets of $[d]$ on the weighting function $\mu(\cdot)$ is necessary to ensure the uniqueness of the minimizer. Consider the case when 
the maximum interaction order equals the number of features, i.e. $\ell=d$, there are $2^{d}$ variables with $2^{d}$ equalities. That is, $\f(S) -\sum_{T \subseteq S} \ex_S(\f,d) = 0$ for all $S \subseteq [d]$. In this case, we can not have any 
$S \subseteq [d]$ such that $\mu(S) = 0$ due to the lack of equations. 

In this special case of $\ell=d$, we have the following closed-form expression. We note that these results are independent of the weighting function as long as we have $\mu(S) > 0$ for all $S \subseteq [d]$.

\begin{proposition}
\label{pro:ell_equals_d}
When the maximum interaction order $\ell=d$, the minimizer of Eqn.\eqref{eqn:constrained_weighted_regresion} the Möbius transform of $\f$, i.e. $\ex_S(\f,d) = a(\f,S) = \sum_{T \subseteq S} (-1)^{|S| - |T|} \f(T) $ for all subsets $S \subseteq [d]$.
\end{proposition}

Then we provide the expression of partial derivatives of the objective in Eqn.\eqref{eqn:weighted_regression} with respect to each variable $\ex_A(\f,\ell)$ for all $A \subseteq [d]$ with $|A| \leq \ell$.

\begin{proposition}
\label{pro:parital_derivative_general}
The partial derivative of Eqn.\eqref{eqn:weighted_regression} with respect to $\Expl_A(\f, \ell)$ is 
\begin{equation}
\label{eqn:partial_derivative_general}
-2 \sum_{ \substack{S: S \supseteq A, \\ \mu(S) < \infty}} \mu(S) \f(S) + 2\sum_{S \in \setlessell} \Expl_S(\f, \ell) \sum_{\substack{L: L \supseteq S \cup A, \\ \mu(L) < \infty}} \mu(L)
\ \ \text{ for all } \ \ 
A \in \setlessell.
\end{equation}
\end{proposition}
This proposition is frequently used in our proof as we solve the minimization problem. Next, the following proposition illustrates how to solve the constrained regression problem via Lagrangian.

\begin{proposition}
\label{pro:constrained_closed_matrix_form_solution} 
Any Faith-Interaction index $\Expl(\f,\ell)$ with respect to a proper weighting function $\mu(\cdot)$ with $\mu(\text{\O}) = \mu([d]) = \infty$ has the form:
\begin{equation}
\small
\begin{bmatrix}
\lambda_{\text{\O}} \\
\lambda_{[d]} \\
\ex_{\text{\O}}(\f,\ell) \\
\cdots \\
\ex_{S}(\f,\ell) \\
\ex_{T}(\f,\ell) \\
\cdots \\
\end{bmatrix}
=
\underbrace{\begin{bmatrix}
0, & 0, & 1,  & \cdots, & 0, & 0, &\cdots \\ 
0, & 0, & 1,  & \cdots, & 1, & 1, & \cdots \\
-\frac{1}{2}, & -\frac{1}{2}, & \bar{\mu}(\text{\O}), & \cdots, & \bar{\mu}(S), & \bar{\mu}(T), & \cdots \\
\cdots,  & \cdots, & \cdots, & \cdots, & \cdots, & \cdots, & \cdots, \\
0, & -\frac{1}{2}, & \bar{\mu}(S), & \cdots, & \bar{\mu}(S), & \bar{\mu}(S \cup T), & \cdots \\
0, & -\frac{1}{2}, & \bar{\mu}(T), & \cdots, & \bar{\mu}(S \cup T), & \bar{\mu}(T), & \cdots \\
\cdots,  & \cdots, & \cdots, & \cdots, & \cdots, & \cdots, & \cdots \\
\end{bmatrix}^{-1}}_{\mathbf{M}^{-1}}
\underbrace{\begin{bmatrix}
\f(\text{\O}) \\
\f([d]) \\
\bar{\f}(\text{\O}) \\
\cdots \\
\bar{\f}(S) \\
\bar{\f}(T) \\
\cdots \\
\end{bmatrix}}_{\mathbf{y}},
\end{equation}
where $\lambda_{\text{\O}}$ and $\lambda_{[d]}$ are Lagrange multipliers with respect to the constraints on the empty set and the full set, 
$\bar{\mu}(S) = \sum_{L \supseteq S , \mu(L) < \infty} \mu(L)$, and $\bar{\f}(S) = \sum_{L \supseteq S , \mu(L) < \infty} \mu(L)\f(L)$.

Formally, the matrix $\mathbf{M} \in \mathbb{R}^{(d_\ell+2) \times (d_\ell+2)}$ and the vector $\mathbf{y} \in \mathbb{R}^{d_\ell+2}$ have the following definitions: we overuse the notations $\lambda_{\text{\O}},\lambda_{[d]}$ and let the rows and columns of $\mathbf{M}$ are indexed by $\{\lambda_{\text{\O}},\lambda_{[d]}, \text{\O},\cdots,S,T,\cdots \}$, which are corresponding to variables $\lambda_{\text{\O}},\lambda_{[d]}, \ex_{\text{\O}(\f,\ell),\cdots, \ex_S(\f,\ell),\ex_T(\f,\ell)}$
\begin{equation*}
\label{eqn:M_y_definitions}
\mathbf{M}_{S,T} =
\begin{cases}
1 & \text{ if } S = (\lambda_{\text{\O}}) \wedge( T = \text{\O}). \\
0 & \text{ if } (S = \lambda_{\text{\O}} )\wedge ( T \neq \text{\O}). \\
1 & \text{ if } (S = \lambda_{[d]}) \wedge (T \subseteq \mathcal{S}_\ell). \\
0 & \text{ if } (S = \lambda_{[d]}) \wedge (T \in \{\lambda_{\text{\O}},\lambda_{[d]} \}). \\
-\frac{1}{2} & \text{ if } (S = \emptyset) \wedge (T = \lambda_{[d]}).  \\
0 & \text{ if } (S \in \mathcal{S}_{\ell} \backslash \text{\O}) \wedge ( T =  \lambda_{\text{\O}} ). \\
-\frac{1}{2} & \text{ if } (S \in \mathcal{S}_{\ell}) \wedge (T = \lambda_{[d]}). \\
\bar{\mu}(S \cup T) & \text{, otherwise.} 
\end{cases}, 
\ \text{ and } \ 
\mathbf{y}_{S} =
\begin{cases}
\f(\text{\O}) & \text{ if } S = \lambda_{\text{\O}}. \\
\f([d]) & \text{ if } S = \lambda_{[d]} \\
\bar{\f}(S) & \text{ otherwise. }\\
\end{cases}
\end{equation*}
where we use $\mathbf{M}_{S,T}$ to denote the entry of the intersection of $S^{th}$ row and $T^{th}$ column.
\end{proposition}

\newpage

\section{Proof of Propositions}
\label{sec:proofs_propositions}

In this section, we provide the proof of theorems and propositions in Section \ref{sec:faithful_interaction_indices}. Before going to the main proof, we introduce some new notations.

\subsection{Proof of Proposition \ref{pro:closed_matrix_form_solution}}
\begin{proof} 
Now we transform the problem into a linear regression problem using matrix representations: let the feature matrix 
$$
\mathbf{X} \in \{ 0,1\}^{2^{d} \times d_\ell}
\text{  indexed with  } \mathbf{X}_{S,T} = \mathbbm{1}[(T \subseteq S) \vee (T = \text{\O})], 
\text{ where } S \subseteq [d]
\text{ and } T \in\mathcal{S}_\ell.
$$
We note that the feature matrix $\mathbf{X}$ is indexed with two sets $S$ and $T$, denoting its rows and columns. Each row of $S$ can also be expressed as $\mathbf{X}_S = \expd(S)$, where $\expd(S) \in \mathbb{R}^{d_\ell}$ with $\expd(S)[T] = \mathbbm{1}[(T \subseteq S) \vee (T = \text{\O})]$.

Then we define the weight matrix:
$$
\sqrt{\mathbf{W}} \in \{ 0,1\}^{2^{d} \times 2^{d}}
\text{ is a diagonal matrix with each entry on the diagonal } \sqrt{\mathbf{W}}_{S,S} = \sqrt{\mu(S)},
$$
where $S \subseteq [d]$.
The function values of $\f(\cdot)$ on each subset can be written into a vector:
$$
\mathbf{Y} \in \mathbb{R}^{2^{d}}
\text{  indexed with }
\mathbf{Y}_S = \f(S)
\text{ where } S \subseteq [d].
$$
With the above definitions, Equation \eqref{eqn:weighted_regression_in_proof} can be viewed as
$$
\min \norm{\sqrt{\mathbf{W} }(\mathbf{Y} -\mathbf{X} \Expl(\f,\ell))}_2^2 
= \min \norm{\mathbf{Y_w}-\mathbf{X_w}\Expl(\f,\ell)}_2^2,
$$
where $\mathbf{Y_w} = \sqrt{\mathbf{W} }\mathbf{Y}$ and $\mathbf{X_w} = \sqrt{\mathbf{W} }\mathbf{X}$. This is a linear regression problem with $\mathbf{X_w}$ being design matrix and $\mathbf{Y_w}$ being the response vector. Since it has a unique minimizer by Proposition \ref{pro:unique_minimizer}, we can apply the closed-form solution:
\begin{align*}
\ex(\f,\ell)
& = \left(\mathbf{X_w}^{T} \mathbf{X_w} \right)^{-1}
\mathbf{X_w}^T \mathbf{Y_w}  \\
& = \left(\sum_{S \subseteq [d]} \sqrt{\mu(S)}{\expd}(S)  {\sqrt{\mu(S)}\expd}(S)^T \right)^{-1}\!\!\!\sum_{S \subseteq [d]}  \left(\sqrt{\mu(S)}{\expd}(S) \right) \left(\sqrt{\mu(S)}\f(S) \right), \\
& = \left(\sum_{S \subseteq [d]} \mu(S){\expd}(S) {\expd}(S)^T \right)^{-1}\!\!\!\sum_{S \subseteq [d]} \mu(S) \f(S) {\expd}(S).
\end{align*}

\end{proof}

\subsection{Proof of Proposition \ref{pro:linearity}}
\begin{proof}
By Proposition \ref{pro:unique_minimizer}, Faith-Interaction indices are the unique minimizer for some weighting function $\mu(\cdot)$. In below, we discuss two cases: when $\mu(S)$ is finite for all subsets $S$ and when $\mu(S)$ is infinity for some subsets $S$.

First, when the coalition weighting function is finite such that $\mu(S) < \infty$ for all subsets $S \subseteq [d]$, by Proposition \ref{pro:closed_matrix_form_solution}, $\ex(\f,\ell)$ has a linear relation with respect to the set function $\f(\cdot)$. Therefore, it satisfies the interaction linearity axiom.  

Next, we solve the case when some coalition function $\mu(\cdot)$ has an infinity measure on some subsets $S$. Denote $\mathcal{T} = \{T: \mu(T) = \infty, T \subseteq [d] \}$ be the set containing all subsets of $[d]$ with infinity weights. We obtain the unique minimizer by solving this constrained minimization problem with a Lagrange multiplier. Specifically, we denote $F(\ex) = \sum_{S:S \subseteq [d], \mu(S) < \infty}  \mu(S) \left( \f(S) - \sum_{T \subseteq S , |T| \leq \ell}\Expl_T(\f,\ell) \right)^2$ as our objective and solve the following $d_\ell + |\mathcal{T}|$ equations:
\begin{equation}
\label{eqn:lagrangian}
\begin{cases}
\frac{\partial F(\ex)}{\partial \ex_A } = \sum_{T: T \subseteq \mathcal{T} \wedge T \supseteq A } \lambda_T 
& \text{ for all } A \in \setlessell.
\\
\f(A) - \sum_{T \subseteq A , |T| \leq \ell}\Expl_T(\f,\ell) = 0 
& \text{ for all } A \subseteq \mathcal{T}.
\end{cases}    
\end{equation}

By Definition \ref{def:Dpq} and Proposition \ref{pro:parital_derivative_general}, the partial derivative of Eqn.\eqref{eqn:weighted_regression} with respect to $\Expl_A(\f, \ell)$ is 
$$
\frac{\partial F(\ex)}{\partial \ex_A } =-2 \sum_{ \substack{S: S \supseteq A, \\ \mu(S) < \infty}} \mu(S) \f(S) + 2\sum_{S \in \setlessell} \bar{\mu}(S \cup A) \Expl_S(\f, \ell).
$$
Combing the above, we solve the following equations:
\begin{align}
\label{eqn:lagrangian2}
(i) \ \  & -2 \sum_{ \substack{S: S \supseteq A, \\ \mu(S) < \infty}} \mu(S) \f(S) + 2\sum_{S \in \setlessell} \bar{\mu}(S \cup A) \Expl_S(\f, \ell) = \sum_{T: (T \subseteq \mathcal{T} )\wedge (T \supseteq A) } \lambda_T 
& ,\text{ for all } A \in \setlessell. \nonumber \\
(ii) \ \  & \f(A) - \sum_{T \subseteq A , |T| \leq \ell}\Expl_T(\f,\ell) = 0 
& ,\text{ for all } A \subseteq \mathcal{T}.
\end{align}
Denote $\lambda = [\lambda_T]_{T \subseteq \mathcal{T}} \subseteq \mathbb{R}^{|\mathcal{T}|}$ as a vector consisting of all multiplier $\lambda_T$.
Let $[\lambda^{(1)}, \Expl(\f_1,\ell)],[\lambda^{(2)},\Expl(\f_2,\ell)] \in \mathbb{R}^{|\mathcal{T}|+ d_\ell}$ be the solution of Eqn.\eqref{eqn:lagrangian2}
( the minimizers of Eqn.\eqref{eqn:constrained_weighted_regresion}) with respect to set functions $\f_1(\cdot)$ and $\f_2(\cdot)$. 

Now we prove that $[\lambda^{(1+2)}, \Expl(\f_{1+2},\ell)] =  [\alpha_1\lambda^{(1)} + \alpha_2 \lambda^{(2)}, \alpha_1\Expl(\f_1,\ell) + \alpha_2 \Expl(\f_2,\ell)]$ is the solution of Eqn.\eqref{eqn:lagrangian2} with respect to the function $\f_{1+2}= \alpha_1 \f_1 + \alpha_2 \f_2$.

First of all, for equation (i), $\lambda$, $\ex(\f,\ell)$ and $\f$ have linear relation. Therefore, we have, for all $A \in \setlessell$, 

\begin{align*}
&-2 \sum_{ \substack{S: S \supseteq A, \\ \mu(S) < \infty}} \mu(S) \f_{1+2}(S) + 2\sum_{S \in \setlessell} \bar{\mu}(S \cup A) \Expl_S(\f_{1+2}, \ell) \\
& = \alpha_1 \left(-2 \sum_{ \substack{S: S \supseteq A, \\ \mu(S) < \infty}} \mu(S) \f_{1}(S) + 2\sum_{S \in \setlessell} \bar{\mu}(S \cup A) \Expl_S(\f_{1}, \ell) \right)  \\
& + \alpha_2 \left(-2 \sum_{ \substack{S: S \supseteq A, \\ \mu(S) < \infty}} \mu(S) \f_{2}(S) + 2\sum_{S \in \setlessell} \bar{\mu}(S \cup A) \Expl_S(\f_{2}, \ell) \right)  \\
& = \alpha_1 \sum_{T: (T \subseteq \mathcal{T} )\wedge (T \supseteq A) } \lambda^{(1)}_T + \alpha_2 \sum_{T:(T \subseteq \mathcal{T} )\wedge (T \supseteq A) } \lambda^{(2)}_T   \\
& = \sum_{T: (T \subseteq \mathcal{T} )\wedge (T \supseteq A) } \lambda^{(1+2)}_T  .
\end{align*}

Secondly, for equation (ii), we also have, for all $A \subseteq \mathcal{T}$,
{ \small
\begin{align*}
& \f_{1+2}(A) - \sum_{T \subseteq A , |T| \leq \ell}\Expl_T(\f_{1+2},\ell) 
& = \alpha_1 \left( \f_1(A) - \sum_{T \subseteq A , |T| \leq \ell}\Expl_T(\f_1,\ell) \right) + \alpha_2 \left( \f_2(A) - \sum_{T \subseteq A , |T| \leq \ell}\Expl_T(\f_2,\ell) \right) 
& = 0.
\end{align*}}

Therefore, $[\lambda^{(1+2)}, \Expl(\f_{1+2},\ell)] =  [\alpha_1\lambda^{(1)} + \alpha_2 \lambda^{(2)}, \  \alpha_1\Expl(\f_1,\ell) + \alpha_2 \Expl(\f_2,\ell)]$ is the solution of Eqn.\eqref{eqn:lagrangian2} with respect to the function $\f_{1+2}= \alpha_1 \f_1 + \alpha_2 \f_2$. Hence, the Faith-Interaction indices satisfy the interaction linearity axiom.
\end{proof}

\subsection{Proof of Proposition \ref{pro:symmetry}}

\subsubsection{Sufficient Condition:}
\begin{proof}
We prove that if the proper weighting functions are permutation invariant, then Faith-Interaction indices $\Expl$ satisfy the interaction symmetry axiom.

Suppose indexes $i$ and $j$ are symmetric with respect to the function $\f$. That is,  $\f(S \cup i) = \f(S \cup j)$ for any set $S \subseteq [d] \backslash \{i,j\}$. Our goal is to prove that indexes $i$ and $j$ in the corresponding explanation are also symmetric, i.e.
$\ex_{S \cup i}(\f,\ell) = \ex_{S \cup j}(\f,\ell)$ for any set $S \subseteq [d] \backslash \{i,j\}$ with $|S| < \ell$. 

By Proposition \ref{pro:unique_minimizer},
Eqn.\eqref{eqn:weighted_regression} has a unique minimizer. The objective function can be written into the following form:

{ \small
\begin{align*}
&  \argmin_{\Expl(\f,\ell) \in \mathbb{R}^{d_\ell}} \sum_{\substack{S \subseteq [d] \\ \mu(S) < \infty}}  \mu(S) \left( \f(S) - \sum_{L \subseteq S , |L| \leq \ell}\Expl_L(\f,\ell) \right)^2  \\
 & = \sum_{\substack{S \subseteq [d] \backslash \{i,j\} \\ \mu(S) < \infty }}  \mu(S) \left( (\f(S) - \sum_{L \subseteq S , |L| \leq \ell}\Expl_L(\f,\ell) \right)^2 + 
 \sum_{\substack{S \subseteq [d] \backslash \{i,j\} \\ \mu(S) < \infty }}  \mu(S \cup i) \left( \f(S \cup i) - \sum_{L \subseteq S \cup i , |L| \leq \ell}\Expl_L(\f,\ell) \right)^2  \\
 & + \sum_{\substack{S \subseteq [d] \backslash \{i,j\}\\ \mu(S) < \infty}}  \mu(S \cup j) \left( \f(S\cup j) - \sum_{L \subseteq S\cup j , |L| \leq \ell}\Expl_L(\f,\ell) \right)^2 \\ 
 & +  \sum_{\substack{S \subseteq [d] \backslash \{i,j\} \\ \mu(S) < \infty}}  \mu(S \cup \{i,j\}) \left( \f(S \cup \{i,j\}) - \sum_{L \subseteq S \cup \{i,j\} , |L| \leq \ell}\Expl_L(\f,\ell) \right)^2 \\
 & \text{ subject to } 
 \f(S) = \sum_{T \subseteq S} \Expl_T(\f,\ell), \ \  \forall \ \  S: \mu(S) = \infty.
\end{align*}}

If $i$ and $j$ in function $\mu(\cdot)$ and $\f(\cdot)$ are interchanged, the above equation becomes

{ \small
\begin{align*}
& \argmin_{\Expl(\f,\ell) \in \mathbb{R}^{d_\ell}} \sum_{S \subseteq [d] \backslash \{i,j\} }  \mu(S) \left( \f(S) - \sum_{L \subseteq S , |L| \leq \ell}\Expl_L(\f,\ell) \right)^2 + 
 \sum_{S \subseteq [d] \backslash \{i,j\} }  \mu(S \cup j) \left( \f(S \cup j) - \sum_{L \subseteq S \cup i , |L| \leq \ell}\Expl_L(\f,\ell) \right)^2  \\
& + \sum_{S \subseteq [d] \backslash \{i,j\} }  \mu(S \cup i) \left( \f(S\cup i) - \sum_{L \subseteq S\cup j , |L| \leq \ell}\Expl_L(\f,\ell) \right)^2 \\
& + 
 \sum_{S \subseteq [d] \backslash \{i,j\} }  \mu(S \cup \{i,j\}) \left( \f(S \cup \{i,j\}) - \sum_{L \subseteq S \cup \{i,j\} , |L| \leq \ell}\Expl_L(\f,\ell) \right)^2 \\
  & \text{ subject to } 
 \f(S) = \sum_{T \subseteq S} \Expl_T(\f,\ell), \ \ \forall \ \  S: \mu(S) = \infty.
\end{align*}}

The equation remains the same since $\mu(S \cup i) = \mu(S \cup j)$  and $\f(S \cup i) = \f(S \cup j)$. Also, the constraints $\f(S) = \sum_{T \subseteq S} \Expl_T(\f,\ell), \ \ \forall \ \  S: \mu(S) = \infty$ remain the same since a proper weighting function is only allowed to have infinity measure on $\mu(\text{\O})$ and $\mu([d])$, which are symmetric to any $i$ and $j$. 

Given that Eqn.\eqref{eqn:weighted_regression} has a unique solution, the above two minimization problems should have the same minimizer. We note that $i$ and $j$ have been interchanged in the set function $\f(\cdot)$ and the weighting function  $\mu$, so $i$ and $j$ should also be symmetric in the minimizer, i.e. $\ex_{S \cup i}(\f,\ell) = \ex_{S \cup j}(\f,\ell)$ for any set $S \subseteq [d] \backslash \{i,j\}$ with $|S| < \ell$. 

\end{proof}

\subsubsection{Necessary Condition}

\begin{proof}
Next, we show that Faith-Interaction indices $\Expl$ satisfy the interaction symmetry axiom only if the proper weighting functions are permutation invariant so that $\mu(S)$ is only a function of $|S|$.

We consider the case when $\ell = d- 1$ and the set function $\f$ is defined as below.
\begin{equation}
\f(S) = \begin{cases}
1  & \text{, if $S = [d]$.} \\
0 & \text{, otherwise.} \\
\end{cases}    
\end{equation}
In this case, $\mathcal{S}_\ell$ consists of all subsets of $[d]$ except for $[d]$. Next, we define a new coalition weighting function $\mu' :2^d \rightarrow \mathbb{R}^{+}$ with 
$$
\mu'(S) =
\begin{cases}
1 & \text{ if } \mu(S) = \infty. \\
\mu(S) & \text{ otherwise.} \\
\end{cases}
$$
We can see that for all $\ex(\f,\ell) \subseteq \mathbb{R}^{d_\ell}$ satisfying  $\f(S) - \sum_{T \subseteq S , |T| \leq \ell}\Expl_T(\f,\ell) =0, \forall S: \mu(S) = \infty, S \subseteq [d]$, the values of objective functions instantiated with $\mu$ and $\mu'$ are the same. That is, 
$$
\sum_{S \subseteq [d]\,:\, \mu(S) < \infty}  \mu(S) \left( \f(S) - \sum_{T \subseteq S , |T| \leq \ell}\Expl_T(\f,\ell) \right)^2
= \sum_{S \subseteq [d]}  \mu'(S) \left( \f(S) - \sum_{T \subseteq S , |T| \leq \ell}\Expl_T(\f,\ell) \right)^2.
$$
Therefore, we can substitute $\mu$ with $\mu'$ and the objective function can be written as 
\begin{align*}
\small
F(\ex) = \sum_{S \subseteq [d]} \mu'(S) \left(\f(S) - \sum_{T \subseteq S, |T| \leq d-1} \Expl_T(\f, \ell)  \right)^2,
\ \text{s.t.} \ & \f(S) - \sum_{T \subseteq S , |T| \leq \ell}\Expl_T(\f,\ell) =0 \;\;,\;\forall S \,:\, \mu(S) = \infty.
\end{align*}
Let $q(S) = \mu'(S) \left(\f(S) - \sum_{T \subseteq S, |T| \leq d-1} \Expl_T(\f, \ell) \right)$ and $p(S) = \sum_{T \supseteq S} q(T)$ for all $S \subseteq [d]$. The partial derivative of $F(\ex)$ with respect to $\ex_L$ is 
$$
\frac{F(\ex)}{\partial \ex_L} 
= -2\sum_{S \supseteq L} \mu'(S) \left(\f(S) - \sum_{T \subseteq S, |T| \leq d-1} \Expl_T(\f, \ell) \right) = -2\sum_{S \supseteq L} q(S) = -2p(L),
$$
where $L \subset [d]$. Before going to the main proof, we first introduce the following claim, which provides a relation between $p(\cdot)$ and $q(\cdot)$.
\begin{claim}
\label{clm:p_q_relation}
For all $S \subseteq [d]$, $q(S) = \sum_{T \supseteq S} (-1)^{|T| - |S|}p(T)$. 
\end{claim}
Also, the following claim states that $p(L)$ for all $L \subseteq [d]$ can not be zero simultaneously. 
\begin{claim}
\label{clm:all_zero}
There is no $\ex(\f,\ell) \in \mathbb{R}^{2^d-1}$ satisfying $p(L) = 0$ for all $L \subseteq [d]$.
\end{claim}
 
With these results in hands, we now prove that 
$q(S_1) = q(S_2)  \neq 0$ for all $S_1, S_2 \subseteq [d]$ with $1 \leq |S_1| = |S_2| \leq d-1$.
Since a proper weighting function is only allowed to have $\mu([d])$ or $\mu(\text{\O})$ to be infinity, we separate the problem into four cases: (1) $\mu(S) < \infty$ for all $S \subseteq [d]$. (2) Only $\mu(\text{\O}) = \infty$. (3) Only $\mu([d]) = \infty$. (4) Only $\mu([d]) = \mu(\text{\O}) = \infty$.

\paragraph{ (1) $\mu(S) < \infty$ for all $S \subseteq [d]$:} 
we solve the minimization problem using partial derivatives:
$$
\frac{F(\ex)}{\partial \ex_S}  = -2p(S) = 0,
\ \text{ for all $S \subset [d]$.}
$$
By Claim \ref{clm:p_q_relation}, for all $S \subset [d]$, we have 
\begin{equation}
\label{eqn:q_p_case1}
q(S) = \sum_{T \supseteq S} (-1)^{|T| - |S|}p(T) = (-1)^{d-|S|}p([d]),
\end{equation}
which implies that $q(S_1) = q(S_2)$ for all $S_1, S_2 \subseteq [d]$ with $1 \leq |S_1| = |S_2| \leq d-1$.

If there exists $q(S) = 0$ for some $S \subset [d]$, then $p([d])$ must also be zero by Eqn.\eqref{eqn:q_p_case1}. Then again by Eqn.\eqref{eqn:q_p_case1}, we have $q(S) = 0$ for all $S \subset [d]$, which is a contradiction by Claim \ref{clm:all_zero}.

\paragraph{ (2) Only $\mu(\text{\O}) = \infty$:} in this case, the only constraint is $q(\text{\O}) = \f(\text{\O}) - \ex_{\text{\O}}(\f,\ell) = 0$, which implies $\ex_{\text{\O}}(\f,\ell) = f(\text{\O})$. Then
we solve the minimization problem using partial derivatives:
$$
\frac{F(\ex)}{\partial \ex_S}  = -2p(S) = 0,
\ \text{ for all $S \subset [d]$ with $ 0 <|S| \leq d-1$.}   
$$
By Claim \ref{clm:p_q_relation}, we have 
\begin{equation}
\label{eqn:q_p_case2}
q(S) = \sum_{T \supseteq S} (-1)^{|T| - |S|}p(T) = (-1)^{d-|S|}p([d]),
\ \text{ for all $S \subset [d]$ with $ 0 <|S| \leq d-1$,}   
\end{equation}
which implies that $q(S_1) = q(S_2)$ for all $S_1, S_2 \subseteq [d]$ with $1 \leq |S_1| = |S_2| \leq d-1$.

If there exists $q(S) = 0$ for some $S \subset [d]$, then $p([d])$ must also be zero by Eqn.\eqref{eqn:q_p_case2}. Then again by Eqn.\eqref{eqn:q_p_case2}, we have $q(S) = 0$ for all $S \subset [d]$ with $ 0 < |S| \leq d-1$. Also, the constraint implies that $q(\text{\O}) = 0$. By Claim \ref{clm:all_zero}, we can not have $q(S) = 0$ for all $S \subseteq [d]$, which is a contradiction.

\paragraph{(3) Only $\mu([d]) = \infty$:} in this case, the only constraint is $q([d]) = \f([d]) - \sum_{T \subseteq [d]} \ex_{T}(\f,\ell) = 0$. Then
we solve the constrained minimization problem using a Lagrange multiplier.
$$
\frac{F(\ex)}{\partial \ex_S}  = -2p(S) = \lambda,
\ \text{ for all $S \subset [d]$ with $ 0 \leq |S| \leq d-1$.}   
$$
By Claim \ref{clm:p_q_relation}, for all $S \subset [d]$ with $ 0 \leq |S| \leq d-1$, we have 
\begin{align}
q(S) & = \sum_{T \supseteq S} (-1)^{|T| - |S|}p(T)  \nonumber \\
& = (-1)^{d-|S|}p([d]) + \sum_{T \supseteq S} (-1)^{|T| - |S|} \frac{-\lambda}{2} \nonumber \\
& = (-1)^{d-|S|}p([d]) -\frac{\lambda}{2} \times \left[(1-1)^{d -|S|} - (-1)^{d-|S|} \right] \nonumber \\
& = (-1)^{d-|S|}p([d]) +\frac{\lambda}{2}  (-1)^{d-|S|} \nonumber \\
& = \frac{\lambda}{2}  (-1)^{d-|S|} 
\ \ \ (p([d]) = q([d]) = 0)
\label{eqn:q_p_case3}
\end{align}
which implies that $q(S_1) = q(S_2)$ for all $S_1, S_2 \subseteq [d]$ with $1 \leq |S_1| = |S_2| \leq d-1$.

If there exists $q(S) = 0$ for some $S \subset [d]$ with $0 \leq |S| \leq d-1$, then we have $\lambda = 0$ by Eqn.\eqref{eqn:q_p_case3}. Then again by Eqn.\eqref{eqn:q_p_case3}, we have $q(S) = 0$ for all $S \subset [d]$ with $ 0 \leq |S| \leq d-1$. Also, the constraint implies that $q([d]) = 0$. By Claim \ref{clm:all_zero}, we can not have $q(S) = 0$ for all $S \subseteq [d]$, which is a contradiction.

\paragraph{(4) Only $\mu(\text{\O}) = \mu([d]) = \infty$:} in this case, the constraints are $q(\text{\O}) = \f(\text{\O}) - \ex_{\text{\O}} = 0$ and $q([d]) = \f([d]) - \sum_{T \subseteq [d]} \ex_{T}(\f,\ell) = 0$. Then
we solve the constrained minimization problem using a Lagrange multiplier.
$$
\frac{F(\ex)}{\partial \ex_S}  = -2p(S) = \lambda,
\ \text{ for all $S \subset [d]$ with $ 1 \leq |S| \leq d-1$.}   
$$
By Claim \ref{clm:p_q_relation}, for all $S \subset [d]$ with $ 1 \leq |S| \leq d-1$, we have 
\begin{align}
q(S) & = \sum_{T \supseteq S} (-1)^{|T| - |S|}p(T)  \nonumber \\
& = (-1)^{d-|S|}p([d]) + \sum_{T \supseteq S} (-1)^{|T| - |S|} \frac{-\lambda}{2} \nonumber \\
& = (-1)^{d-|S|}p([d]) -\frac{\lambda}{2} \times \left[(1-1)^{d -|S|} - (-1)^{d-|S|} \right] \nonumber \\
& = (-1)^{d-|S|}p([d]) +\frac{\lambda}{2}  (-1)^{d-|S|} \nonumber \\
& = \frac{\lambda}{2}  (-1)^{d-|S|} 
\ \ \ (p([d]) = q([d]) = 0)
\label{eqn:q_p_case4}
\end{align}
which implies that $q(S_1) = q(S_2)$ for all $S_1, S_2 \subseteq [d]$ with $1 \leq |S_1| = |S_2| \leq d-1$.

If there exists $q(S) = 0$ for some $S \subset [d]$ with $1 \leq |S| \leq d-1$, then we have $\lambda = 0$ by Eqn.\eqref{eqn:q_p_case4}. Then again by Eqn.\eqref{eqn:q_p_case4}, we have $q(S) = 0$ for all $S \subset [d]$ with $ 1 \leq  |S| \leq d-1$. Also, the constraint implies that $q(\text{\O}) = q([d]) = 0$. By Claim \ref{clm:all_zero}, we can not have $q(S) = 0$ for all $S \subseteq [d]$, which is a contradiction.

By summarizing (1)$\sim$(4), we conclude that $q(S_1) = q(S_2)  \neq 0$ for all $S_1, S_2 \subseteq [d]$ with $1 \leq |S_1| = |S_2| \leq d-1$. That is,
$$
\mu'(S_1) 
\underbrace{\left(\f(S_1) - \sum_{T_1 \subseteq S_1, |T_1| \leq d-1} \Expl_{T_1}(\f, \ell) \right)}_{(i)}
= \mu'(S_2) 
\underbrace{\left(\f(S_2) - \sum_{T_2 \subseteq S_2, |T_2| \leq d-1} \Expl_{T_2}(\f, \ell) \right)}_{(ii)}.
$$
Since $\f(\cdot)$ is symmetric by definition ( only depends on its input size) and the minimizer $\Expl(\f, \ell) $ satisfies the interaction symmetry axiom, we get that 
the mimimizer $\Expl_{T_1}(\f, \ell) = \Expl_{T_2}(\f, \ell) $ for all $|T_1| = |T_2|$.  Therefore, each term in $(i)$ and $(ii)$ have one-to-one correspondence. We have $(i)=(ii)$ and we can simplify the above equation: 
$$
\left(\mu'(S_1) - \mu'(S_2) \right) \times \left(\f(S_1) - \sum_{T_1 \subseteq S_1, |T_1| \leq d-1} \Expl_{T_1}(\f, \ell)  \right) 
= 0,
$$
Since the value in the second bracket is nonzero ($q(S_1) \neq 0$), we can conclude that $\mu'(S_1) = \mu'(S_2)$ for all $S_1, S_2 \subseteq [d]$ with $1 \leq |S_1| = |S_2| \leq d-1$. Also, by definition of $\mu'()$, we have $\mu(S) = \mu'(S)$ for all $S \subseteq [d]$ with $1 \leq |S| \leq d-1$. Therefore, we conclude the weighting function $\mu(\cdot)$ is also symmetric.

\end{proof}

\subsection{Proof of Proposition \ref{pro:dummy}}
\begin{proof}

The minimization problem can be written as follows:
\begin{equation*}
\small
\ex(\f,\ell) = \min_{\Expl \in \mathbb{R}^{d_\ell}} F_\ell(\f,\ex) = \min \sum_{S \subseteq [d]} \mu(S) \left(
\f(S) - \sum_{T \subseteq S, |T| \leq \ell} \Expl_T(\f, \ell)  \right)^2,
\end{equation*}
Without loss of generality, let $\mu(S) = \prod_{j \in S} p_j \prod_{k \not \in S} (1-p_k)$ for some $0 < p_j < 1$.

Now, we prove the minimizer of the above equation satisfies interaction dummy axiom.  More generally, we prove that the minimization problem with a dummy feature can be reduced to another problem with only $d-1$ features and the interaction terms containing the dummy feature is zero.
Formally, we have the following lemma:
\begin{lemma}
\label{lm:dummt_reduce}
Assume that $i^{\text{th}}$ feature of the set function $\f(\cdot)$ is a dummy feature such that $\f(S) = \f(S \cup i)$ for all $S \subseteq [d] \backslash \{i\}$.
Let $\f':2^{d-1} \mapsto \mathbb{R}$ with $\f'(S) = \f(S)$ for all $S \subseteq [d-1]$. Then we have 
\begin{equation}
\label{eqn:dummy_solution}
\begin{cases}
\Expl_{S}(\f,\ell) = \Expl_{S}(\f',\ell), 
& \text{ for all }  \ \ S \subseteq [d] \backslash \{ i\}, 0 \leq |S| \leq \ell. \\
\Expl_{S \cup \{i\}}(\f,\ell) = 0, & \text{ for all }  \ \ S \subseteq [d] \backslash \{ i\}, 0 \leq |S| \leq \ell - 1. 
\end{cases}
\end{equation}
where $\Expl_{S}(\f',\ell)$ is the minimizer of the following problem:
$$
\min \sum_{S \subseteq [d]\backslash i} \mu'(S) \left(
\f'(S) - \sum_{T \subseteq S, |T| \leq \ell} \Expl'_T(\f, \ell)  \right)^2,
\text{ with }
\mu'(S) = \prod_{j \in S} p_j \prod_{k \in [d-1] \backslash S} (1-p_k).
$$
\end{lemma}

\begin{proof}
Without loss of generality, we assume that $d^{th}$ feature is a dummy feature, such that $\f(S \cup d) = \f(S)$ for all $S \subseteq [d-1]$.
Now, we solve the minimization problem  by partial derivatives. 
\begin{equation}
\frac{\partial \ F_\ell(\f,\ex)}{\partial \Expl_S } = 0  \text{  for all  } S \subseteq [d-1], |S| \leq \ell.
\end{equation}
We note that the partial derivative can be calculated as below:
$$
\frac{\partial \ F_\ell(\f,\ex)}{\partial \Expl_S } 
= -2\sum_{T \supseteq S} \mu(T) \left(\f(T) - \sum_{L \subseteq T, |L| \leq \ell} \ex_L(\f,\ell) 
\right)
= 0
\text{  for all  } S \subseteq [d-1], |S| \leq \ell.
$$
Now, for convenience, we denote 
$q(T) = \f(T) - \sum_{L \subseteq T, |L| \leq \ell} \ex_L(\f,\ell)$ for all $T \subseteq [d]$. Then $\ex(\f,\ell)$ satisfies the following equalities:
\begin{equation}
\label{eqn:partial_derivative_dummy}
\sum_{T \supseteq S} \mu(T)q(T) = 0,
\text{  for all  } S \subseteq [d-1], |S| \leq \ell.    
\end{equation}

Similarly, the minimizer $\ex(\f',\ell)$ of $F_\ell(\f',\ex)$ satisfies 
\begin{equation}
\label{eqn:partial_derivative_dummy_prime}\sum_{T: S \subseteq T \subseteq [d-1]} \mu'(T)q'(T) = 0,
\text{  for all  } S \subseteq [d-1], |S| \leq \ell,
\end{equation}
where $q':2^{d-1} \mapsto \mathbb{R}$ with $q'(T) = \f'(T) - \sum_{L \subseteq T, |L| \leq \ell} \ex_L(\f',\ell)$ for all $T \subseteq [d-1]$.

By the definitions of weighting function $\mu(\cdot)$ and $\mu'(\cdot)$, we have 
\begin{equation}
\label{eqn:mu_dummy_property}
\mu(S) = (1-p_d)\mu'(S),
\text{ and }
\mu(S \cup d) = p_d\mu'(S), 
\text{ for all }
S \subseteq [d-1].
\end{equation}
Also, since we have $\ex_L(\f,\ell) = 0$ for all $L$ containing $\{d\}$ and $ \f(T) = \f(T \cup \{ d\})$ ( Eqn.\eqref{eqn:dummy_solution}), for all $T \subseteq [d-1]$, we have 
\begin{equation}
\label{eqn:q_dummy_property1}
q(T \cup \{ d\}) 
= \f(T \cup \{ d\}) - \sum_{L \subseteq T \cup \{ d\}, |L| \leq \ell} \ex_L(\f,\ell) 
= \f(T) - \sum_{L \subseteq T, |L| \leq \ell} \ex_L(\f,\ell) 
= q(T),    
\end{equation}
and
\begin{equation}
\label{eqn:q_dummy_property2}
q'(T) 
= \f'(T) - \sum_{L \subseteq T, |L| \leq \ell} \ex_L(\f',\ell) 
= \f(T) - \sum_{L \subseteq T, |L| \leq \ell} \ex_L(\f,\ell) 
= q(T) = q(T \cup \{ d\}). 
\end{equation}
Now, we prove that $\ex(\f,\ell)$ defined in Eqn.\eqref{eqn:dummy_solution} satisfies the system of linear equations in Eqn.\eqref{eqn:partial_derivative_dummy}.

(1) For all $S \subseteq [d-1]$ with $0 \leq |S| \leq \ell$, we have

\begin{align*}
\sum_{T: S \subseteq T \subseteq [d]} \mu(T) q(T) 
& =  \sum_{T: S \subseteq T \subseteq [d-1]} \mu(T) q(T)
+ \sum_{T: S \subseteq T \subseteq [d-1]} \mu(T \cup \{ d\}) q(T\cup \{ d\}) \\
& = \sum_{T: S \subseteq T \subseteq [d-1]} \left(\mu(T) +\mu(T \cup \{ d\}) \right) q(T) \ \ \text{( Using Eqn.\eqref{eqn:q_dummy_property1} )} \\
& =  \sum_{T: S \subseteq T \subseteq [d-1]} \mu'(T) q(T) \ \
\text{( Using Eqn.\eqref{eqn:mu_dummy_property} )} \\
& = \sum_{T: S \subseteq T \subseteq [d-1]} \mu'(T) q'(T) \ \ 
\text{(Using Eqn.\eqref{eqn:q_dummy_property2} )} \\
&  = 0  \ \ 
(\text{Eqn.\eqref{eqn:partial_derivative_dummy_prime}}).
\end{align*} 

(2) For all $S \subseteq [d]$ containing $\{ d\}$ with $ 1 \leq |S| \leq \ell$, we have 
\begin{align*}
\sum_{T: S \subseteq T \subseteq [d]} \mu(T) q(T)
& = \sum_{T: (S\backslash \{d\}) \subseteq T \subseteq [d-1]} \mu(T \cup \{d\} ) q(T \cup \{d\}) \\
& = \sum_{T: (S\backslash \{d\}) \subseteq T \subseteq [d-1]} p_d \mu'(T) q'(T) \ \ \text{(Using Eqn.\eqref{eqn:mu_dummy_property} and Eqn.\eqref{eqn:q_dummy_property2})} \\
& = p_d \sum_{T: (S\backslash \{d\}) \subseteq T \subseteq [d-1]} \mu'(T) q'(T)  \\
& = 0 
\ \ \text{ (By Eqn.\eqref{eqn:partial_derivative_dummy_prime})} \\
\end{align*}
Therefore, by combining (1) and (2), we have 
$$
\sum_{T: S \subseteq T \subseteq [d]} \mu(T) q(T) = 0 \text{  for all  }
S \subseteq [d], \text{ with } |S| \leq \ell.
$$
That is, Eqn.\eqref{eqn:dummy_solution} is the minimizer of the minimization problem. Consequently, the minimizer satisfies the interaction dummy axiom for all $1 \leq \ell \leq d$.
\end{proof}

\end{proof}

\subsection{Proof of Proposition \ref{pro:efficiency}}
\begin{proof}

\textbf{Sufficient condition:} first of all, we prove that if the proper weighting functions have $\mu(\text{\O}) = \mu([d]) = \infty$, the Faith-Interaction indices satisfy the interaction efficiency axiom. 

By Proposition \ref{pro:unique_minimizer}, the constrained minimization problem has a unique minimizer. Also, the constraints ensures that $\f([d]) = \sum_{T \subseteq [d], |T| \leq \ell} \ex_T(\f,\ell)$ and $\f(\text{\O}) = \ex_{\text{\O}}(\f,\ell)$. Therefore, the minimizer (or the Faith-Interaction indices) satisfies the interaction efficiency axiom.

\textbf{Necessary condition:} we prove that if the Faith-Interaction indices satisfy the interaction efficiency axiom, the corresponding weighting function must satisfy $\mu(\text{\O}) = \mu([d]) = \infty$.

We consider the case when $\ell = d-1$ and the set function $\f$ is defined as below.
\begin{equation}
\f(S) = \begin{cases}
1  & \text{, if $S = [d]$.} \\
0 & \text{, otherwise.} \\
\end{cases}    
\end{equation}
In this case, $\mathcal{S}_\ell$ consists of all subsets of $[d]$ except for $[d]$. We will prove that $\ex(\f,d-1)$ satisfies the interaction efficiency axiom only if $\mu(\text{\O}) = \mu([d]) = \infty$.

First, we define a new coalition weighting function $\mu' :2^d \rightarrow \mathbb{R}^{+}$ with 
$$
\mu'(S) =
\begin{cases}
1 & \text{ if } \mu(S) = \infty. \\
\mu(S) & \text{ otherwise.} \\
\end{cases}
$$
We can see that for all $\ex(\f,\ell) \subseteq \mathbb{R}^{d_\ell}$ satisfying  $\f(S) - \sum_{T \subseteq S , |T| \leq \ell}\Expl_T(\f,\ell) =0, \forall S: \mu(S) = \infty, S \subseteq [d]$, the values of objective functions instantiated with $\mu$ and $\mu'$ are the same. That is, 
$$
\sum_{S \subseteq [d]\,:\, \mu(S) < \infty}  \mu(S) \left( \f(S) - \sum_{T \subseteq S , |T| \leq \ell}\Expl_T(\f,\ell) \right)^2
= \sum_{S \subseteq [d]}  \mu'(S) \left( \f(S) - \sum_{T \subseteq S , |T| \leq \ell}\Expl_T(\f,\ell) \right)^2.
$$
Therefore, we can substitute $\mu$ with $\mu'$ and the objective function can be written as 
\begin{align*}
F(\ex) = \sum_{S \subseteq [d]} \mu'(S) \left(\f(S) - \sum_{T \subseteq S, |T| \leq d-1} \Expl_T(\f, \ell)  \right)^2,
\ \text{s.t.} \ & \f(S) - \sum_{T \subseteq S , |T| \leq \ell}\Expl_T(\f,\ell) =0 \;\;,\;\forall S \,:\, \mu(S) = \infty.
\end{align*}

Let $q(S) = \mu'(S) \left(\f(S) - \sum_{T \subseteq S, |T| \leq d-1} \Expl_T(\f, \ell) \right)$ and $p(S) = \sum_{T \supseteq S} q(T)$ for all $S \subseteq [d]$. The partial derivative of $F(\ex)$ with respect to $\ex_L$ is 
$$
\frac{F(\ex)}{\partial \ex_L} 
= -2\sum_{S \supseteq L} \mu'(S) \left(\f(S) - \sum_{T \subseteq S, |T| \leq d-1} \Expl_T(\f, \ell) \right) = -2\sum_{S \supseteq L} q(S) = -2p(L),
$$
where $L \subset [d]$. Recall that Claim \ref{clm:p_q_relation} in the proof of Proposition \ref{pro:symmetry} states that 
\begin{equation}
\label{eqn:p_q_relation}
q(S) = \sum_{T \supseteq S} (-1)^{|T| - |S|}p(T), \text{ for all $S \subseteq [d]$ }.
\end{equation}

Also, Claim \ref{clm:all_zero} ensures that $p(L)$ for all $L \subseteq [d]$ can not be zero simultaneously:
\begin{equation}
\label{eqn:all_zero}
\text{ There is no $\ex(\f,\ell) \in \mathbb{R}^{2^d-1}$ satisfying $p(L) = 0$ for all $L \subseteq [d]$.}
\end{equation}
 
With these results in hands, we now prove that 
we must have $\mu(\text{\O}) = \mu([d]) = \infty$. Otherwise, we will have $q([d]) \neq 0$ or $q(\text{\O}) \neq 0$.
Since a proper weighting function is only allowed to have $\mu([d])$ or $\mu(\text{\O})$ to be infinity, we only need to discuss three cases: (1) $\mu(S) < \infty$ for all $S \subseteq [d]$. (2) Only $\mu(\text{\O}) = \infty$. (3) Only $\mu([d]) = \infty$. 

\paragraph{ (1) $\mu(S) < \infty$ for all $S \subseteq [d]$:} 
we solve the minimization problem using partial derivatives:
$$
\frac{F(\ex)}{\partial \ex_S}  = -2p(S) = 0,
\ \text{ for all $S \subset [d]$ with $ 0 \leq |S| \leq d-1$.}   
$$
By Claim \ref{clm:p_q_relation}, for all $S \subset [d]$, we have 
\begin{equation}
\label{eqn:q_p_case1_eff}
q(S) = \sum_{T \supseteq S} (-1)^{|T| - |S|}p(T) = (-1)^{d-|S|}p([d]),
\ \text{ for all $S \subset [d]$ with $ 0 \leq |S| \leq d-1$.}   
\end{equation}

If $q(\text{\O}) = 0$, then $p([d])$ must also be zero by Eqn.\eqref{eqn:q_p_case1_eff}. Then again by Eqn.\eqref{eqn:q_p_case1_eff}, we have $q(S) = 0$ for all $S \subset [d]$, which is a contradiction by Claim \ref{clm:all_zero}.

\paragraph{ (2) Only $\mu(\text{\O}) = \infty$:} in this case, the only constraint is $q(\text{\O}) = \f(\text{\O}) - \ex_{\text{\O}}(\f,\ell) = 0$, which implies $\ex_{\text{\O}}(\f,\ell) = f(\text{\O})$. Then
we solve the minimization problem using partial derivatives:
$$
\frac{F(\ex)}{\partial \ex_S}  = -2p(S) = 0,
\ \text{ for all $S \subset [d]$ with $ 0 <|S| \leq d-1$.}   
$$
By Claim \ref{clm:p_q_relation}, we have 
\begin{equation}
\label{eqn:q_p_case2_eff}
q(S) = \sum_{T \supseteq S} (-1)^{|T| - |S|}p(T) = (-1)^{d-|S|}p([d]),
\ \text{ for all $S \subset [d]$ with $ 0 <|S| \leq d-1$.}   
\end{equation}

If $q([d]) = 0$, then by Eqn.\eqref{eqn:q_p_case2_eff}, we have $q(S) = 0$ for all $S \subset [d]$ with $ 0 < |S| \leq d-1$. Also, the constraint implies that $q(\text{\O}) = 0$. By Claim \ref{clm:all_zero}, we can not have $q(S) = 0$ for all $S \subseteq [d]$, which is a contradiction.

\paragraph{(3) Only $\mu([d]) = \infty$:} in this case, the only constraint is $q([d]) = \f([d]) - \sum_{T \subseteq [d]} \ex_{T}(\f,\ell) = 0$. Then
we solve the constrained minimization problem using a Lagrange multiplier.
$$
\frac{F(\ex)}{\partial \ex_S}  = -2p(S) = \lambda,
\ \text{ for all $S \subset [d]$ with $ 0 \leq |S| \leq d-1$.}   
$$
By Claim \ref{clm:p_q_relation}, for all $S \subset [d]$ with $ 0 \leq |S| \leq d-1$, we have 
\begin{align}
q(S) & = \sum_{T \supseteq S} (-1)^{|T| - |S|}p(T)  \nonumber \\
& = (-1)^{d-|S|}p([d]) + \sum_{T \supseteq S} (-1)^{|T| - |S|} \frac{-\lambda}{2} \nonumber \\
& = (-1)^{d-|S|}p([d]) -\frac{\lambda}{2} \times \left[(1-1)^{d -|S|} - (-1)^{d-|S|} \right] \nonumber \\
& = (-1)^{d-|S|}p([d]) +\frac{\lambda}{2}  (-1)^{d-|S|} \nonumber \\
& = \frac{\lambda}{2}  (-1)^{d-|S|} 
\ \ \ (\text{ By } p([d]) = q([d]) = 0).
\label{eqn:q_p_case3_eff}
\end{align}

If $q(\text{\O}) = 0$, then we have $\lambda = 0$ by Eqn.\eqref{eqn:q_p_case3_eff}. Then again by Eqn.\eqref{eqn:q_p_case3_eff}, we have $q(S) = 0$ for all $S \subset [d]$ with $ 0 \leq |S| \leq d-1$. Also, the constraint implies that $q([d]) = 0$. By Claim \ref{clm:all_zero}, we can not have $q(S) = 0$ for all $S \subseteq [d]$, which is a contradiction.

Let come back to the main proof of Proposition \ref{pro:efficiency}. By summarizing (1)$\sim$(3), we do not have $q(\text{\O}) = q([d])= 0$ and $\ex(\f,d-1)$ does not satisfy the interaction efficiency axiom in these cases. Therefore, the proper weighting function must have $\mu(\text{\O}) = \mu([d]) = \infty$.

\end{proof}

\subsection{Proof of Proposition \ref{pro:faith_shap_cardinal_prob}}

\begin{proof}
By Theorem \ref{thm:faith_shap}, for all $S \subseteq [d]$ with $|S| = \ell$, we have 
$$
\ex^{\text{F-Shap}}_{S}(\f,\ell) 
= \frac{(2\ell -1)!}{((\ell-1)!)^2} \sum_{T \subseteq [d] \backslash S} \frac{(\ell+|T|-1)!(d-|T|-1)!}{ (d+\ell-1)!} \Delta_S(\f(T))
= \sum_{T \subseteq [d] \backslash S} p^\ell_{|T|} \Delta_S(\f(T)).
$$
We next show that $\sum_{t=0}^{d-\ell} {d-\ell \choose t} p^\ell_t = 1$. The following derivation is based on the following property of beta functions:
$$
B(\alpha,\beta) = \int_{x=0}^1 x^{\alpha-1}(1-x)^{\beta-1} dx = \frac{(\alpha-1)!(\beta-1)!}{(\alpha+\beta-1)!}
\text{ for all } \alpha, \beta \in \mathbb{N}.
$$
Then we have
\begin{align*}
\sum_{t=0}^{d-\ell} {d-\ell \choose t} p^\ell_t 
& = \sum_{t=0}^{d-\ell} {d-\ell \choose t}\frac{(2\ell -1)!(\ell+t-1)!(d-t-1)!}{((\ell-1)!)^2 (d+\ell-1)!} \\
& = \frac{(2\ell -1)!}{((\ell-1)!)^2} \sum_{t=0}^{d-\ell} {d-\ell \choose t} \int_{x=0}^1 x^{\ell+t-1}(1-x)^{d-t-1} dx  \\
& \ \ (\text{ Using the definition of Beta function } B(\ell+t,d-t) ) \\
& = \frac{(2\ell -1)!}{((\ell-1)!)^2} \int_{x=0}^1  x^{\ell-1}(1-x)^{\ell-1}\sum_{t=0}^{d-\ell} {d-\ell \choose t}  x^{t}(1-x)^{d-\ell-t} dx  \\
& = \frac{(2\ell -1)!}{((\ell-1)!)^2} \int_{x=0}^1  x^{\ell-1}(1-x)^{\ell-1}\left(x + (1 -x)\right)^{d-\ell} dx  \\
& = \frac{(2\ell -1)!}{((\ell-1)!)^2} \int_{x=0}^1  x^{\ell-1}(1-x)^{\ell-1} dx  \\
& = \frac{(2\ell -1)!}{((\ell-1)!)^2} B(\ell,\ell) \\
& = 1.
\end{align*}
\end{proof}

\subsection{Proof of Proposition \ref{pro:faith_shap_path_integral}}

Again, we use the following property of beta functions.
$$
B(\alpha,\beta) = \int_{x=0}^1 x^{\alpha-1}(1-x)^{\beta-1} dx = \frac{(\alpha-1)!(\beta-1)!}{(\alpha+\beta-1)!}
\text{ for all } \alpha, \beta \in \mathbb{N}.
$$
Also, we have the following equalities \citet{hammer2012boolean}.
$$
g(x) = \sum_{T \subseteq [d]} a(\f,T) \prod_{i \in T} x_i,
\text{ and }
\Delta_S g(x) = \sum_{T \supseteq S} a(\f,T) \prod_{i \in T\backslash S} x_i,
$$
where $a(\f,\cdot)$ is the Mobius transform of $\f$.

Then for all $S \in \setlessell$ with $|S| = \ell$, we have
\begin{align*}
\int_{x=0}^1 \Delta_S g(x,\cdots,x) dI_x(\ell,\ell)
& = \frac{1}{B(\ell,\ell)} \int_{x=0}^1 x^{\ell-1}(1-x)^{\ell-1} \Delta_S g(x,\cdots,x) dx \\
& = \frac{1}{B(\ell,\ell)} \int_{x=0}^1 x^{\ell-1}(1-x)^{\ell-1} \sum_{W \supseteq S} \left(
a(\f,W) \prod_{i \in W \backslash S }x \right) dx \\
& = \frac{1}{B(\ell,\ell)} \int_{x=0}^1 x^{\ell-1}(1-x)^{\ell-1} \sum_{U \subseteq [d] \backslash S} \left(
a(\f,U \cup S) x^{|U|} \right) dx  \\
& \,\, (\text{ by setting } U = W \backslash S) \\
& = \frac{1}{B(\ell,\ell)} \sum_{U \subseteq [d] \backslash S} a(\f,U \cup S) \int_{x=0}^1 x^{\ell+|U|-1}(1-x)^{\ell-1}  dx \\
& = \frac{1}{B(\ell,\ell)} \sum_{U \subseteq [d] \backslash S}  \frac{(\ell+|U|-1)!(\ell-1)!}{(2\ell+|U|-1)!} a(\f,U \cup S)\\
& = \frac{1}{B(\ell,\ell)} \sum_{T \supseteq  S} \frac{(|T|-1)!(\ell-1)!}{(|T|+\ell-1)!}  a(\f,T)\\
& = 
a(\f,S) + (-1)^{\ell - |S| }\frac{|S|}{\ell + |S|} {\ell \choose |S| } \sum_{T \supset S, |T| > \ell} \frac{ {|T| - 1 \choose \ell }}{ { |T| + \ell -1 \choose \ell + |S|}} a(\f,T) \\
& = \ex^{\text{F-Shap}}_S(\f,\ell)
\ \ (\text{ by Eqn.\eqref{eqn:faith_shapley}}).
\end{align*}

\subsection{Proof of Proposition \ref{pro:ab_condition}}
In this proof, we use the notation of cumulative weighting function introduced in Definition \ref{def:cumulative_weighting} and the notation of $D^p_q$ in Definition \ref{def:Dpq}. 

Also, we utilize the results in the proof of Claim \ref{clm:mubar_closeform} in Section \ref{sec:clm_d_linear_relation}, where we show that there exists constants $ c^{(1)}_q , c^{(2)}_q,  c^{(3)}_q \in \mathbb{R}$ for $1 \leq q \leq d$ such that 
$$ \frac{D^{p}_{q}}{D^{p}_{q+1}} = c^{(1)}_q p +  c^{(2)}_q \text{ and } \frac{D^{p+1}_{q}}{D^{p}_{q+1}} =  c^{(1)}_q  p +  c^{(3)}_q
\ \text{ for } \ 0 \leq p \leq d-q-1.$$ 

Moreover, the relation between constants can be obtained recursively on $q$:
\begin{equation}
\label{eqn:ABC0}
c^{(1)}_0 = \frac{b-a^2}{(1-a)(a-b)} 
\ \ , c^{(2)}_0 = \frac{a-b}{(1-a)(a-b)}
\ \ , c^{(3)}_0= \frac{a(a-b)}{(1-a)(a-b)} ,
\end{equation}
and 
\begin{equation}
\label{eqn:abc_q_relation}
c^{(1)}_{q+1} = \frac{c^{(1)}_q}{c^{(1)}_q + c^{(2)}_q - c^{(3)}_q},
\ \ c^{(2)}_{q+1} = \frac{c^{(1)}_q + c^{(2)}_q}{c^{(1)}_q + c^{(2)}_q - c^{(3)}_q},
\ \ c^{(3)}_{q+1}=  \frac{c^{(3)}_q}{c^{(1)}_q + c^{(2)}_q - c^{(3)}_q}.
\end{equation}

Now we come back to the proof of Proposition \ref{pro:ab_condition}.
We first show that $c^{(1)}_q \geq 0 ,c^{(2)}_q >0$ and $c^{(3)}_q > 0$ for $q=0,1,2...,d$.
From Eqn. \eqref{eqn:ABC0}, by using the condition $1 > a > b \geq a^2 > 0$, we have 
$$
c^{(1)}_0 = \frac{b-a^2}{(1-a)(a-b)} \geq 0, 
\ \ c^{(2)}_0 = \frac{a-b}{(1-a)(a-b)} > 0, 
\ \ c^{(3)}_0= \frac{a(a-b)}{(1-a)(a-b)} > 0.
$$
Then, since $c^{(2)}_0 - c^{(3)}_0 = 1$ and $c^{(2)}_{q+1} - c^{(3)}_{q+1} = \frac{c^{(1)}_q + c^{(2)}_q - c^{(3)}_q}{c^{(1)}_q+c^{(2)}_q-c^{(3)}_q} = 1$ from Eqn.\eqref{eqn:abc_q_relation}, we get that the denominators of $c^{(1)}_{q+1},c^{(2)}_{q+1}$ and  $c^{(3)}_{q+1}$ are positive.
Then we get that $c^{(1)}_q \geq 0 ,c^{(2)}_q >0$ and $c^{(3)}_q > 0$ for $q=0,2...,d$. Since the ratio $D^{p}_q / D^p_{q+1} >0$ is positive and $D^p_0 = \mubar{p} > 0$ by Lemma \ref{clm:mubar_closeform}, we  conclude that $D^p_q > 0$ for all $p,q$ with $0 \leq p+q \leq d$.

Lastly, we note that $\mu_p = \sum_{j=p}^{d} (-1)^{j-p}{d-p \choose j-p } \mubar{j} =  D^p_{d-p} > 0$ for $0 \leq p \leq d$ by Claim \ref{clm:mubar_mu_relation}. Therefore, we get that $\mu(S) > 0$ for all $S \subseteq [d]$.

\newpage

\subsection{Proof of Proposition \ref{pro:guideline_ab} }
We first transform the ratios into the form of cumulative weighting functions as in Definition \ref{def:cumulative_weighting}.
$$
\frac{\mubar{d}}{\mubar{d-1}}
= \frac{\mu_d}{\mu_d + \mu_{d-1}}
= \frac{r_1}{r_1+1}
\ \ \text{ and }\ \ 
\frac{\mubar{d-1}}{\mubar{d-2}}
= \frac{\mu_d + \mu_{d-1}}{\mu_d + 2\mu_{d-1} + \mu_{d-2}}
= \frac{r_1r_2 + r_2}{1+2r_2+r_1r_2} 
$$
By Claim \ref{clm:mubar_closeform}, we have 
$$
\frac{\mubar{d}}{\mubar{d-1}} 
= \frac{a(a-b)+(d-1)(b-a^2)}{(a-b)+(d-1)(b-a^2)} 
= \frac{A + d-1}{B+d-1}
\ \ \text{ and }\ \ 
\frac{\mubar{d-1}}{\mubar{d-2}}
= \frac{a(a-b)+(d-2)(b-a^2)}{(a-b)+(d-2)(b-a^2)} 
= \frac{A + d-2}{B+d-2},
$$
where we let $A = \frac{a^2-ab}{b-a^2}$ and $B = \frac{a-b}{b-a^2}$. Next, by combining the above equations, we can solve $A$ and $B$ in terms of $r_1$ and $r_2$.

\begin{align*}
1 & = (A+d-1) - (A+d-2) = \frac{\mubar{d}}{\mubar{d-1}} (B+d-1)
- \frac{\mubar{d-1}}{\mubar{d-2}} (B+d-2)
\end{align*}
$$ \Rightarrow 
B = -(d-1) + \frac{1 - \frac{\mubar{d-1}}{\mubar{d-2}}}
{\frac{\mubar{d}}{\mubar{d-1}} - \frac{\mubar{d-1}}{\mubar{d-2}}}
= -(d-1) + \frac{(r_1+1)(r_2+1)}{r_1 - r_2}
$$
Similarly, we get
$$ 
A = -(d-1) + \frac{1 - \frac{\mubar{d-1}}{\mubar{d-2}}}
{1 - \frac{\mubar{d-1}}{\mubar{d}} \times \frac{\mubar{d-1}}{\mubar{d-2}}}
= -(d-1) + \frac{r_1(r_2+1)}{r_1 - r_2}.
$$
Next, by solving $A$ and $B$ in terms of $a$ and $b$, we get
$$
a  =  \frac{A}{B} 
\ \ \text{ and } \ \ 
b = \frac{A}{B} \times \frac{A+1}{B+1}.
$$
By plugging $A$ and $B$ into the above equation, we get the form of $a$ and $b$ in terms of $r_1$ and $r_2$:
$$
a = \frac{ r_1(r_2+1) -(d-1)(r_1-r_2) }
{(r_1+1)(r_2+1) - (d-1)(r_1-r_2)  } 
\ \text{ and } \ 
b = \frac{ r_1(r_2+1)  -(d-2)(r_1-r_2) }
{(r_1+1)(r_2+1) - (d-2)(r_1-r_2) } a.
$$
We next prove that given $1 > r_1 > r_2 > \frac{(d-2)r_1}{r_1 + d -1} > 0$, we have 
$1 > a > b \geq a^2 > 0$. First, we prove that $B > A > 0 $:
\begin{align*}
B-A 
& = \frac{(r_1+1)(r_2+1)}{r_1 - r_2}
- \frac{r_1(r_2+1) }{r_1 - r_2} \\
& = \frac{r_2+1}{r_1 - r_2} > 0 \\
\end{align*}
Also,
\begin{align*}
A > 0
& \Leftrightarrow  \frac{r_1(r_2+1)}{r_1 - r_2} > d-1 \\
& \Leftrightarrow 
 r_1(r_2+1)> (d-1)r_1 - (d-1)r_2 \\
& \Leftrightarrow 
(d-1 +r_1)r_2 > (d-2)r_1  \\
& \Leftrightarrow 
r_2 > \frac{(d-2)r_1 }{d-1 + r_1} \\
\end{align*}
Therefore, we have  $B > A > 0 $. Then we have
\begin{align*}
B > A > 0 
& \Rightarrow 1 > \frac{A}{B} >  \frac{A(A+1)}{B(B+1)}  \geq  \frac{A^2}{B^2} > 0\\
& \Rightarrow 1 > a > b \geq a^2 > 0.
\end{align*}

\subsection{Proof of Proposition \ref{pro:strictly_convex}}
\begin{proof}
First, we deal with the case when the coalition weighting function $\mu(\cdot)$ is finite such that $\mu(S) \in \mathbb{R}^+$ for all $S \subseteq [d]$. 
Recall that the objective is defined as following:
\begin{equation}
\label{eqn:weighted_regression_in_proof}
\sum_{S \subseteq [d]}  \mu(S) \left( \f(S) - \sum_{T \subseteq S , |T| \leq \ell}\Expl_T(\f,\ell) \right)^2.
\end{equation}
Obviously, this is a convex function since $\mu(S) > 0$ for all $S \subseteq [d]$. Now we show that it is additionally a strictly convex function:

We first rewrite Eqn.\eqref{eqn:weighted_regression_in_proof} into a matrix form. Let the feature matrix 
$$
\mathbf{X} \in \{ 0,1\}^{2^{d} \times d_\ell}
\text{  indexed with  } \mathbf{X}_{S,T} = \mathbbm{1}[(T \subseteq S) \vee (T = \text{\O})], 
\text{ where } S \subseteq [d]
\text{ and } T \in\mathcal{S}_\ell.
$$
We note that the feature matrix $\mathbf{X}$ is indexed with two sets $S$ and $T$, denoting its rows and columns. Each row of $S$ can also be expressed as $\mathbf{X}_S = \expd(S)$, where $\expd(S) \in \mathbb{R}^{d_\ell}$ with $\expd(S)[T] = \mathbbm{1}[(T \subseteq S) \vee (T = \text{\O})]$.

Then we define the weight matrix:
$$
\sqrt{\mathbf{W}} \in \{ 0,1\}^{2^{d} \times 2^{d}}
\text{ is a diagonal matrix with each entry on the diagonal } \sqrt{\mathbf{W}}_{S,S} = \sqrt{\mu(S)},
$$
where $S \subseteq [d]$.
The function values of $\f(\cdot)$ on each subset can be written into a vector:
$$
\mathbf{Y} \in \mathbb{R}^{2^{d}}
\text{  indexed with }
\mathbf{Y}_S = \f(S)
\text{ where } S \subseteq [d].
$$
With the above definitions, Equation \eqref{eqn:weighted_regression_in_proof} can be viewed as
\begin{equation}
\label{eqn:weighted_regression_matrix_form}
\norm{\sqrt{\mathbf{W} }(\mathbf{Y} -\mathbf{X} \Expl(\f,\ell))}_2^2 = \norm{\mathbf{Y_w}-\mathbf{X_w}\Expl(\f,\ell)}_2^2,
\end{equation}
where $\mathbf{Y_w} = \sqrt{\mathbf{W} }\mathbf{Y}$ and $\mathbf{X_w} = \sqrt{\mathbf{W} }\mathbf{X}$. The Hessian matrix can be expressed as $2\mathbf{X_w}^T \mathbf{X_w}$.

We now prove that the Hessian matrix is positive definite. 
Let $\mathbf{b}$ be any vector in  $\mathbb{R}^{d_\ell}$. 
$$
\mathbf{b}^T \left(2\mathbf{X_w}^T\mathbf{X_w} \right) \mathbf{b}
= 2\norm{\mathbf{X_w} \mathbf{b}}_2^2  
= 2\norm{\sqrt{\mathbf{W}}\mathbf{X} \mathbf{b}}_2^2 
\geq 0.
$$
We get that $\mathbf{b}^T \mathbf{X_w}^T\mathbf{X_w} \mathbf{b} = 0$ if and only if $\norm{\sqrt{\mathbf{W}}\mathbf{X} \mathbf{b}}_2 = 0$. Since $\sqrt{\mathbf{W}}$ is a diagonal matrix with each entry is a positive number, we get
$$
\norm{\sqrt{\mathbf{W}}\mathbf{X} \mathbf{b}}_2 = 0 
\Leftrightarrow
\mathbf{X} \mathbf{b} = \mathbf{0}
\Leftrightarrow
\expd(S) \mathbf{b} = 0, \ \ \forall \ S \subseteq [d].
$$

We prove that this also implies that $\mathbf{b}_L = 0$ for all $L \in \setlessell$ by using induction on the size of $T$.
\begin{enumerate}
    \item If we plug in $S = \text{\O}$, we get $\mathbf{b}_{\text{\O}} = 0$.
    \item Assume $\mathbf{b}_L = 0$ for all $L \subseteq [d]$ with $|L| \leq k$.
    \item For all subsets $L$ with size $k+1 \leq \ell$, we have 
    $$
    \sum_{T \subseteq L} \mathbf{b}_T 
    = \mathbf{b}_L +  \sum_{T \subseteq L, |T| \leq k} \mathbf{b}_T 
    =0.
    $$
    Therefore, we have $\mathbf{b}_L =0$.
\end{enumerate}
We obtain that $\mathbf{b}^T \mathbf{X_w}^T\mathbf{X_w} \mathbf{b} = 0$ if and only if $\mathbf{b} = \mathbf{0}$. Therefore, the Hessian matrix is positive definite and Eqn.\eqref{eqn:weighted_regression_in_proof} is strictly convex.

\end{proof}

\subsection{Proof of Proposition \ref{pro:unique_minimizer}}

\begin{proof}
If the coalition weighting function $\mu$ is finite, Proposition \ref{pro:strictly_convex} has shown that the objective is strictly convex and therefore has a unique minimizer. Now we deal with the case when $\mu(\text{\O})$ and $\mu([d])$ are allowed to be infinite.

We first define a new coalition weighting function $\mu' :2^d \rightarrow \mathbb{R}^{+}$ with 
$$
\mu'(S) =
\begin{cases}
1 & \text{ if } \mu(S) = \infty. \\
\mu(S) & \text{ otherwise.} \\
\end{cases}
$$
We can see that for all $\ex(\f,\ell) \subseteq \mathbb{R}^{d_\ell}$ satisfying  $\f(S) - \sum_{T \subseteq S , |T| \leq \ell}\Expl_T(\f,\ell) =0, \forall S: \mu(S) = \infty, S \subseteq [d]$, the values of objective functions instantiated with $\mu$ and $\mu'$ are the same. That is, 
$$
\sum_{S \subseteq [d]\,:\, \mu(S) < \infty}  \mu(S) \left( \f(S) - \sum_{T \subseteq S , |T| \leq \ell}\Expl_T(\f,\ell) \right)^2
= \sum_{S \subseteq [d]}  \mu'(S) \left( \f(S) - \sum_{T \subseteq S , |T| \leq \ell}\Expl_T(\f,\ell) \right)^2.
$$
Therefore, we can substitute $\mu$ with $\mu'$ and use the fact that  $\sum_{S \subseteq [d]}  \mu'(S) \left( \f(S) - \sum_{T \subseteq S , |T| \leq \ell}\Expl_T(\f,\ell) \right)^2$ is a strictly convex function by Propostion \ref{pro:strictly_convex}.

Then, since there exists at least a solution $\ex(\f,\ell) \subseteq \mathbb{R}^{d_\ell}$ for the set of linear equations $\f(S) - \sum_{T \subseteq S , |T| \leq \ell}\Expl_T(\f,\ell) =0, \forall S: \mu(S) = \infty, S \subseteq [d]$, we have at least one minimizer of Eqn.\eqref{eqn:constrained_weighted_regresion}. Suppose that there exist two minimizers $\ex_1(\f,\ell)$ and $\ex_2(\f,\ell)$. Since $\ex_1(\f,\ell)$ and $\ex_2(\f,\ell)$ both satisfy the set of linear equations $\f(S) - \sum_{T \subseteq S , |T| \leq \ell}\Expl_T(\f,\ell) =0, \forall S: \mu(S) = \infty, S \subseteq [d]$, $(\ex_1(\f,\ell) + \ex_2(\f,\ell)) / 2$ also satisfy it. However, since we have a strictly convex objective,
$$
F(\ex) = \sum_{S \subseteq [d]}  \mu'(S) \left( \f(S) - \sum_{T \subseteq S , |T| \leq \ell}\Expl_T(\f,\ell) \right)^2,
$$
we have $F(\ex_1) + F(\ex_2) > \frac{F((\ex_1+\ex_2)/2)}{2}$, which is a contradiction. Therefore, we have a unique minimizer for Eqn.\eqref{eqn:constrained_weighted_regresion}.

\end{proof}

\subsection{Proof of Proposition \ref{pro:ell_equals_d}}

\begin{proof}
First, by Proposition \ref{pro:unique_minimizer}, Eqn.\eqref{eqn:constrained_weighted_regresion} has a unique minimizer. Next, we prove that $\ex_S(\f,d) = a(\f,S)$ is the only minimzer. Specifically, \citet{grabisch2000equivalent} has shown that $\ex_S(\f,d) = a(\f,S)$ satisfies
$$
\f(S) - \sum_{T \subseteq S , |T| \leq \ell}\Expl_T(\f,\ell) 
= \f(S) - \sum_{T \subseteq S , |T| \leq \ell} a(\f,T) 
= 0 
\text{, for all }
S \subseteq [d].
$$
This implies that Eqn.\eqref{eqn:constrained_weighted_regresion} is zero. However, since Eqn.\eqref{eqn:constrained_weighted_regresion} is always non-negative,
$\ex_S(\f,d) = a(\f,S)$ is the only minimizer.
\end{proof}

\subsection{Proof of Proposition \ref{pro:parital_derivative_general}}
\begin{proof}

The objective can be expressed as a quadratic function of $\Expl_A(\f, \ell)$:
\begin{equation}
\label{eqn:quadratic_objective}
\sum_{S \subseteq [d], \mu(S) < \infty}  \mu(S) \left( \f(S) - \sum_{T \subseteq S , |T| \leq \ell}\Expl_T(\f,\ell) \right)^2 = 
a_A \Expl_A(\f, \ell)^2 + b_A\Expl_A(\f, \ell) + c_A.
\end{equation}

We now solve the coefficients $a_A, b_A, c_A$. First, the leading coefficient is
$$
a_A = \sum_{\substack{S:S \supseteq A, \\ \mu(S) < \infty }} \mu(S).
$$
Secondly, for any subset $B \in \setlessell, B \neq A$,  we note that $\Expl_A(\f_R, \ell)$ and $\Expl_B(\f_R, \ell)$ appear in the same bracket for all subsets $S \supseteq (A \cup B)$ in Eqn.\eqref{eqn:quadratic_objective}. Hence, the coefficient of the first order term is 
\begin{align*}
b_A
& = \left[\sum_{\substack{B:B \in \setlessell, \\ B \neq A}}  \sum_{\substack{S:S \supseteq A \cup B, \\ 
\mu(S) < \infty}} 2\mu(S)\Expl_B(\f, \ell) \right] - 2 \sum_{\substack{S:S \supseteq{A}, \\ \mu(S) < \infty}} \mu(S) \f(S) \\
& =2 \left[\sum_{\substack{B:B \in \setlessell, \\ B \neq A}} \Expl_B(\f, \ell) \sum_{\substack{S:S \supseteq A \cup B, \\ 
\mu(S) < \infty}} \mu(S) \right] - 2 \sum_{\substack{S:S \supseteq{A}, \\ \mu(S) < \infty}} \mu(S) \f(S). \\
\end{align*}

Combining the above, the partial derivative is
\begin{align*}
2a_A \Expl_A(\f, \ell)+ b_A =  
-2 \sum_{ \substack{S: S \supseteq A, \\ \mu(S) < \infty}} \mu(S) \f(S) + 2\sum_{S \in \setlessell} \Expl_S(\f, \ell) \sum_{\substack{L: L \supseteq S \cup T, \\ \mu(L) < \infty}} \mu(L).
\end{align*}

\end{proof}

\subsection{Proof of Proposition \ref{pro:constrained_closed_matrix_form_solution}}

\begin{proof}
We solve the constrained minimization problem via Lagrangian multiplier. Denote the objective $F(\ex) = \sum_{S \subseteq [d], 1 \leq |S| \leq d-1}  \mu(S) \left( \f(S) - \sum_{T \subseteq S , |T| \leq \ell}\Expl_T(\f,\ell) \right)^2$. Then we have the following equalities.
$$
\begin{cases}
\frac{\partial F(\ex)}{\partial \ex_A} 
= \lambda_{[d]} \cdot \frac{\partial }{\partial \ex_A} \left( -\f([d]) + \sum_{T \subseteq [d], |T| \leq \ell} \ex_T(\f,\ell) \right) = \lambda_{[d]}
& \text{ for all } A \in \setlessell \backslash \{\text{\O} \}. \\
\frac{\partial F(\ex)}{\partial \ex_{\text{\O}}} 
= \lambda_{\text{\O}} + \lambda_{[d]} \frac{\partial }{\partial \ex_{\text{\O}}} \left( -\f([d]) + \sum_{T \subseteq [d], |T| \leq \ell} \ex_T(\f,\ell) \right) =
\lambda_{\text{\O}} +\lambda_{[d]} \\
\sum_{T \subseteq [d], |T| \leq \ell} \ex_T = \f([d]) \\
\ex_{\text{\O}}(\f,\ell) = \f(\text{\O}).
\end{cases}
$$
By Proposition \ref{pro:parital_derivative_general}, we have 
\begin{align*}
\frac{\partial F(\ex)}{\partial \ex_A} 
& = -2 \sum_{ \substack{S: S \supseteq A, \\ \mu(S) < \infty}} \mu(S) \f(S) + 2\sum_{S \in \setlessell} \Expl_S(\f, \ell) \sum_{\substack{L: L \supseteq S \cup A, \\ \mu(L) < \infty}} \mu(L) \\
& = -2 \bar{\f}(A) + 2 \sum_{S \in \setlessell} \bar{\mu}(S \cup A)\Expl_S(\f, \ell). \\
\end{align*}
Combining the above two equations, we then have
$$
\begin{cases}
- \frac{1}{2}\lambda_{[d]} + \sum_{S \in \setlessell} \bar{\mu}(S \cup A)\Expl_S(\f, \ell) 
 =  \bar{\f}(A)
& \text{ for all } A \in \setlessell \backslash \{\text{\O} \}. \\
- \frac{1}{2}\lambda_{\text{\O}} - \frac{1}{2}\lambda_{[d]} + \sum_{S \in \setlessell} \bar{\mu}(S)\Expl_S(\f, \ell)
 =  \bar{\f}(\text{\O}) \\
\ex_{\text{\O}}(\f,\ell) = \f(\text{\O}).\\
\sum_{T \subseteq [d], |T| \leq \ell} \ex_T = \f([d]). \\
\end{cases}
$$
Now we write the system of linear equations into the matrix form.
$$
\mathbf{M} \begin{bmatrix}
\lambda_{\text{\O}} \\
\lambda_{[d]} \\
\ex_{\text{\O}}(\f,\ell) \\
\cdots \\
\ex_{S}(\f,\ell) \\
\ex_{T}(\f,\ell) \\
\cdots \\
\end{bmatrix}
= \mathbf{y}.
$$
By Proposition \ref{pro:unique_minimizer}, we know the system of linear equations has a unique solution, we have that the matrix $\mathbf{M}$ is invertible and therefore the solution can be expressed as $\mathbf{M}^{-1}\mathbf{y}$.

\end{proof}

\newpage

\section{Proof of Theorems}
\label{sec:proofs_theorems}
In this section, we provide the proof for Theorem \ref{thm:thmdummysymm} and \ref{thm:faith_shap}.

\subsection{Extra Notations}
\label{sec:extra_notations}

First, we introduce the cumulative weighting function. This function appears naturally in the partial derivatives of  Eqn.\eqref{eqn:weighted_regression} and Eqn.\eqref{eqn:constrained_weighted_regresion} with respect to each variable $\ex_S(\f,\ell)$ (we will show it in the later proof).  

\begin{definition}
\label{def:cumulative_weighting}
The cumulative weighting  function $\bar{\mu}:2^d \rightarrow \mathbb{R}^+$ such that $\bar{\mu}(S) = \sum_{\substack{T:T \supseteq S, \mu(T) < \infty}} \mu(T)$ for all subset $S \subseteq [d]$.
\end{definition}

When the function $\mu(S)$ only depends on the size of the input $|S|$, we simplify the notations by $\mu_{|S|} = \mu(S)$ and  $\bar{\mu}(S)
= \mubar{|S|} 
= \sum_{T \supseteq S, \mu(T) < \infty} \mu_{|T|} = \sum_{i: |S| \leq i \leq d, \mu_i < \infty}{d-|S| \choose i-|S|}\mu_{i}$ for all subsets $S \subseteq [d]$ to simplify the notation. Also, we introduce the notation of $D^{p}_q$ which will be used when solving first-order conditions of weighted regression problems.

\begin{definition}
\label{def:Dpq}
When the weighting function $\mu(\cdot)$ only depends on its input size, we define
$D^{p}_q = \sum_{j=0}^{q}{q \choose j} (-1)^{j}\mubar{p+j}$ for all $p,q \in \{0,1,2,...,d\}$ with $0 \leq p+q \leq d$. 
\end{definition}

In the proof in this section, the binomial coefficient ${ n \choose k}$ has a more general definition: for integers $n$ and $k$,
\begin{equation}
{ n \choose k}
= \begin{cases}
\frac{n!}{k!(n-k)!} & \text{  , if  } \ \ n \geq k \geq 0 \\
0 & \text{ , otherwise. }
\end{cases}
\end{equation}

\subsection{Proof of Theorem \ref{thm:thmdummysymm}}
\label{sec:proof_theorem_dummysymm}

We separate the proof of Theorem \ref{thm:thmdummysymm} into two parts: the sufficient condition and the necessary condition.

\subsubsection{Sufficient Condition}
\label{sec:suf_condition}

\begin{proof}
First of all, we prove the sufficient condition: suppose that the weighting function is in the following form:
\begin{align*}
\mu(S) & \propto \sum_{i=|S|}^{d} {d- |S| \choose i-|S|}(-1)^{i-|S|} g(a,b,i), \  \text{ where }
g(a,b,i) = 
\begin{cases}
1 & \text{ , if } \ i = 0. \\
\prod_{j=0}^{j=i-1} \frac{a(a-b) + j(b-a^2)}{a-b + j(b-a^2)}
& \text{ , if } \   1 \leq i \leq d. \\
\end{cases}
\end{align*} 
for some $a,b \in \R^{+}$ with $a>b$ such that $\mu(S) > 0$ for all $S \subseteq [d]$. Then we prove that the minimizer of Eqn.\eqref{eqn:weighted_regression} given the above weighting function satisfies interaction linearity, symmetry and dummy axioms. 

Since the $\mu(\cdot)$ defined in Eqn.\eqref{eqn:thmdummysymm} only depends on the size of the input set, to simplify the notations, we use $\mu_{|S|} = \mu(S)$ and  $\bar{\mu}(S) = \mubar{|S|} =  \sum_{i=|S|}^{d}{d-|S| \choose i-|S|}\mu_{i}$ to denote the weighting function and the cumulative weighting function (Definition \ref{def:cumulative_weighting}) for all subsets $S \subseteq [d]$. Also, since multiplying a scalar to $\mu$ does not change the minimizer of Eqn.\eqref{eqn:weighted_regression}, without loss of generality, we assume that $\mu(S) = \sum_{i=|S|}^{d} {d- |S| \choose i-|S|}(-1)^{i-|S|} g(a,b,i)$.

Then, we derive some properties of the weighting function $\mu(\cdot)$ the cumulative weighting function $\bar{\mu}(\cdot)$ and the operator $D_{i}^{t}$. We delay the proof of Claim \ref{clm:mubar_mu_relation}-\ref{clm:d_ri_positive} to Section \ref{sec:proofs_claims}.

\begin{claim}
\label{clm:mubar_mu_relation}
For all weighting function $\mu:\{0,1\}^d \rightarrow \mathbb{R}^+$, we have
$$\mu(S) = \sum_{T \supseteq S} (-1)^{|T| - |S|} \bar{\mu}(T) \text{  , for all $S \subseteq [d]$.  } 
$$
\end{claim}

Then the cumulative weighting function $\mubar{t}$ can be computed as following:

\begin{claim}
\label{clm:mubar_closeform}
When $\mu(S) = \sum_{i=|S|}^{d} {d- |S| \choose i-|S|}(-1)^{i-|S|} g(a,b,i)$, the cumulative weighting function is 
$$
\mubar{t} = 
\begin{cases}
\prod_{j=0}^{t-1} \frac{a(a-b) + j(b-a^2)}{a -b + j(b-a^2)} & \text{ , if } 1 \leq t \leq d. \\
1 & \text{ , if } t=0.\\
\end{cases}
$$
\end{claim}

Also, the following claim states that the operator $D^t_i$ is positive for all $t,i \in \{0,1,2,...,d\}$ with $0 \leq t+i \leq d$. 
\begin{claim}
\label{clm:d_ri_positive}
When $\mu(\cdot)$ is finite and permutation-invariant, we have $D^{p}_q > 0$ for all $p,q \in \{0,1,2,...,d\}$ with $0 \leq p+q \leq d$. 
\end{claim}

By Proposition \ref{pro:linearity}, the Faith-Interaction index satisfies interaction linearity axiom. 
In addition, by Proposition \ref{pro:symmetry},  since the weighting $\mu(S)$ only depends on $|S|$, the Faith-Interaction index also satisfies the interaction symmetry axiom. Consequently, we only need to prove that the Faith-Interaction index satisfies the interaction dummy axiom.

Next, we introduce the basis function.
\begin{definition}
\label{def:basis_function}
For any subset $R \subseteq [d]$ with $|R| = r$, the basis function with respect to $R$ is defined below:
$$
\f_R(S) = \begin{cases}
1, & \text{if  } \ S \supseteq R. \\
0, & \text{otherwise.}
\end{cases}
$$
\end{definition}

This is known as \textit{unanimity game} in game theory community. We note that only elements inside $R$ actually contribute to the function value. That is, elements belong to $[d] \backslash R$ are dummy elements. Formally, we have $\f(S \cup i) = \f(S)$ for any $i \in [d] \backslash R$ and $S \subseteq [d]\backslash \{i\}$.

An important property of the basis functions is that any function $\f:\{0,1\}^d \rightarrow \mathbb{R}$ can be expressed as a linear combination of the $2^d$ basis functions. 
Then by the interaction linearity axiom, the minimizer of Eqn.\eqref{eqn:weighted_regression} with respect to $\f(\cdot)$ can be represented as the same linear combination of minimizer of these basis functions. 
In the following lemma, we show that if the minimizers of Eqn.\eqref{eqn:weighted_regression} with respect to these $2^d$ basis functions satisfy the interaction dummy axiom, then all functions satisfy dummy axiom. Therefore, it is sufficient to prove that these minimizers satisfy the interaction dummy axiom for these $2^d$ basis functions.
\begin{lemma}
\label{lm:f_decomposable}
Let $\ex(\f,\ell)$ be a Faith-Interaction indice with respect to a proper weighting function $\mu:2^d \mapsto \mathbb{R}^+ \cup \{\infty \}$.
If $\ex_S(\f_R,\ell) = 0$ for all $\ell \in [d]$ and for all $S \in \setlessell, R \subseteq [d]$ with $S \cap ([d] \backslash R) \neq \text{\O}$, then the Faith-Interaction indices with respect to the weighting function $\mu(\cdot)$ satisfy interaction dummy axiom. 
\end{lemma}
\begin{proof}
By Lemma 3 in \citet{shapley1953value}, any function $\f:\{0,1\}^d \rightarrow \mathbb{R}$ can be expressed as a linear combination of these $2^d$ basis functions, such that
$$
\f = \sum_{R \subseteq [d]} c_R \f_R 
\ \  \text{  with  }  \ \ 
c_R = \sum_{T \subseteq R} (-1)^{|R| - |T|} \f(T).
$$
$c_R$ here is the Möbius coefficient.
By Proposition \ref{pro:linearity}, the Faith-Interaction indice $\ex$ satisfies the interaction linearity axiom, which implies that $\Expl(\f,\ell)$ can be expressed as the following form:
$$
\Expl(\f,\ell) = \sum_{R \subseteq [d]} c_R \Expl(\f_R,\ell).
$$
Suppose that we have some dummy feature $i \in [d]$ such that $\f(T \cup i) = \f(T)$ for any $ T \subseteq [d] \backslash \{ i\}$, then for any $R \supseteq \{ i\}$, we have 
$$
c_R 
= \sum_{T \subseteq R} (-1)^{|R| - |T|} \f(T)
= \sum_{T \subseteq R \backslash \{ i\}} (-1)^{|R| - |T|} ( \f(T) - \f(T \cup i)) = 0.
$$
Therefore, the function $\f$ is the linear combinations of $\f_R(\cdot)$ for some $R$ not containing $i$. However, for these subsets, by the definition of the basis function, we have $\f_R(S) = \f_R(S \cup i)$ for any $S \subseteq [d] \backslash \{i\}$. Since we have $\ex_S(\f_R, \ell) = 0$ for all $S \in \setlessell$ with $S \cap ([d] \backslash R) \neq \text{\O}$, we get $\ex_T(\f_R, \ell) = 0$ for any $T$ containing the dummy feature $i$.
Consequently, we have 
$$
\Expl_T(\f,\ell) = \sum_{R \subseteq [d]} c_R \Expl_T(\f_R,\ell) = \sum_{R \subseteq [d]\backslash \{i\}} c_R \Expl_T(\f_R,\ell) = 0.
$$
We note that the above proof holds for any $\ell \in [d]$. Therefore, the Faith-Interaction indices with respect to $\mu(\cdot)$ satisfy interaction dummy axiom.
\end{proof}

Now we come back to the proof of Theorem \ref{thm:thmdummysymm}. By Lemma \ref{lm:f_decomposable}, we only need to prove that for any $R \subseteq [d]$, the minimizer of Eqn.\eqref{eqn:weighted_regression} with respect to the basis function $\f_R$, $\ex(\f_R,\ell)$, satisfies $\ex_S(\f_R,\ell) = 0$ if $S$ containing any dummy element in $[d] \backslash R$. 

The objective function of the weighted least square problem with respect to the basis function $\f_R$ can be written as follows:

\begin{align}
\label{eqn:weighted_ls}
F_R(\ex) & = \sum_{S \subseteq [d]} \mu_{|S|} \left[\sum_{T \subseteq S, |T| \leq \ell} \Expl_T(\f_R, \ell) - \f_R(S) \right]^2
\\
\label{eqn:weighted_ls_2}
& = \sum_{S \supseteq R, S \subseteq [d]} \mu_{|S|} \left[ \sum_{T \subseteq S, |T| \leq \ell} \Expl_T(\f_R, \ell) - 1 \right]^2 + \sum_{S \not\supseteq R, S \subseteq [d]} \mu_{|S|} \left[ \sum_{T \subseteq S, |T| \leq \ell} \Expl_T(\f_R, \ell)  \right]^2.
\end{align}
Note that the weighting function $\mu$ is defined in Eqn.\eqref{eqn:thmdummysymm}. Let $r = |R|$ denote the size of the set $R$. Now, we separate the problem into three cases: (1) $d \geq \ell \geq  r \geq 0$. (2) $d \geq \ell + r $ and $r > \ell$. (3) $\ell + r > d \geq r > \ell \geq 1$.

\paragraph{(1)  $d \geq \ell \geq  r \geq 0$:} 

\begin{lemma}
If $\f_R$ is a basis function with $|R| = r \leq \ell$, the unique minimizer of Eqn.\eqref{eqn:weighted_ls} is
\begin{equation}
\Expl_T(\f_R, \ell)
= \begin{cases}
1 & \text{, if } \ \ T = R. \\
0 & \text{, otherwise.}
\end{cases}
\end{equation}
\end{lemma}

\begin{proof}
By Proposition \ref{pro:unique_minimizer}, Eqn.\eqref{eqn:weighted_ls} has a unique minimizer. 
If we plug in the above definition of  $\Expl_T(\f_R, \ell)$ to Eqn.\eqref{eqn:weighted_ls_2}, we get that 
$$
\sum_{T \subseteq S, |T| \leq \ell} \Expl_T(\f_R, \ell) - \f_R(S) = 0 
\ \ \text{ for all } \ \
S \subseteq [d].
$$
This implies that $F_R(\ex) = 0$. Since the objective is always non-negative, this is the unique minimizer of Eqn.\eqref{eqn:weighted_ls}.
\end{proof}

\paragraph{(2) $d \geq \ell + r$ and $r > \ell$:} 
Next, we solve Eqn.\eqref{eqn:weighted_ls_2} by using partial derivatives. By Proposition \ref{pro:parital_derivative_general}, for all $A \in \setlessell$, we have 
\begin{align}
\frac{\partial \ F(\ex)}{\partial \Expl_A}
& = -2 \sum_{ \substack{S: S \supseteq A, \\ \mu(S) < \infty}} \mu(S) \f_R(S) + 2\sum_{S \in \setlessell} \Expl_S(\f, \ell) \sum_{\substack{L: L \supseteq S \cup A, \\ \mu(L) < \infty}} \mu(L)
\nonumber \\
& = -2\mubar{|A \cup R|} + 2\sum_{S \in \setlessell} \mubar{|S \cup A|} \Expl_S(\f_R, \ell).
\label{eqn:partial_derivative}
\end{align}

Now we utilize the symmetric structure in the basis function $\f_R(\cdot)$ and weighting function $\mu(\cdot)$.
In the basis functions, there are only two kinds of input elements, which are elements in $R$ and not in $R$. Therefore, for $i^{th}$ order interactions terms $\Expl_{T}(\f_R,\ell)$ where $|T| = i$, there are at most $i+1$ distinct values. Each value corresponds to the set with $j$ elements in $R$ for $j=0,1,...i$. That is, by the interaction symmetry axiom, there are only $i+1$ different importance value for $i^{th}$ order interactions terms (since if $|T_1| = |T_2|$ and $|T_1 \cap R| = |T_2 \cap R|$ then $T_1$ and $T_2$ are symmetric and $\Expl_{T_1}(\f_R,\ell) = \Expl_{T_2}(\f_R,\ell)$), so there are $ 1 + 2 +... + (\ell + 1)= \frac{(\ell + 2)(\ell + 1)}{2}$ kinds of values in the minimizer $\Expl(\f_R,\ell)$. We then introduce a new notation system that utilizes the symmetric structure. First, we use a vector $\bfb$
to represent these $\frac{(\ell + 2)(\ell + 1)}{2}$ values.

\begin{definition}
\label{def:bij}
The vector
$ \bfb \in \mathbb{R}^{\frac{(\ell +1)(\ell + 2)}{2}} \text{ is indexed with  }
\bfb_{i,j} = \Expl_S(\f_R,\ell) \text{  with  } |S| = i \text{ and } |S \backslash R| = j,$
where $i, j$ are integers with $ 0 \leq i \leq \ell$ and $0 \leq j \leq i$.
\end{definition}

The term $\bfb_{i,j}$ means the importance score of an $i^{th}$ order interaction (of size $i$) term with $i-j$ elements lying in $R$ and $j$ element lying in $[d] \backslash R$. Note that in this definition, $\bfb_{0,0} = \ex_{\text{\O}(\f_R,\ell)}$ means the bias term in the weighted linear regression. Now we can apply this new notation to rewrite Eqn. \eqref{eqn:partial_derivative}.

\begin{lemma}
\label{lm:bij_parital_derivative}
The partial derivative of $F_R(\ex)$ with respect to $\bfb_{i,j} $ is 
$$
\frac{\partial F_R(\ex)}{\partial \bfb_{i,j}} 
= -2\mubar{r+j}  
+ 2 \sum_{p=0}^{\ell} \sum_{q=0}^{p} \left( \sum_{\rho=0}^{i-j} \sum_{\sigma=0}^{j} { i-j \choose \rho} { r-(i-j) \choose p-q-\rho} { j \choose \sigma}{ d-r-j \choose q-\sigma} \mubar{i+p-\rho-\sigma} \bfb_{p,q} \right), 
$$
where $r = |R|$.
\end{lemma}

\begin{proof}
Let $\Expl_A(\f_R, \ell) = \bfb_{i,j} $, where $i = |A|$ and $ j = |A \backslash R| $. Then the first term in Eqn. \eqref{eqn:partial_derivative} is $-2\mubar{|A \cup R|} = -2\mubar{r+j}.$ We know that for any 

\begin{equation}
\label{eqn:mapping_S_to_b}
\mubar{|S \cup A|}\Expl_S(\f_R, \ell)
= \mubar{ \{ |S| + |A| - |S \cap  A| \}} \bfb_{ \{|S|,|S \backslash R| \}} =  \mubar{ \{ |S| + |A| - |S \cap R \cap A | - |(S \cap A \backslash R  | \} } \bfb_{ \{|S|,|S \backslash R| \}}
\end{equation}

for all $S \in \setlessell$. The equation depends on the four sets, which are $A \cap R, A \backslash R, S \cap R$ and $S \backslash R$. 
Therefore, in the following proof of this lemma, we split set $S$ to $S \cap R$ and $S \backslash R$ and consider them in different cases.

We now let $\Expl_S(\f_R, \ell) = \bfb_{p,q}$, so that $p(S) = |S|$ and $q(S) = |S \backslash R|$. Assume $\rho(S) = |S \cap R \cap A |$ and $\sigma(S) = |S  \cap A \backslash R |$. 
Eqn. \eqref{eqn:mapping_S_to_b} can be written as
$$
\mubar{|S \cup A|}\Expl_S(\f_R, \ell) = \mubar{p(S) + i - \rho(S) - \sigma(S)} \bfb_{p,q}.
$$

We can calculate the number of sets $S$ that satisfies the constraints $p(S) = p_0, q(S) = q_0, \rho(S) = \rho_0, \sigma(S) = \sigma_0$.
Since $S = (S \cap R) + (S \backslash R)$, the number of $S$ (satisfies the constraints) equals to the number set $ (S \cap R)$ (satisfies the constraints) times the number of set $ (S \backslash R)$ (satisfies the constraints), since $S$ is determined given $S\cap R$ and $ (S \backslash R)$. We calculate the number of set $(S \cap R)$ and set $ (S \backslash R)$ (that satisfies the constraints) respectively.

\begin{enumerate}
    \item First, we observe that the number of $S \cap R$ (that satisfies the constraints) is equal to the number of ways to choose $|S \cap R|$ elements from $R$ (that satisfies the constraints), and $|S \cap R| = p(S) - q(S) = p_0 - q_0.$ Choosing $p_0 - q_0$ elements from $R$ can be further viewed as choosing $|(S \cap R) \cap A|$ elements from $R\cap A$ and $|(S \cap R) \backslash A|$ elements from $R\backslash A$. We note that $|(S \cap R) \cap A| = \rho_0$ and $|(S \cap R) \backslash A| = |(S \cap R)| - |(S \cap R) \cap A| = p_0 - q_0 - \rho_0$.
    Therefore, there are ${|R \cap A| \choose \rho_0 }{ |R  \backslash A | \choose p_0-q_0-\rho_0} = {i-j \choose \rho}{r -(i-j) \choose (p-q)-\rho}$ ways to select $|S \cap R|$ elements from $R$.
    
    \item  Secondly, we observe that the number of $S \backslash R$ (that satisfies the constraints) is equal to the number of ways to choose $|S \backslash R|$ elements from $[d] \backslash R$ (that satisfies the constraints), and $|S \backslash R| = q(S) =  q_0.$ Choosing $q_0$ elements from $[d] \backslash R$ can be further viewed as choosing $|S  \cap A \backslash R|$ elements from $A \backslash R$ and $|S \backslash R \backslash A|$ elements from $[d] \backslash R \backslash A$.
    We note that $|S  \cap A \backslash R| = \sigma_0 $ and $|S \backslash R \backslash A| = q_0- \sigma_0$.
    Therefore, there are ${| A \backslash R| \choose  p_0-q_0 -\sigma_0 }{ |[d]\backslash R \backslash A| \choose q_0 -\sigma_0} = {j \choose \sigma}{d-r-j \choose q -\sigma}$ ways to select $|S \cap R|$ elements from $R$.
    
    For elements in $S \backslash R$, there should be $\sigma =  |(S \backslash R) \cap ( A \backslash R )|$ elements 
    from $ A \backslash R = ([d]\backslash R) \cap A$ and $q - \sigma$ elements  from $ ([d]\backslash R) \backslash A$ to satisfy the constraint $\sigma$. Therefore,  
    there are ${ |([d]\backslash R) \cap A| \choose \sigma }{| ([d]\backslash R) \backslash A| \choose q- \sigma}  = {j \choose \sigma}{d-r-j \choose q-\sigma}$  ways to select 
    elements of $S \backslash R$ from $[d]\backslash R$.
    
\end{enumerate}

Note that we have $\rho \leq |A \cap R| = i-j$ and $\sigma \leq |A \backslash R| = j $. Since every $S$ can map to some constraints $p,q,\rho, \sigma$, by summation over all possible $p,q, \rho, \sigma$, we can get the partial derivative of the objective with respect to $\bfb_{i,j}$.

Also, if there are not enough elements to be selected, i.e. $|R \backslash A| < p_0-q_0-\rho_0$, ${|R \backslash A| \choose p_0-q_0-\rho_0} = 0$, so the number of ways for selection is zero.

\end{proof}

Now, since the minimizer of Eqn.\eqref{eqn:weighted_regression} must satisfy $\frac{\partial F_R(\ex)}{\partial \bfb_{i,j}} $ for all $0 \leq i \leq \ell$ and $0 \leq j \leq i$ , we can write the system of $\frac{(\ell+2)(\ell+1)}{2}$ equations into a matrix form, $\mathbf{M}\bfb= \mathbf{Y}$, with definitions below.

\begin{definition}
The coefficient matrix $\mathbf{M} \in \mathbb{R}^{\frac{(\ell +1)(\ell + 2)}{2} \times \frac{(\ell +1 ) (\ell + 2)}{2}}$, whose rows and columns are indexed with 2 iterators respectively. The value of each entry is 
$$
\mathbf{M}_{\{i,j\}, \{p,q\}} = \sum_{\rho=0}^{i-j} \sum_{\sigma=0}^{j} { i-j \choose \rho } { r-(i-j) \choose p-q-\rho} { j \choose \sigma}{ d-r-j \choose q-\sigma} \mubar{i+p-\rho-\sigma}, $$
where $ 0 \leq i \leq \ell$, $0 \leq j \leq i$,  $ 0 \leq p \leq \ell$ and $0 \leq q \leq p$.
\end{definition}

\begin{definition}
 $\mathbf{Y} \in \mathbb{R}^{\frac{(\ell + 1) (\ell+2)}{2}}$ is a column vector with each entry $\mathbf{Y}_{j} = \mubar{r+j}$.
\end{definition}

Now we prove the interaction dummy axiom holds for the basis function $\f_R$. That is,
$\bfb_{i,j} = 0$ if $j > 0$ for all $i^{th}$ order interaction, $0 \leq i \leq \ell$ since $j > 0$ means there are some elements lying in $[d] \backslash R$.

\begin{lemma}
\label{lm:simplified_matrix}
Assume we have a system of $\upsilon$ linear equations with $\upsilon$ unknowns, $\mathbb{A} \mathbf{v} = \mathbb{C}$, where $\mathbb{A} \in \mathbb{R}^{\upsilon \times \upsilon}$ is the coefficient matrix, $\mathbf{v} \in \mathbb{R}^{\upsilon}$ is a vector of unknowns, and $\mathbb{C} \in \mathbb{R}^{\upsilon}$ is a vector of real numbers. Define $P \subseteq{\{1,2,...,\upsilon \}}$ as a set of indexes, and denote $\mathbb{A}_{P} \in \mathbb{R}^{\upsilon \times |P|}$ as a submatrix of $\mathbb{A}$ by only taking columns in $\mathbb{A}$ whose indexes are in $P$.
If $\mathbb{A}\mathbf{v}  = \mathbb{C}$ has a unique solution $\mathbf{v} $ and $rank([\mathbb{A}_P, \mathbb{C}]) = |P|$, then the solution $\mathbf{v}_i = 0$ if $i \notin P$. 
\end{lemma}

\begin{proof}
Since $\mathbb{A}\mathbf{v} = \mathbb{C}$ has a unique solution, by Rouché–Capelli theorem, we have $rank(\mathbb{A}) = rank([\mathbb{A},\mathbb{C}]) = \upsilon$, which equals to the number of columns in $\mathbb{A}$. Hence, the columns of $\mathbb{A}$ are linear independent. The column spaces of $\mathbb{A}_P$ consists of $|P|$ columns from $\mathbb{A}$, so we have $rank(\mathbb{A}_P) = |P|$.

Let $\mathbf{v} _P \in \mathbb{R}^{|P|}$ consist of values in $x$ whose indexes are in $P$.
By Rouché–Capelli theorem, $rank(\mathbb{A}_P)= rank([\mathbb{A}_P,\mathbb{C}]) = |P|$ implies the system of linear equations, $\mathbb{A}_P\mathbf{v}_P=\mathbb{C}$, has a unique solution $\mathbf{v}_P$. 
Now we construct the solution of $\mathbb{A}\mathbf{v}  = \mathbb{C}$ by $\mathbf{v}_P$: let
$$
\mathbf{v} _i = \begin{cases}
\mathbf{v} _{Pj} \text{ if } i \in P, \text{ where $i$ the $j^{th}$ element of $P$} \\
0 \text{ , otherwise, } \\
\end{cases}
$$
where $\mathbf{v}_{Pj}$ denote the $j^{th}$ element of $\mathbf{v}_P$. We can easily verify that $\mathbf{v}$ is the unique solution of $\mathbb{A}\mathbf{v} = \mathbb{C}$. 

\end{proof}

Lemma \ref{lm:simplified_matrix} tells us that if we aim to prove that some unknown variables are zero in a system of linear equations, we can alternatively prove that the rank of a simplified augmented matrix $[\mathbb{A}_P,\mathbb{C}]$ equals $|P|$. 

Then, if the interaction dummy axiom holds, the terms $\bfb_{i,j}$ with $i \geq j > 0$ should be zero since these interaction terms contain at least one dummy element ( that is outside $R$). Therefore, by Lemma \ref{lm:simplified_matrix},
we now consider columns corresponding to $\bfb_{0,0}, \bfb_{1,0},..., \bfb_{\ell,0}$. These columns correspond to interaction terms that only contain elements in $R$. 
We simplify the matrix $\mathbf{M}$ in the following way:

Put
$$
\mathbf{M'} \in \mathbb{R}^{\frac{(\ell+1)(\ell+2)}{2} \times (\ell+1)} \text{, whose columns correspond to } \bfb_{i,0} \text{ of } \mathbf{M} \text{ for } 0 \leq i \leq \ell
$$
\begin{equation}
\label{eqn:m_definition}
\small
\text{ with each entry } \mathbf{M'}_{\{i,j\},\{p, 0\} } = \sum_{\rho=0}^{i-j} { i-j \choose \rho} { r-(i-j) \choose p- \rho} \mubar{i+p-\rho}  \text{ for } 0 \leq i \leq \ell,  0 \leq p \leq \ell  \text{ and }  0 \leq j \leq i.
\end{equation}

The entry $\mathbf{M'}_{\{i,j\},\{p, 0\}}$ can be interpreted as the coefficient of $\bfb_{p,0}$ in the equation $\frac{\partial F_R(\ex)}{ \partial \bfb_{i,j}} = 0$.
Since we have already known that the system of linear equations, $\mathbf{M}b=\mathbf{Y}$, has a unique solution by Proposition \ref{pro:unique_minimizer}, if we can prove that the rank of the matrix $\mathbf{Q} = [\mathbf{M'}, \mathbf{Y}]$ equals to $\ell + 1$, we can conclude that $\bfb_{i,j} = 0 $ for all $i \geq j > 0$ by Lemma \ref{lm:simplified_matrix}. It implies the interaction dummy axiom holds for the basis function $\f_R$. To calculate the rank of matrix $\mathbf{Q}$, we first define some notations.

\begin{definition}
\label{def:row_combination}
We define a function  $\Re_{i,j}^k(\cdot): \mathbb{R}^{\frac{(\ell+1)(\ell+2)}{2}  \times (\ell+1) } \rightarrow  \mathbb{R}^{ \ell+1} $, which takes a matrix as input and outputs a weighted summation of rows. Formally, for any matrix $\mathbf{B} \in \mathbb{R}^{\frac{(\ell+1)(\ell+2)}{2}  \times (\ell+1)}$, let
$$
\Re_{i,j}^k(\mathbf{B}) = \sum_{\sigma=0}^{\sigma=k} {k \choose \sigma} (-1)^{\sigma} \mathbf{B}_{\{i,j+\sigma \}}
$$
be a combination of rows of $\mathbf{B}$
for any $0 \leq i \leq \ell $, $0 \leq j \leq i$ and $0 \leq k \leq i-j$, where $\mathbf{B}_{\{i,j+\rho\}}$ is denoted as the $\{i, j+\rho\}$th row of the matrix $\mathbf{B}$.
\end{definition}
$\Re_{i,j}^k (\mathbf{Q})$ can be interpreted as some row operations during the Gaussian elimination process along the rows corresponding to the interaction terms of size $i$, $\bfb_{i,\rho}$ for some $0 \leq \rho \leq i$, in the matrix $\mathbf{Q}$.
We note that $\Re_{i,j}^0(\mathbf{B}) = \mathbf{B}_{\{i,j\}}$ if $k = 0$.
Then we prove the following lemma. 

\begin{definition}
\label{def:Pst}
Define
$$
\mathbf{P}_{s,t} = 
\begin{bmatrix}
\Re_{s,s}^{0}(\mathbf{Q}) \\
\Re_{s+1,s}^{1}(\mathbf{Q}) \\
...\\
\Re_{s+t, s }^{t}(\mathbf{Q}) \\
\end{bmatrix}, 
\text{ and } 
\mathbf{P'}_{s,t} = 
\begin{bmatrix}
\mathbf{P}_{s,t} \\
\Re_{s+t-1, s }^{t-1}(\mathbf{Q}) - \Re_{s+t, s}^{t-1}(\mathbf{Q}) \\
\end{bmatrix}, 
$$
where $\mathbf{P}_{s,t} \in \mathbb{R}^{(t+1) \times (\ell + 2)}$ and $\mathbf{P'}_{s,t} \in \mathbb{R}^{(t+2) \times (\ell + 2)}$ for 
$0 \leq s \leq \ell-1 $ and $1 \leq t \leq \ell - s$.

\end{definition}

\begin{lemma}
\label{lm:Pst_value}
Following Definition \ref{def:Dpq} and \ref{def:Pst} ,  
\begin{equation}
\small
\mathbf{P'}_{s,t} =
\begin{bmatrix}
D_{0}^{s}, & {r \choose 1 }D_{0}^{s+1}, & ...&
{r \choose t-1 }D_{0}^{s+t-1}, & {r \choose t }D_{0}^{s+t}, &...  & { r \choose \ell  }D_{0}^{s+\ell},  & D_{0}^{r+s} \\
0, & {r-1 \choose 0 }D_{1}^{s+1}, & ...& 
{r-1 \choose t-2 }D_{1}^{s+t-1}, & 
{r-1 \choose t-1 }D_{1}^{s+t}, & ... & { r-1 \choose \ell -1 } D_{1}^{s+\ell}, & D_{1}^{r+s} \\
0, & 0, & ...&  
{r-2 \choose t-3 }D_{2}^{s+t-1}, & 
{r-2 \choose t-2 }D_{2}^{s+t}, & ... & { r-2 \choose \ell-2 } D_{2}^{s+\ell}, & D_{2}^{r+s} \\
., & ., & ... &., &., & ... & .,& . \\
0, & 0,& ...& 0 , & 
{r-t \choose 0 }  D_{t}^{s+t}, & ... 
 & { r-t \choose \ell-t } D_{t}^{s+\ell}, & D_{t}^{r+s}  \\
0,& 0,& ...& {r-t+1 \choose 0} D_{t}^{s+t-1}, & 
[{r-t+1 \choose 1}- {r-t \choose 0 } ] D_{t}^{s+t}, &... & [{r-t+1 \choose \ell-t+1}- {r-t \choose \ell-t } ] D_{t}^{s+\ell}, & 0  \\
\end{bmatrix}.
\end{equation}
Formally, for all $0 \leq t' \leq t+1$ and $0 \leq p \leq \ell+1$, the $(p+1)^{th}$ element of  $(t'+1)^{\ th}$ row of $\mathbf{P'}_{s,t}$ is 
\begin{equation}
\begin{cases}
0 & \text{ if } 0 \leq t' \leq t \text{ and } p < t' \\
{r-t' \choose p-t'} D_{t'}^{s+p} 
& \text{ if } 0 \leq t' \leq t  \text{ and }  t' \leq p \leq \ell\\
D_{t'}^{r+s} & \text{ if } 0 \leq t' \leq t \text{ and } p = \ell+1 \\
0 & \text{ if }    t'=t+1 \text{ and } p < t-1  \\
{r-t+1 \choose 0} D_{t}^{s+t-1}  & \text{ if } t'=t+1 \text{ and } p = t-1\\
[{r-t+1 \choose p-t+1}- {r-t \choose p-t } ] D_{t}^{s+p} & \text{ if }    t'=t+1 \text{ and } t \leq  p \leq \ell  \\
0 & \text{ if }    t'=t+1 \text{ and } p =\ell + 1 \\
\end{cases}.
\end{equation}
\end{lemma}

\begin{proof}
We first introduce two claims.
\begin{claim}
\label{clm:row_operation}
For $0 \leq s' \leq \ell -1$, $0 \leq t' \leq \ell-s'$ and $0 \leq p'\leq \ell$, 
\begin{equation}
\sum_{\sigma=0}^{\sigma=t'} {t' \choose \sigma} (-1)^{\sigma} \mathbf{M'}_{\{s'+t', s'+\sigma \},\{p',0\}} 
= \begin{cases}
0 & \text{ if } p' < t' \\
{r-t' \choose p-t'} D_{t'}^{s'+p'}
  & \text{ if }  t' \leq p' \leq \ell
\end{cases}
\end{equation}

\end{claim}

\begin{claim}
\label{clm:row_decomp}
For any matrix $\mathbf{B} \in \mathbb{R}^{\frac{(\ell+1)(\ell+2)}{2}  \times (\ell+1)}$, $\Re_{s+t,s}^{t-1}(\mathbf{B}) = \Re_{s+t,s}^{t}(\mathbf{B}) + \Re_{s+t,s+1}^{t-1}(\mathbf{B})$.
\end{claim}
We leave the proof of Claim \ref{clm:row_operation} and \ref{clm:row_decomp} to Section \ref{sec:proofs_claims}.
We now compute value of each entry of $\mathbf{P'}_{s,t}$ by cases.
\begin{enumerate}
    \item If $0 \leq t' \leq t$ and $p < t'$, by plugging in $s'=s$, $t'=t'$ and $p'=p$ to Claim \ref{clm:row_operation}, the value of $(p+1)^{th}$ element of $(t'+1)^{th}$ row of $\mathbf{P'}_{s,t}$ is
    $$
    \sum_{\sigma=0}^{\sigma=t'} {t' \choose \sigma} (-1)^{\sigma} \mathbf{M'}_{\{s+t', s+\sigma \},\{p,0\}} 
    = 0 
    $$
    
    \item If $0 \leq t' \leq t$ and $t' \leq p \leq \ell$, by plugging in $s'=s$, $t'=t'$ and $p'=p$ to Claim \ref{clm:row_operation}, the value of $(p+1)^{th}$ element of $(t'+1)^{th}$ row of $\mathbf{P'}_{s,t}$ is
    $$
    \sum_{\sigma=0}^{\sigma=t'} {t' \choose \sigma} (-1)^{\sigma} \mathbf{M'}_{\{s+t', s+\sigma \},\{p,0\}} 
    = {r-t' \choose p-t'} D_{t'}^{s+p}
    $$
    \item If $0 \leq t' \leq t$ and $p=\ell + 1$, the value of $(\ell+2)^{th}$ element of $(t'+1)^{th}$ row of $\mathbf{P'}_{s,t}$ is 
    $$
    \sum_{\sigma=0}^{\sigma=t'} {t' \choose \sigma} (-1)^{\sigma} \mathbf{Y}_{s+t', s+\sigma } 
    = \sum_{\sigma=0}^{\sigma=t'} {t' \choose \sigma} (-1)^{\sigma} \mubar{r+s+\sigma } 
    = D_{t'}^{r+s}
    $$
    
    \item If $t'=t+1$ and $p < \ell+1$, by Claim \ref{clm:row_decomp}, the  $(t+2)^{th}$ row of $\mathbf{P'}_{s,t}$ becomes $ \Re_{s+t-1,s}^{t-1} - \Re_{s+t,s+1}^{t-1} - \Re_{s+t,s}^{t} $.
    The value of $p^{th}$ element this row is
    \begin{align*}
    & \sum_{\sigma=0}^{t-1} {t-1 \choose \sigma} (-1)^{\sigma} ( \mathbf{M'}_{\{s+t-1, s+\sigma \},\{p,0\}} -   \mathbf{M'}_{\{s+t, s+\sigma+1 \},\{p,0\}} ) 
    - \sum_{\sigma=0}^{t} {t \choose \sigma} (-1)^{\sigma} \mathbf{M'}_{\{s+t, s+\sigma \},\{p,0\}} \\
    & =  
    \begin{cases}
    0 & \text{ if } p < t-1 \\
    {r-t+1 \choose p-t+1} D_{t-1}^{s+p} 
    - {r-t+1 \choose p-t+1} D_{t-1}^{s+p+1}
    & \text{ if }  p = t-1  \\
    {r-t+1 \choose p-t+1} D_{t-1}^{s+p} 
    - {r-t+1 \choose p-t+1} D_{t-1}^{s+p+1}
    - {r-t \choose p-t } D_{t}^{s+p}
    & \text{ if }  t \leq p \leq \ell   \\
    \end{cases}
    \ \ \ (\text{Claim \ref{clm:row_operation}}) \\
    & =  
    \begin{cases}
    0 & \text{ if } p < t-1 \\
    {r-t+1 \choose p-t+1} D_{t}^{s+p} 
    & \text{ if }  p = t-1  \\
    [{r-t+1 \choose p-t+1}- {r-t \choose p-t } ] D_{t}^{s+p} & \text{ if }  t \leq p \leq \ell   \\
    \end{cases}
    \end{align*}
    
    \item If $t'=t+1$ and $p=\ell + 1$, the value of $(\ell+2)^{th}$ element of $(t+2)^{th}$ row of $\mathbf{P'}_{s,t}$ is 
    \begin{align*}
    & \sum_{\sigma=0}^{t-1} {t-1 \choose \sigma} (-1)^{\sigma} \mathbf{Y}_{s+t-1, s+\sigma } 
    - \sum_{\sigma=0}^{t-1} {t-1 \choose \sigma} (-1)^{\sigma} \mathbf{Y}_{s+t, s+\sigma }  \\
    & = \sum_{\sigma=0}^{t-1} {t-1 \choose \sigma} (-1)^{\sigma} \mubar{r+s+\sigma } 
    - \sum_{\sigma=0}^{t-1} {t-1 \choose \sigma} (-1)^{\sigma} \mubar{r+s+\sigma }  \\
    & = 0 \\
    \end{align*}

\end{enumerate}

\end{proof}

\begin{lemma}
\label{lm:submatrix_rank}
The rank of matrices $\mathbf{P}_{s,t}$ and $\mathbf{P'}_{s,t}$ are $t+1$ and therefore the vector $\Re_{s+t, s }^{t-1}(\mathbf{Q})$ lies in the span of $\{ \Re_{s,s}^{0}(\mathbf{Q}),
\Re_{s+1,s}^{1}(\mathbf{Q}),
...,
\Re_{s+t, s }^{t}(\mathbf{Q}) \}$.
\end{lemma}
\begin{proof}

By lemma \ref{lm:Pst_value}, we know value of each entry of $\mathbf{P}_{s,t}$ and $\mathbf{P'}_{s,t}$. The first $t+1$ columns of $\mathbf{P}_{s,t}$ is

\begin{equation}
\mathbf{P}_{s,t}^{(sub)} =
\begin{bmatrix}
D_{0}^{s}, & {r \choose 1 }D_{0}^{s+1}, & ...&
{r \choose t-1 }D_{0}^{s+t-1}, & {r \choose t }D_{0}^{s+t}\\
0, & {r-1 \choose 0 }D_{1}^{s+1}, & ...& 
{r-1 \choose t-2 }D_{1}^{s+t-1}, & 
{r-1 \choose t-1 }D_{1}^{s+t} \\
0, & 0, & ...&  
{r-2 \choose t-3 }D_{2}^{s+t-1}, & 
{r-2 \choose t-2 }D_{2}^{s+t} \\
., & ., & ... &., &. \\
0, & 0,& ...& 0 , & 
{r-t \choose 0 }  D_{t}^{s+t}  \\
\end{bmatrix}.
\end{equation}

This is an upper triangular matrix. Also, the value on the diagonal is positive since ${r-i \choose 0} = 1$ and $D_{t}^{s+t} > 0$ ( by Claim \ref{clm:d_ri_positive} )
for all $0 \leq i \leq t \leq \ell$ and $r > \ell$. Therefore, the submatrix $\mathbf{P}_{s,t}^{1(sub)}$ is full rank, so that $rank(\mathbf{P}_{s,t}^{(sub)}) = t+1$. It also implies that the rank of $\mathbf{P}_{s,t}$ is $t+1$ since the rank of $\mathbf{P}_{s,t}$ is larger or equal to the rank of column spaces of $\mathbf{P}_{s,t}^{1(sub)}$ and less or equal to the number of rows in $\mathbf{P}_{s,t}$, which are both $t+1$.

Next, we calculate the rank of the matrix $\mathbf{P'}_{s,t}$. We first show that every $(t+2) \times (t+2)$ submatrix of $\mathbf{P'}_{s,t}$ below has rank $t+1$.
The submatrix $\mathbf{P}_{s,t}^{'(sub,i)}$ consists of the first $t$ columns, the $i^{th}$ column and the last column for any $i$ with  $t+1 \leq i \leq \ell$. This matrix can be written as following:  

\begin{equation}
\mathbf{P}_{s,t}^{'(sub,i)} =
\begin{bmatrix}
D_{0}^{s}, & {r \choose 1 }D_{0}^{s+1}, & ...&
{r \choose t-1 }D_{0}^{s+t-1}, & {r \choose i }D_{0}^{s+i}, & D_0^{r+s}\\
0, & {r-1 \choose 0 }D_{1}^{s+1}, & ...& 
{r-1 \choose t-2 }D_{1}^{s+t-1}, & 
{r-1 \choose i-1 }D_{1}^{s+i}, & D_1^{r+s} \\
0, & 0, & ...&  
{r-2 \choose t-3 }D_{2}^{s+t-1}, & 
{r-2 \choose i-2 }D_{2}^{s+i}, & D_2^{r+s} \\
., & ., & ... &., &.,& . \\
0, & 0,& ...&  {r-t+1 \choose 0 }  D_{t-1}^{s+t-1}, & 
{r-t+1 \choose i-t+1 }  D_{t-1}^{s+i}, & D_{t-1}^{r+s}  \\
0, & 0,& ...& 0 , & 
{r-t \choose i-t }  D_{t}^{s+i}, & D_t^{r+s}  \\
0,& 0,& ...& {r-t+1 \choose 0} D_{t}^{s+t-1}, &
[{r-t+1 \choose i-t+1}- {r-t \choose i-t } ] D_{t}^{s+i},&  0  \\
\end{bmatrix}.
\end{equation}
The determinant of $\mathbf{P}_{s,t}^{(sub,i)}$ is 
\begin{align*}
|\mathbf{P}_{s,t}^{(sub,i)}| 
& = \left( \prod_{j=0}^{t-2} {r-j \choose 0} D_{j}^{s+j} \right) \times |\mathbf{P}_{s,t}^{(3 \times 3,i)}|,
\end{align*}
where $\mathbf{P}_{s,t}^{(3 \times 3,i)} \in \mathbb{R}^{3 \times 3}$ is defined below:
$$
\mathbf{P}_{s,t}^{(3 \times 3,i)} 
= 
\begin{bmatrix}
 {r-t+1 \choose 0 }  D_{t-1}^{s+t-1}, & 
{r-t+1 \choose i-t+1 }  D_{t-1}^{s+i}, & D_{t-1}^{r+s}  \\
0 , & 
{r-t \choose i-t }  D_{t}^{s+i}, & D_{t}^{r+s}  \\
{r-t+1 \choose 0} D_{t}^{s+t-1}, &
[{r-t+1 \choose i-t+1}- {r-t \choose i-t } ] D_{t}^{s+i},&  0  \\
\end{bmatrix}
= \begin{bmatrix}
  D_{t-1}^{s+t-1}, & 
{r-t+1 \choose i-t+1 }  D_{t-1}^{s+i}, & D_{t-1}^{r+s}  \\
0 , & 
{r-t \choose i-t }  D_{t}^{s+i}, & D_{t}^{r+s}  \\
D_{t}^{s+t-1}, &
[{r-t+1 \choose i-t+1}- {r-t \choose i-t } ] D_{t}^{s+i},&  0  \\
\end{bmatrix}.
$$
$\mathbf{P}_{s,t}^{(3 \times 3,i)} $ is the right bottom $3 \times 3$ submatrix of $\mathbf{P}_{s,t}^{(sub,i)}$. 
Before we show that the determinant of $\mathbf{P}_{s,t}^{(3 \times 3,i)}$ is zero, we introduce a property of $D^{t}_i$.
\begin{claim}
\label{clm:d_linear_relation}
When $\mu(\cdot)$ is defined in Eqn.\eqref{eqn:thmdummysymm}, for all $p,q \in \{0,1,2,...,d-1\}$ with $0 \leq p+q \leq d-1$, the ratios of $D^{p}_{q}$ and $D^{p}_{q+1}$ can be written as an affine function of $p$, which is equivalent to
$\frac{D^{p}_{q}}{D^{p}_{q+1}} = c^{(1)}_q p +  c^{(2)}_q $ for some constants $ c^{(1)}_q , c^{(2)}_q \in \mathbb{R}$ depending on $q$.
\end{claim}
The proof of Claim \ref{clm:d_linear_relation} is deferred to Section \ref{sec:proofs_claims}.

The determinant of $\mathbf{P}_{s,t}^{(3 \times 3,i)}$ is
\begin{align*}
|\mathbf{P}_{s,t}^{(3 \times 3,i)}| 
& = && - D_{t-1}^{s+t-1} D_{t}^{r+s}\left[{r-t+1 \choose i-t+1}- {r-t \choose i-t } \right] D_{t}^{s+i}\\
& && + D_{t}^{s+t-1} \left[  
{r-t+1 \choose i-t+1 }  D_{t-1}^{s+i}  D_{t}^{r+s} - D_{t-1}^{r+s} {r-t \choose i-t } D_{t}^{s+i}
\right] \ \ \ (\text{Expand with the first column}) \\ \\
& = && D_{t}^{s+t-1}D_{t}^{r+s}D_{t}^{s+i} \Bigg[
\left[-{r-t+1 \choose i-t+1}+{r-t \choose i-t } \right][c^{(1)}_{t-1}(s+t-1)+c^{(2)}_{t-1}] \\
& && + {r-t+1 \choose i-t+1 }[(c^{(1)}_{t-1}(s+i)+ c^{(2)}_{t-1}] -  {r-t \choose i-t }[c^{(1)}_{t-1}(r+s) + c^{(2)}_{t-1}]
\Bigg]
\ \ \ (\text{Claim \ref{clm:d_linear_relation}}) \\
& = && D_{t}^{s+t-1}D_{t}^{r+s}D_{t}^{s+i}
\left[ {r-t+1 \choose i-t+1}(i-t+1) 
- {r-t \choose i-t }(r-t+1)
\right]c^{(1)}_{t-1} \\
& = && 0.
\end{align*}
It implies that the submatrix $\mathbf{P}_{s,t}^{(sub,i)}$ is not full rank. 
Now if we take away the bottom row and the second last column of $\mathbf{P}_{s,t}^{(sub,i)}$, the submatrix becomes 
$$
\begin{bmatrix}
D_{0}^{s}, & {r \choose 1 }D_{0}^{s+1}, & ...&
{r \choose t-1 }D_{0}^{s+t-1}, &  D_0^{r+s}\\
0, & {r-1 \choose 0 }D_{1}^{s+1}, & ...& 
{r-1 \choose t-2 }D_{1}^{s+t-1}, & D_1^{r+s} \\
0, & 0, & ...&  
{r-2 \choose t-3 }D_{2}^{s+t-1},  & D_2^{r+s} \\
., & ., & ... &., & . \\
0, & 0,& ...&  {r-t+1 \choose 0 }  D_{t-1}^{s+t-1}, & 
 D_{t-1}^{r+s}  \\
0, & 0,& ...& 0 , 
& D_t^{r+s}  \\
\end{bmatrix}.
$$
This is again a full rank upper triangular matrix since its diagonal entries are all positive. Therefore, the second last (or the $(t+1)^{th}$) column of $\mathbf{P}_{s,t}^{1(sub,i)}$ lies in the span of the other $t+1$ columns. That is, the $i^{th}$ column of matrix $\mathbf{P'}_{s,t}$ can be expressed as the first $t$ columns and the last column of the matrix $\mathbf{P'}_{s,t}$ for all $i$ with $t+1 \leq i \leq \ell$.
Therefore, the rank of $\mathbf{P'}_{s,t}$ is $t+1$.

Finally, we note that matrices $\mathbf{P}_{s,t}$ and $\mathbf{P'}_{s,t}$ only differ in the last row, but their rank is the same. Consequently, the last row, $\Re_{s+t, s }^{t-1}(\mathbf{Q})$, lies in the span of the first $t+1$ rows, which are $\{ \Re_{s,s}^{0} (\mathbf{Q}),
\Re_{s+1,s}^{1} (\mathbf{Q}),
...,
\Re_{s+t, s }^{t} (\mathbf{Q}) \}$.

\end{proof}

\begin{lemma}
\label{lm:overall_rank}
Given the results in Lemma \ref{lm:submatrix_rank} and following Definition \ref{def:row_combination}, then 
$$
\calA = \{
\Re_{0,0}^{0}(\mathbf{Q}) ,
\Re_{1,0}^{1}(\mathbf{Q}),
...,
\Re_{\ell,0}^{\ell}(\mathbf{Q}) \}
$$
forms the basis of row space of matrix $\mathbf{Q}$ and therefore implies that $rank(\mathbf{Q}) = \ell + 1$.
\end{lemma}
\begin{proof}

We prove a stronger version of this lemma: $\Re_{i,j}^k(\mathbf{Q})$ lies in the span of $\calA$ for all $0 \leq i \leq \ell $, $0 \leq k \leq i$, and $0 \leq j \leq i-k$. This result immediately implies that the row $\mathbf{Q}_{\{i,j\}} = \Re_{i,j}^0$ in matrix $\mathbf{Q}$ also lies in the span of $\mathcal{A}$.

We prove it by induction on two indices, which are $i$ and $k$. 
\begin{enumerate}
    \item When $i=0$, $k$ and $j$ can only be zero. We have $\Re_{0,0}^0(\mathbf{Q}) \in \spann(\calA)$. 
    \item When $i=i'$, suppose that $\Re_{i,j}^k(\mathbf{Q}) \in \spann(\calA)$ holds for $0 \leq i \leq i'$, $0 \leq k \leq i$ and $0 \leq j \leq i-k$.
    \item When $i=i'+1$, we prove that $\Re_{i'+1,j}^k(\mathbf{Q})$  lies in the span of $\calA$ for all $0 \leq k \leq i'+1$, and $0 \leq j \leq i-k$ by using another induction on $k$. In this induction process, $k$ is taking value from $i'+1$ to $0$.

\begin{enumerate}
    \item When $k=i'+1$, we know $\Re_{i'+1,i'+1}^{i'} \in \spann(\calA)$ by plugging  $s=0$ and $t=i'+1$ in Lemma \ref{lm:submatrix_rank} . 
    
    \item Suppose $\Re_{i'+1, j}^{k'}(\mathbf{Q}) \in \spann(\calA)$ for $k=k'+1,.., i'+1$ and $0\leq j \leq i'+1 - k$ for some $0 \leq k' \leq i'$
    \item When $k=k'$,
    by Lemma \ref{lm:submatrix_rank}, we know 
    $$
    \Re_{i'+1,i'-k'}^{k'}(\mathbf{Q}) \in \spann(\Re_{i'-k',i'-k'}^{0}(\mathbf{Q}),...,\Re_{i',i'-k'}^{k'}(\mathbf{Q}),\Re_{i'+1,i'-k'}^{k'+1}(\mathbf{Q})) 
    $$ 
    by plugging in $s=i'-k'$ and $t=k'+1$. Since we have 
    $$
    \spann(\Re_{i'-k',i'-k'}^{0}(\mathbf{Q}),...,\Re_{i',i'-k'}^{k'}(\mathbf{Q})) \subseteq \spann(\calA)
    $$ 
    by induction hypothesis on $i$ and $\Re_{i'+1,i'-k'}^{k'+1}(\mathbf{Q}) \in \spann(\calA)$ by the induction hypothesis on $k$, we have 
    \begin{equation}
    \label{eqn:row_i_k_lie_A}
    \Re_{i'+1,i'-k'}^{k'}(\mathbf{Q}) \in \spann(\Re_{i'-k',i'-k'}^{0}(\mathbf{Q}),...,\Re_{i',i'-k'}^{k'}(\mathbf{Q}),\Re_{i'+1,i'-k'}^{k'+1}(\mathbf{Q})) \subseteq \spann(\calA).
    \end{equation}
    
    Then we observe that 
    \begin{equation}
    \label{eqn:row_relation}
    \Re_{i'+1,j}^k(\mathbf{Q}) = 
    \begin{cases}
    \Re_{i'+1,i'-k}^k(\mathbf{Q})  + \sum_{\rho=j}^{i'-k-1} \Re_{i'+1,\rho}^{k+1}(\mathbf{Q}) & \text{ if   } 0 \leq j < i'-k. \\
    \Re_{i'+1,i'-k}^k(\mathbf{Q}) & \text{ if   } j = i'-k \\
    \Re_{i'+1,i'-k}^k(\mathbf{Q}) - \Re_{i'+1,i'-k}^{k+1}(\mathbf{Q})  & \text{ if   } j = i'-k+1.\\
    \end{cases}    
    \end{equation}
    The Eqn. \eqref{eqn:row_relation} follows from 
    \begin{align*}
    \Re_{i'+1, \rho}^{k+1}(\mathbf{Q})
    & = \sum_{\sigma=0}^{k+1} {k+1 \choose \sigma} (-1)^{\sigma} \mathbf{Q}_{\{i'+1,\rho+\sigma \}} \\
    & = \sum_{\sigma=0}^{k} {k \choose \sigma} (-1)^{\sigma} \mathbf{Q}_{\{i'+1,\rho+\sigma \}}
    + \sum_{\sigma=1}^{k+1} {k \choose \sigma-1} (-1)^{\sigma} \mathbf{Q}_{\{i'+1,\rho+\sigma \}} \\
    & = \sum_{\sigma=0}^{k} {k \choose \sigma} (-1)^{\sigma} \mathbf{Q}_{\{i'+1,\rho+\sigma \}}
    - \sum_{\sigma=0}^{k} {k \choose \sigma} (-1)^{\sigma} \mathbf{Q}_{\{i'+1,\rho+\sigma +1\}} \\
    & = \Re_{i'+1, \rho}^{k}(\mathbf{Q}) - \Re_{i'+1, \rho+1}^{k}(\mathbf{Q}),
    \end{align*}
    Then Eqn. \eqref{eqn:row_relation} can be attained by summing over the above equation from $\rho =j$ to $\rho=i'-k-1$.
    
    By Eqn. \eqref{eqn:row_relation}, we know that the vector $\Re_{i'+1,j}^{k'}(\mathbf{Q})$ can be expressed as the combination of some terms in the form of $\Re_{i'+1,\rho}^{k'+1}(\mathbf{Q})$ and $\Re_{i'+1,i'-k' }^{k'}(\mathbf{Q})$. Both terms lie in the span of $A$ by the induction hypothesis on $k$ and Eqn. \eqref{eqn:row_i_k_lie_A}.
    
    Therefore, we can conclude that $\Re_{i'+1,j}^{k'}(\mathbf{Q})$ also lies in the span of $\calA$ for all $0 \leq j \leq i'+1-k'$. It implies that it also holds when $k=k'$ and we establish the induction step on $k$.
\end{enumerate}
After completing mathematical induction proof on $k$, we know that $\Re_{i'+1,j}^k(\mathbf{Q})$ lies in the span of $\calA$ for all $0 \leq k \leq i'+1$ and $0 \leq j \leq i'+1-k$. Consequently, we also finish the induction step on $i$ (when $i= i'+1$).

\end{enumerate}

Finally, we can deduce that $\Re_{i,j}^k(\mathbf{Q})$ lies in the span of $\calA$ for any $0 \leq i \leq \ell $, $0 \leq j \leq i$ and $0 \leq k \leq i-j$. Then we know every row $\mathbf{Q}_{\{i,j\}} = \Re_{i,j}^0(\mathbf{Q})$ in matrix $\mathbf{Q}$ also lies in the span of $\mathcal{A}$. It immediately implies $rank(\mathbf{Q}) = \ell +1$.

\end{proof}

By Lemma \ref{lm:overall_rank}, the rank of the matrix $ \mathbf{Q} = [\mathbf{M'}, \mathbf{Y}]$ equals to $\ell + 1$. This in turn implies that $\bfb_{i,j} = 0 $ for all $i \geq j > 0$ by Lemma \ref{lm:simplified_matrix}. Therefore, the corresponding Faith-Interaction indices satisfy the interaction dummy axiom when $d \geq \ell + r$ and $r > \ell$.

\paragraph{(3) $ r + \ell > d \geq \ell + r$ and $r > \ell$:} Now, we generalize the results in the second case to the last case. We recall that $\bfb_{i,j}$ denotes the interaction indices with $j$ elements in $R$ and $j$ elements in $[d] \backslash R$ (Definition \ref{def:bij}). However, 
when $r + \ell > d$, there are some $\bfb_{i,j}$ that do not exist since there are not enough elements outside $R$.
For example, $\bfb_{\ell,\ell}$ does not exist since there are only $d-r < \ell$ elements outside $R$. 

In this case, we can still compute the matrix of linear equations,  $\mathbf{Q'} = [ \mathbf{M'}, \mathbf{Y}]$, but some rows do not exist. Particularly, all rows corresponding to $\bfb_{i,j}$ with $d-r < j \leq i$ do not exist.
Nevertheless, removing rows do not
increases the rank of the matrix $\mathbf{Q}$. Therefore, by Lemma \ref{lm:overall_rank}, the rank of $\text{rank}(\mathbf{Q'}) \leq \text{rank}(\mathbf{Q}) = \ell + 1 $. Also, we note that the columns in the coefficient matrix $\mathbf{M}$ are linearly independent (since it has a unique solution by Proposition \ref{pro:unique_minimizer}). This implies that the columns in the reduced coefficient matrix $\mathbf{M'}$ are also linearly independent (since $\mathbf{M'}$ is a submatrix of $\mathbf{M}$). Therefore, the rank of the reduced coefficient matrix $\mathbf{M'}$ equals to the number of columns in $\mathbf{M'}$, which is $\ell+1$. 

Overall, we have $ \ell + 1 \geq \text{rank}(\mathbf{Q'}) \geq \text{rank}(\mathbf{M'}) = \ell +1 $. That is, $\text{rank}(\mathbf{Q'}) = \ell +1$, which in turn implies that $\bfb_{i,j} = 0 $ for all $i \geq j > 0$ by Lemma \ref{lm:simplified_matrix}. Therefore, the corresponding Faith-Interaction indices satisfy the interaction dummy axiom when $ r + \ell > d \geq \ell + r$ and $r > \ell$.

Therefore, by summarizing (1)-(3),
we conclude that the interaction dummy axiom holds for all basis function $\f_R$. This result can be generalized to any function $\f(\cdot)$ by applying Lemma \ref{lm:f_decomposable}.

In conclusion, 
the Faith-Interaction indices with respect to the weighting function defined in Eqn.\eqref{eqn:thmdummysymm} satisfy the interaction linearity, symmetry and dummy axioms for all set functions $\f(\cdot):2^d \mapsto \mathbb{R}$ and all maximum interaction order $1 \leq \ell \leq d$.
\end{proof}

\subsubsection{Proof of Necessary Condition of Theorem \ref{thm:thmdummysymm}}
\label{sec:necessity_theorem_dummysymm}

Now we prove the necessary condition of Theorem \ref{thm:thmdummysymm}. That is, Faith-Interaction indices $\Expl$ with a finite weighting function 
satisfy interaction linearity, symmetry, and dummy axioms only if the weighting function $\mu$ has the form in Eqn.\eqref{eqn:thmdummysymm}.

From Proposition \ref{pro:symmetry}, the Faithful-Interaction indices satisfy interaction symmetry axiom if and only if $\mu(S)$ only depends on the size of the input set $|S|$. Therefore, the weighting function must be symmetric. The following lemma show that the weighting function must be in the form in Eqn.\eqref{eqn:thmdummysymm} if the corresponding Faith-Interaction indices satisfy interaction dummy axiom.

\begin{lemma}
\label{lm:necessity_dummy}
Faith-Interaction indices $\Expl$ with a finite and permutation-invariant weighting function 
satisfy the interaction dummy axiom only if the weighting function $\mu$ has the following form: 
\begin{align*}
\mu(S) & \propto \sum_{i=|S|}^{d} {d- |S| \choose i-|S|}(-1)^{i-|S|} g(a,b,i), \  \text{ where }
g(a,b,i) = 
\begin{cases}
1 & \text{ , if } \ i = 0. \\
\prod_{j=0}^{j=i-1} \frac{a(a-b) + j(b-a^2)}{a-b + j(b-a^2)}
& \text{ , if } \   1 \leq i \leq d. \\
\end{cases}
\end{align*} 
for some $a,b \in \R^{+}$ with $a>b$ such that $\mu(S) > 0$ for all $S \subseteq [d]$.
\end{lemma}

\begin{proof}

We now solve the case when the function $\f = \f_R$ is a basis function. Recall that the definition of basis functions is :
$$
\f_R(S) = \begin{cases}
1, & \text{if}\ S \supseteq R \\
0, & \text{otherwise}. \\
\end{cases}
$$ 
Since the weighting function $\mu(\cdot)$ if finite and only depends on the size of the input set, we use the definitions in Section \ref{sec:extra_notations}: $\mu_{|S|} = \mu(S)$ and  $\mubar{|S|} = \sum_{T \supseteq S} \mu(T) = \sum_{i=|S|}^d { d-|S| \choose i-|S|} \mu_i$ for all $S \subseteq [d]$. Since  Faithful-Interaction indices should hold for all maximum interaction orders $1 \leq \ell \leq d$, we restrict the maximum interaction order to $\ell=1$.

Now we use the following proposition from \citet{ding2008formulas} to prove the necessary condition.

\begin{proposition}
\label{pro:closed_form_finite_mu}
(\citet{ding2008formulas}, equation (8)) When the maximum interaction order $\ell = 1$ (no interaction terms) and the set function $\f = \f_R$ is a basis function for some $R \subseteq [d]$ with $|R| = r$, and the weighting function $\mu(S)$ is permutation-invariant and is normalized so that $\mubar{0} = \sum_{S \subseteq [d]} \mu(S) = 1$, then the minimizer of Eqn.\eqref{eqn:weighted_regression}, $\Expl(\f_R, \ell) \in \mathbb{R}^{d+1}$, has the following form:
$$
\Expl_S(\f_R, \ell) = 
\begin{cases}
-\frac{\mubar{1} + (d-1)\mubar{2}}{\varrho}\mubar{r} + \frac{\mubar{1}}{\varrho} \left(r\mubar{r} + (d-r)\mubar{r+1} \right) & \text{, if $S = \phi$} \\
\frac{1}{\varrho(\mubar{1} - \mubar{2})} \left( (\mubar{1}^2 - \mubar{1}\mubar{2})\mubar{r} + \varrho \mubar{r} - (\mubar{1}^2 - \mubar{2})(r\mubar{r} + (d-r)\mubar{r+1} ) \right) & \text{, if $S = \{ i\}$ for $ i \in R$} \\
\frac{1}{\varrho(\mubar{1} - \mubar{2})} \left( (\mubar{1}^2 - \mubar{1}\mubar{2})\mubar{r} + \varrho \mubar{r+1} - (\mubar{1}^2 - \mubar{2})(r\mubar{r} + (d-r)\mubar{r+1} ) \right) & \text{, if $S = \{ i\}$ for $ i \notin R$} \\
\end{cases}
$$
where $\varrho = d\mubar{1}^2 - (d-1)\mubar{2} - \mubar{1}$. 
\end{proposition}
Note that the basis function satisfies $\f_R(S \cup i) = \f_R(S)$ for $i \notin R$ and $S \subseteq [d] \backslash i$. Since the minimizer satisfies the interaction dummy axiom,
we have $\Expl_{\{i\}}(\f_R, \ell) =  0$ for all $i \notin R$, which implies
\begin{equation*}
 (\mubar{1}^2 - \mubar{1}\mubar{2})\mubar{r} + \varrho \mubar{r+1} - (\mubar{1}^2 - \mubar{2})(r\mubar{r} + (d-r)\mubar{r+1} ) = 0.
\end{equation*}

Let $\mubar{1} = a$ and $\mubar{2} = b$. By plugging in $\varrho = d\mubar{1}^2 - (d-1)\mubar{2} - \mubar{1}$, we have 
$$
(a^2 - ab)\mubar{r} + (da^2 - (d-1)b - a) \mubar{r+1} - ra^2\mubar{r} - (d-r)\mubar{r+1}a^2 + rb\mubar{r} + b(d-r)\mubar{r+1} = 0.
$$
By rearranging, we have
\begin{equation}
\label{eqn:mubar_recursive}
\frac{\mubar{r+1}}{\mubar{r}} = \frac{a(a-b) + r(b-a^2)}{(a-b) + r(b-a^2)}.
\end{equation}

Without loss of generality, we assume that $\bar{\mu}_0 = \sum_{S \subseteq [d]} \mu(S) = 1$, so we have $\mubar{0} =1, \mubar{1} = a$ and $\mubar{2} =b$. Since Eqn. \eqref{eqn:mubar_recursive} holds for all $1 \leq r \leq d$, then we can solve $\mubar{k}$ for $3 \leq k \leq d$ by rewinding the above recursive equation.
\begin{equation}
\label{eqn:mubar_closed_form}
\mubar{k} = \prod_{j=0}^{k-1} \frac{a(a-b) + j(b-a^2)}{(a-b) + j(b-a^2)}.
\end{equation}
Then by applying Lemma \ref{clm:mubar_mu_relation}, we can have 
\begin{align*}
\mu(S) = \mu_{|S|} 
& =  \sum_{i=|S|}^{d} {d- |S| \choose i-|S|}(-1)^{i-|S|} \mubar{i} \\
& = \sum_{i=|S|}^{d} {d- |S| \choose i-|S|}(-1)^{i-|S|} \prod_{j=0}^{j=i-1} \frac{a(a-b) + j(b-a^2)}{a-b + j(b-a^2)}, \\
\end{align*}
for some $a,b \in \R^{+}$ with $a>b$ such that $\mu(S) > 0$ for all $S \subseteq [d]$.
 
\end{proof}

\newpage

\subsection{Proof of Theorem \ref{thm:faith-banzhaf}}

\begin{proof}

\textbf{Sufficiency:} Below we show that the Faith-Banzhaf index satisfies interaction linearity, symmetry, dummy, and generalized 2-efficiency axiom.

By Proposition \ref{pro:linearity}, the Faith-Banzhaf index satisfies interaction linearity axiom. By Proposition \ref{pro:symmetry}, the Faith-Banzhaf index satisfies the interaction symmetry axiom. By Proposition \ref{pro:dummy}, when specifying $p_i = \frac{1}{2}$ for all $1 \leq i \leq d$, $\mu(S)=\prod_{i \in S} p_i \prod_{j \not \in S} (1-p_j) = \frac{1}{2^d}$ is the same as the weighting function used in this theorem. Therefore, the Faith-Banzhaf index satisfies the interaction dummy axiom. 

Note that the closed-form expression in Eqn.\eqref{eqn:faith_banzhaf2} is from Proposition 7.1 of \citep{grabisch2000equivalent}. Also, for all $S$ with size $\ell$, \citep{grabisch2000equivalent} has shown that these values coincide with the Banzhaf interaction indices, which has the following form:
$$
\Expl_S^{\text{F-Bzf}}(\f,\ell) = 
\sum_{T \subseteq [d] \backslash S}\frac{1}{2^{d-|S|}} \Delta_S(\f(T))
\ \ \text{ for all } S \in \setlessell 
\text{ with } |S| = \ell.
$$
Also, recall that the generalized 2-efficiency axiom is only defined on the highest-order interaction, where the Faith-Banzhaf and the Banzhaf interaction indices overlap. Therefore, Faith-Banzhaf indices also satisfy the generalized 2-efficiency axiom.

\textbf{Necessity:} Below we show that the Faith-Banzhaf index is the only index that satisfies interaction linearity, symmetry, dummy, and generalized 2-efficiency axiom.

First, when $\ell = d$, 
by Proposition \ref{pro:ell_equals_d}, we have 
$\ex_{[d]}(\f,d) = a(\f,[d])$ for all $d \in \mathbb{N}$. We then apply the following  results for the Banzhaf interaction index.

\begin{claim}
\label{clm:generalized_2eff}
\citep[Theorem~4]{grabisch1999axiomatic}
If an interaction index that satisfies the interaction linearity, symmetry, dummy, generalized 2-efficiency axioms and $\ex_{[d]}(\f,d) = a(\f,[d])$ for all $d \in \mathbb{N}$, then its highest order terms must have the following form:
$$
\ex_S(\f,\ell) = \sum_{T \subseteq [d] \backslash S}\frac{1}{2^{d-|S|}} \Delta_S(\f(T)),
\forall S \in \setlessell \text{ with } |S| = \ell. 
$$
\end{claim}
Therefore, by the above claim, the highest order terms (with $|S| = \ell$) must coincide with the Banzhaf interaction indices. Then we prove that the weighting function should be a constant, i.e. $\mu(S) = c$ for some constant $c>0$ for all $S \subseteq [d]$. 

Now, we consider the case when the maximum interaction order $\ell = 1$ and $\f = \f_R$ is a basis function, where $\f_R(S) = 1$ for all $S \supseteq R$ and $0$ otherwise.
Since the Faithful-Interaction indices that satisfy the interaction symmetry axiom, the corresponding weighting function is symmetric by Proposition \ref{pro:symmetry}. Then, we let $\mubar{|S|} = \sum_{T \supseteq S} \mu(T) = \sum_{i=|S|}^d { d-|S| \choose i-|S|} \mu_i$ for all $S \subseteq [d]$. 

By Proposition \ref{pro:closed_form_finite_mu}, the minimizer of Eqn.\eqref{eqn:weighted_regression} is
$$
\Expl_S(\f_R, \ell) = 
\begin{cases}
-\frac{\mubar{1} + (d-1)\mubar{2}}{\varrho}\mubar{r} + \frac{\mubar{1}}{\varrho} \left(r\mubar{r} + (d-r)\mubar{r+1} \right) & \text{, if $S = \phi$} \\
\frac{1}{\varrho(\mubar{1} - \mubar{2})} \left( (\mubar{1}^2 - \mubar{1}\mubar{2})\mubar{r} + \varrho \mubar{r} - (\mubar{1}^2 - \mubar{2})(r\mubar{r} + (d-r)\mubar{r+1} ) \right) & \text{, if $S = \{ i\}$ for $ i \in R$} \\
\frac{1}{\varrho(\mubar{1} - \mubar{2})} \left( (\mubar{1}^2 - \mubar{1}\mubar{2})\mubar{r} + \varrho \mubar{r+1} - (\mubar{1}^2 - \mubar{2})(r\mubar{r} + (d-r)\mubar{r+1} ) \right) & \text{, if $S = \{ i\}$ for $ i \notin R$} \\
\end{cases},
$$
where  $\varrho = d\mubar{1}^2 - (d-1)\mubar{2} - \mubar{1}$.

Since we have known that the $\ex_{i}(\f_R,\ell)$ must coincide with the Banzhaf interaction indices 
by Claim \ref{clm:generalized_2eff}, we have $\Expl_{\{i\}}(\f_R, \ell) = 0$ for $i \not \in R$ and $\Expl_{\{i\}}(\f_R, \ell) = \frac{1}{2^{|R|-1}}$ for $i \in R$. Equivalently, we have
\begin{equation}
\label{eqn:basis_function_banzhaf}
\begin{cases}
\frac{1}{\varrho(\mubar{1} - \mubar{2})} \left( (\mubar{1}^2 - \mubar{1}\mubar{2})\mubar{r} + \varrho \mubar{r} - (\mubar{1}^2 - \mubar{2})(r\mubar{r} + (d-r)\mubar{r+1} ) \right) = \frac{1}{2^{r-1}}. \\
\frac{1}{\varrho(\mubar{1} - \mubar{2})} \left( (\mubar{1}^2 - \mubar{1}\mubar{2})\mubar{r} + \varrho \mubar{r+1} - (\mubar{1}^2 - \mubar{2})(r\mubar{r} + (d-r)\mubar{r+1} ) \right) = 0.
\end{cases}.
\end{equation}

Subtracting the first equality with the second one, we have  
\begin{equation}
\label{eqn:mubar_recursive2}
\frac{1}{\varrho(\mubar{1} - \mubar{2})} \cdot  \varrho (\mubar{r} -\mubar{r+1} )  = \frac{1}{2^{r-1}}
\Rightarrow 
\frac{\mubar{r} - \mubar{r+1}}{\mubar{1} - \mubar{2}}
= \frac{1}{2^{r-1}}
\Rightarrow \frac{\mubar{r+1} - \mubar{r+2}}{\mubar{r} - \mubar{r+1}} = \frac{1}{2}.
\end{equation}
Note that it holds for all $1 \leq r \leq d-2$. 

Next, by solving $\Expl_{\{i\}}(\f_R, \ell) = 0$ for $i \not \in R$ and for all $1 \leq r \leq d-1$, we have the following result from Eqn.\eqref{eqn:mubar_recursive}:
$$
\mubar{k} = \prod_{j=0}^{k-1} \frac{a(a-b) + j(b-a^2)}{(a-b) + j(b-a^2)}, 
\text{ for some $1 > a = \mubar{1} > b = \mubar{2} > 0$.}
$$

Then by plugging in the above expression to Eqn.\eqref{eqn:mubar_recursive2}, we have
\begin{align*}
0 
& = \mubar{r} - 3 \mubar{r+1} + 2\mubar{r+2} \\
& =  \mubar{r} \left( 
1 - 3 \cdot \frac{a(a-b) + r(b-a^2)}{(a-b) + r(b-a^2)} 
+ 2 \cdot \frac{a(a-b) + r(b-a^2)}{(a-b) + r(b-a^2)} \cdot \frac{a(a-b) + (r+1)(b-a^2)}{(a-b) + (r+1)(b-a^2)} 
\right) \\
& = \mubar{r} \left( 
\left(1 - \frac{a(a-b) + r(b-a^2)}{(a-b) + r(b-a^2)} 
\right)
- 2  \frac{a(a-b) + r(b-a^2)}{(a-b) + r(b-a^2)}  \left( 1-   \frac{a(a-b) + (r+1)(b-a^2)}{(a-b) + (r+1)(b-a^2)} \right)
\right) \\
& = \mubar{r} \left( \frac{(1-a)(a-b)}{(a-b) + r(b-a^2)} - 2  \frac{a(a-b) + r(b-a^2)}{(a-b) + r(b-a^2)} \cdot \frac{(1-a)(a-b)}{(a-b) + (r+1)(b-a^2)} 
\right) \\
& = \frac{\mubar{r}(1-a)(a-b)}{(a-b) + r(b-a^2)} \left(1  - 2  \frac{a(a-b) + r(b-a^2)}{(a-b) + (r+1)(b-a^2)}  \right) \\
& = \frac{\mubar{r}(1-a)(a-b)}{(a-b) + r(b-a^2)} \left( \frac{(a-b) + (r+1)(b-a^2) - 2a(a-b) - 2r(b-a^2)}{(a-b) + (r+1)(b-a^2)}  \right) \\
& = \frac{\mubar{r}(1-a)(a-b)}{(a-b) + r(b-a^2)} \cdot \frac{a - rb +(r-3)a^2 +2ab}{(a-b) + (r+1)(b-a^2)}   \\
& = \frac{\mubar{r}(1-a)(a-b)}{\left[(a-b) + r(b-a^2) \right]\left[(a-b) + (r+1)(b-a^2) \right]} 
\cdot \left[a(1-3a+2b) -r(b-a^2) \right].
\end{align*}
Since we have $1 > a = \mubar{1} > b = \mubar{2} > 0$ and $\mubar{r} > 0$,  we have $a(1-3a+2b) -r(b-a^2) = 0$ for all $1 \leq r \leq d-2$. Then for all $d \geq 4$, we have 
\begin{equation}
\label{eqn:ab_relation}
a(1-3a+2b) -(b-a^2) = a(1-3a+2b) -2(b-a^2) = 0.    
\end{equation}

It implies $b=a^2$ and $a(1-3a+2b)  = 0$. That is, $a = \frac{1}{2}$ and $b = \frac{1}{4}$. By substituting these into Eqn.\eqref{eqn:mubar_recursive}, we have $\mubar{k} = \frac{1}{2^k}$. Then by applying Lemma \ref{clm:mubar_mu_relation}, we get 
\begin{align*}
\mu(S) = \mu_{|S|} 
& =  \sum_{i=|S|}^{d} {d- |S| \choose i-|S|}(-1)^{i-|S|} \mubar{i} = \frac{1}{2^d}, \ \  \forall S \subseteq [d].
\end{align*}

Finally, for $d=3$, from Eqn.\eqref{eqn:ab_relation}, we have 
$$
a(1-3a+2b) -(b-a^2) = 0
\Rightarrow (2a-1)(b-a) = 0
\Rightarrow a = \frac{1}{2}.
$$
Next, we plug in $r=d$ into the first equality in Eqn.\eqref{eqn:basis_function_banzhaf}, we have 
\begin{align}
\frac{1}{2^{d-1}}
& = \frac{(\mubar{1}^2 - \mubar{1}\mubar{2})\mubar{d} + \varrho \mubar{d} - (\mubar{1}^2 - \mubar{2})(d\mubar{d})}{(d\mubar{1}^2 - (d-1)\mubar{2} - \mubar{1})(\mubar{1} - \mubar{2})}  \nonumber \\
& =\mubar{d} \frac{(\mubar{1}^2 - \mubar{1}\mubar{2}) + (d\mubar{1}^2 - (d-1)\mubar{2} - \mubar{1}) - d(\mubar{1}^2 - \mubar{2})}{(d\mubar{1}^2 - (d-1)\mubar{2} - \mubar{1})(\mubar{1} - \mubar{2})}  \nonumber \\
& =\mubar{d} \frac{(\mubar{1}^2 - \mubar{1}\mubar{2}) + \mubar{2}- \mubar{1}}{(d\mubar{1}^2 - (d-1)\mubar{2} - \mubar{1})(\mubar{1} - \mubar{2})}  \nonumber \\
& = \mubar{d} \frac{\mubar{1}-1}{d\mubar{1}^2 - (d-1)\mubar{2} - \mubar{1}}  \nonumber \\
& =  \frac{\mubar{d}( a-1 )}{da^2 - (d-1)b - a} , \ \ (\text{By definition, we have } a=\mubar{1} \text{ and } b=\mubar{2}) 
\label{eqn:r=d}
\end{align}
By plugging in $d=3$, $a=\frac{1}{2}$ and $\mubar{3} = \frac{ab(a-b) + 2b(b-a^2)}{(a-b) + j(b-a^2)}
$ from Eqn.\eqref{eqn:mubar_recursive}, we have 
\begin{align*}
\frac{1}{4} 
& =  \frac{\mubar{d}( a-1 )}{da^2 - (d-1)b - a}  \\
& = \frac{ab(a-b) + 2b(b-a^2)}{(a-b) + 2(b-a^2)} \times \frac{a-1}{3a^2 - 2b - a} \ \ \ \text{By plugging in } d=3  \\
& = \frac{b(-a^2-ab+2b)}{a+b-2a^2} \times \frac{a-1}{3a^2 - 2b - a}  \\
& = \left(-\frac{1}{4}+\frac{3b}{2} \right) \times \frac{-\frac{1}{2}}{\frac{1}{4} - 2b }
\ \ \ \text{By plugging in } a = \frac{1}{2}.
\end{align*}
Therefore, we have $b=\frac{1}{4}$, which implies that $\mu(S) = \frac{1}{2^d}$ using the same argument.

We note that for $d=2$, any weighting function satisfies $\mu_1= c$ and $\mu_0 = \mu_2 = \frac{1}{2} -c $ for some constant $c > 0$ leads to Banzhaf interaction values for $\ell=1,2$.

\end{proof}

\newpage

\subsection{Proof of Theorem \ref{thm:faith_shap}}
In this section, we present
proof for the sufficient condition, the closed-form expressions of Faithful Shapley index, and the necessary condition of Theorem \ref{thm:faith_shap}. For the sufficient condition, we provide two proof. The first one is in Section \ref{sec:faithshap_sufficiency} and is simpler. The second one is in Section \ref{sec:alternative_faithshap_sufficiency} and is similar to the proof of Theorem \ref{thm:thmdummysymm}. 
Next, we derive the closed-form expression of Faithful Shapley  Interaction indices in Section \ref{sec:closed_form_solution_faith_shap}. Finally, we provide the proof of the necessary condition in Section \ref{sec:faithshap_necessity}. 

\subsubsection{Proof of Sufficient Condition of Theorem \ref{thm:faith_shap}}
\label{sec:faithshap_sufficiency}

In the following, we prove that under the weighting function defined in Eqn. \eqref{eqn:faithshap_weight}, the minimizers of Eqn.\eqref{eqn:constrained_weighted_regresion} satisfy interaction linearity, symmetry, efficiency, and dummy axioms.

\textbf{Interaction linearity, symmetry, efficiency axiom:} The minimizers of Eqn.\eqref{eqn:constrained_weighted_regresion} satisfy interaction linearity, symmetry, and efficiency axioms by Proposition \ref{pro:linearity}, \ref{pro:symmetry} and \ref{pro:efficiency}.

\textbf{Interaction dummy axiom:} Below we prove that the optimal solution satisfies the interaction dummy axiom.
The constrained optimization problem can be written as follows:

\begin{equation*}
\small
\ex^{\text{F-Shap}}(\f,\ell) = \min_{\Expl \in \mathbb{R}^{d_\ell}} F_\ell(\f,\ex) = \min \sum_{S \subseteq [d], 1 \leq |S| \leq d-1} \mu_{|S|} \left(
\f(S) - \sum_{T \subseteq S, |T| \leq \ell} \Expl_T(\f, \ell)  \right)^2,
\end{equation*}
\begin{equation}
\label{eqn:constrained_ls_proof}
\text{ subject to } \Expl_{\text{\O}}(\f,\ell) = \f(\text{\O}) \text{   and   }  \sum_{T \subseteq [d], |T| \leq \ell} \Expl_{T}(\f,\ell) = \f([d]),
\end{equation}
where we use the notations:$\mu_{|S|} = \mu(S)$ and  $\bar{\mu}(S)
= \mubar{|S|} 
= \sum_{T \supseteq S, \mu(T) < \infty} \mu_{|T|} = \sum_{i: |S| \leq i \leq d, \mu_i < \infty}{d-|S| \choose i-|S|}\mu_{i}$ since the weighting function only depends on the size of input sets. Also, since multiplying a scalar to $\mu$ does not change the minimizer of Eqn.\eqref{eqn:weighted_regression}, without loss of generality, we let 
$$\mu(S) = \frac{(d-|S|-1)!(|S|-1)!}{(d-1)!} = B(|S|,d-|S|) \propto \frac{d-1}{\binom{d} {|S|}\,|S|\,(d-|S|)},
\text{ for all } S \subseteq [d] 
\text{ with } 1 \leq |S| \leq d-1,
$$
where $B(\cdot,\cdot)$ is a beta function.
Now, we prove the minimizer of Eqn.\eqref{eqn:con_weighted_ls_basis} satisfies interaction dummy axiom.  More generally, we prove that the constrained minimization problem with a dummy feature can be reduced to another problem with only $d-1$ features and the interaction terms containing the dummy feature is zero.
Formally, we have the following lemma:
\begin{lemma}
\label{lm:faith_shap_reduce}
Assume that $i^{\text{th}}$ feature of the set function $\f(\cdot)$ is a dummy feature such that $\f(S) = \f(S \cup i)$ for all $S \subseteq [d] \backslash \{i\}$.
Let $\f':2^{d-1} \mapsto \mathbb{R}$ with $\f'(S) = \f(S)$ for all $S \subseteq [d-1]$. Then we have 
\begin{equation}
\label{eqn:dummy_faith_solution}
\begin{cases}
\Expl_{S}^{\text{F-Shap}}(\f,\ell) = \Expl_{S}^{\text{F-Shap}}(\f',\ell), 
& \text{ for all }  \ \ S \subseteq [d] \backslash \{ i\}, 0 \leq |S| \leq \ell. \\
\Expl_{S \cup \{i\}}^{\text{F-Shap}}(\f,\ell) = 0, & \text{ for all }  \ \ S \subseteq [d] \backslash \{ i\}, 0 \leq |S| \leq \ell - 1. 
\end{cases}
\end{equation}
\end{lemma}

\begin{proof}
Without loss of generality, we assume that $d^{th}$ feature is a dummy feature, such that $\f(S \cup d) = \f(S)$ for all $S \subseteq [d-1]$.

Now, we solve the constrained minimization problem (Eqn.\eqref{eqn:constrained_ls_proof}) by Lagrange multipliers. 
\begin{equation}
\begin{cases}
\frac{\partial \ F_\ell(\f,\ex)}{\partial \Expl_S } = \lambda & \text{  for all  }
S \subseteq [d], 1 \leq |S| \leq \ell.  \\
\sum_{S \in \setlessell} \Expl_S(\f, \ell) = \f([d]). \\
\Expl_{\text{\O}}(\f, \ell) = \f(\emptyset).
\end{cases},
\end{equation}
We note that the partial derivative can be calculated as below:
$$
\frac{\partial \ F_\ell(\f,\ex)}{\partial \Expl_S } 
= -2\sum_{T \supseteq S, T \neq [d]} \mu(T) \left(\f(T) - \sum_{L \subseteq T, |L| \leq \ell} \ex_L(\f,\ell) 
\right)
= \lambda,
\text{  for all  }
S \subseteq [d], 1 \leq |S| \leq \ell.
$$
Now, for convenience, we denote 
$q(T) = \f(T) - \sum_{L \subseteq T, |L| \leq \ell} \ex^{\text{F-Shap}}_L(\f,\ell)$ for all $T \subseteq [d]$. Then we have
\begin{equation}
\label{eqn:lagrange2}
\begin{cases}
\sum_{T: S \subseteq T \subset [d]} \mu(T)q(T) = \frac{-\lambda}{2} & \text{  for all  }
S \subseteq [d], 1 \leq |S| \leq \ell.  \\
q([d]) = 0. \\
q(\text{\O}) = 0.
\end{cases}
\end{equation}
Similarly, the minimizer $\ex^{\text{F-Shap}}(\f',\ell)$ of $F_\ell(\f',\ex)$ satisfies 
\begin{equation}
\label{eqn:lagrange_for_f_prime}
\begin{cases}
\sum_{T: S \subseteq T \subset [d-1]} \mu'(T)q'(T) = \frac{-\lambda'}{2} & \text{  for all  }
S \subseteq [d-1], 1 \leq |S| \leq \ell.  \\
q'([d-1]) = 0.\\
q'(\text{\O}) = 0.
\end{cases}
\end{equation}
where $q:2^{d-1} \mapsto \mathbb{R}$ with $q'(T) = \f'(T) - \sum_{L \subseteq T, |L| \leq \ell} \ex_L^{\text{F-Shap}}(\f',\ell)$ for all $T \subseteq [d-1]$ and $\mu'(T) = B(|T|, d-1-|T|)$ for all $T \subseteq [d-1]$ with $1 \leq |T| \leq d-2$.

We prove that $\ex^{\text{F-Shap}}(\f,\ell)$ defined in Eqn.\eqref{eqn:dummy_faith_solution} satisfies the system of linear equations in Eqn.\eqref{eqn:lagrange2}.

First of all, by definitions of $\f'(\cdot)$ and $\ex^{\text{F-Shap}(\f,\ell)}$, we have 
$$
\ex_{\text{\O}}^{\text{F-Shap}}(\f,\ell)
= \ex_{\text{\O}}^{\text{F-Shap}}(\f,\ell) = \f'(\emptyset) = \f(\emptyset)
\Rightarrow q(\text{\O}) = 0,
$$
and 
$$
\sum_{S \in \setlessell} \Expl_S^{\text{F-Shap}}(\f, \ell)
= \sum_{S \subseteq [d-1], |S| \leq \ell} \ex_S^{\text{F-Shap}}(\f',\ell) 
=\f'([d-1]) = \f([d])
\Rightarrow q([d]) = 0.
$$

Before we prove that
$$
\sum_{T: S \subseteq T \subset [d]} \mu(T) q(T) = \frac{-\lambda}{2}  \text{  for all  }
S \subseteq [d], 1 \leq |S| \leq \ell,
$$
for some $\lambda \in \mathbb{R}$, we first derive some relations between weighting function $\mu(\cdot)$ and $\mu'(\cdot)$ and $q(\cdot)$ and $q'(\cdot)$.

Since we have $\ex^{\text{F-Shap}}_L(\f,\ell) = 0$ for all $L$ containing $\{d\}$ and $ \f(T) = \f(T \cup \{ d\})$ ( Eqn.\eqref{eqn:dummy_faith_solution}), for all $T \subseteq [d-1]$, we have 
\begin{equation}
\label{eqn:q_property1}
q(T \cup \{ d\}) 
= \f(T \cup \{ d\}) - \sum_{L \subseteq T \cup \{ d\}, |L| \leq \ell} \ex^{\text{F-Shap}}_L(\f,\ell) 
= \f(T) - \sum_{L \subseteq T, |L| \leq \ell} \ex^{\text{F-Shap}}_L(\f,\ell) 
= q(T),    
\end{equation}
and
\begin{equation}
\label{eqn:q_property2}
q'(T) 
= \f'(T) - \sum_{L \subseteq T, |L| \leq \ell} \ex^{\text{F-Shap}}_L(\f',\ell) 
= \f(T) - \sum_{L \subseteq T, |L| \leq \ell} \ex^{\text{F-Shap}}_L(\f,\ell) 
= q(T) = q(T \cup \{ d\}). 
\end{equation}

By using a property of the beta function, for positive integers $1 \leq i \leq d-2$,  we have
\begin{equation}
\label{eqn:mu_property}
\mu'_{i} = B(i,d-1-i) = B(i+1,d-1-i) + B(i,d-i) = \mu_{i+1} + \mu_{i},
\end{equation}

Also, for $2 \leq i \leq d-1$, we have 
\begin{align}
\mu_{i} = B(i,d-i)
& = \frac{(i-1)!(d-i-1)!}{(d-1)!}  \nonumber \\
& = \frac{(i-2)!(d-i-1)!}{(d-2)!} \cdot \frac{i-1}{d-1} \nonumber \\
& = B(i-1,d-i) \cdot \frac{i-1}{d-1} \nonumber \\
& = \frac{(i-1)\mu'_{i-1}}{d-1}.
\label{eqn:mu_property2}
\end{align}

(1) For all $S \subseteq [d-1]$ with $1 \leq |S| \leq \ell$, we have

\begin{align*}
\sum_{T: S \subseteq T \subset [d]} \mu(T) q(T) 
& =  \sum_{T: S \subseteq T \subseteq [d-1]} \mu(T) q(T)
+ \sum_{T: S \subseteq T \subset [d-1]} \mu(T \cup \{ d\}) q(T\cup \{ d\}) \\
& = \mu([d-1])q([d-1]) + \sum_{T: S \subseteq T \subset [d-1]} \left(\mu(T) +\mu(T \cup \{ d\}) \right) q(T) \ \ \text{( Using Eqn.\eqref{eqn:q_property1} )} \\
& =  \mu([d-1])q([d-1]) + \sum_{T: S \subseteq T \subset [d-1]} \mu'(T) q(T) \ \
\text{( Using Eqn.\eqref{eqn:mu_property} )} \\
& = \sum_{T: S \subseteq T \subset [d-1]} \mu'(T) q'(T) \ \ 
\text{(Using Eqn.\eqref{eqn:q_property2} and  $q([d-1]) = q'([d-1]) = 0$ )} \\
&  = \frac{-\lambda'}{2}  \ \ 
(\text{Eqn.\eqref{eqn:lagrange_for_f_prime}}).
\end{align*} 

(2) For all $S = \{d\}$, we have  
\begin{align*}
\sum_{T: S \subseteq T \subset [d]} \mu(T) q(T)
& = \sum_{T \subset [d-1]} \mu(T \cup \{d\} ) q(T \cup \{d\}) \\
& = \sum_{T \subset [d-1]}  \frac{|T|\mu'(T)}{d-1}  q'(T) \ \ 
\text{(Using Eqn.\eqref{eqn:mu_property2}, Eqn.\eqref{eqn:q_property2}, and $q(\text{\O}) = q(\{d\}) = 0$)} \\
& = \frac{1}{d-1}\sum_{i \in [d-1]} \sum_{T: \{i\} \subseteq T \subset [d-1]} \mu'(T) q'(T) \\
& = \frac{1}{d-1}\sum_{i \in [d-1]} -\frac{\lambda'}{2} \ \ 
(\text{Eqn.\eqref{eqn:lagrange_for_f_prime}})\\
& = -\frac{\lambda'}{2}.
\end{align*}

(3) For all $S \subseteq [d]$ containing $\{ d\}$ with $ 2 \leq |S| \leq \ell$, we have 
\begin{align*}
\sum_{T: S \subseteq T \subset [d]} \mu(T) q(T)
& = \sum_{T: (S\backslash \{d\}) \subseteq T \subset [d-1]} \mu(T \cup \{d\} ) q(T \cup \{d\}) \\
& = \sum_{T: (S\backslash \{d\}) \subseteq T \subset [d-1]}   \frac{|T|\mu'(T)}{d-1}  q'(T) \ \ 
\text{(Eqn.\eqref{eqn:mu_property2} and Eqn.\eqref{eqn:q_property2})} \\
& = \sum_{T: (S\backslash \{d\}) \subseteq T \subset [d-1]}   \frac{(|T|-|S|+1)\mu'(T)}{d-1}  q'(T) +
\sum_{T: (S\backslash \{d\}) \subseteq T \subset [d-1]}   \frac{(|S|-1 )\mu'(T)}{d-1}  q'(T)  \\
& = \sum_{i \in [d] \backslash S} \frac{1}{d-1}  \sum_{T: (S\backslash \{d\} \cup i) \subseteq T \subset [d-1]}   \mu'(T)  q'(T) +
\frac{(|S|-1)}{d-1} \cdot \sum_{T: S\backslash \{d\} \subseteq T \subset [d-1]} \mu'(T)  q'(T)  \\
& =  \sum_{i \in [d] \backslash S} \frac{-\lambda'}{2(d-1)} - \frac{|S|-1}{d-1} \cdot
\frac{\lambda'}{2}  \ \ 
(\text{Eqn.\eqref{eqn:lagrange_for_f_prime}}) \\
& = (d-|S|) \frac{-\lambda'}{2(d-1)} - \frac{|S|-1}{d-1} \cdot
\frac{\lambda'}{2} \\
& = \frac{-\lambda'}{2}.
\end{align*}
Therefore, combining (1), (2), and (3), we have 
$$
\sum_{T: S \subseteq T \subset [d]} \mu(T) q(T) = \frac{-\lambda'}{2}  \text{  for all  }
S \subseteq [d], 1 \leq |S| \leq \ell.
$$
That is, Eqn.\eqref{eqn:dummy_faith_solution} is the minimizer of Eqn.\eqref{eqn:constrained_ls_proof}. Consequently, the minimizer of Eqn.\eqref{eqn:constrained_ls_proof} satisfies interaction dummy axiom for all $1 \leq \ell \leq d$.
\end{proof}

\newpage

\subsubsection{Alternative proof of Sufficient Condition of Theorem \ref{thm:faith_shap}}
\label{sec:alternative_faithshap_sufficiency}

In the following, we present an alternative proof of sufficient condition of that under the weighting function defined in Eqn. \eqref{eqn:faithshap_weight}, the minimizers of Theorem \ref{thm:faith_shap}. This proof is similar to Theorem \ref{thm:thmdummysymm}. 

\textbf{Interaction linearity, symmetry, efficiency axiom:} The minimizers of Eqn.\eqref{eqn:constrained_weighted_regresion} satisfy interaction linearity, symmetry, and efficiency axioms by Proposition \ref{pro:linearity}, \ref{pro:symmetry} and \ref{pro:efficiency}.

\textbf{Interaction dummy axiom:} Below we prove that the optimal solution satisfies the interaction dummy axiom. 
The constrained optimization problem can be written as follows:

\begin{align*}
\small
\ex^{\text{F-Shap}}(\f_R,\ell) 
& = \min_{\Expl \in \mathbb{R}^{d_\ell}} F_R(\ex) \\
& = \min \sum_{S \supseteq R, S \subseteq [d]} \mu_{|S|} \left( \sum_{T \subseteq S, |T| \leq \ell} \Expl_T(\f_R, \ell) - 1 \right)^2 + \sum_{S \not\supseteq R, S \subseteq [d]} \mu_{|S|} \left( \sum_{T \subseteq S, |T| \leq \ell} \Expl_T(\f_R, \ell) \right)^2,
\end{align*}
\begin{equation}
\label{eqn:con_weighted_ls_basis}
\text{ subject to } \Expl_{\text{\O}}(\f,\ell) = \f(\text{\O}) \text{   and   }  \sum_{T \subseteq [d], |T| \leq \ell} \Expl_{T}(\f,\ell) = \f([d]), 
\end{equation}
where we use the notations:$\mu_{|S|} = \mu(S)$ and  $\bar{\mu}(S)
= \mubar{|S|} 
= \sum_{T \supseteq S, \mu(T) < \infty} \mu_{|T|} = \sum_{i: |S| \leq i \leq d, \mu_i < \infty}{d-|S| \choose i-|S|}\mu_{i}$ since the weighting function only depends on the size of input sets. Also, since multiplying a scalar to $\mu$ does not change the minimizer of Eqn.\eqref{eqn:weighted_regression}, without loss of generality,
, without loss of generality, we assume that $\mu(S) = \frac{d-1}{\binom{d} {|S|}\,|S|\,(d-|S|)}$ for $S \subseteq [d]$ with $1 \leq |S| \leq d-1$.

First, by Lemma \ref{lm:f_decomposable}, we only need to prove that the minimizers of all the basis functions $\f_R$ ( Definition \ref{def:basis_function} ) satisfy the interaction dummy axiom. That is, the minimizer $\ex^{\text{F-Shap}}_S(\f,\ell) = 0$ for all $S$ containing dummy features in $[d] \backslash R$, i.e.
$S \in \setlessell$ with $S \cap ([d] \backslash R) \neq \text{\O}$ .

Let $r = |R|$ denote the size of the set $R$. Now, we separate the problem into three cases: (1) $d \geq \ell \geq  r \geq 0$. (2) $d \geq \ell + r $ and $r > \ell$. (3) $\ell + r > d \geq r > \ell \geq 1$.

\paragraph{(1)  $d \geq \ell \geq  r \geq 0$:} In this case, the minimizer is trivial.

\begin{lemma}
\label{lm:solution_r_small_thm2}
If $\f_R$ is a basis function with $|R| = r \leq \ell \leq d$, the minimizer of Eqn.\eqref{eqn:con_weighted_ls_basis} is
\begin{equation}
\ex^{\text{F-Shap}}_T(\f_R, \ell)
= \begin{cases}
1 & \text{, if  } \ \  T = R. \\
0 & \text{, otherwise.}
\end{cases}
\end{equation}
\end{lemma}

\begin{proof}
If we plug in $\ex^{\text{F-Shap}}_T(\f_R, \ell)$ to Eqn.\eqref{eqn:con_weighted_ls_basis}, we get $F_R(\ex)= 0$. The solution also satisfies the constraints in Eqn. \eqref{eqn:con_weighted_ls_basis}:
\begin{equation}
\label{eqn:constrained_minimizer_r_small}
\Expl^{\text{F-Shap}}_{\text{\O}}(\f,\ell) = \f(\text{\O}) \ \  \text{   and   }  \sum_{T \subseteq [d], |T| \leq \ell} \Expl_{T}^{\text{F-Shap}}(\f,\ell) = 1 = \f([d]). 
\end{equation}
Since $F(\ex)$ is always non-negative, by Proposition \ref{pro:unique_minimizer}, it is the unique minimizer of Eqn.\eqref{eqn:con_weighted_ls_basis}.
\end{proof}
We note that the minimizer in Eqn.\eqref{eqn:constrained_minimizer_r_small} satisfies the interaction dummy axiom, i.e. $\ex_S(\f,\ell) = 0$ if $S \in \setlessell$ with $S \cap R \neq \text{\O}$.

\paragraph{(2) $d \geq \ell + r$ and $r > \ell$:} 
We solve the constrained optimization problem by using Lagrange multiplier. 

\begin{equation}
\label{eqn:lagrange}
\begin{cases}
\frac{\partial \ F(\ex)}{\partial \Expl_S } = \lambda & \text{  for all  }
S \subseteq [d], 1 \leq |S| \leq \ell.  \\
\sum_{S \in \setlessell} \Expl_S(\f_R, \ell) = 1. \\
\Expl_{\text{\O}}(\f_R, \ell) = 0.
\end{cases},
\end{equation}

where $\lambda \in \mathbb{R}$ is the Lagrange multiplier. Now we utilize the symmetry structure in the basis function $\f_R$. By proposition \ref{pro:unique_minimizer}, Eqn.\eqref{eqn:lagrange} has a unique minimizer $\ex^{\text{F-Shap}}(\f_R, \ell)$.

In the basis functions, there are only two kinds of input elements, which are elements in $R$ and not in $R$. Therefore, for $i^{th}$ order interactions terms $\Expl_{T}(\f_R,\ell)$ where $|T| = i$, there are at most $i+1$ distinct values, which has $j$ elements in $R$, $j=0,1,...i$. That is, by the interaction symmetry axiom, there are only $i+1$ different importance value for $i^{th}$ order interactions terms (since if $|T_1| = |T_2|$ and $|T_1 \cap R| = |T_2 \cap R|$ then $T_1$ and $T_2$ are symmetry and $\Expl_{T_1}(\f_R,\ell) = \Expl_{T_2}(\f_R,\ell)$), so there are $ 1 + 2 +... + (\ell + 1)= \frac{(\ell + 2)(\ell + 1)}{2}$ kinds of values in the optimal solution $\Expl(\f_R,\ell)$. Since we have known that $\Expl_{\text{\O}}(\f_R, \ell) = 0$, we then introduce a new notation system to represent the rest $\frac{(\ell + 2)(\ell + 1)}{2}-1$ values.

\begin{definition}
The minimizer
$ \bfb \in \mathbb{R}^{\frac{(\ell +1)(\ell + 2)}{2} - 1 } \text{ is indexed with  }
\bfb_{i,j} = \Expl_S(\f_R,\ell) \text{  with  } |S| = i \text{ and } |S \backslash R| = j,$
where $i, j$ are integers with $0 \leq j \leq i \leq \ell$ and $i+j > 0$.
\end{definition}

The term $\bfb_{i,j}$ means the importance score of an $i^{th}$ order interaction (of size $i$) term with $i-j$ elements lying in $R$ and $j$ element lying in $[d] \backslash R$. Now we can apply the new notation system to rewrite the system of linear equations (Eqn.\eqref{eqn:lagrange}). By plugging in the closed-form solution of partial derivatives with Lemma \ref{lm:bij_parital_derivative}, for all $0 \leq j \leq i \leq \ell, i+j > 0 $, we have 

\begin{equation}
\label{eqn:constrained_eqn1}
-\mubar{r+j}  
+  \sum_{p,q :0 \leq q \leq p \leq \ell, p+q > 0} \left( \sum_{\rho=0}^{i-j} \sum_{\sigma=0}^{j} { i-j \choose \rho} { r-(i-j) \choose p-q-\rho} { j \choose \sigma}{ d-r-j \choose q-\sigma} \mubar{i+p-\rho-\sigma} \bfb_{p,q} \right) = \frac{\lambda}{2}.
\end{equation}
and
\begin{equation}
\label{eqn:constrained_eqn2}
\sum_{p,q:0 \leq q \leq p \leq \ell, p+q > 0} { r \choose p-q }{ d-r \choose q } \bfb_{p,q} = 1.
\end{equation}
Moreover, we write them into the matrix form,
$\mathbf{M}
\begin{bmatrix} 
\lambda \\
\bfb \\ 
\end{bmatrix}= \mathbf{Y}$ defined below.

\begin{definition}
The coefficient matrix $\mathbf{M} \in \mathbb{R}^{\frac{(\ell +1)(\ell + 2)}{2} \times \frac{(\ell +1 ) (\ell + 2)}{2}}$, whose rows and columns are indexed with 2 iterators respectively. The value of each entry is 
$$
\mathbf{M}_{\{i,j\}, \{p,q\}} =
\begin{cases}
0 \ \  &  \text{ if } \ \ i=0, j=0, p=0 \text{ and } q = 0 .\\
-1/2 \ \  &  \text{ if } \ \ i+j > 0,  p=0 \text{ and } q = 0 .\\
{ r \choose p-q }{ d-r \choose q } \ \  &  \text{ if } \ \ i=0, j=0, p+q > 0 .\\
\sum_{\rho=0}^{i-j} \sum_{\sigma=0}^{j} { i-j \choose \rho } { r-(i-j) \choose p-q-\rho} { j \choose \sigma}{d-r-j \choose q-\sigma} \mubar{i+p-s-t} & \text{, otherwise}. \\
\end{cases}
$$
where $i,j,p,q$ satisfy the constraints: $ 0 \leq j \leq i \leq \ell$, and $ 0 \leq q \leq p \leq \ell$. 
\end{definition}

\begin{definition}
\label{def:Y_def_thm2}
$\mathbf{Y} \in \mathbb{R}^{\frac{(\ell + 1) (\ell+2)}{2}}$ is a column vector with each entry $$
\mathbf{Y}_{i,j} = 
\begin{cases}
1 & \text{ if } i = 0 \text{ and } j = 0 \\
\mubar{r+j} & \text{ if } 0 \leq j \leq i \leq \ell, i + j > 0 \\
\end{cases}
$$ 
\end{definition}
We note that the interaction dummy axiom holds for the basis function $\f_R$ if and only if $\bfb_{i,j} = 0$ for all $i,j > 0$  because $j > 0$ means this interaction term contains dummy features, which are lying in $[d] \backslash R$. 

Recall that Lemma \ref{lm:simplified_matrix} states that if we want to prove that some unknown variables are zero in a system of linear equations, we can prove that the rank of a simplified augmented matrix equals to the number of non-zero variables. To use Lemma  \ref{lm:simplified_matrix} to prove it, we now only consider columns corresponding to non-dummy elements (the column indexed with $p=q=0$), $\bfb_{1,0},..., \bfb_{\ell,0}$ and simplify the matrix $\mathbf{M}$ in the following way:

Put
$$
\mathbf{M'} \in \mathbb{R}^{\frac{(\ell+1)(\ell+2)}{2} \times (\ell+1)} \text{, whose columns correspond to } \lambda, \bfb_{1,0},..., \bfb_{\ell,0} \text{ of } \mathbf{M}  
$$
\begin{equation}
\label{eqn:m_definition_thm2}
\text{ with each entry } \mathbf{M'}_{\{i,j\},\{p, 0\} } = 
\begin{cases}
0 \ \  &  \text{ if } \ \ i=0, j=0 \text{ and } p = 0 \\
-1/2 \ \  &  \text{ if } \ \ i+j > 0, \text{ and } p = 0 \\
{ r \choose p } \ \  &  \text{ if } \ \ i=0, j=0 \text{ and } p > 0 \\
\sum_{\rho=0}^{i-j} { i-j \choose \rho} { r-(i-j) \choose p- \rho} \mubar{i+p-\rho}  & \text{, otherwise.} \\
\end{cases}
\end{equation}

The entry $\mathbf{M'}_{\{i,j\},\{p, 0\}}$ can be interpreted as the coefficient of $\bfb_{p,0}$ in the equation $\frac{\partial F(\ex)}{ \partial \bfb_{i,j}} = 0$.
Since we have already known that the system of linear equations, $\mathbf{M}b=\mathbf{Y}$, has a unique solution by Proposition \ref{pro:unique_minimizer}, if we can prove that the rank of the matrix $\mathbf{Q} = [\mathbf{M'}, \mathbf{Y}]$ equals to $\ell + 1$, we get that $\bfb_{i,j} = 0 $ for all $i \geq j > 0$ by Lemma \ref{lm:simplified_matrix}. It implies the interaction dummy axiom holds for the basis function $\f_R$. To calculate the rank of matrix $\mathbf{Q}$, we first define some notations.

\begin{definition}
\label{def:Pst_thm2}
Following Definition \ref{def:row_combination}, we define 
$$
\mathbf{P}_{s,t} = 
\begin{bmatrix}
\mathbf{z_s}  \\
\Re_{s+1,s}^{1}(\mathbf{Q}) \\
...\\
\Re_{s+t, s }^{t}(\mathbf{Q}) \\
\end{bmatrix}, 
\text{ and } 
\mathbf{P'}_{s,t} = 
\begin{bmatrix}
\mathbf{P}_{s,t} \\
\Re_{s+t-1, s }^{t-1}(\mathbf{Q}) - \Re_{s+t, s}^{t-1}(\mathbf{Q}) \\
\end{bmatrix}, 
$$
where 
$$
\mathbf{z_s} 
= \Re_{\max(1,s), \max(1,s)}^{0}(\mathbf{Q})
= \begin{cases}
\Re_{1, 1}^{0}(\mathbf{Q}) & \text{ , if } s = 0.\\
\Re_{s,s}^{0}(\mathbf{Q}) & \text{ , if } s > 0.\\
\end{cases}
$$
and 
$\mathbf{P}_{s,t} \in \mathbb{R}^{(t+1) \times (\ell + 2)}$ and $\mathbf{P'}_{s,t} \in \mathbb{R}^{(t+2) \times (\ell + 2)}$ for 
$0 \leq s \leq \ell-1 $ and $1 \leq t \leq \ell - s$.

\end{definition}

\begin{lemma}
\label{lm:Pst_value_thm2}
Following Definition \ref{def:Dpq} and \ref{def:Pst_thm2}, for any $s \geq 0$, let $s'=\max(1,s)$,
\begin{equation}
\label{eqn:def_pst_value_thm2}
\mathbf{P}^{'}_{s,t} =
\begin{bmatrix}
-\frac{1}{2}, & {r \choose 1 }D_0^{s'+1}, & ...&
{r \choose t-1 }D_0^{s'+t-1}, & {r \choose t }D_0^{s'+t}, &...  & { r \choose \ell  }D_0^{s'+\ell},  & D_0^{r+s'} \\
0, & {r-1 \choose 0 }D_{1}^{s+1}, & ...&
{r-1 \choose t-2 }D_{1}^{s+t-1}, & 
{r-1 \choose t-1 }D_{1}^{s+t}, & ... & { r-1 \choose \ell -1 } D_{1}^{s+\ell}, & D_{1}^{r+s} \\
0, &  0, & ...&  
{r-2 \choose t-3 }D_{2}^{s+t-1}, & 
{r-2 \choose t-2 }D_{2}^{s+t}, & ... & { r-2 \choose \ell-2 } D_{2}^{s+\ell}, & D_{2}^{r+s} \\
., & ., & ... &., &., & ... & .,& .  \\
0, & 0,& ...& {r-t+1 \choose 0 }  D_{t-1}^{s+t-1}, & 
{r-t+1 \choose 1 }  D_{t-1}^{s+t}, & ... & { r-t+1 \choose \ell-t+1 } D_{t-1}^{s+\ell}, & D_{t-1}^{r+s}  \\
0, &  0,& ...& 0 , & 
{r-t \choose 0 }  D_{t}^{s+t}, & ... 
 & { r-t \choose \ell-t } D_{t}^{s+\ell}, & D_{t}^{r+s}  \\
 0,& 0,& ...& {r-t+1 \choose 0} D_{t}^{s+t-1}, & 
[{r-t+1 \choose 1}- {r-t \choose 0 } ] D_{t}^{s+t}, &... & [{r-t+1 \choose \ell-t+1}- {r-t \choose \ell-t } ] D_{t}^{s+\ell}, & 0  \\
\end{bmatrix}.
\end{equation}

Formally, for all $1 \leq t' \leq t+1$ and $0 \leq p \leq \ell+1$, the $(p+1)^{th}$ element of  $(t'+1)^{\ th}$ row of $\mathbf{P}^{'}_{s,t} $ is 
\begin{equation}
\begin{cases}
-\frac{1}{2} & \text{ if } t' = 0 \text{ and } p=0 \\
{r \choose p}D_0^{s'+p} & \text{ if }  t' = 0 \text{ and } 1 \leq p \leq \ell \\
D_0^{r+s'} & \text{ if }  t' = 0 \text{ and } p = \ell + 1 \\
0 & \text{ if } 1 \leq t' \leq t \text{ and } p < t' \\
{r-t' \choose p-t'} D_{t'}^{s+p} 
& \text{ if } 1 \leq t' \leq t  \text{ and }  t' \leq p \leq \ell\\
D_{t'}^{r+s} & \text{ if } 1 \leq t' \leq t \text{ and } p = \ell+1 \\
0 & \text{ if }    t'=t+1 \text{ and } p < t-1  \\
{r-t+1 \choose 0} D_{t}^{s+t-1}  & \text{ if } t'=t+1 \text{ and } p = t-1\\
[{r-t+1 \choose p-t+1}- {r-t \choose p-t } ] D_{t}^{s+p} & \text{ if }    t'=t+1 \text{ and } t \leq  p \leq \ell  \\
0 & \text{ if }    t'=t+1 \text{ and } p =\ell + 1 \\
\end{cases}.
\end{equation}
\end{lemma}
\begin{proof}
First, we observe that the right-down $\mathbb{R}^{(\frac{(\ell+1)(\ell+2)}{2}-1) \times \ell} $ submatrix of $\mathbf{M'}$ is the same as in Eqn.\eqref{eqn:m_definition} in the proof of Theorem \ref{thm:thmdummysymm}\footnote{$D$ is a function of the cumulative weighting function to a real number. Although the value of cumulative function $\bar{\mu}(\cdot)$ in Theorem \ref{thm:thmdummysymm} and \ref{thm:faith_shap} is different, the coefficient of $D$ is the same.}. Therefore, the corresponding entries in the matrix $\mathbf{P}^{'}_{s,t}$ should also be the same as in Lemma \ref{lm:Pst_value} except for the first row and the first column of matrix $\mathbf{P}^{'}_{s,t}$.

Next, we calculate the first row of the matrix $\mathbf{P}^{'}_{s,t}$, which is the $(s',s')^{th}$ row of $\mathbf{M}'$. The first $\ell+1$ elements can be obtained with Eqn. \eqref{eqn:m_definition_thm2} and the last element of it can be obtained with Definition \ref{def:Y_def_thm2}.

Then, we calculate the first column of the matrix $\mathbf{P}^{'}_{s,t}$.

\begin{enumerate}
    \item If $1 \leq t' \leq t$ and $p=0$, the value of $1^{st}$ element of $(t'+1)^{th}$ row of $\mathbf{P'}_{s,t}$ is 
    $$
    \sum_{\sigma=0}^{\sigma=t'} {t' \choose \sigma} (-1)^{\sigma} (\frac{-1}{2})
    = \frac{-1}{2} (1-1)^{t'}
    = 0
    $$

    \item If $t'=t+1$ and $p=0$, the value of the first element of $(t'+1)^{th}$ row of $\mathbf{P'}_{s,t}$ is 
    \begin{align*}
    & \frac{-1}{2} \sum_{\sigma=0}^{t-1} {t-1 \choose \sigma} (-1)^{\sigma} 
    + \frac{-1}{2}  \sum_{\sigma=0}^{t-1} {t-1 \choose \sigma} (-1)^{\sigma} = 0  \\
    \end{align*}
\end{enumerate}

\end{proof}

\begin{lemma}
\label{lm:submatrix_rank_thm2}
Following Definition \ref{def:Pst_thm2}, the rank of matrices $\mathbf{P}_{s,t}$ and $\mathbf{P}^{'}_{s,t}$ are $t+1$ and therefore the vector $\Re_{s+t, s }^{t-1}(\mathbf{Q})$ lies in the span of $\{ \Re_{s',s'}^{0}(\mathbf{Q}),
\Re_{s+1,s}^{1}(\mathbf{Q}),
...,
\Re_{s+t, s }^{t}(\mathbf{Q}) \}$.
\end{lemma}
\begin{proof}
Before proving this lemma, we first introduce some properties of the value of $D^p_1$ (Definition \ref{def:Dpq}) if the weighting function is defined in Eqn. \eqref{eqn:faithshap_weight}.

\begin{claim}
\label{clm:dpq_value_thm2}
When the weighting function $\mu(\cdot)$ defined in Eqn. \eqref{eqn:faithshap_weight}, for all $p,q \in \{1,2,\ldots, d-1 \}$ with $1 \leq p+q \leq d$, we have
$$
D^p_q 
= \sum_{j=0}^q { q \choose j}(-1)^j \mubar{p+j} 
\propto \frac{d-1}{dq { p+q-1 \choose p-1 } }.
$$ 
\end{claim}

\begin{claim}
\label{clm:d_linear_relation_thm2}
When the weighting function $\mu(S)$ is defined in Eqn. \eqref{eqn:faithshap_weight}, for all $p,q \in \{0,1,2,...,d-1\}$ with $0 \leq p+q \leq d-1$, the ratios of $D^{p}_{q}$ and $D^{p}_{q+1}$ can be written as an affine function of $p$, which is equivalent to
$\frac{D^{p}_{q}}{D^{p}_{q+1}} = c^{(1)}_q p +  c^{(2)}_q $ for some constants $ c^{(1)}_q , c^{(2)}_q \in \mathbb{R}$ dependent on $q$.
\end{claim}

By Lemma \ref{lm:Pst_value_thm2}, we know value of each entry of $\mathbf{P}_{s,t}$ and $\mathbf{P'}_{s,t}$. The first $t+1$ columns of $\mathbf{P}_{s,t}$ is

\begin{equation}
\mathbf{P}_{s,t}^{(sub)} =
\begin{bmatrix}
-\frac{1}{2}, & {r \choose 1 }D_0^{s'+1}, & ...&
{r \choose t-1 }D_0^{s'+t-1}, & {r \choose t }D_0^{s'+t} \\
0, & {r-1 \choose 0 }D_{1}^{s+1}, & ...&
{r-1 \choose t-2 }D_{1}^{s+t-1}, & 
{r-1 \choose t-1 }D_{1}^{s+t} \\
0, & 0, & ...&  
{r-2 \choose t-3 }D_{2}^{s+t-1}, & 
{r-2 \choose t-2 }D_{2}^{s+t} \\
., & ., & ... &., &. \\
0, & 0,& ...& 0 , & 
{r-t \choose 0 }  D_{t}^{s+t}  \\
\end{bmatrix}.
\end{equation}

This is a upper triangular matrix and the values on the diagonal are nonzero  by Claim \ref{clm:dpq_value_thm2}. Therefore, the submatrix $\mathbf{P}_{s,t}^{(sub)}$ is full rank, so that $rank(\mathbf{P}_{s,t}^{(sub)}) = t+1$. It also implies that the rank of $\mathbf{P}_{s,t}$ is $t+1$ since the rank of $\mathbf{P}_{s,t}$ is always not smaller than the rank of column spaces of $\mathbf{P}_{s,t}^{1(sub)}$ and not larger then the number  of rows in $\mathbf{P}_{s,t}$, which are both $t+1$.

Next, we calculate the rank of the matrix $P'_{s,t}$. We first show that every $(t+2) \times (t+2)$ submatrix has rank $t+1$. The submatrix consists of the first $t$ columns, the $i^{th}$ column and the last column for all $t+1 \leq i \leq \ell$, which is as following:  

\begin{equation}
\mathbf{P}_{s,t}^{'(sub,i)} =
\begin{bmatrix}
-\frac{1}{2}, & {r \choose 1 }D_0^{s'+1}, & ...&
{r \choose t-1 }D_0^{s'+t-1}, & {r \choose i }D_0^{i+1}, & D_0^{r+s'} \\
0, & {r-1 \choose 0 }D_{1}^{s+1}, & ...& 
{r-1 \choose t-2 }D_{1}^{s+t-1}, & 
{r-1 \choose i-1 }D_{1}^{s+i}, & D_1^{r+s} \\
0, & 0, & ...&  
{r-2 \choose t-3 }D_{2}^{s+t-1}, & 
{r-2 \choose i-2 }D_{2}^{s+i}, & D_2^{r+s} \\
., & ., & ... &., &.,& . \\
0, & 0,& ...&  {r-t+1 \choose 0 }  D_{t-1}^{s+t-1}, & 
{r-t+1 \choose i-t+1 }  D_{t-1}^{s+i}, & D_{t-1}^{r+s}  \\
0, & 0,& ...& 0 , & 
{r-t \choose i-t }  D_{t}^{s+i}, & D_t^{r+s}  \\
0,& 0,& ...& {r-t+1 \choose 0} D_{t}^{s+t-1}, &
[{r-t+1 \choose i-t+1}- {r-t \choose i-t } ] D_{t}^{s+i},&  0  \\
\end{bmatrix}.
\end{equation}

The determinant of $\mathbf{P}_{s,t}^{'(sub,i)}$ is 
\begin{align*}
|\mathbf{P}_{s,t}^{(sub,i)}| 
& = 
-\frac{1}{2}
\prod_{j=1}^{t-2} {r-j \choose 0} D_{j}^{s+j} \times |P_{s,t}^{(3 \times 3,i)}|,
\end{align*}
where
$$
\mathbf{P}_{s,t}^{(3 \times 3,i)} 
= 
\begin{bmatrix}
 {r-t+1 \choose 0 }  D_{t-1}^{s+t-1}, & 
{r-t+1 \choose i-t+1 }  D_{t-1}^{s+i}, & D_{t-1}^{r+s}  \\
0 , & 
{r-t \choose i-t }  D_{t}^{s+i}, & D_{t}^{r+s}  \\
{r-t+1 \choose 0} D_{t}^{s+t-1}, &
[{r-t+1 \choose i-t+1}- {r-t \choose i-t } ] D_{t}^{s+i},&  0  \\
\end{bmatrix}
= \begin{bmatrix}
  D_{t-1}^{s+t-1}, & 
{r-t+1 \choose i-t+1 }  D_{t-1}^{s+i}, & D_{t-1}^{r+s}  \\
0 , & 
{r-t \choose i-t }  D_{t}^{s+i}, & D_{t}^{r+s}  \\
D_{t}^{s+t-1}, &
[{r-t+1 \choose i-t+1}- {r-t \choose i-t } ] D_{t}^{s+i},&  0  \\
\end{bmatrix}
$$
is the right bottom $3 \times 3$ submatrix of $\mathbf{P}_{s,t}^{(sub,i)}$.

Now we prove that the determinant of $P_{s,t}^{(3 \times 3,i)}$ is zero.
The determinant of $\mathbf{P}_{s,t}^{(3 \times 3,i)}$ is
\begin{align*}
|\mathbf{P}_{s,t}^{(3 \times 3,i)}| 
& = && - D_{t-1}^{s+t-1} D_{t}^{r+s}\left[{r-t+1 \choose i-t+1}- {r-t \choose i-t } \right] D_{t}^{s+i}\\
& && + D_{t}^{s+t-1} \left[  
{r-t+1 \choose i-t+1 }  D_{t-1}^{s+i}  D_{t}^{r+s} - D_{t-1}^{r+s} {r-t \choose i-t } D_{t}^{s+i}
\right] \ \ \ (\text{Expand with the first column}) \\ \\
& = && D_{t}^{s+t-1}D_{t}^{r+s}D_{t}^{s+i} \Bigg[
\left[-{r-t+1 \choose i-t+1}+{r-t \choose i-t } \right][c^{(1)}_{t-1}(s+t-1)+c^{(2)}_{t-1}] \\
& && + {r-t+1 \choose i-t+1 }[(c^{(1)}_{t-1}(s+i)+ c^{(2)}_{t-1}] -  {r-t \choose i-t }[c^{(1)}_{t-1}(r+s) + c^{(2)}_{t-1}]
\Bigg]
\ \ \ (\text{Lemma \ref{clm:d_linear_relation}}) \\
& = && D_{t}^{s+t-1}D_{t}^{r+s}D_{t}^{s+i}
\left[ {r-t+1 \choose i-t+1}(i-t+1) 
- {r-t \choose i-t }(r-t+1)
\right]c^{(1)}_{t-1} \\
& = && 0
\end{align*}
It implies that the submatrix $\mathbf{P}_{s,t}^{(sub,i)}$ is not full rank. 
Now if we take away the bottom row and the second last column of $\mathbf{P}_{s,t}^{(sub,i)}$, the submatrix becomes 
$$
\begin{bmatrix}
-\frac{1}{2}, & {r \choose 1 }D_0^{s'+1}, & ...&
{r \choose t-1 }D_0^{s'+t-1},  & D_0^{r+s'} \\
0, & {r-1 \choose 0 }D_{1}^{s+1}, & ...&
{r-1 \choose t-2 }D_{1}^{s+t-1}, & D_1^{r+s} \\
0, & 0, & ...&  
{r-2 \choose t-3 }D_{2}^{s+t-1},  & D_2^{r+s} \\
., & ., & ... &., & . \\
0, & 0,& ...&  {r-t+1 \choose 0 }  D_{t-1}^{s+t-1}, & 
 D_{t-1}^{r+s}  \\
0, & 0,& ...& 0 , 
& D_t^{r+s}  \\
\end{bmatrix}.
$$
This is again a full rank upper triangular matrix since its diagonal entries are all nonzero. Therefore, the second last ($(t+1)^{th}$) column of $\mathbf{P}_{s,t}^{'(sub,i)}$ lies in the span of the other $t+1$ columns. That is, the $i^{th}$ column of matrix $\mathbf{P}'_{s,t}$ can be expressed as the first $t$ columns and the last column of the matrix $P'_{s,t}$ for all $t+1 \leq i \leq \ell$.
Therefore, the rank of $\mathbf{P}'_{s,t}$ is $t+1$.

Finally, we note that matrices $\mathbf{P}_{s,t}$ and $\mathbf{P}'_{s,t}$ only differ in the last row, but their rank is the same. Consequently, the last row, $\Re_{s+t, s }^{t-1}$, lies in the span of the first $t+1$ rows, which are $\{ \Re_{0,0}^{0},
\Re_{s+1,s}^{1},
...,
\Re_{s+t, s }^{t} \}$.
\end{proof}

\begin{lemma}
\label{lm:overall_rank_thm2}
Given the results in Lemma \ref{lm:submatrix_rank_thm2}  and following Definition \ref{def:row_combination}, then 
$$
\calA = \{
\Re_{1,1}^{0}(\mathbf{Q}) ,
\Re_{1,0}^{1}(\mathbf{Q}),
...,
\Re_{\ell,0}^{\ell}(\mathbf{Q}) \}
$$
forms the basis of row space of matrix $\mathbf{Q}$ and therefore $rank(\mathbf{Q}) = \ell + 1$.
\end{lemma}
\begin{proof}

We prove a stronger version of this lemma: $\Re_{i,j}^k(\mathbf{Q})$ lies in the span of $\calA$ for all $0 \leq i \leq \ell $, $0 \leq k \leq i$, and $0 \leq j \leq i-k$. This results immediately imply that row $\mathbf{Q}_{\{i,j\}} = \Re_{i,j}^0$ in matrix $\mathbf{Q}$ also lies in the span of $\mathcal{A}$.

We prove it by induction on two indices, which are $i$ and $k$. 
\begin{enumerate}
    \item When $i=0$, $k$ and $j$ can only be zero. We prove that $\Re_{0,0}^{0}(\mathbf{Q}) = c\Re_{1,0}^{1}(\mathbf{Q})$ for some constant $c$. With Eqn. \eqref{eqn:m_definition_thm2} and  Lemma \ref{lm:Pst_value_thm2} (by applying $s=0$ to the second row of $\mathbf{P}'_{s,t}$), we can calculate the ratio of the $(p+1)^{th}$ element of $\Re_{1,0}^{1}(\mathbf{Q})$ and $\Re_{0,0}^{0}(\mathbf{Q})$. 
    \begin{enumerate}
        \item If $p=0$, the first element of both $\Re_{1,0}^{1}(\mathbf{Q})$ and $\Re_{0,0}^{0}(\mathbf{Q})$ are zero.
        \item If $0 < p \leq \ell$, we have 
        \begin{align*}
         \frac{1}{c} = \frac{{r-1 \choose p-1} D_1^{p} }{ {r \choose p} }  
         = D_1^{p} \frac{p}{r}  
         \propto \frac{d-1}{dp} \frac{p}{r} \ \ \ (\text{Claim \ref{clm:dpq_value_thm2}}) 
          = \frac{d-1}{dr}.
        \end{align*}
        \item If $p=\ell+1$, by Definition \ref{def:Y_def_thm2} , we have
        $$
        \frac{1}{c}
        = \frac{ D_1^{r} }{ 1 } 
        \propto \frac{d-1}{dr}\ \ \ (\text{Claim \ref{clm:dpq_value_thm2}}).
        $$
    \end{enumerate}
    
    Therefore, we have $\Re_{0,0}^0(\mathbf{Q}) \in \spann(\calA)$. 
    
    \item When $i=1$, $\Re_{1,1}^{0}(\mathbf{Q})$ and $\Re_{1,0}^{1}(\mathbf{Q})$ are all in the set $\calA$.
    
    \item When $i=i'$, suppose that $\Re_{i,j}^k(\mathbf{Q}) \in \spann(\calA)$ holds for $0 \leq i \leq i'$, $0 \leq k \leq i$ and $0 \leq j \leq i-k$.
    \item When $i=i'+1$, we prove $\Re_{i'+1,j}^k(\mathbf{Q})$  lies in the span of $\calA$ for all $0 \leq k \leq i'+1$, and $0 \leq j \leq i-k$ by using another induction on $k$. In this induction process, $k$ is taking value from $i'+1$ to $0$.

\begin{enumerate}
    \item When $k=i'+1$, by Lemma \ref{lm:submatrix_rank_thm2}, we know $\Re_{i'+1,i'+1}^{i'} \in \spann(\calA)$ by plugging in $s=0$ and $t=i'+1$. 
    
    \item Suppose $\Re_{i'+1, j}^{k'}(\mathbf{Q}) \in \spann(\calA)$ for $k=k'+1,.., i'+1$ and $0\leq j \leq i'+1 - k$ for some $0 \leq k' \leq i'$
    \item When $k=k'$,
    by Lemma \ref{lm:submatrix_rank}, we know 
    $$
    \Re_{i'+1,i'-k'}^{k'}(\mathbf{Q}) \in \spann(\Re_{i'-k',i'-k'}^{0}(\mathbf{Q}),...,\Re_{i',i'-k'}^{k'}(\mathbf{Q}),\Re_{i'+1,i'-k'}^{k'+1}(\mathbf{Q})) 
    $$ 
    by plugging in $s=i'-k'$ and $t=k'+1$. Since we have 
    $$
    \spann(\Re_{i'-k',i'-k'}^{0}(\mathbf{Q}),...,\Re_{i',i'-k'}^{k'}(\mathbf{Q})) \subseteq \spann(\calA)
    $$ 
    by induction hypothesis on $i$ and $\Re_{i'+1,i'-k'}^{k'+1}(\mathbf{Q}) \in \spann(\calA)$ by the induction hypothesis on $k$, we have 
    \begin{equation}
    \label{eqn:row_i_k_lie_A2}
    \Re_{i'+1,i'-k'}^{k'}(\mathbf{Q}) \in \spann(\Re_{i'-k',i'-k'}^{0}(\mathbf{Q}),...,\Re_{i',i'-k'}^{k'}(\mathbf{Q}),\Re_{i'+1,i'-k'}^{k'+1}(\mathbf{Q})) \subseteq \spann(\calA).
    \end{equation}

    Then we observe that 
    \begin{equation}
    \label{eqn:row_relation2}
    \Re_{i'+1,j}^k(\mathbf{Q}) = 
    \begin{cases}
    \Re_{i'+1,i'-k}^k(\mathbf{Q})  + \sum_{\rho=j}^{i'-k-1} \Re_{i'+1,\rho}^{k+1}(\mathbf{Q}) & \text{ if   } 0 \leq j < i'-k. \\
    \Re_{i'+1,i'-k}^k(\mathbf{Q}) & \text{ if   } j = i'-k \\
    \Re_{i'+1,i'-k}^k(\mathbf{Q}) - \Re_{i'+1,i'-k}^{k+1}(\mathbf{Q})  & \text{ if   } j = i'-k+1.\\
    \end{cases}    
    \end{equation}
    The Eqn.\eqref{eqn:row_relation} follows from 
    \begin{align*}
    \Re_{i'+1, \rho}^{k+1}(\mathbf{Q})
    & = \sum_{\sigma=0}^{k+1} {k+1 \choose \sigma} (-1)^{\sigma} \mathbf{Q}_{\{i'+1,\rho+\sigma \}} \\
    & = \sum_{\sigma=0}^{k} {k \choose \sigma} (-1)^{\sigma} \mathbf{Q}_{\{i'+1,\rho+\sigma \}}
    + \sum_{\sigma=1}^{k+1} {k \choose \sigma-1} (-1)^{\sigma} \mathbf{Q}_{\{i'+1,\rho+\sigma \}} \\
    & = \sum_{\sigma=0}^{k} {k \choose \sigma} (-1)^{\sigma} \mathbf{Q}_{\{i'+1,\rho+\sigma \}}
    - \sum_{\sigma=0}^{k} {k \choose \sigma} (-1)^{\sigma} \mathbf{Q}_{\{i'+1,\rho+\sigma +1\}} \\
    & = \Re_{i'+1, \rho}^{k}(\mathbf{Q}) - \Re_{i'+1, \rho+1}^{k}(\mathbf{Q}),
    \end{align*}
    Then Eqn. \eqref{eqn:row_relation2} can be attained by summing over the above equation from $\rho =j$ to $\rho=i'-k-1$.
    
    By Eqn. \eqref{eqn:row_relation2} $\Re_{i'+1,j}^{k'}(\mathbf{Q})$ can be expressed as the combination of some terms in the form of $\Re_{i'+1,\rho}^{k'+1}(\mathbf{Q})$ and $\Re_{i'+1,i'-k' }^{k'}(\mathbf{Q})$. Both terms lie in the span of $A$ by the induction hypothesis on $k$ and Eqn. \eqref{eqn:row_i_k_lie_A2}.
    
    Therefore, we can conclude that $\Re_{i'+1,j}^{k'}(\mathbf{Q})$ also lies in the span of $\calA$ for all $0 \leq j \leq i'+1-k'$. It implies that it also holds when $k=k'$ and we establish the induction step on $k$.
\end{enumerate}
After completing mathematical induction proof on $k$, we know that $\Re_{i'+1,j}^k(\mathbf{Q})$ lies in the span of $\calA$ for all $0 \leq k \leq i'+1$ and $0 \leq j \leq i'+1-k$. Consequently, we also finish the induction step on $i$ (when $i= i'+1$).

\end{enumerate}

Finally, we can deduce that $\Re_{i,j}^k(\mathbf{Q})$ lies in the span of $\calA$ for any $0 \leq i \leq \ell $, $0 \leq j \leq i$ and $0 \leq k \leq i-j$. Then we know that  every row $\mathbf{Q}_{\{i,j\}} = \Re_{i,j}^0(\mathbf{Q})$ in matrix $\mathbf{Q}$ also lies in the span of $\mathcal{A}$, which in turn implies $rank(\mathbf{Q}) = \ell +1$.

\end{proof}

\paragraph{(3) $ r + \ell > d \geq \ell + r$ and $r > \ell$:} Now, we generalize the results in the second case to the last case. We recall that $\bfb_{i,j}$ denotes the interaction indices with $j$ elements in $R$ and $j$ elements in $[d] \backslash R$ (Definition \ref{def:bij}). However, 
when $r + \ell > d$, there are some $\bfb_{i,j}$ that do not exist since there are not enough elements outside $R$.
For example, $\bfb_{\ell,\ell}$ does not exist since there are only $d-r < \ell$ elements outside $R$. 

In this case, we can still compute the matrix of linear equations,  $\mathbf{Q'} = [ \mathbf{M'}, \mathbf{Y}]$, but some rows do not exist. Particularly, all rows corresponding to $\bfb_{i,j}$ with $d-r < j \leq i$ do not exist.
Nevertheless, removing rows do not
increases the rank of the matrix $\mathbf{Q}$. Therefore, by Lemma \ref{lm:overall_rank_thm2}, the rank of $\text{rank}(\mathbf{Q'}) \leq \text{rank}(\mathbf{Q}) = \ell + 1 $. Also, we note that the columns in the coefficient matrix $\mathbf{M}$ are linearly independent (since it has a unique solution by Proposition \ref{pro:unique_minimizer}). This implies that the columns in the reduced coefficient matrix $\mathbf{M'}$ are also linearly independent (since $\mathbf{M'}$ is a submatrix of $\mathbf{M}$). Therefore, the rank of the reduced coefficient matrix $\mathbf{M'}$ equals to the number of columns in $\mathbf{M'}$, which is $\ell+1$. 

Overall, we have $ \ell + 1 \geq \text{rank}(\mathbf{Q'}) \geq \text{rank}(\mathbf{M'}) = \ell +1 $. That is, $\text{rank}(\mathbf{Q'}) = \ell +1$, which in turn implies that $\bfb_{i,j} = 0 $ for all $i \geq j > 0$ by Lemma \ref{lm:simplified_matrix}. Therefore, the corresponding Faith-Interaction indices satisfy the interaction dummy axiom when $ r + \ell > d \geq \ell + r$ and $r > \ell$.

Therefore, by summarizing (1)-(3),
we conclude that the interaction dummy axiom holds for all basis function $\f_R$. This result can be generalized to any function $\f(\cdot)$ by applying Lemma \ref{lm:f_decomposable}.

In conclusion, 
the Faith-Interaction indices with respect to the weighting function defined in Eqn.\eqref{eqn:thmdummysymm} satisfy the interaction linearity, symmetry and dummy axioms for all set functions $\f(\cdot):2^d \mapsto \mathbb{R}$ and all maximum interaction order $1 \leq \ell \leq d$.

\newpage

\subsubsection{Closed-form Solution of Faith-Shap}
\label{sec:closed_form_solution_faith_shap}

In this section, we solve the constrained weighted linear regression problem defined in Eqn.\eqref{eqn:constrained_weighted_regresion} with the weighting function defined in Eqn.\eqref{eqn:faithshap_weight}. We start with solving Faith-Shap indices for basis functions (Lemma \ref{lm:faithshap_for_basis_function}) and then extend the results to general set functions (Lemma \ref{lm:faithshap_for_general_functions}). Lastly, we provide another expression for the highest order terms of Faithful Shapley Interaction indices, i.e. $\ex_S(\f,\ell)$ for $|S| = \ell$, in terms of discrete derivatives (Lemma \ref{lm:faithshap_highest_order}).

First of all, we solve the closed-form solution when the set function $\f(\cdot)$ is a basis function.

\begin{lemma}
\label{lm:faithshap_for_basis_function}
(Faith-Shap for basis functions)
Consider the basis function $\f_R$ defined as 
$ \f_R(S) = 1$ if $S \supseteq R$, otherwise $0$, where $R \subseteq [d]$. Let $\ell$ be the maximum interaction order. Let $\Expl^{\text{F-Shap}}(\f_R,\ell)$ be the solution of the constrained weighted linear regression problem defined in Eqn.\eqref{eqn:constrained_weighted_regresion} with the weighting function defined in Eqn.\eqref{eqn:faithshap_weight}. Then, for all $S \subseteq [d]$ with $|S| \leq \ell$,
\begin{equation}
\label{eqn:solution_basis_function}
\Expl^{\text{F-Shap}}_S(\f_R,\ell)= 
\begin{cases}
(-1)^{\ell - |S| }\frac{|S|}{\ell + |S|} {\ell \choose |S| } \frac{ {|R| - 1 \choose \ell }}{ { |R| + \ell -1 \choose \ell + |S|}} 
& \text{, if } S \subseteq R, |S| \geq 1, \text{ and } |R| > \ell. \\
0 & \text{, if } S \not \subseteq R \text{ and } |R| > \ell. \\
0 & \text{, if } S = \text{\O} \text{ and } |R| > \ell . \\
1 & \text{, if } S = R \text{ and } |R| \leq \ell. \\
0 & \text{, if } S \neq R \text{ and } |R| \leq \ell. \\
\end{cases}
\end{equation}
\end{lemma}
\begin{proof}

For clarity, we let $r = |R|$ and $s = |S|$.
We first deal with the cases when $|R| = r \leq \ell$. By Lemma \ref{lm:solution_r_small_thm2}, the minimizer of Eqn.\eqref{eqn:con_weighted_ls_basis} is 
$$
\ex^{\text{F-Shap}}_T(\f_R, \ell)
= \begin{cases}
1 & \text{, if  } \ \  T = R. \\
0 & \text{, otherwise.}
\end{cases}
$$

Next, we consider the case when $r > \ell \geq 1$ and $d > \ell + r$ . If $S = \text{\O}$, by the constraints, we must have $\Expl^{\text{F-Shap}}_{\text{\O}}(\f,\ell) = \f(\text{\O}) = 0$. If $S \not \subseteq R$, since elements outside $R$ are dummy features, we have $\Expl^{\text{F-Shap}}_S(\f,\ell) = 0$ by the dummy axiom. Also, by symmetric axiom, for all $S \subseteq R$, $\Expl^{\text{F-Shap}}_S(\f,\ell)$ only depends on the size of the set $S$. Therefore, there are only $\ell$ kinds of different values, which are $\Expl^{\text{F-Shap}}_S(\f,\ell)$ for $S \subseteq R$ with $|S| = 1,\cdots,\ell$. For convenience, we use $\Expl^{\text{F-Shap}}_{i}(\f,\ell)$ to deonte $\Expl^{\text{F-Shap}}_{S}(\f,\ell)$ with $|S| = i$ and $S \subseteq R$. 

We solve the problem via Lagrange multiplier. Following the same simplification process in Section \ref{sec:alternative_faithshap_sufficiency}, we get the matrix $\mathbf{P}_{0,\ell}$, which is the matrix in Eqn.\eqref{eqn:def_pst_value_thm2} in Lemma \ref{lm:Pst_value_thm2} and Definition \ref{def:Pst_thm2}. $\mathbf{P}_{0,\ell} \in \mathbb{R}^{(\ell+1) \times (\ell+2)}$ can be expressed as follows.
\begin{equation}
\mathbf{P}_{0,\ell} =
\left(
\begin{array}{@{}rrrrrr|c@{}}
-\frac{1}{2}, & {r \choose 1 }D_0^{2}, & ...&
{r \choose \ell-2 }D_0^{\ell-1}, & {r \choose \ell-1 }D_0^{\ell}, & { r \choose \ell  }D_0^{\ell+1},  & D_0^{r+1} \\
0, & {r-1 \choose 0 }D_{1}^{1}, & ...&
{r-1 \choose \ell-3 }D_{1}^{\ell-2}, & 
{r-1 \choose \ell-2 }D_{1}^{\ell-1}, &  { r-1 \choose \ell -1 } D_{1}^{\ell}, & D_{1}^{r} \\
0, &  0, & ...&  
{r-2 \choose \ell-4 }D_{2}^{\ell-2}, & 
{r-2 \choose \ell-3 }D_{2}^{\ell-1}, &  { r-2 \choose \ell-2 } D_{2}^{\ell}, & D_{2}^{r} \\
., & ., & ... &., &., &  .,& .  \\
0, & 0,& ...& 0, & 
{ r-\ell+1 \choose 0 } D_{\ell-1}^{\ell-1} & { r-\ell+1 \choose 1 } D_{\ell-1}^{\ell}, & D_{\ell-1}^{r}  \\
0, &  0,& ...& 0 , & 
0, & { r-\ell \choose 0 } D_{\ell}^{\ell}, & D_{\ell}^{r}  \\
\end{array} \right)
\end{equation}

Formally, for $0 \leq i \leq \ell $ and $0 \leq j \leq \ell + 1$, the value of the element in $(i+1)^{th}$ row and $(j+1)^{th}$ column is 
\begin{equation}
\begin{cases}
-\frac{1}{2} & \text{ if } i = 0 \text{ and } j=0. \\
{r \choose j}D_0^{j+1} & \text{ if }  i = 0 \text{ and } 1 \leq j \leq \ell. \\
D_0^{r+1} & \text{ if }  i = 0 \text{ and } j = \ell + 1. \\
0 & \text{ if } 1 \leq i \leq \ell \text{ and } j < i. \\
{r-i \choose j-i} D_{i}^{j} 
& \text{ if } 1 \leq i \leq \ell  \text{ and }  i \leq j \leq \ell.\\
D_{i}^{r} & \text{ if } 1 \leq i \leq \ell \text{ and } j = \ell+1 .\\
\end{cases}.
\end{equation}

From Definition \ref{def:Pst_thm2}, this matrix is the augmented matrix that is obtained form applying Gaussian elimination process to the original matrix for solving a system of linear equations. Row $i+1$ in $\mathbf{P}_{0,\ell}$ corresponds to $\Expl^*_{S}(\f,\ell)$ with $|S| = i$ and $S \subseteq R$ ( except that the first row maps to Lagrange multiplier $\lambda$). That is, 
\begin{equation}
\label{eqn:def_p0ell}
\underbrace{\left(
\begin{array}{@{}rrrrrr}
-\frac{1}{2}, & {r \choose 1 }D_0^{2}, & ...&
{r \choose \ell-2 }D_0^{\ell-1}, & {r \choose \ell-1 }D_0^{\ell}, & { r \choose \ell  }D_0^{\ell+1}\\
0, & {r-1 \choose 0 }D_{1}^{1}, & ...&
{r-1 \choose \ell-3 }D_{1}^{\ell-2}, & 
{r-1 \choose \ell-2 }D_{1}^{\ell-1}, &  { r-1 \choose \ell -1 } D_{1}^{\ell} \\
0, &  0, & ...&  
{r-2 \choose \ell-4 }D_{2}^{\ell-2}, & 
{r-2 \choose \ell-3 }D_{2}^{\ell-1}, &  { r-2 \choose \ell-2 } D_{2}^{\ell}\\
., & ., & ... &., &., &  . \\
0, & 0,& ...& 0, & 
{ r-\ell+1 \choose 0 } D_{\ell-1}^{\ell-1} & { r-\ell+1 \choose 1 } D_{\ell-1}^{\ell} \\
0, &  0,& ...& 0 , & 
0, & { r-\ell \choose 0 } D_{\ell}^{\ell} \\
\end{array} \right)}_{\text{ upper triangular matrix}}
\left( 
\begin{array}{@{}c}
\lambda \\
\Expl^{\text{F-Shap}}_{1}(\f_R,\ell)  \\
\Expl^{\text{F-Shap}}_{2}(\f_R,\ell)  \\
. \\
\Expl^{\text{F-Shap}}_{\ell-1}(\f_R,\ell) \\
\Expl^{\text{F-Shap}}_{\ell}(\f_R,\ell)  \\
\end{array}
\right)
= 
\left( 
\begin{array}{@{}c}
D_{0}^{r+1}\\
D_{1}^{r} \\
D_{2}^{r}  \\
. \\
D_{\ell-1}^{r}  \\
D_{\ell}^{r}  \\
\end{array}
\right),
\end{equation}
where $\Expl^{\text{F-Shap}}_{i}(\f,\ell) = \Expl^{\text{F-Shap}}_{S}(\f,\ell)$ with $|S| = i$ and $S \subseteq R$.
Now, we can verify the solution 
$$
D^p_q 
\propto \frac{d-1}{dq { p+q-1 \choose p-1 } } \propto \frac{(p-1)!(q-1)!}{(p+q-1)!}. 
$$

we can solve $\Expl^{\text{F-Shap}}(\f,\ell)$ with the above matrix equation. 
Claim \ref{clm:dpq_value_thm2} gives us the value of $D_{q}^{p}$, for all $p,q \in \{1,2,\ldots, d-1 \}$, we have 
\begin{equation}
\label{eqn:def_Dpq_closed_form}
D^p_q 
\propto \frac{d-1}{dq { p+q-1 \choose p-1 } } \propto \frac{(p-1)!(q-1)!}{(p+q-1)!}.  
\end{equation}

We begin by the following claim.

\begin{claim}
\label{clm:closed_form_basis_function}
For $ 1 \leq i \leq \ell < r$, 
$$
\sum_{j=0}^{\ell - i} 
{ r- i \choose j} \frac{D_{i}^{j+i}}{D^{r}_{i}} 
\frac{(r-1)!(\ell + j+i -1)!}{(r+\ell-1)!(j+i-1)!} { r - j-i - 1 \choose \ell - i-j} (-1)^{\ell - i - j} 
= 1.
$$
\end{claim}

Now, by plugging in 
$$\ex_i^{\text{F-Shap}}(\f_R,\ell) = (-1)^{\ell - i }\frac{i}{\ell + i} {\ell \choose i } \frac{ {r - 1 \choose \ell }}{ { r + \ell -1 \choose \ell + i}}
$$ 
to the $(i+1)^{th}$ row of matrix Eqn.\eqref{eqn:def_p0ell}, we have
\begin{align*}
& \sum_{j=0}^{\ell-i} {r-i \choose j} D^{j+i}_i \ex_{i+j}^{\text{F-Shap}}(\f_R,\ell) \\
& = \sum_{j=0}^{\ell-i} {r-i \choose j} D^{j+i}_i (-1)^{\ell - i -j}\frac{i}{\ell + i+j} {\ell \choose i+j } \frac{ {r- 1 \choose \ell }}{ { r + \ell -1 \choose \ell + i+j}}
\\
& = \sum_{j=0}^{\ell-i} {r-i \choose j} D^{j+i}_i \frac{(r-1)!(\ell + j+i -1)!}{(r+\ell-1)!(j+i-1)!} { r - j-i - 1 \choose \ell - i-j} (-1)^{\ell - i - j} 
\\
& = D^{r}_i \ \ \ (\text{By Claim \ref{clm:closed_form_basis_function}}).
\end{align*}
Therefore, the solution in Eqn.\eqref{eqn:solution_basis_function} satisfies the system of linear equations.

Finally, we deal with the case when $\ell +r > d \geq r > \ell \geq 1$. We define a new set function $\f_R':2^{d+\ell} \mapsto \mathbb{R}$ with $\f'_R(S) =
\f(S \cap [d])$ for all $S \subseteq [d + \ell]$. Then we can easily see that $\f'_R(\cdot)$ is also a basis function and 
features $d+1,\cdots, d+ \ell$ are dummy nods, which have no effect to the set function. Then by Lemma \ref{lm:faith_shap_reduce}, we have $\ex_S^{\text{F-Shap}}(\f_R,\ell) = \ex_S^{\text{F-Shap}}(\f'_R,\ell)$ for all $S \subseteq [d]$ with $|S| \leq \ell$. That is, the values of these interaction terms are the same but $\f'_R(\cdot)$ has more features. We can apply the previous results since $d' = d+ \ell > \ell + r$. We note that Eqn.\eqref{eqn:solution_basis_function} does not depend on the number of features $d$, so Eqn.\eqref{eqn:solution_basis_function} is the minimizer in this case.

\end{proof}

\begin{lemma}
\label{lm:faithshap_for_general_functions}
(Faith-Shap for general set functions)
For any set function $\f:2^d \mapsto \mathbb{R}$,
the Faith-Shap interaction indices have the following form:
\begin{equation*}
\ex^{\text{F-Shap}}_S(\f,\ell) =
a(S) + (-1)^{\ell - |S| }\frac{|S|}{\ell + |S|} {\ell \choose |S| } \sum_{T \supset S, |T| > \ell} \frac{ {|T| - 1 \choose \ell }}{ { |T| + \ell -1 \choose \ell + |S|}} a(T), \;\; \forall S \in \mathcal{S}_\ell,
\end{equation*}
where $a(S)$ is the Möbius transform of $\f$.
\end{lemma}

By Lemma 3 in \citet{shapley1953value}, any set function $\f: 2^d \rightarrow \mathbb{R}$ can be written into a linear combination of basis functions:
$$
\f(S) = \sum_{R \subseteq [d]} a(R)\f_R(S).
$$
Using the linearity axiom, we can get extend the closed-form minimizers for basis functions in Lemma \ref{lm:faithshap_for_basis_function} to general functions as follows:
\begin{align*}
\Expl^{\text{F-Shap}}(\f,\ell) 
& = \sum_{R \subseteq [d]} a(R)\Expl^{\text{F-Shap}}(\f_R,\ell).
\end{align*}
Then we could obtain 
$\ex_S^{\text{F-Shap}}$ by using the above equation. For any $S \in \setlessell$.
\begin{align*}
& \ex^{\text{F-Shap}}_S(\f,\ell) \\
& = \sum_{R \subseteq [d]} a(R)\Expl^{\text{F-Shap}}_S(\f_R,\ell) \\
& = a(S)\Expl^{\text{F-Shap}}_S(\f_S,\ell) 
+ \sum_{R: S \not \subseteq R, R \subseteq [d] } a(R)\Expl^{\text{F-Shap}}_S(\f_R,\ell) 
 \\
& + \sum_{R: S \subset R, |R| \leq \ell, R \subseteq [d] } a(R)\Expl^{\text{F-Shap}}_S(\f_R,\ell) 
+ \sum_{R: S \subset R, |R| > \ell, R \subseteq [d] } a(R)\Expl^{\text{F-Shap}}_S(\f_R,\ell) 
\\
& = a(S) + 0 + 0 +  (-1)^{\ell - |S| }\frac{|S|}{\ell + |S|} {\ell \choose |S| } \sum_{R: S \subset R \subseteq [d], |R| > \ell} \frac{ {|R| - 1 \choose \ell }}{ { |R| + \ell -1 \choose \ell + |S|}} a(R)
\ \ \ \text{(Lemma \ref{lm:faithshap_for_basis_function}) } \\
& = a(S) + (-1)^{\ell - |S| }\frac{|S|}{\ell + |S|} {\ell \choose |S| } \sum_{T \supset S, |T| > \ell} \frac{ {|T| - 1 \choose \ell }}{ { |T| + \ell -1 \choose \ell + |S|}} a(T).
\end{align*}

\begin{lemma}
\label{lm:faithshap_highest_order}
(Faith-Shap in the form of discrete derivatives)
The highest order term of Faithful Shapley Interaction indices have the following form:
\footnote{We have tried to solve expressions for general orders in terms of discrete derivatives. However, even for the second-highest term  $|S| = \ell -1$, there is no clear expressions. Specifically, for $|S| = \ell -1$, we have
$$
\Expl^{\text{F-Shap}}_S(\f,\ell) = 
\frac{(2\ell-3)!}{(\ell- 2)!(\ell-1)!}\sum_{T \subseteq [d] \backslash S}\frac{(|T|+\ell-2)!(d-|T|-1)!}{(d-1+l)!}(\ell^2 -\ell d+ 2 \ell |T| - 2\ell - |T| +1)\Delta_S f(T) \text{   for   } |S| = \ell-1.
$$}
$$
\Expl_S^{\text{F-Shap}}(\f,\ell) = \frac{(2\ell -1)!}{((\ell-1)!)^2 } 
\sum_{T \subseteq [d] \backslash S}\frac{(\ell+|T|-1)!(d-|T|-1)!}{(d+\ell-1)!}   \Delta_S(\f(T))
\ \ \text{ for all } S \in \setlessell 
\text{ with } |S| = \ell.
$$
\end{lemma}

Using the formula for basis functions in 
Lemma \ref{lm:faithshap_for_basis_function}, for any $S \subseteq [d]$ with $|S| = \ell$, we get 
\begin{align*}
& \Expl^{\text{F-Shap}}_S(\f,\ell)  = \sum_{R \subseteq [d], R \supseteq S} a(R)\Expl^{\text{F-Shap}}_S(\f_R,\ell) \\
& = \sum_{ W \subseteq [d]  \backslash S} a(W \cup S)\Expl_S^{\text{F-Shap}}(\f_{W \cup S},\ell)  \ \ \ 
\text{ where } W = R \backslash S \text{ and } \Expl^{\text{F-Shap}}_S(\f_R,\ell) = 0 \text{ for all } S \not \subseteq R. \\
& =  a(S) + 
\sum_{ \substack{W \subseteq [d] \backslash S \\ |W| > \ell - |S|}} a(W \cup S) (-1)^{\ell - |S| }  \frac{(|W \cup S|-1)!(\ell + |S| -1)! (|W|-1)!}{(|W \cup S| -\ell-1)! (|W \cup S| +\ell-1)! (|S|-1)!(\ell-|S|)!} \\
& =  \Delta_S(\f(\text{\O})) + 
\underbrace{\sum_{ \substack{W \subseteq [d] \backslash S \\ |W| > \ell - |S|}} a(W \cup S) (-1)^{\ell - |S| }  \frac{(|W \cup S|-1)!(\ell + |S| -1)! (|W|-1)!}{(|W \cup S| -\ell-1)! (|W \cup S| +\ell-1)! (|S|-1)!(\ell-|S|)!}}_{(i)}.
\end{align*}
The second last equality is obtained by applying Lemma \ref{lm:faithshap_for_basis_function} and the last equality is due to Claim \ref{clm:mobius_to_discrete_derivaitive} below.

\begin{claim}
\label{clm:mobius_to_discrete_derivaitive}
(\citet{sundararajan2020shapley} ,Lemma 2) Möbius coefficients and discrete derivatives are related by following relation:
$$
a(T \cup S) = \sum_{W \subseteq T} (-1)^{|T| - |W|} \Delta_S (\f(W))
$$
for $S$ and $T$ such that $S \cap T = \text{\O}$.
\end{claim}

Now, we analyze (i).
{\small
\begin{align}
(i) & = \frac{(-1)^{\ell - |S| }(\ell + |S| -1)!}{(|S|-1)! (\ell-|S|)! } 
\sum_{\substack{W \subseteq [d] \backslash S \\ |W| > \ell - |S|}} a(W \cup S)   \frac{(|W \cup S|-1)! (|W|-1)!}{(|W \cup S| -\ell-1)! (|W \cup S| +\ell-1)! } \nonumber \\
& =  \frac{(-1)^{\ell - |S| }(\ell + |S| -1)!}{(|S|-1)! (\ell-|S|)! } 
\sum_{\substack{W \subseteq [d] \backslash S \\ |W| > \ell - |S|}}
\sum_{ U \subseteq W} (-1)^{|U|-|W|} \Delta_S(\f(U))
\frac{(|W \cup S|-1)! (|W|-1)!}{(|W \cup S| -\ell-1)! (|W \cup S| +\ell-1)! } \nonumber \\
& \ \ \ \ \ \ \ \text{ (Using Claim \ref{clm:mobius_to_discrete_derivaitive}) }
\nonumber \\
& = \frac{(-1)^{\ell - |S| }(\ell + |S| -1)!}{(|S|-1)! (\ell-|S|)! } 
\sum_{U \subseteq [d] \backslash S} \Delta_S(\f(U))
\sum_{ \substack{ W \supseteq U, W \subseteq [d]\backslash S \\ |W| > \ell - |S|}} (-1)^{|U|-|W|}
\frac{(|W \cup S|-1)! (|W|-1)!}{(|W \cup S| -\ell-1)! (|W \cup S| +\ell-1)! }.
\label{eqn:expression_(i)_closed_form}
\end{align}}
Now, we analyze the inner sum. For clarity, we let $w = |W|$, $s = |S|$ and $u = |U|$. We also use the properties of beta functions: $B(m,n) = \frac{(m-1)!(n-1)!}{(m+n-1)!} = \int_{0}^1 x^{m-1}(1-x)^{n-1} dx$ for any $m,n \subseteq \mathbb{N}$.

\begin{align*}
& \sum_{ W \supseteq U, W \subseteq [d]\backslash S, |W| > \ell - |S|} (-1)^{|U|-|W|} 
\frac{(|W \cup S|-1)! (|W|-1)!}{(|W \cup S| -\ell-1)! (|W \cup S| +\ell-1)! } \\
& = \sum_{w = \max(u,\ell-s+1)}^{d-s} { d-s-u\choose w-u } (-1)^{u-w}
\frac{(w+s-1)! (w-1)!}{(w + s -\ell-1)! (w + s +\ell-1)! } \\
& = \frac{1}{(\ell-1)!}\sum_{w = \max(u,\ell-s+1)}^{d-s} { d-s-u\choose w-u } (-1)^{u-w}
\frac{(\ell-1)!(w+s-1)!}{ (w + s +\ell-1)! } \prod_{i=1}^{\ell-s}(w-i) \\
& = \frac{1}{(\ell-1)!}\sum_{w = \max(u,\ell-s+1)}^{d-s} { d-s-u\choose w-u } (-1)^{u-w}B(w+s,\ell) \prod_{i=1}^{\ell-s}(w-i) \\
& \ \ \ \text{(Definition of Beta function)}
 \\
& = \frac{1}{(\ell-1)!}\sum_{w = \max(u,\ell-s+1)}^{d-s} { d-s-u\choose w-u } (-1)^{u-w} \left(\prod_{i=1}^{\ell-s}(w-i) \right) \int_{0}^1 x^{w+s-1}(1-x)^{\ell-1} dx \\
& \ \ \ \text{(Property of Beta function)}
 \\
& = \frac{1}{(\ell-1)!}  \int_{0}^1 x^{s-1} (1-x)^{\ell-1} \sum_{w = \max(u,\ell-s+1)}^{d-s} { d-s-u\choose w-u } (-1)^{u-w} \left(\prod_{i=1}^{\ell-s}(w-i) \right) x^{w} dx  \\
& \ \ \ \text{(Exchange of integration)}
 \\
& = \frac{1}{(\ell-1)!}  \int_{0}^1 x^{\ell-1} (1-x)^{\ell-1} \sum_{w = \max(u,1)}^{d-\ell} { d-\ell-u\choose w-u } (-1)^{u-w} x^{w} dx  
\ \ \ (\text{Using } s = \ell).\\
\end{align*}

\paragraph{(1)} For the case when $u \geq 1$, we have 
\begin{align*}
& \frac{1}{(\ell-1)!}  \int_{0}^1 x^{\ell-1} (1-x)^{\ell-1} \sum_{w = \max(u,1)}^{d-\ell} { d-\ell-u\choose w-u } (-1)^{u-w} x^{w} dx  \\
& = \frac{1}{(\ell-1)!}  \int_{0}^1 x^{\ell-1+u} (1-x)^{\ell-1} \sum_{w'=0}^{d-\ell-u} { d-\ell-u\choose w' } (-1)^{w'} x^{w'} dx  
\ \ \ (\text{let } w' = w-u)\\
& = \frac{1}{(\ell-1)!}  \int_{0}^1 x^{\ell-1+u} (1-x)^{\ell-1} (1-x)^{d-\ell-u} dx   \\
& = \frac{1}{(\ell-1)!}  \int_{0}^1 x^{\ell-1+u} (1-x)^{d-u-1} dx   \\
& = \frac{ B(\ell+u, d-u)}{(\ell-1)!}  \\
& = \frac{(\ell+u-1)!(d-u-1)!}{(d+\ell-1)!(\ell-1)!}.  \\
\end{align*}

\paragraph{(2)} For the case when $u =0$, we have 
\begin{align*}
& \frac{1}{(\ell-1)!}  \int_{0}^1 x^{\ell-1} (1-x)^{\ell-1} \sum_{w = \max(u,1)}^{d-\ell-u} { d-\ell-u\choose w-u } (-1)^{u-w} x^{w} dx  \\
& = \frac{1}{(\ell-1)!}  \int_{0}^1 x^{\ell-1} (1-x)^{\ell-1} \sum_{w=1}^{d-\ell} { d-\ell \choose w } (-1)^{w} x^{w} dx  \\
& = \frac{1}{(\ell-1)!}  \int_{0}^1 x^{\ell-1} (1-x)^{\ell-1} \sum_{w=0}^{d-\ell} { d-\ell \choose w } (-1)^{w} x^{w} dx  - \frac{1}{(\ell-1)!}  \int_{0}^1 x^{\ell-1} (1-x)^{\ell-1} dx  \\
& = \frac{1}{(\ell-1)!}  \int_{0}^1 x^{\ell-1} (1-x)^{\ell-1} (1-x)^{d-\ell} dx  - \frac{ B(\ell,\ell)}{(\ell-1)!}  \\
& = \frac{1}{(\ell-1)!}  \int_{0}^1 x^{\ell-1} (1-x)^{d-1} dx   - \frac{ B(\ell,\ell)}{(\ell-1)!} \\
& = \frac{ B(\ell, d)}{(\ell-1)!}  - \frac{B(\ell,\ell)}{(\ell-1)!} \\
& = \frac{(\ell+u-1)!(d-u-1)!}{(d+\ell-1)!(\ell-1)!} - \frac{(\ell-1)!}{(2\ell-1)!}. \\
\end{align*}

Now, combining the two cases and plugg into Eqn.\eqref{eqn:expression_(i)_closed_form}, we have 
{\small
\begin{align*}
(i) & = \frac{(-1)^{\ell - |S| }(\ell + |S| -1)!}{(|S|-1)! (\ell-|S|)! } 
\sum_{U \subseteq [d] \backslash S} \Delta_S(\f(U))
\sum_{ \substack{ W \supseteq U, W \subseteq [d]\backslash S \\ |W| > \ell - |S|}} (-1)^{|U|-|W|}
\frac{(|W \cup S|-1)! (|W|-1)!}{(|W \cup S| -\ell-1)! (|W \cup S| +\ell-1)! }  \\
& = \frac{(2\ell -1)!}{(\ell-1)! } 
\sum_{U \subseteq [d] \backslash S} \Delta_S(\f(U))
\sum_{ \substack{ W \supseteq U, W \subseteq [d]\backslash S \\ |W| > \ell - |S|}} (-1)^{|U|-|W|}
\frac{(|W \cup S|-1)! (|W|-1)!}{(|W \cup S| -\ell-1)! (|W \cup S| +\ell-1)! } \\
& \ \ \ (\text{using } s = |S| = \ell)\\
& = \frac{(2\ell -1)!}{((\ell-1)!)^2 } 
\sum_{U \subseteq [d] \backslash S}\frac{(\ell+u-1)!(d-u-1)!}{(d+\ell-1)!}   \Delta_S(\f(U)) - \Delta_S(\f(\text{\O})).
\end{align*}}
Therefore, by substituting $U$ with $T$, we get the desired result:
$$
\Expl_S^{\text{F-Shap}}(\f,\ell) = \frac{(2\ell -1)!}{((\ell-1)!)^2 } 
\sum_{T \subseteq [d] \backslash S}\frac{(\ell+|T|-1)!(d-|T|-1)!}{(d+\ell-1)!}   \Delta_S(\f(T)).
$$

\newpage

\subsubsection{Proof of Necessary Condition of Theorem \ref{thm:faith_shap}}
\label{sec:faithshap_necessity}

Now we prove the necessary condition of Theorem \ref{thm:faith_shap}. That is, Faith-Interaction indices $\Expl$ 
satisfy linearity, symmetry, efficiency and dummy axioms only if the weighting function $\mu$ has the form in Eqn.\eqref{eqn:faithshap_weight}.

From Proposition \ref{pro:symmetry}, the Faithful-Interaction indices satisfy the interaction symmetry axiom if and only if $\mu(S)$ only depends on the size of the input set $|S|$. Therefore, the weighting function must be symmetric.
Also, by Proposition \ref{pro:efficiency}, Faithful-Interaction indices satisfy the interaction efficiency axiom if and only if $\mu(\text{\O}) = \mu([d]) = \infty$. Therefore, the weighting function must be permutation-invariant and has infinity measure on the empty set and the full set.

Then, the following lemma show that the weighting function must be in the form of Eqn.\eqref{eqn:faithshap_weight} if the corresponding Faith-Interaction indices additionally satisfy dummy axiom.

\begin{lemma}
\label{lm:necessity_dummy_shapley}
Faith-Interaction indices $\Expl$ with a  weighting function, which is permutation-invariant and has $\mu(\text{\O}) = \mu([d]) = \infty$, 
satisfy the interaction dummy axiom only if the weighting function $\mu$ has the following form: 
$$
\mu(S) \propto
\frac{d-1}{\binom{d} {|S|}\,|S|\,(d-|S|)}
\text{   for all   } S \subseteq [d]
\text{   with   } 1 \leq |S| \leq d-1, 
\ \ \text{ and } \ \ 
\mu(\text{\O}) = \mu([d]) = \infty
$$
\end{lemma}

\begin{proof}

We now solve the case when the function $\f = \f_R$ is a basis function for some $R \subseteq [d]$. Recall that the definition of basis functions is :
$$
\f_R(S) = \begin{cases}
1, & \text{if}\ S \supseteq R \\
0, & \text{otherwise}. \\
\end{cases}
$$ 
Since the weighting function $\mu(\cdot)$ if finite and only depends on the size of the input set, we use the definitions in Section \ref{sec:extra_notations}: $\mu_{|S|} = \mu(S)$ and  $\mubar{|S|} = \sum_{T \supseteq S} \mu(T) = \sum_{i=|S|}^d { d-|S| \choose i-|S|} \mu_i$ for all $S \subseteq [d]$. Since  Faithful-Interaction indices should hold for all maximum interaction orders $1 \leq \ell \leq d$, we restrict the maximum interaction order to $\ell=1$. We now use the following results from \citet{ding2008formulas}.

\begin{proposition}
(\citet{ding2008formulas}, Theorem 16) When the maximum interaction order $\ell = 1$ (no interaction terms) and the set function $\f = \f_R$ is a basis function for some $R \subseteq [d]$ with $|R| = r$, and the weighting function $\mu(S)$ is permutation-invariant and has $\mu(\text{\O}) = \mu([d]) = \infty$, the minimizer of Eqn.\eqref{eqn:constrained_weighted_regresion}, $\Expl(\f_R, \ell) \in \mathbb{R}^{d+1}$, has the following form:
\begin{equation}
\label{eqn:ding_thm16}
\Expl_S(\f_R, \ell) = 
\begin{cases}
0 & \text{, if $S = \phi$.} \\
\frac{d-r}{d} \cdot  \frac{\mubar{r}- \mubar{r+1}}{\mubar{1} - \mubar{2}}   + \frac{1}{d}
& \text{, if $S = \{ i\}$ for $ i \in R$.} \\
\frac{-r}{d} \cdot \frac{\mubar{r}- \mubar{r+1}}{\mubar{1} - \mubar{2}}  + \frac{1}{d} & \text{, if $S = \{ i\}$ for $ i \notin R$.} \\
\end{cases}
\end{equation}
Note that when $R = [d]$, we let $\mubar{d+1} = 0$ for clarity.
\end{proposition}
Since the optimal solution $\Expl(\f_R,\ell)$ satisfies the symmetry, dummy and efficiency axiom, we should have 
\begin{equation}
\label{eqn:shapley_for_basis_function}
\Expl_{\{i\}}(\f_R, \ell) =  
\begin{cases}
\frac{1}{r} & \text{, if } i \in R. \\
0 & \text{, otherwise.}
\end{cases}
\ \ \text{, and } \ \ 
\Expl^*_{\text{\O}}(\f_R, \ell) =  0.    
\end{equation}

By comparing Eqn.\eqref{eqn:ding_thm16} and Eqn.\eqref{eqn:shapley_for_basis_function}, we obtain
$$
 \frac{\mubar{r}- \mubar{r+1}}{\mubar{1} - \mubar{2}} = \frac{1}{r} 
\ \ \ \text{ for }  1 \leq r \leq d-1.
$$
By letting $\mubar{1} - \mubar{2} = k$ and  plugging in the definition of $\mubar{r}$, i.e. $\mubar{r} = \sum_{i=r}^d { d-r \choose i-r} \mu_i$, we get 
$$
 \sum_{i=r}^{d-1} { d-1-r \choose i-r} \mu_i = \frac{k}{r} 
\ \ \ \text{ for }  1 \leq r \leq d-1.
$$
There are $d-1$ unknown parameters, $\mu_1,\cdots,\mu_{d-1}$, and $d-1$ equations. Hence, the solution is uniquely determined (in terms of $k$).
By solving the equation from $r=d-1$ to $r=1$, we obtain this unique solution:
$$
\mu_{|S|} = \mu(S) \propto
\frac{d-1}{\binom{d} {|S|}\,|S|\,(d-|S|)}
\text{   for all   } S \subseteq [d]
\text{   with   } 1 \leq |S| \leq d-1.
$$

\end{proof}

Now, by Lemma \ref{lm:necessity_dummy_shapley}, we conclude that the necessary condition of Theorem \ref{thm:faith_shap} holds. 

\newpage

\newpage

\section{proof of Claims}
\label{sec:proofs_claims}
In this sections, we give the omitted proof of claims that are used in the our proof in Appendix \ref{sec:proofs_propositions} and \ref{sec:proofs_theorems}.

\subsection{Proof of Claim \ref{clm:p_q_relation}}
\begin{proof}
For any $S \subseteq [d]$,
\begin{align*}
\sum_{T \supseteq S} (-1)^{|T| - |S|}p(T)
& = \sum_{T \supseteq S} (-1)^{|T| - |S|} \sum_{L \supseteq T} q(L) \\
& = \sum_{L \supseteq S} q(L)  \sum_{T: S \subseteq T \subseteq L} (-1)^{|T| - |S|} \\
& = \sum_{L \supseteq S} q(L)  \sum_{t=0}^{|L| - |S|} {|L| - |S| \choose t} (-1)^{t} \\
& = \sum_{L \supseteq S} q(L)  (1-1)^{|L|-|S|} \\
& = q(S) \\
\end{align*}
\end{proof}

\subsection{Proof of Claim \ref{clm:all_zero}}
\begin{proof}
Suppose that there exists a $\ex(\f,\ell) \in \mathbb{R}^{2^d-1}$ satisfying $p(L) = 0$ for all $L \subseteq [d]$.

By Claim \ref{clm:p_q_relation}, we have 
$q(S) = \sum_{T \supseteq S} (-1)^{|T| - |S|}p(T) = 0$ for all $S \subseteq [d]$. That is,
\begin{equation*}
0 = q(S) = \begin{cases}
\mu'(S) \left(1 - \sum_{T \subseteq S, |T| \leq d-1} \Expl_T(\f, \ell) \right)  & \text{ , if $S = [d]$} \\
\mu'(S) \left(\sum_{T \subseteq S, |T| \leq d-1} \Expl_T(\f, \ell) \right) & \text{, otherwise.} \\
\end{cases}
\end{equation*}
By the definition of $\mu'(\cdot)$, we have $\mu'(S) > 0$ for all $S \subseteq [d]$. We have 
\begin{align*}
& \sum_{T \subseteq S, |T| \leq d-1} \Expl_T(\f, \ell) =1 & \text{ , if $S = [d]$.} \\
& \sum_{T \subseteq S, |T| \leq d-1} \Expl_T(\f, \ell)  = 0 & \text{ , if $S \subset [d].$} 
\end{align*}
If we plug in $S = \phi$, we get $\Expl_{\phi}(\f,\ell) = 0$. Then we plug in $T = \{ i\}$ for some $1 \leq i \leq d$, we get $\Expl_{\phi}(\f,\ell) + \Expl_{\{i\}}(\f,\ell) = 0$, which implies $\Expl_{\{i\}}(\f,\ell) = 0$. 

Similarly, we obtain $\Expl_{T}(\f,\ell) = 0$ for all  $T \in \setlessell$ by simple induction. However, this $\Expl(\f,\ell)$ fails to satisfy the first equality:
$$
1 - \sum_{T \subseteq S, |T| \leq d-1} \Expl_T(\f, \ell)  = 1 \neq 0,
$$ 
which is a contradiction. Therefore, there is no $\ex(\f,\ell) \in \mathbb{R}^{2^d-1}$ satisfying $p(L) = 0$ for all $L \subseteq [d]$.

\end{proof}

\subsection{Proof of Claim \ref{clm:mubar_mu_relation}}
\begin{proof}
\begin{align*}
\sum_{T \supseteq S} (-1)^{|T| - |S|} \bar{\mu}(T) 
& = \sum_{T \supseteq S} (-1)^{|T| - |S|} \sum_{L \supseteq T} \mu(L)    \\
& = \sum_{L,T: L \supseteq T \supseteq S}  (-1)^{|T| - |S|} \mu(L) \\
& = \sum_{L: L \supseteq S} \mu(L) \sum_{T:  L \supseteq T \supseteq S} (-1)^{|T| - |S|}  \\
& = \sum_{L: L \supseteq S} \mu(L) \sum_{i=|S|}^{|L|} { |L| - |S| \choose i - |S|}(-1)^{i- |S|}  \\
& = \sum_{L: L \supseteq S} \mu(L) (1-1)^{|L| - |S|}  \\
& = \mu(S)\\
\end{align*}

\end{proof}

\subsection{Proof of Claim \ref{clm:mubar_closeform}}

\begin{proof}
Let $\prodd_t = \prod_{j=0}^{t-1} \frac{a(a-b) + j(b-a^2)}{a -b + j(b-a^2)}$. By definition, we have

\begin{align*}
\mubar{t} 
& = \sum_{k=t}^{d} {d-t \choose k-t } \mu_k \\
& = \sum_{k=t}^{d} {d-t \choose k-t } 
\sum_{i=k}^{d} {d-k \choose i-k}(-1)^{i-k} \prod_{j=0}^{j=i-1} \frac{a(a-b) + j(b-a^2)}{a-b + j(b-a^2)} \\
& = \sum_{k=t}^{d} {d-t \choose k-t } 
\sum_{i=k}^{d} {d-k \choose i-k}(-1)^{i-k} \prodd_i   \\
& = \sum_{k=t}^{d} \sum_{i=k}^{d} \prodd_i {d-t \choose k-t }  {d-k \choose i-k}(-1)^{i-k} \\
& = \sum_{i=t}^{d} \prodd_i  \sum_{k=t}^{i} (-1)^{i-k}{d-t \choose k-t }  {d-k \choose i-k} \\
& = \sum_{i=t}^{d} \prodd_i  \sum_{k=t}^{i} (-1)^{i-k} \frac{(d-t)!(d-k)!}{(k-t)!(d-k)!(i-k)!(d-i)!} \\
& = \sum_{i=t}^{d} \prodd_i \frac{(d-t)!}{(d-i)!} \sum_{k=t}^{i}(-1)^{i-k}  \frac{1}{(k-t)!(i-k)!} \\
& = \sum_{i=t}^{d} \prodd_i \frac{(d-t)!}{(d-i)!(i-t)!} \sum_{k=t}^{i} (-1)^{i-k} \frac{(i-t)!}{(k-t)!(i-k)!} \\
& = \sum_{i=t}^{d} \prodd_i {d-t \choose i-t} \sum_{k=t}^{i}(-1)^{i-k} {i-t \choose i-k} \\
& = \prodd_t  + \sum_{i=t+1}^{d} \prodd_i(-1)^{i-k} {d-t \choose i-t} (1 - 1)^{i-t} \\
& = \prodd_t \\
\end{align*}
Therefore, we have $\mubar{t} = \prodd_t =  \prod_{j=0}^{t-1} \frac{a(a-b) + j(b-a^2)}{a -b + j(b-a^2)}$ for all $1 \leq t \leq d$. Also, when $t=0$, since we assume that $\sum_{S \subseteq [d]} \mu(S) = 1$, we have $\mu_0 = 1$.

\end{proof}

\subsection{Proof of Claim \ref{clm:d_ri_positive}}

\begin{proof}
We prove $D^{p}_q > 0$ for all $p,q \in \{0,1,2,...,d\}$ with $0 \leq p+q \leq d$ by induction. 

(i) First, for all $p$ with $0 \leq p \leq d$, we have 
\begin{align*}
D^{p}_{d-p}
& = \sum_{j=0}^{d-p}{d-p \choose j} (-1)^{j}\mubar{p+j} 
\ \ \ \text{ (Definition \ref{def:Dpq}) } \\
& = \sum_{T: S \subseteq T \subseteq [d] }{d - |S| \choose |T| - |S|} (-1)^{|T|-|S|}\bar{\mu}(T) 
\ \ \ \text{ , for some } S \text{ with } |S| = p \\
& = \mu(S) \ \ \ (\text{Claim }  \ref{clm:mubar_mu_relation} ) > 0.
\end{align*}
(ii) Also, we have $D^p_0 = \mubar{p} = \sum_{j=0}^{d-p} {d-p \choose j } \mu_{j+p} > 0$.

Then, for all $p,q \in \{0,1,2,...,d-1\}$ with $0 \leq p+q \leq d$, we have

\begin{align*}
D^{p}_{q}
& = \sum_{j=0}^{q}{q \choose j} (-1)^{j}\mubar{p+j} 
\ \ \ \text{ (Definition \ref{def:Dpq}) } \\
& = \mubar{p} + \mubar{p+q} (-1)^{p+q} + \sum_{j=1}^{q-1} \left({q-1 \choose j} + {q-1 \choose j-1} \right) (-1)^{j}\mubar{p+j} 
\ \ \ \text{ (Pascal's rule) } \\
& =  \mubar{p} + \mubar{p+q}(-1)^{p+q} + \sum_{j=1}^{q-1} \left( {q-1 \choose j}  (-1)^{j}\mubar{p+j} + {q-1 \choose j-1}  (-1)^{j}\mubar{p+j} \right) \\
& =  \mubar{p} + \mubar{p+q} (-1)^{p+q} + \sum_{j=1}^{q-1}  {q-1 \choose j}  (-1)^{j}\mubar{p+j} + \sum_{j=1}^{q-1}{q-1 \choose j-1}  (-1)^{j}\mubar{p+j}  \\
& = \sum_{j=0}^{q-1}  {q-1 \choose j}  (-1)^{j}\mubar{p+j} - \sum_{j=0}^{q-1}{q-1 \choose j}  (-1)^{j}\mubar{p+1+j} \\
& = D^{p}_{q-1} - D^{p+1}_{q}.
\end{align*}
Therefore, we have $D^{p}_{q} + D^{p+1}_{q} =  D^{p}_{q-1}$ (iii). Now, we use this formula to prove $D^{p}_{q}$ is positive by using induction. 

\begin{itemize}
    \item When $p=d$, we have $D^p_0 > 0$ by (ii).
    \item Assume that when $p = i$, $D^p_q > 0$ for all $0 \leq  q \leq d-p$. 
    \item When $p = i-1$, we now prove that $D^p_q > 0$ for all $0 \leq  q \leq d-p$. By (i) and (ii),
    we have $D^p_0 > 0 $ and $D^{p}_{d-p} > 0$. Then by (iii), we have $D^{p}_{d-p} + D^{p+1}_{d-p} =  D^{p}_{d-p-1}$, since we have already had $D^{p}_{d-p} > 0$ and $ D^{p+1}_{d-p} > 0$ by induction hypothesis, we have $D^{p}_{d-p-1} > 0$. Then by calling (iii) recursively from $q= d-p$ to $q=1$, we conclude that $D^p_q > 0$ for all $0 \leq  q \leq d-p$, which also conclude the induction proof.
\end{itemize}

\end{proof}

\subsection{Proof of Claim \ref{clm:row_operation}}
\begin{proof}
First of all, we develop the following equality.

\begin{claim}
\label{clm:I_value}
For any integers $\ell,r, t' p', \rho$ with constraints $1 \leq \ell < r $, $0 \leq t' \leq \ell$, $ 0 \leq p' \leq \ell$ and $0 \leq \rho \leq t'$, we have
\begin{equation}
\sum_{\sigma=0}^{t'-\rho} {t' \choose \sigma} (-1)^{\sigma} {t'-\sigma \choose \rho} { r-t'+\sigma \choose p'- \rho} = \begin{cases}
0 & \text{ if } p' < t' \\
(-1)^{t'-\rho}{r-t' \choose p' - t' } {t' \choose \rho} & \text{ if }  t' \leq p' \leq \ell \\
\end{cases}
\end{equation}
\end{claim}
The proof is delayed to Section \ref{sec:proof_claim_I_value}
By using the definition of $M'$ in Eqn.\eqref{eqn:m_definition} , we have 
\begin{align*}
\sum_{\sigma=0}^{t'} {t' \choose \sigma} (-1)^{\sigma} M'_{\{s'+t', s'+\sigma \},\{p',0\}} 
& = 
\sum_{\sigma=0}^{t'} {t' \choose \sigma} (-1)^{\sigma} \sum_{\rho=0}^{t'-\sigma} {t'-\sigma \choose \rho} { r-t'+\sigma \choose p'- \rho} \mubar{s'+t'+p'-\rho}  \\
& = 
\sum_{\sigma=0}^{t'}
\sum_{\rho=0}^{t'-\sigma} {t' \choose \sigma} (-1)^{\sigma}  {t'-\sigma \choose \rho} { r-t'+\sigma \choose p'- \rho} \mubar{s'+t'+p'-\rho} 
 \\
& = 
\sum_{\rho=0}^{t'}  \sum_{\sigma=0}^{t'-\rho} {t' \choose \sigma} (-1)^{\sigma} {t'-\sigma \choose \rho} { r-t'+\sigma \choose p'- \rho} \mubar{s'+t'+p'-\rho} 
 \\ 
& = 
\sum_{\rho=0}^{t'}  \mubar{s'+t'+p'-\rho}  \sum_{\sigma=0}^{t'-\rho} {t' \choose \sigma} (-1)^{\sigma} {t'-\sigma \choose \rho} { r-t'+\sigma \choose p'- \rho} 
\\ 
\end{align*}

Then, by plugging in the results of Claim. \ref{clm:I_value}, we have 
\begin{align*}
\sum_{\sigma=0}^{t'} {t' \choose \sigma} (-1)^{\sigma} M'_{\{s'+t', s'+\sigma \},\{p',0\}} 
& = 
\sum_{\rho=0}^{t'}  \mubar{s'+t'+p'-\rho}  \sum_{\sigma=0}^{t'-\rho} {t' \choose \sigma} (-1)^{\sigma} {t'-\sigma \choose \rho} { r-t'+\sigma \choose p'- \rho}  \\
& = \begin{cases}
0 & \text{ if } p' < t' \\
{r-t' \choose p'-t'} \sum_{\rho=0}^{t'}  (-1)^{t'-\rho} {t' \choose \rho} \mubar{s'+t'+p'-\rho} 
  & \text{ if }  t' \leq p' \leq \ell \\
\end{cases} \\
& = \begin{cases}
0 & \text{ if } p' < t' \\
{r-t' \choose p'-t'} D_{t'}^{s'+p'}
  & \text{ if }  t' \leq p' \leq \ell  \text{   ( Definition \ref{def:Dpq})}\\
\end{cases}
\end{align*}

\end{proof}

\subsection{Proof of Claim \ref{clm:row_decomp}}

\begin{proof}
By Definition \ref{def:row_combination} and Pascal's rule, we have
\begin{align*}
\Re_{s+t,s}^{t} 
& = \sum_{\sigma=0}^{t} {t \choose \sigma} (-1)^{\sigma} Q_{\{s+t,s+\sigma \}} \\
& = \sum_{\sigma=0}^{t} [{t-1 \choose \sigma-1} + {t-1 \choose \sigma}](-1)^{\sigma} Q_{\{s+t,s+\sigma \}} \\
& = \sum_{\sigma=1}^{t} {t-1 \choose \sigma-1} (-1)^{\sigma} Q_{\{s+t,s+\sigma \}} 
+  \sum_{\sigma=0}^{t-1} {t-1 \choose \sigma}(-1)^{\sigma} Q_{\{s+t,s+\sigma \}}  \\
& = -\sum_{\sigma=0}^{t-1} {t-1 \choose \sigma} (-1)^{\sigma} Q_{\{s+t,s+\sigma+1 \}} +  \sum_{\sigma=0}^{t-1} {t-1 \choose \sigma} (-1)^{\sigma} Q_{\{s+t,s+\sigma \}} \\
& = -\Re_{s+t,s+1}^{t-1} +\Re_{s+t,s}^{t-1}.   \\
\end{align*}
By rearanging, we have $\Re_{s+t,s}^{t-1} = \Re_{s+t,s}^{t} +  \Re_{s+t,s+1}^{t-1}$.

\end{proof}

\subsection{Proof of Claim \ref{clm:d_linear_relation}}
\label{sec:clm_d_linear_relation}
\begin{proof}
We prove a stronger version of the claim:
$$ \frac{D^{p}_{q}}{D^{p}_{q+1}} = c^{(1)}_q p +  c^{(2)}_q \text{ and } \frac{D^{p+1}_{q}}{D^{p}_{q+1}} =  c^{(1)}_q  p +  c^{(3)}_q $$ 
for some constants $ c^{(1)}_q , c^{(2)}_q,  c^{(3)}_q \in \mathbb{R}$ dependent on $q$. We prove it by induction.
\begin{enumerate}
    \item For $q = 0$, 
    \begin{equation}
    \label{eqn:AB0}
    \frac{D^p_0}{D^p_1} 
    = \frac{\mubar{p}}{ \mubar{p}-\mubar{p+1}}
    = \frac{1}{1-\mubar{p+1}/\mubar{p}} 
    = \frac{1}{1- \frac{a(a-b) + p(b-a^2)}{a -b + t(b-a^2)}}
    = \frac{a-b + p(b-a^2) }{(1-a)(a-b)} 
    = c^{(1)}_0 p + c^{(2)}_0
    \end{equation}
    \begin{equation}
    \label{eqn:AC0}
    \frac{D^{p+1}_0}{D^p_1} 
    = \frac{\mubar{p+1}}{ \mubar{p}-\mubar{p+1}}
    = \frac{1}{\mubar{p}/\mubar{p+1}-1} 
    = \frac{1}{\frac{a-b + p(b-a^2)}{a(a-b) + p(b-a^2)} - 1}
    = \frac{a(a-b) + p(b-a^2) }{(1-a)(a-b)} 
    = c^{(1)}_0 p + c^{(3)}_0 
    \end{equation}

    \item When $q = k$, suppose we have 
    $$\frac{D^{p}_{k}}{D^{p}_{k+1}} = c^{(1)}_k p + c^{(2)}_k, \text{ and } \frac{D^{p+1}_{k}}{D^{p}_{k+1}} = c^{(1)}_k p + c^{(3)}_q$$ for some constants $c^{(1)}_k, c^{(2)}_k, c^{(3)}_k \in \mathbb{R}$.

    \item When $q = k+1$, by the assumption in step 2, we have 
    $$ \frac{D^{p+1}_{k+1}} {D^p_{k+1}}
    = \frac{D^{p+1}_{k}/  D^{t}_{k+1}} {D^{p+1}_{k}/D^{p+1}_{k+1}}
    = \frac{c^{(1)}_k p + c^{(3)}_k}{c^{(1)}_k p + c^{(1)}_k + c^{(2)}_k}.
    $$
    Then, 
    
    \begin{align}
    \label{eqn:a_k_relation}
     \frac{D^p_{k+1}}{D^p_{k+2}} 
    & = \frac{D^p_{k+1}}{D^p_{k+1} - D^{p+1}_{k+1}} 
    =  \frac{1}{1 - D^{p+1}_{k+1}/ D^p_{k+1}}
    = \frac{c^{(1)}_k p + c^{(1)}_k + c^{(2)}_k}{c^{(1)}_kt + c^{(1)}_k + c^{(2)}_k - c^{(1)}_k t - c^{(3)}_k} \nonumber \\
    & = \frac{c^{(1)}_k p}{c^{(1)}_k + c^{(2)}_k - c^{(3)}_k} + \frac{c^{(1)}_k + c^{(2)}_k}{c^{(1)}_k + c^{(2)}_k - c^{(3)}_k}  
    \end{align}

    Also, 
    \begin{equation}
    \label{eqn:c_k_relation}
    \small
    \frac{D^{p+1}_{k+1}}{D^p_{k+2}} 
    = \frac{D^{p+1}_{k+1}}{D^p_{k+1} - D^{p+1}_{k+1}}
    = \frac{1}{D^p_{k+1} /D^{p+1}_{k+1} - 1}
    = \frac{c^{(1)}_k p + c^{(3)}_k}{c^{(1)}_k + c^{(2)}_k - c^{(3)}_k}
    = \frac{c^{(1)}_k p }{c^{(1)}_k + c^{(2)}_k - c^{(3)}_k} + \frac{c^{(3)}_k}{c^{(1)}_k + c^{(2)}_k - c^{(3)}_k}
    \end{equation}
    The first equality in equations \eqref{eqn:a_k_relation} and \eqref{eqn:c_k_relation} is due to 
    \begin{align}
    D^{p}_{k+2} 
    & = \sum_{i=0}^{k+2} {k+2 \choose i }(-1)^i \mubar{p+i}  
    \nonumber \\
    & =  \mubar{p} + \sum_{i=1}^{k+2} ( {k+1 \choose i-1} + {k+1 \choose i})(-1)^i \mubar{p+i} 
     \nonumber\\
    & = \sum_{i=0}^{k+1} {k \choose i }(-1)^i \mubar{p+i} - \sum_{i=0}^{k+1} {k+1 \choose i }(-1)^i \mubar{p+i+1}
    \nonumber \\
    & =  D^{p}_{k+1} - D^{p+1}_{k+1} \nonumber \\
    \label{eqn:D_recursive_relation}
    \end{align}
    
    Therefore, by Eqn. \eqref{eqn:a_k_relation}, \eqref{eqn:c_k_relation}, we can get 
    \begin{equation}
    \label{eqn:abc_k_relation}
    c^{(1)}_{k+1} = \frac{c^{(1)}_k}{c^{(1)}_k + c^{(2)}_k - c^{(3)}_k},
    \ \ c^{(2)}_{k+1} = \frac{c^{(1)}_k + c^{(2)}_k}{c^{(1)}_k + c^{(2)}_k - c^{(3)}_k},
    \ \ c^{(3)}_{k+1}=  \frac{c^{(3)}_k}{c^{(1)}_k + c^{(2)}_k - c^{(3)}_k}
    \end{equation}
    So we can conclude that these two ratio have affine relationship with $p$.

\end{enumerate}

\end{proof}

\subsection{Proof of Claim \ref{clm:dpq_value_thm2}}
We first prove that $D^p_1 = \mubar{p} - \mubar{p+1} \propto \frac{d-1}{pd}$ for $1 \leq p \leq d-1$.

\begin{align*}
\mubar{p}
& = \sum_{i=0}^{d-p-1} { d-p \choose i } \mu_{i+p} \ \ \ ( \text{ Definition } \ref{def:cumulative_weighting})  \\
& \propto 
\sum_{i=0}^{d-p-1} { d-p \choose i } \frac{d-1}{{d \choose i+p}(i+p)(d-i-p) } \ \ \ ( \mu_d=0)\\
& =
(d-1)\sum_{i=0}^{d-p-1}  \frac{ (d-p)!(i+p-1)!(d-i-p-1)!}{i!(d-p-i)!d!  } \\
& = 
\frac{(d-1)(d-p)!}{d!} \sum_{i=0}^{d-p-1}  \frac{(i+p-1)!}{i!(d-p-i)} \\
& = 
\frac{(d-1)(d-p)!}{d!} \sum_{i=0}^{d-p-1}  \frac{\prod_{j=1}^{p-1}(i+j) }{d-p-i} \\
\end{align*}
Then, we have 
\begin{align*}
D^p_1 
& = \mubar{p} - \mubar{p+1} \\
& \propto 
\frac{(d-1)(d-p)!}{d!} \sum_{i=0}^{d-p-1}  \frac{\prod_{j=1}^{p-1}(i+j) }{d-p-i}
- \frac{(d-1)(d-p-1)!}{d!} \sum_{i=0}^{d-p-2}  \frac{\prod_{j=1}^{p}(i+j) }{d-p-i-1}
\\
& =
\frac{(d-1)(d-p-1)!}{d!}
\left[
(p-1)! +
\sum_{i=0}^{d-p-2}  \left[ \frac{(d-p)\prod_{j=1}^{p-1}(i+j+1) }{d-p-i-1} 
- \frac{\prod_{j=1}^{p}(i+j) }{d-p-i-1}
\right]\right] \\
& =
\frac{(d-1)(d-p-1)!}{d!}
\left[
(p-1)! +
\sum_{i=0}^{d-p-2} \frac{\prod_{j=1}^{p-1}(i+j+1)}{d-p-i-1} \left[ d-p-i-1
\right]\right] \\
& =
\frac{(d-1)(d-p-1)!(p-1)!}{d!}
\left[
1 +
\sum_{i=0}^{d-p-2} \ \  \frac{\prod_{j=1}^{p-1}(i+j+1)}{(p-1)!}
\right] \\
& =
\frac{(d-1)(d-p-1)!(p-1)!}{d!}
\sum_{i=p-1}^{d-2}  
{ i \choose p-1 }\\
& = 
\frac{(d-1)(d-p-1)!(p-1)! {d-1 \choose p }}{d!} \ \ \ (*)\\
& = 
\frac{(d-1)(d-p-1)!(p-1)!(d-1)!}{d!p!(d-p-1)!} \\
& = 
\frac{d-1}{dp} \\
\end{align*}

where $(*)$ follows from using Pascal's rule for multiple times:
\begin{align*}
{d-1 \choose p} 
& = {d-2 \choose p-1} + {d-2 \choose p}   \\
& = {d-2 \choose p-1} + {d-3 \choose p-1} +  {d-3 \choose p} \\
& = {d-2 \choose p-1} + {d-3 \choose p-1} + \cdots +  {p \choose p-1} + {p \choose p} \\
& = \sum_{i=p-1}^{d-2} {i \choose p-1} \\
\end{align*}

Then we prove the following equation by induction on $q$. We have proved the basic case when $q=1$. Let assume that it holds for $q-1$, such that for all $0 \leq p \leq d-q$, 
\begin{align*}
D^p_{q-1}
= \sum_{j=0}^{q-1} 
{ q-1 \choose j}(-1)^j \mubar{p+j}
\propto 
\frac{d-1}{d(q-1) { p+q-2 \choose p-1 } }.
\end{align*}

Then, we have
\begin{align*}
D^p_{q}
& = D^p_{q-1} - D^{p+1}_{q-1} 
\ \ \ \ \ (\text{Eqn. \eqref{eqn:D_recursive_relation}}) \\
& \propto 
\frac{d-1}{d} \times \frac{1}{(q-1) { p+q-2 \choose p-1 } } - \frac{1}{(q-1) { p+q-1 \choose p } } \\
& = 
\frac{d-1}{d} \times
\frac{1}{(q-1) { p+q-2 \choose p-1 } }
\left( 1 - \frac{p}{ p+q-1 } \right)\\
& = 
\frac{d-1}{d} \times
\frac{1}{{ p+q-2 \choose p-1 }( p+q-1) } \\
& = 
\frac{d-1}{d} \times
\frac{(p-1)!(q-1)!}{(p+q-1)! } \\
& = 
\frac{d-1}{dq {p+q-1 \choose p-1}} \\
\end{align*}
Therefore, we complete the induction proof.

\subsection{Proof of Claim \ref{clm:d_linear_relation_thm2}}
\begin{proof}
By Claim \ref{clm:dpq_value_thm2}, we have
\begin{align*}
\frac{D^{p}_{q}}{D^{p}_{q+1}}  
& = \frac{\frac{1}{q { p+q-1 \choose p-1 } }}{\frac{1}{(q+1) { p+q \choose p-1 } }} 
= \frac{q+1}{q} \times \frac{p+q}{q+1}
= \frac{p}{q} + 1 \\
\end{align*}
The constants are $ c^{(1)}_q = \frac{1}{q}  \text{  and  } c^{(2)}_q  = 1$.

\end{proof}

\subsection{Proof of Claim \ref{clm:closed_form_basis_function}}

First of all, by Claim \ref{clm:dpq_value_thm2}, we plug in $D^p_q 
= \frac{d-1}{dq { p+q-1 \choose p-1 } }$ to the equation.

\begin{proof}

\begin{align}
& \sum_{j=0}^{\ell - i} 
{ r- i \choose j} \frac{D_{i}^{j+i}}{D^{r}_{i}} 
\frac{(r-1)!(\ell + i+j -1)!}{(r+\ell-1)!(j+i-1)!} { r - j-i - 1 \choose \ell - i-j} (-1)^{\ell - i - j}  \nonumber\\
& = \frac{(r+i-1)!(r-i)!}{(r+\ell-1)!(r-\ell-1)!}\sum_{j=0}^{\ell - i} 
\frac{(\ell + i+ j -1)!}{(r-i-j)j!(j+2i-1)!(\ell-i-j)!}(-1)^{\ell - i - j}   \nonumber\\
& = 
\frac{(r+i-1)!(r-i)!}{(r+\ell-1)!(r-\ell-1)!}\sum_{j=0}^{\ell - i} 
\frac{(\ell-i)!(-1)^{\ell-i-j}}{(r-i-j)j!(\ell-i-j)!}\cdot 
\frac{(\ell + i+j -1)!}{(j+2i-1)!(\ell-i)!}  \nonumber\\
& = 
\frac{(r+i-1)!(r-i)!}{(r+\ell-1)!(r-\ell-1)!}
\underbrace{\sum_{j=0}^{\ell - i} 
 \frac{(-1)^{\ell-i-j}{\ell -i \choose \ell -i - j}}{r-i-j} \cdot
{\ell+i+j-1 \choose \ell-i} }_{(i)}
\label{eqn:(i)_substitute}
\end{align}
Now we use generating function to prove the above equation equals 1. 

First, we look at the coefficients of the generating function $p_1(x) = \int x^{r-\ell-1}(1-x)^{\ell-i} dx$. 
By using the fact that the 
coefficient of $x^{r-i-n-1}$ in the polynomial $x^{r-\ell-1}(1-x)^{\ell-i} $ is
$(-1)^{\ell-i-n}{\ell -i \choose \ell-i-n}$ for some $0 \leq n \leq \ell-i$, we have 
$$
p_1(x) = \int x^{r-\ell-1}(1-x)^{\ell-i} dx = \sum_{n=0}^{\ell-i} \frac{(-1)^{\ell-i-n}{\ell-i \choose \ell-i-n}}{r-i-n} x^{r-i-n},
$$
where we set the constant term in the integration is zero.

Secondly, we consider another polynomial \footnote{Generally, for any non-negative integer $k$ and non-zero real value $a$, we have $\sum_{n=0}^{\infty}a^n {n+k \choose k} x^n = \frac{1}{(1-ax)^{k+1}}$.} 
$$p_2(x) =\frac{1}{(1-x)^{\ell-i+1}}
= \sum_{m=0}^{\infty}{ m + \ell -i \choose \ell-i}x^m.$$ 

Finally, we deal with the coefficient of $x^{r+i-1}$ term of the polynomial $p_1(x)p_2(x)$. By combining the above two equalities, this coefficient is 
\begin{align*}
& \sum_{n=0}^{\ell-i} \left( \frac{(-1)^{\ell-i-n}{\ell-i \choose \ell-i-n}}{r-i-n} x^{r-i-n}
 \cdot {(2i+n-1)+\ell-i \choose \ell-i }x^{2i+n-1} \right) \\
& = x^{r+i-1} \underbrace{\sum_{n=0}^{\ell-i} \frac{(-1)^{\ell-i-n}{\ell-i \choose \ell-i-n}}{r-i-n}
\cdot {\ell+i+n-1 \choose \ell-i } }_{(ii)}. 
\end{align*}

Now, if we compare (i) and (ii), we can see that their values are the same. Therefore, (i) is actually the coefficient of $x^{r+i-1}$ term of the polynomial $p_1(x)p_2(x)$.

We now analyze the polynomial $p_1(x) = \int x^{r-\ell-1}(1-x)^{\ell-i} dx$.
\begin{claim}
\label{clm:int_closed_form}
Assume that the constant term in the polynomial $p_1(x) = \int x^{r-\ell-1}(1-x)^{\ell-i} dx$ is zero so that $p_1(0)=0$. Then we have 
$$
p_1(x) = \int x^{r-\ell-1}(1-x)^{\ell-i} dx 
= (1-x)^{\ell-i+1}p(x) + \frac{(r-\ell-1)! }{\prod_{k=0}^{r-\ell-1} (\ell-i+k+1)},
$$
where $p(x)$ is a polynomial with degree at most $\leq r-\ell-1$.
\end{claim}

We delay the proof to Section \ref{sec:proof_int_closed_form}. Now, we are able to calculate the polynomial $p_1(x)p_2(x)$:
\begin{align*}
p_1(x)p_2(x) 
& = \frac{\int x^{r-\ell-1}(1-x)^{\ell-i} dx }{(1-x)^{\ell-i+1}} \\
& = p(x) + \frac{(r-\ell-1)! }{\prod_{k=0}^{r-\ell-1} (\ell-i+k+1)} \cdot \frac{1}{{(1-x)^{\ell-i+1}}}
\ \  \text{(Claim \ref{clm:int_closed_form})} \\
& =  p(x) + \frac{(r-\ell-1)! }{\prod_{k=0}^{r-\ell-1} (\ell-i+k+1)}  \sum_{m=0}^{\infty}{ m + \ell -i \choose \ell-i}x^m. \\
\end{align*}
We note that $p(x)$ is a polynomial with degree $\leq r-\ell-1$. Recall that our goal is to calculate the coefficient of $x^{r+i-1}$ term of the polynomial $p_1(x)p_2(x)$, so $p(x)$ has nothing to do with it. Therefore, we have 
\begin{align*}
(i) & = \frac{(r-\ell-1)! }{\prod_{k=0}^{r-\ell-1} (\ell-i+k+1)} { (r+i-1) + \ell -i \choose \ell-i} \\
& = \frac{(r-\ell-1)!(\ell-i)! }{(r-i)!} { r+\ell-1 \choose \ell-i} \\
& = \frac{(r-\ell-1)!(r+\ell-1)! }{(r-i)!(r+i-1)!}. \\
\end{align*}

By substituting the value of (i) to Eqn.\eqref{eqn:(i)_substitute}, we get
$$
\sum_{j=0}^{\ell - i} 
{ r- i \choose j} \frac{D_{i}^{j+i}}{D^{r}_{i}} 
\frac{(r-1)!(\ell + j+i -1)!}{(r+\ell-1)!(j+i-1)!} { r - j-i - 1 \choose \ell - i-j} (-1)^{\ell - i - j} 
= 1.
$$

\end{proof}

\subsection{Proof of Claim \ref{clm:I_value}}
\label{sec:proof_claim_I_value}
\begin{proof}
Before we prove Claim \ref{clm:I_value}, we first derive the following equality.
\begin{claim}
\label{clm:rho_poly}
For any $x',\sigma \in \mathbb{R}$ and $y' \in \mathbb{N} + \{0\}$, we have
$$ 
\prod_{j=0}^{y'} (x'+\sigma-j) =
C_0 +
\sum_{j=0}^{y'} C_{j+1} \sigma(\sigma-1)...(\sigma-j) 
\text{ with } 
C_i = \frac{1}{i!}\prod_{j=0}^{i-1} (y'+1-j) \prod_{k=0}^{y'-i}(x'-k)
$$
\end{claim}

\begin{proof}
Let
$$
g(\sigma) = \prod_{j=0}^{y'} (x'+\sigma-j)  
$$
We note that $g(\sigma)$ is a $(y'+1)^{\text{th}}$-polynomial in terms of $\sigma$. Since $\{\sigma(\sigma-1)...(\sigma-j)  \}_{j=0}^{y'}\cup \{ 1\}$ is a basis for this polynomial, $g(\sigma)$ can be expressed as $C_{(0)}(x') +
\sum_{j=0}^{y'} C_{(j+1)}(x') \sigma(\sigma-1)...(\sigma-j) $ with some functions $C_{(0)}(x'),...,C_{(y'+1)}(x')$ of $x'$. Now we proof that 

$$
C_{(i)}(x') =  \frac{1}{i!}\prod_{j=0}^{i-1} (y'+1-j) \prod_{k=0}^{y'-i}(x'-k) = C_i
$$
by induction.
At each induction step, we prove that 
$ C_i = C_{(i)}(x') = \frac{1}{i!}( \sum_{j=0}^{i} (-1)^{i-j}{ i \choose j }g(j) )
=  \frac{1}{i!}\prod_{j=0}^{i-1} (y'+1-j) \prod_{k=0}^{y'-i}(x'+i+k)$.
\begin{enumerate}
    \item When $i=0$, by plugging in $\sigma = 0$ in $g(\sigma)$, we have
    $$
    C_{(0)}(x')  = g(0) 
    = \prod_{j=0}^{y'} (x'-j)
    = \prod_{k=0}^{y'} (x'-k) 
    = C_0.
    $$ 
    
    \item When $i=0,1,...,i'$ for some $i'>0$, assume that we have 
    $$
    C_{i}
    = C_{(i)}(x')
    = \frac{1}{i!} \left( \sum_{j=0}^{i} (-1)^{i-j}{ i \choose j }g(j) \right)
    = \frac{1}{i!}\prod_{j=0}^{i-1} (y'+1-j) \prod_{k=0}^{y'-i}(x'-k).
    $$
    
    \item When $i=i'+1$, by plugging in $\sigma = i'+1$, we have
    $$
    g(i'+1) = \prod_{j=0}^{y'} (x'+i'+1-j) 
    = C_{(0)}(x') + \sum_{j=0}^{i'} C_{(j+1)}(x') (i'+1)(i')...(i'+1-j) 
    $$
    Now we can express $C_{(i'+1)}(x')$ as $C_{(0)}(x'),..., C_{(i')}(x')$ and $g(i'+1)$. First, we proof $C_{(i'+1)}(x') =  \frac{1}{(i'+1)!}\sum_{k=0}^{i'+1} 
    { i'+1 \choose k}
     (-1)^{i'+1-k}  g(k) $.
    \begin{align*}
    & C_{(i'+1)}(x') \\
    & = \frac{1}{(i'+1)!}\left(g(i'+1) - C_{(0)}(x') - \sum_{j=0}^{i'-1}  (i'+1)(i')...(i'+1-j) C_{(j+1)}(x')\right) \\
    & = \frac{1}{(i'+1)!}\left(g(i'+1) - g(0) - \sum_{j=0}^{i'-1}(i'+1)(i')...(i'+1-j)  \frac{1}{(j+1)!} \left[ \sum_{k=0}^{j+1} (-1)^{j+1-k}{ j+1 \choose k }g(k) \right] \right) \\
    & = \frac{1}{(i'+1)!}\left(g(i'+1) - g(0) - \sum_{j=0}^{i'-1}
    { i'+1 \choose j+1 }
    \left[ \sum_{k=0}^{j+1} (-1)^{j+1-k}{ j+1 \choose k }g(k) \right] \right) \\
    & = \frac{1}{(i'+1)!}\left(g(i'+1)  - \sum_{j=0}^{i'}
    { i'+1 \choose j }
    \left[ \sum_{k=0}^{j} (-1)^{j-k}{ j \choose k }g(k) \right] \right) \\
    & = \frac{1}{(i'+1)!}\left(g(i'+1)  - \sum_{k=0}^{i'}  \sum_{j=k}^{i'}
    { i'+1 \choose j }
    (-1)^{j-k}{ j \choose k }g(k)  \right) \\
    & = \frac{1}{(i'+1)!}\left(g(i'+1)  - \sum_{k=0}^{i'}  g(k) \sum_{j=k}^{i'} 
    \frac{(-1)^{j-k}(i'+1)!j!}{j!(i'+1-j)!k!(j-k)!}\right) \\
    & = \frac{1}{(i'+1)!}\left(g(i'+1)  - \sum_{k=0}^{i'}  g(k)
    \frac{(i'+1)!}{k!(i'-k+1)!}
    \sum_{j=0}^{i'-k} 
    \frac{(-1)^{j}(i'-k+1)!}{(i'-j+k+1)!(j)!}\right) \\
    & = \frac{1}{(i'+1)!}\left(g(i'+1)  - \sum_{k=0}^{i'}  g(k)
    { i'+1 \choose k}
    ( (-1)^{i'-k} + \sum_{j=0}^{i'-k+1} 
    \frac{(-1)^{j}(i'-k+1)!}{(i'-j+k+1)!(j)!} )\right) \\
    & = \frac{1}{(i'+1)!}\left(g(i'+1)  - \sum_{k=0}^{i'}  g(k)
    { i'+1 \choose k}
    ( (-1)^{i'-k} + (1-1)^{i'-k+1} )\right) \\
     & = \frac{1}{(i'+1)!}\sum_{k=0}^{i'+1}  g(k)
    { i'+1 \choose k}
     (-1)^{i'+1-k} \\
    \end{align*}
    Secondly, we prove that $C_{(i'+1)}(x')
    =  \frac{1}{(i'+1)!}\prod_{j=0}^{i'} (y'+1-j) \prod_{k=0}^{y'-i}(x'+i'+1+k) 
    = C_{i'+1}$.
    
    \begin{align*}
    C_{(i'+1)}(x')
    & = \frac{1}{(i'+1)!}( \sum_{j=0}^{i'+1} (-1)^{i'+1-j}{ i'+1 \choose j }g(j) ) \\
    & = \frac{1}{(i'+1)!}
    \left(  \sum_{j=0}^{i'} (-1)^{i'-j}{ i' \choose j } g(j+1)-\sum_{j=0}^{i'} (-1)^{i'-j}{ i' \choose j }g(j) \right) 
    \ \ \ (\text{Pascal's rule} )\\
    & = \frac{1}{(i'+1)!}
    \left(  \sum_{j=0}^{i'} (-1)^{i'-j}{ i' \choose j } g(j+1)
    -\sum_{j=0}^{i'} (-1)^{i'-j}{ i' \choose j }g(j) \right) \\
    & = \frac{1}{(i'+1)!}
    (i'!C_{(i')}(x'+1) - i'!C_{(i')}(x') ) \\
    & =  \frac{1}{(i'+1)!}(\prod_{j=0}^{i'-1} (y'+1-j) \prod_{k=0}^{y'-i'}(x'+1-k) - \prod_{j=0}^{i'-1} (y'+1-j) \prod_{k=0}^{y'-i'}(x'-k) ) \\
    & =  \frac{\prod_{j=0}^{i'-1} (y'+1-j) \prod_{k=0}^{y'-i'-1}(x'-k)}{(i'+1)!}( (x'+1) -  (x'-y'+i') ) \\
    & =  \frac{\prod_{j=0}^{i'} (y'+1-j) \prod_{k=0}^{y'-i'-1}(x'-k)}{(i'+1)!}\\
    & = C_{i'+1} \\
    \end{align*}
\end{enumerate}
Then we complete the induction proof of Claim \ref{clm:rho_poly}.
\end{proof}

Now we come back to the proof of Claim \ref{clm:I_value}.

Since we may have $r-t'+\sigma < t'=\rho$, we expand ${ r-t'+\sigma \choose p'- \rho}$ with ${ r-t'+\sigma \choose p'- \rho} = \frac{ \prod_{j=0}^{p'-\rho-1} (r-t'+\sigma-j)}{(p'-\rho)!}$ to avoid the denominator being zero.
\begin{align*}
& \sum_{\sigma=0}^{t'-\rho} {t' \choose \sigma} (-1)^{\sigma} {t'-\sigma \choose \rho} { r-t'+\sigma \choose p'- \rho} \\
& = \sum_{\sigma=0}^{t'-\rho} 
\frac{(-1)^{\sigma}t'!(t'-\sigma)!}
{\sigma!(t'-\sigma)!\rho!(t'-\sigma-\rho)!} 
\frac{ \prod_{j=0}^{p'-\rho-1} (r-t'+\sigma-j)}{(p'-\rho)!}
\\
& = \frac{t'!}{\rho!(p'-\rho)!}
\sum_{\sigma=0}^{t'-\rho} 
\frac{(-1)^{\sigma} \prod_{j=0}^{p'-\rho-1} (r-t'+\sigma-j)}
{\sigma!(t'-\sigma-\rho)!}
\\
& = \frac{t'!}{\rho!(p'-\rho)!}
\sum_{\sigma=0}^{t'-\rho} 
\frac{(-1)^{\sigma} (C_0 + \sum_{j=0}^{p'-\rho-1} C_{j+1}
\sigma(\sigma-1)...(\sigma-j) ) } 
{\sigma!(t'-\sigma-\rho)!} \\
& (\text{By plugging in $x'=r-t'$ and $y'=p-\rho -1$ in Claim.\ref{clm:rho_poly}.})
\\
& = 
\underbrace{\frac{ t'!}{\rho!(p'-\rho)!} 
\left[C_0 \sum_{\sigma=0}^{t'-\rho} 
\frac{(-1)^{\sigma} }
{\sigma!(t'-\sigma-\rho)!}  + \sum_{j=0}^{p'-\rho-1} C_{j+1} 
\sum_{\sigma=0}^{t'-\rho} 
\frac{(-1)^{\sigma} \sigma(\sigma-1)...(\sigma-j) }
{\sigma!(t'-\sigma-\rho)!} \right]}_{(I)}.
\\
\end{align*}

Then we separate $(I)$ into two cases. 
First, when $\sigma \leq j'$, we have $\sigma(\sigma-1)...(\sigma-j)= 0$. Therefore, we have 
\begin{align*}
(I) & = \frac{ t'!}{\rho!(p'-\rho)!} 
\left[C_0 \sum_{\sigma=0}^{t'-\rho} 
\frac{(-1)^{\sigma} }
{\sigma!(t'-\sigma-\rho)!}  + \sum_{j=0}^{t'-\rho-1} C_{j+1} 
\sum_{\sigma=j+1}^{t'-\rho} 
\frac{(-1)^{\sigma} }
{(\sigma-j-1)!(t'-\sigma-\rho)!} \right] 
\\
& = \frac{t'!}{\rho!(p'-\rho)!}
\sum_{j=0}^{t'-\rho}  C_j \sum_{\sigma=j}^{t'-\rho} 
\frac{(-1)^{\sigma} }
{(\sigma-j)!(t'-\sigma-\rho)!} 
\\
& = \frac{t'!}{\rho!(t'-\rho)!}
\sum_{j=0}^{t'-\rho} C_j (t'-\rho -j)! \sum_{\sigma=j}^{t'-\rho} 
(-1)^{\sigma} {t'-\rho -j \choose t'-\rho-\sigma}
\\
& = \frac{t'!}{\rho!(t'-\rho)!}
\sum_{j=0}^{t'-\rho}  C_j (t'-\rho -j)! (1-1)^{t'-\rho-j} (-1)^j
\\
& =
\frac{t'!C_{t'-\rho}}{\rho!(p'-\rho)!} (-1)^{t'-\rho} \\
& =
\frac{t'!}{\rho!(p'-\rho)!}
\frac{(p'-\rho)!(r-t')!}{(t'-\rho)!(p'-t')!(r-p')!} (-1)^{t'-\rho} \\
& (\text{By plugging in $x'=r-t'$ and $y'=p-\rho -1$ in Claim.\ref{clm:rho_poly}.}) \\
& =
{r-t' \choose p' - t' }{t' \choose \rho} (-1)^{t'-\rho} \\
\end{align*}

Secondly, if $p' < t'$, 
\begin{align*}
(I)& = \frac{ t'!}{\rho!(p'-\rho)!} 
\left[\frac{ C_0(-1)^{\sigma} }{\sigma!(t'-\sigma-\rho)!}\sum_{\sigma=0}^{t'-\rho} 
\frac{(-1)^{\sigma} }
{\sigma!(t'-\sigma-\rho)!}  + \sum_{j=0}^{t'-\rho-1} C_{j+1} 
\sum_{\sigma=0}^{t'-\rho} 
\frac{(-1)^{\sigma} }
{(\sigma-j-1)!(t'-\sigma-\rho)!} \right] 
\\
& = \frac{t'!}{\rho!(p'-\rho)!}
\sum_{j=0}^{t'-\rho}  C_j \sum_{\sigma=j}^{t'-\rho} 
\frac{(-1)^{\sigma} }
{(\sigma-j)!(t'-\sigma-\rho)!} \ \ \ ( \sigma \leq j \Rightarrow \sigma(\sigma-1)...(\sigma-j)= 0)
\\
& = \frac{t'!}{\rho!(t'-\rho)!}
\sum_{j=0}^{p'-\rho} C_j (t'-\rho -j)! \sum_{\sigma=j}^{t'-\rho} 
(-1)^{\sigma} {t'-\rho -j \choose t'-\rho-\sigma}
\\
& = \frac{t'!}{\rho!(t'-\rho)!}
\sum_{j=0}^{p'-\rho}  C_j (t'-\rho -j)! (1-1)^{t'-\rho-j}
\\
& =0
\end{align*}
\end{proof}

\subsection{Proof of Claim \ref{clm:int_closed_form}}
\label{sec:proof_int_closed_form}

\begin{proof}
We prove a more general form: for any $a,b \in \mathbb{N} + \{0\}$, we have 
\begin{equation}
\label{eqn:int_closed_form}
g_{a,b}(x) = \int x^{a}(1-x)^{b} dx 
= 
(1-x)^{b+1}p_{a,b}(x) + \frac{a! }{\prod_{k=0}^{a} (b+k+1)},
\end{equation}
where $g_{a,b}(x),p_{a,b}(x)$ are polynomials depended on $a,b$ with $g_{a,b}(0)=0$ . 

Now we prove Eqn.\eqref{eqn:int_closed_form} by induction on $a$. For $a=0$, we have 
$$
\int (1-x)^{b} dx  = \frac{-(1-x)^{b+1}}{b+1} + \frac{1}{b+1}.
$$
Assume that Eqn.\eqref{eqn:int_closed_form} holds for $a=n$. 
Then, for $a,b \in \mathbb{N}$, using integration by parts, we have 
\begin{align*}
& \int x^{n+1}(1-x)^{b} dx \\
& = \frac{-1}{b+1}\int x^{n+1} d (1-x)^{b+1} \\
& = \frac{-1}{b+1}x^{n+1}(1-x)^{b+1} + \frac{1}{b+1} \int (1-x)^{b+1} d x^{n+1}  \\
& =  \frac{-1}{b+1}x^{n+1}(1-x)^{b+1} + \frac{n+1}{b+1} \int x^{n}(1-x)^{b+1} dx.
\end{align*}
We note that the first term is divisible by $(1-x)^{b+1}$, while the second term can be obtained by induction hypothesis. Therefore, we have 
\begin{align*}
g_{n+1,b}(x)
& = \int x^{n+1}(1-x)^{b} dx \\
& = (1-x)^{b+1}p_{n+1,b}(x) + \frac{n+1}{b+1}\frac{n! }{\prod_{k=0}^{n} (b+k+2)}   \\
& = (1-x)^{b+1}p_{n+1,b}(x) + \frac{(n+1)! }{\prod_{k=0}^{n+1} (b+k+1)}.
\end{align*}
By plugging in $a=r-\ell-1$ and $b=\ell-1$, we have the desired results. Also, since the degree of $p_1(x) = \int x^{r-\ell-1}(1-x)^{\ell-i} dx$ is at most $r-i$, $p(x)$ should have degree less or equal to $r-\ell-1$.

\end{proof}

\end{document}

%% file: tables.tex
\begin{table}[ht]
\centering
\small
\resizebox{0.9\textwidth}{!}{
\begin{tabular}{m{0.05\linewidth}|m{0.33\linewidth}|m{0.13\linewidth}|m{0.33\linewidth}|m{0.13\linewidth}} 
\toprule
Index & 
\multicolumn{4}{l}{
\begin{tabular}{m{0.80\linewidth}|m{0.13\linewidth}} Sentences & Predicted Prob. \\
\end{tabular}} \\ 
\midrule
\multirow{34}{*}{1} & \multicolumn{4}{l}{\begin{tabular}{m{0.80\linewidth}|m{0.13\linewidth}} I have Never forgot this movie. All these years and it has remained in my life. & 0.992 \\ \end{tabular}} \\ 
\cmidrule[1.5pt]{2-5} & \multicolumn{2}{c|}{\textit{Faithful Shapley indices}} & \multicolumn{2}{c}{\textit{Shapley Taylor indices}}  \\ 
\cmidrule{2-5}
& Feature (interactions) & Scores & Feature (interactions) & Scores \\
\cmidrule{2-5}
& Never, forgot & 0.818 & Never, forgot & 1.077 \\
\cmidrule{2-5}
& life & 0.383 & Never, life & -0.211 \\
\cmidrule{2-5}
& forgot & -0.254 & remained, movie & -0.177 \\
\cmidrule{2-5}
& and & 0.168 & Never, this & -0.160 \\
\cmidrule{2-5}
& it & 0.168 & forgot, life & -0.149 \\
\cmidrule{2-5}
& Never & -0.163 & and, forgot & -0.149 \\
\cmidrule{2-5}
& years & 0.156 & in, life & -0.143 \\
\cmidrule{2-5}
& All & 0.132 & Never, it & -0.122 \\
\cmidrule{2-5}
& my & 0.126 & Never, movie & -0.114 \\
\cmidrule{2-5}
& has & 0.120 & have, Never & -0.110 \\
\cmidrule{2-5}
& have & 0.112 & I, have & 0.106 \\
\cmidrule{2-5}
& Never, life & -0.106 & forgot, in & -0.105 \\
\cmidrule{2-5}
& forgot, it & -0.096 & Never, All & -0.104 \\
\cmidrule{2-5}
& my, life & -0.086 & years, life & -0.101 \\
\cmidrule{2-5}
& this & 0.081 & it, forgot & -0.101 \\
\cmidrule[1.5pt]{2-5}
& \multicolumn{2}{c|}{\textit{Shapley interaction indices }} & \multicolumn{2}{c}{\textit{Integrated Hessian}}  \\
\cmidrule{2-5}
& Never, forgot & 1.166 & this, this & 1.796 \\
\cmidrule{2-5}
& these, life & -0.194 & in, in & -1.406 \\
\cmidrule{2-5}
& have, Never & -0.164 & my, my & -1.357 \\
\cmidrule{2-5}
& Never, it & -0.154 & ., . & 1.123 \\
\cmidrule{2-5}
& it, forgot & -0.149 & have, this & 1.050 \\
\cmidrule{2-5}
& forgot, life & -0.148 & movie, movie & 1.038 \\
\cmidrule{2-5}
& forgot, remained & 0.146 & it, it & -0.972 \\
\cmidrule{2-5}
& have, forgot & -0.139 & never, this & 0.820 \\
\cmidrule{2-5}
& and, Never & -0.136 & in, my & 0.792 \\
\cmidrule{2-5}
& it, life & -0.135 & in, . & 0.719 \\
\cmidrule{2-5}
& and, it & 0.131 & ., . & -0.602 \\
\cmidrule{2-5}
& forgot, in & -0.129 & this, in & 0.554 \\
\cmidrule{2-5}
& I, it & -0.125 & remained, remained & -0.532 \\
\cmidrule{2-5}
& years, Never & -0.121 & this, life & 0.526 \\
\cmidrule{2-5}
& forgot, my & -0.119 & remained, my & 0.526 \\
\cmidrule[1.5pt]{2-5}
& \multicolumn{4}{c}{ Archipalego
} \\
\cmidrule{2-5}
& \multicolumn{4}{c}{\includegraphics[width=\textwidth]{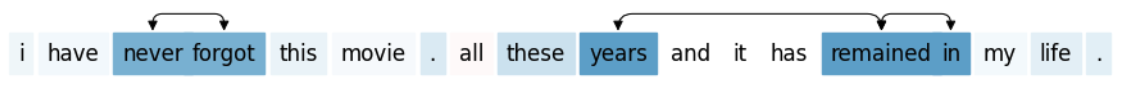}
} \\
\bottomrule    
\end{tabular}}
\vspace{0.2cm}
\caption{Top-15 important feature (interactions) for different methods for language dataset. The predicted probability is the out probability of the sentence having positive sentiment. The archipalego algorithm is run with $k=3$ interactions. }
\label{tab:index1}
\end{table}

\begin{table}[ht]
\centering
\small
\resizebox{0.9\textwidth}{!}{
\begin{tabular}{m{0.05\linewidth}|m{0.33\linewidth}|m{0.13\linewidth}|m{0.33\linewidth}|m{0.13\linewidth}} 
\toprule
Index & 
\multicolumn{4}{l}{
\begin{tabular}{m{0.80\linewidth}|m{0.13\linewidth}} Sentences & Predicted Prob. \\
\end{tabular}} \\ 
\midrule
\multirow{34}{*}{2} & \multicolumn{4}{l}{\begin{tabular}{m{0.80\linewidth}|m{0.13\linewidth}}TWINS EFFECT is a poor film in so many respects. The only good element is that it doesn't take itself seriously. & 0.012 \\ \end{tabular}} \\ 
\cmidrule[1.5pt]{2-5} & \multicolumn{2}{c|}{\textit{Faithful Shapley indices}} & \multicolumn{2}{c}{\textit{Shapley Taylor indices}}  \\ 
\cmidrule{2-5}
& Feature (interactions) & Scores & Feature (interactions) & Scores \\
\cmidrule{2-5}
& poor & -0.341 & only, good & -0.450 \\
\cmidrule{2-5}
& respects & 0.297 & EFFECT, good & -0.182 \\
\cmidrule{2-5}
& only, good & -0.243 & good, is & -0.171 \\
\cmidrule{2-5}
& poor, only & 0.206 & poor, film & -0.169 \\
\cmidrule{2-5}
& good & 0.176 & only, element & -0.168 \\
\cmidrule{2-5}
& poor, respects & -0.173 & doesn't, poor & 0.151 \\
\cmidrule{2-5}
& doesn't & -0.169 & only, poor & 0.150 \\
\cmidrule{2-5}
& poor, good & 0.122 & respects, poor & -0.149 \\
\cmidrule{2-5}
& only, doesn't & 0.115 & itself, poor & -0.142 \\
\cmidrule{2-5}
& poor, doesn't & 0.111 & respects, good & -0.137 \\
\cmidrule{2-5}
& many & 0.095 & it, doesn't & -0.108 \\
\cmidrule{2-5}
& it & 0.084 & it, only & -0.098 \\
\cmidrule{2-5}
& itself & 0.083 & take, seriously & 0.095 \\
\cmidrule{2-5}
& element & 0.076 & doesn't, good & -0.094 \\
\cmidrule{2-5}
& poor, many & -0.070 & doesn't, only & 0.093 \\
\cmidrule[1.5pt]{2-5}
& \multicolumn{2}{c|}{\textit{Shapley interaction indices }} & \multicolumn{2}{c}{\textit{Integrated Hessian}}  \\
\cmidrule{2-5}
& only, good & -0.280 & ., . & 19.441 \\
\cmidrule{2-5}
& a, only & -0.259 & is, . & 3.506 \\
\cmidrule{2-5}
& only, poor & 0.223 & the, . & 2.802 \\
\cmidrule{2-5}
& poor, good & 0.171 & only, . & 2.374 \\
\cmidrule{2-5}
& doesn't, poor & 0.159 & take, . & 2.004 \\
\cmidrule{2-5}
& respects, poor & -0.154 & it, . & 1.463 \\
\cmidrule{2-5}
& good, is & -0.150 & good, . & 1.462 \\
\cmidrule{2-5}
& The, good & -0.146 & seriously, . & 1.372 \\
\cmidrule{2-5}
& poor, element & -0.146 & doesn, . & 1.282 \\
\cmidrule{2-5}
& doesn't, only & 0.142 & is, . & 1.226 \\
\cmidrule{2-5}
& doesn't, take & -0.130 & that, . & 1.169 \\
\cmidrule{2-5}
& respects, only & -0.120 & ', ' & 1.143 \\
\cmidrule{2-5}
& good, element & -0.119 & ., . & -1.001 \\
\cmidrule{2-5}
& it, take & -0.117 & is, is & 0.996 \\
\cmidrule{2-5}
& so, that & -0.112 & itself, . & 0.914 \\
\cmidrule[1.5pt]{2-5}
& \multicolumn{4}{c}{ Archipalego
} \\
\cmidrule{2-5}
& \multicolumn{4}{c}{\includegraphics[width=\textwidth]{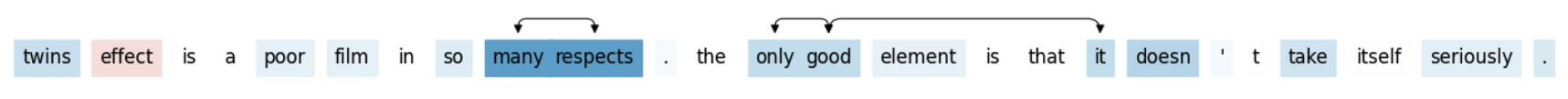}
} \\
\bottomrule    
\end{tabular}}
\vspace{0.2cm}
\caption{Top-15 important feature (interactions) for different methods for language dataset. The predicted probability is the out probability of the sentence having positive sentiment. The archipalego algorithm is run with $k=3$ interactions. }
\label{tab:index2}
\end{table}

\begin{table}[ht]
\centering
\small
\resizebox{0.9\textwidth}{!}{
\begin{tabular}{m{0.05\linewidth}|m{0.33\linewidth}|m{0.13\linewidth}|m{0.33\linewidth}|m{0.13\linewidth}} 
\toprule
Index & 
\multicolumn{4}{l}{
\begin{tabular}{m{0.80\linewidth}|m{0.13\linewidth}} Sentences & Predicted Prob. \\
\end{tabular}} \\ 
\midrule
\multirow{34}{*}{3} & \multicolumn{4}{l}{\begin{tabular}{m{0.80\linewidth}|m{0.13\linewidth}}
I rented this movie to get an easy, entertained view of the history of Texas. I got a headache instead. & 0.026 \\ \end{tabular}} \\ 
\cmidrule[1.5pt]{2-5} & \multicolumn{2}{c|}{\textit{Faithful Shapley indices}} & \multicolumn{2}{c}{\textit{Shapley Taylor indices}}  \\ 
\cmidrule{2-5}
& Feature (interactions) & Scores & Feature (interactions) & Scores \\
\cmidrule{2-5}
& instead & -0.321 & headache, instead & 0.268 \\
\cmidrule{2-5}
& headache, instead & 0.252 & view, instead & -0.178 \\
\cmidrule{2-5}
& headache & -0.205 & headache, Texas & -0.139 \\
\cmidrule{2-5}
& easy & 0.158 & rented, instead & 0.137 \\
\cmidrule{2-5}
& view & 0.130 & instead, easy & -0.125 \\
\cmidrule{2-5}
& history & 0.123 & got, headache & -0.118 \\
\cmidrule{2-5}
& rented & -0.122 & entertained, instead & -0.115 \\
\cmidrule{2-5}
& Texas & 0.101 & rented, headache & 0.109 \\
\cmidrule{2-5}
& entertained & 0.095 & got, easy & -0.108 \\
\cmidrule{2-5}
& rented, instead & 0.085 & got, history & -0.105 \\
\cmidrule{2-5}
& Texas, headache & -0.069 & a, I & -0.100 \\
\cmidrule{2-5}
& history, instead & -0.064 & view, history & 0.100 \\
\cmidrule{2-5}
& the & 0.059 & got, rented & -0.100 \\
\cmidrule{2-5}
& entertained, instead & -0.057 & got, a & -0.099 \\
\cmidrule{2-5}
& this & 0.052 & history, an & 0.094 \\
\cmidrule[1.5pt]{2-5}
& \multicolumn{2}{c|}{\textit{Shapley interaction indices }} & \multicolumn{2}{c}{\textit{Integrated Hessian}}  \\
\cmidrule{2-5}
& headache, instead & 0.333 & ., . & -13.363 \\
\cmidrule{2-5}
& easy, I & -0.248 & i, . & -2.441 \\
\cmidrule{2-5}
& movie, instead & -0.226 & ., a & -2.117 \\
\cmidrule{2-5}
& history, to & -0.162 & ., . & -1.420 \\
\cmidrule{2-5}
& Texas, an & -0.135 & texas, . & 1.171 \\
\cmidrule{2-5}
& rented, easy & -0.135 & this, . & -1.087 \\
\cmidrule{2-5}
& entertained, easy & 0.130 & ., i & -1.056 \\
\cmidrule{2-5}
& to, easy & -0.115 & to, . & -1.035 \\
\cmidrule{2-5}
& view, instead & -0.114 & of, . & 0.843 \\
\cmidrule{2-5}
& entertained, instead & -0.112 & entertained, entertained & -0.761 \\
\cmidrule{2-5}
& of, instead & -0.102 & headache, headache & 0.753 \\
\cmidrule{2-5}
& an, easy & 0.095 & history, history & -0.688 \\
\cmidrule{2-5}
& instead, easy & -0.093 & ., got & -0.673 \\
\cmidrule{2-5}
& of, easy & -0.087 & view, . & 0.657 \\
\cmidrule{2-5}
& get, easy & -0.085 & to, to & -0.599 \\
\cmidrule[1.5pt]{2-5}
\cmidrule[1.5pt]{2-5}
& \multicolumn{4}{c}{ Archipalego
} \\
\cmidrule{2-5}
& \multicolumn{4}{c}{\includegraphics[width=\textwidth]{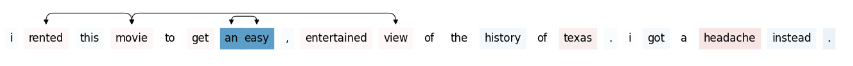}
} \\
\bottomrule     
\end{tabular}}
\vspace{0.2cm}
\caption{Top-15 important feature (interactions) for different methods for language dataset. The predicted probability is the out probability of the sentence having positive sentiment. The archipalego algorithm is run with $k=3$ interactions. }
\label{tab:index3}
\end{table}

\begin{table}[ht]
\centering
\small
\resizebox{0.9\textwidth}{!}{
\begin{tabular}{m{0.05\linewidth}|m{0.33\linewidth}|m{0.13\linewidth}|m{0.33\linewidth}|m{0.13\linewidth}} 
\toprule
Index & 
\multicolumn{4}{l}{
\begin{tabular}{m{0.80\linewidth}|m{0.13\linewidth}} Sentences & Predicted Prob. \\
\end{tabular}} \\ 
\midrule
\multirow{34}{*}{4} & \multicolumn{4}{l}{\begin{tabular}{m{0.80\linewidth}|m{0.13\linewidth}}Truly appalling waste of space. Me and my friend tried to watch this film to its conclusion but had to switch it off about 30 minutes from the end.
 & 0.002 \\ \end{tabular}} \\ 
\cmidrule[1.5pt]{2-5} & \multicolumn{2}{c|}{\textit{Faithful Shapley indices}} & \multicolumn{2}{c}{\textit{Shapley Taylor indices}}  \\ 
\cmidrule{2-5}
& Feature (interactions) & Scores & Feature (interactions) & Scores \\
\cmidrule{2-5}
& waste & -0.345 & appalling, waste & 0.298 \\
\cmidrule{2-5}
& appalling, waste & 0.257 & Truly, waste & -0.296 \\
\cmidrule{2-5}
& appalling & -0.251 & switch, it & -0.248 \\
\cmidrule{2-5}
& Truly & 0.169 & tried, waste & 0.230 \\
\cmidrule{2-5}
& waste, tried & 0.167 & but, watch & -0.210 \\
\cmidrule{2-5}
& friend & 0.162 & friend, waste & -0.184 \\
\cmidrule{2-5}
& space & 0.149 & friend, tried & -0.172 \\
\cmidrule{2-5}
& tried & -0.134 & friend, but & -0.169 \\
\cmidrule{2-5}
& Truly, waste & -0.118 & Truly, but & -0.145 \\
\cmidrule{2-5}
& watch & 0.087 & but, waste & 0.145 \\
\cmidrule{2-5}
& off & -0.086 & waste, watch & -0.140 \\
\cmidrule{2-5}
& and & 0.078 & waste, off & 0.138 \\
\cmidrule{2-5}
& waste, friend & -0.074 & had, space & -0.128 \\
\cmidrule{2-5}
& waste, space & -0.058 & Truly, film & 0.126 \\
\cmidrule{2-5}
& of & 0.055 & 30, waste & 0.124 \\
\cmidrule[1.5pt]{2-5}
& \multicolumn{2}{c|}{\textit{Shapley interaction indices }} & \multicolumn{2}{c}{\textit{Integrated Hessian}}  \\
\cmidrule{2-5}
& tried, watch & -0.365 & the, the & -31.568 \\
\cmidrule{2-5}
& appalling, waste & 0.293 & the, end & -13.784 \\
\cmidrule{2-5}
& tried, waste & 0.259 & the, . & 9.472 \\
\cmidrule{2-5}
& from, end & 0.230 & end, end & -5.719 \\
\cmidrule{2-5}
& conclusion, its & 0.228 & end, . & 3.522 \\
\cmidrule{2-5}
& Truly, waste & -0.210 & from, the & 3.390 \\
\cmidrule{2-5}
& waste, space & -0.202 & ., . & 2.616 \\
\cmidrule{2-5}
& space, off & -0.191 & had, the & 1.540 \\
\cmidrule{2-5}
& to, its & -0.180 & from, end & 1.441 \\
\cmidrule{2-5}
& to, space & 0.166 & 30, the & -0.959 \\
\cmidrule{2-5}
& Me, its & 0.162 & its, the & 0.941 \\
\cmidrule{2-5}
& Truly, of & 0.155 & minutes, . & 0.821 \\
\cmidrule{2-5}
& the, Me & -0.154 & off, the & 0.796 \\
\cmidrule{2-5}
& but, waste & 0.148 & ., . & 0.779 \\
\cmidrule{2-5}
& had, space & -0.146 & to, the & 0.737 \\
\cmidrule[1.5pt]{2-5}
& \multicolumn{4}{c}{ Archipalego
} \\
\cmidrule{2-5}
& \multicolumn{4}{c}{\includegraphics[width=\textwidth]{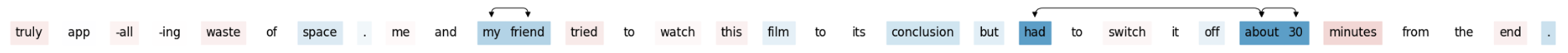}
} \\
\bottomrule      
\end{tabular}}
\vspace{0.2cm}
\caption{Top-15 important feature (interactions) for different methods for language dataset. The predicted probability is the out probability of the sentence having positive sentiment. The archipalego algorithm is run with $k=3$ interactions. }
\label{tab:index4}
\end{table}

\begin{table}[ht]
\centering
\small
\resizebox{0.9\textwidth}{!}{
\begin{tabular}{m{0.05\linewidth}|m{0.33\linewidth}|m{0.13\linewidth}|m{0.33\linewidth}|m{0.13\linewidth}} 
\toprule
Index & 
\multicolumn{4}{l}{
\begin{tabular}{m{0.80\linewidth}|m{0.13\linewidth}} Sentences & Predicted Prob. \\
\end{tabular}} \\ 
\midrule
\multirow{34}{*}{5} & \multicolumn{4}{l}{\begin{tabular}{m{0.80\linewidth}|m{0.13\linewidth}}I still remember watching Satya for the first time. I was completely blown away. & 0.994 \\ \end{tabular}} \\ 
\cmidrule[1.5pt]{2-5} & \multicolumn{2}{c|}{\textit{Faithful Shapley indices}} & \multicolumn{2}{c}{\textit{Shapley Taylor indices}}  \\ 
\cmidrule{2-5}
& Feature (interactions) & Scores & Feature (interactions) & Scores \\
\cmidrule{2-5}
& remember & 0.337 & blown, away & 0.345 \\
\cmidrule{2-5}
& blown, away & 0.293 & the, first & 0.191 \\
\cmidrule{2-5}
& time & 0.281 & time, first & 0.182 \\
\cmidrule{2-5}
& Satya & 0.208 & watching, for & -0.169 \\
\cmidrule{2-5}
& remember, blown & -0.158 & time, away & -0.167 \\
\cmidrule{2-5}
& watching & 0.153 & time, Satya & -0.151 \\
\cmidrule{2-5}
& blown & 0.146 & time, still & -0.145 \\
\cmidrule{2-5}
& time, away & -0.127 & still, watching & -0.144 \\
\cmidrule{2-5}
& completely, away & -0.101 & I, watching & -0.131 \\
\cmidrule{2-5}
& Satya, time & -0.091 & watching, first & -0.128 \\
\cmidrule{2-5}
& remember, time & -0.073 & remember, away & 0.118 \\
\cmidrule{2-5}
& I, watching & -0.071 & Satya, away & -0.118 \\
\cmidrule{2-5}
& completely, blown & 0.063 & was, watching & -0.115 \\
\cmidrule{2-5}
& first, blown & -0.053 & remember, blown & -0.110 \\
\cmidrule{2-5}
& first & 0.049 & completely, away & -0.107 \\
\cmidrule[1.5pt]{2-5}
& \multicolumn{2}{c|}{\textit{Shapley interaction indices }} & \multicolumn{2}{c}{\textit{Integrated Hessian}}  \\
\cmidrule{2-5}
& blown, away & 0.318 & ., . & 4.759 \\
\cmidrule{2-5}
& was, remember & 0.237 & was, . & 1.866 \\
\cmidrule{2-5}
& remember, blown & -0.180 & blown, blown & 1.552 \\
\cmidrule{2-5}
& time, Satya & -0.167 & i, . & 1.185 \\
\cmidrule{2-5}
& the, first & 0.144 & was, was & 1.105 \\
\cmidrule{2-5}
& blown, first & -0.133 & i, . & 1.063 \\
\cmidrule{2-5}
& completely, blown & 0.126 & blown, away & 0.889 \\
\cmidrule{2-5}
& time, away & -0.119 & satya, . & 0.857 \\
\cmidrule{2-5}
& I, was & 0.093 & for, for & -0.763 \\
\cmidrule{2-5}
& watching, blown & 0.087 & remember, remember & -0.745 \\
\cmidrule{2-5}
& I, watching & -0.083 & for, time & -0.745 \\
\cmidrule{2-5}
& time, blown & 0.080 & i, i & -0.727 \\
\cmidrule{2-5}
& watching, away & -0.078 & ., . & -0.616 \\
\cmidrule{2-5}
& remember, watching & -0.076 & watching, satya & 0.592 \\
\cmidrule{2-5}
& remember, was & -0.074 & completely, . & 0.579 \\
\cmidrule[1.5pt]{2-5}
& \multicolumn{4}{c}{ Archipalego
} \\
\cmidrule{2-5}
& \multicolumn{4}{c}{\includegraphics[width=\textwidth]{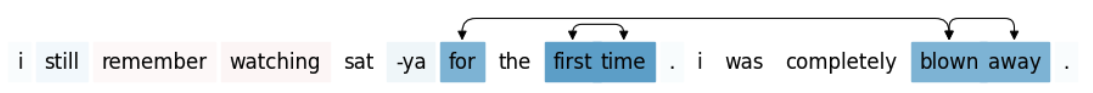}
} \\
\bottomrule       
\end{tabular}}
\vspace{0.2cm}
\caption{Top-15 important feature (interactions) for different methods for language dataset. The predicted probability is the out probability of the sentence having positive sentiment. The archipalego algorithm is run with $k=3$ interactions. }
\label{tab:index5}
\end{table}